\documentclass[final]{alt2026} 

\title[Suspicious Alignment of SGD]{Suspicious Alignment of SGD:\\A Fine-Grained Step Size Condition Analysis}
\usepackage{times}
\usepackage{amsmath, amsfonts, amssymb}
\usepackage{bm}
\usepackage{booktabs}
\usepackage{array}
\usepackage{thmtools} 
\usepackage{comment}
\usepackage{graphicx}
\usepackage{grffile}
\usepackage{adjustbox}
\usepackage{caption}
\usepackage{subcaption}
\usepackage{float}

\usepackage{multirow}

\usepackage{hyperref}

\newtheorem*{theorem-restated}{Theorem}
\newtheorem*{assumption-restated}{Assumption}
\newtheorem*{corollary-restated}{Corollary}
\newtheorem*{lemma-restated}{Lemma}
\newtheorem{assumption}[theorem]{Assumption}

\def\vxi{{\bm{\xi}}}
\def\mA{{\bm{A}}}

\def\mC{{\bm{C}}}
\def\mD{{\bm{D}}}

\def\mM{{\bm{M}}}

\def\mP{{\bm{P}}}

\def\mU{{\bm{U}}}

\def\mM{{\bm{M}}}

\def\mLambda{{\bm{\Lambda}}}
\def\mSigma{{\bm{\Sigma}}}

\altauthor{%
 \Name{Shenyang Deng} \Email{shenyang.deng.gr@dartmouth.edu}\\
 \addr Dartmouth College
 \AND
 \Name{Boyao Liao} \Email{bxl307@student.bham.ac.uk}\\
 \addr University of Birmingham
 \AND
 \Name{Zhuoli Ouyang} \Email{Zhuoli.Ouyang@dartmouth.edu}\\
 \addr Dartmouth College
 \AND
 \Name{Tianyu Pang} \Email{tianyu.pang.gr@dartmouth.edu}\\
 \addr Dartmouth College
 \AND
 \Name{Minhak Song} \Email{minhaksong@kaist.ac.kr}\\
 \addr KAIST
 \AND
 \Name{Yaoqing Yang} \Email{yaoqing.yang@dartmouth.edu}\\
 \addr Dartmouth College
}

\usepackage{xcolor}

\newif\ifcomments
\commentsfalse

\newcommand{\yaoqing}[1]{\ifcomments \textcolor{orange}{[Yaoqing:\ #1]} \fi}
\newcommand{\addressedyaoqing}[1]{\ifcomments \textcolor{teal}{[Yaoqing (addressed):\ #1]} \fi}

\newcommand{\addressedtianyu}[1]{\ifcomments \textcolor{teal}{[Tianyu (addressed):\ #1]} \fi}

\newcommand{\shenyang}[1]{\ifcomments \textcolor{blue}{[Shenyang:\ #1]} \fi}

\begin{document}

\maketitle

\begin{abstract}

This paper explores the suspicious alignment phenomenon in stochastic gradient descent (SGD) under ill-conditioned optimization, where the Hessian spectrum splits into dominant and bulk subspaces. This phenomenon describes the behavior of gradient alignment in SGD updates. Specifically, during the initial phase of SGD updates, the alignment between the gradient and the dominant subspace tends to decrease. Subsequently, it enters a rising phase and eventually stabilizes in a high-alignment phase. The alignment is considered ``suspicious'' because, paradoxically, the projected gradient update along this highly-aligned dominant subspace proves ineffective at reducing the loss. The focus of this work is to give a fine-grained analysis in a high-dimensional quadratic setup about how step size selection produces this phenomenon. Our main contribution can be summarized as follows: We propose a step-size condition revealing that in low-alignment regimes, an adaptive critical step size $\eta_t^*$ separates alignment-decreasing ($\eta_t < \eta_t^*$) from alignment-increasing ($\eta_t > \eta_t^*$) regimes, whereas in high-alignment regimes, the alignment is self-correcting and decreases regardless of the step size. We further show that under sufficient ill-conditioning, a step size interval exists where projecting the SGD updates to the bulk space decreases the loss while projecting them to the dominant space increases the loss, which explains a recent empirical observation that projecting gradient updates to the dominant subspace is ineffective. Finally, based on this adaptive step-size theory, we prove that for a constant step size and large initialization, SGD exhibits this distinct two-phase behavior: an initial alignment-decreasing phase, followed by stabilization at high alignment.

\end{abstract}


\section{Introduction}
 Deep neural networks are often optimized over ill-conditioned loss landscapes. Empirical studies by \citet{sagun2016eigenvalues,sagun2017empirical} have shown that the Hessian of the training loss in over-parameterized models typically exhibits a bimodal eigenvalue spectrum: a small number of large eigenvalues corresponding to directions of high curvature are sharply separated from a dense bulk of near-zero eigenvalues that span the vast majority of the parameter space. This spectral structure creates an ill-conditioned loss landscape: steep along a few dominant directions, yet nearly flat across a high-dimensional subspace. Such geometry is referred to in the literature as the river-valley landscape~\citep{wen2024understanding}, where narrow, high-curvature ``valleys'' are embedded within a broad, flat ``river''.
To formalize this structure, let \(\nabla^2 L(\bm{x})\) denote the Hessian at a point \(\bm{x}\), and assume it admits a spectral decomposition with eigenvalues \(\lambda_1 \geq \cdots\geq \lambda_k  >\lambda_{k+1} \geq \cdots\geq \lambda_d > 0\) and orthonormal eigenvectors \(\{\bm{u}_i\}_{i=1}^d\). We partition the spectrum into a \emph{dominant block} \(\mathcal{D} = \{1, \dots, k\}\) and a \emph{bulk block} \(\mathcal{B} = \{k+1, \dots, d\}\), separated by a non-vanishing gap
\[
\mathrm{gap}_1 := \lambda_k - \lambda_{k+1} > 0.
\]
This induces orthogonal projectors onto the corresponding subspaces:
\[
\bm{P}^{\mathcal{D}} = \sum_{i \in \mathcal{D}} \bm{u}_i \bm{u}_i^\top, \qquad
\bm{P}^{\mathcal{B}} = \sum_{i \in \mathcal{B}} \bm{u}_i \bm{u}_i^\top.
\]
Typically, the (river-valley) landscapes are ill-conditioned. In such landscapes, SGD exhibits a striking phenomenon: the gradient \(\nabla L(\bm{x}_t)\) becomes increasingly aligned with the dominant subspace \(\mathcal{D}\) as training progresses. This \emph{tiny subspace} effect was documented by \citet{gur2018gradient}, who showed that late-stage gradients lie almost entirely within the span of the top few Hessian eigenvectors. Other researchers also have similar findings \citep{schneider2024identifying}. Following by those results, a lot of works try to use this property to design more efficient learning algorithms \citep{gressmann2020improving,gauch2022few,li2022low}. To gain a better understanding of the intuition, let's define the alignment by
\[
\theta_t := \frac{\|\bm{P}^{\mathcal{D}} \nabla L(\bm{x}_t)\|_2^2}{\|\nabla L(\bm{x}_t)\|_2^2} \in [0,1].
\]
Intuitively, \(\theta_t \approx 1\) suggests that optimization is focused on the most curved directions, which may seem to be an opportunity to achieve efficiency.
However, above intuition is contradicted by recent findings \citep{song2024does}: in the high-alignment regime (\(\theta_t \approx 1\)), the SGD updates projected onto the dominant space \(\mathcal{D}\) often \emph{fail to decrease the training loss}, while the orthogonal bulk component---despite carrying negligible gradient norm---does make the loss decrease. We call this \textbf{suspicious alignment}: a state where the gradient is mainly supported on \(\mathcal{D}\), yet updates along \(\mathcal{D}\) give non-decreasing or even ascending loss under typical step sizes.
This leads to two questions:
\begin{quote}
\textbf{Problem 1.} \emph{How does the step size $\eta_t$ govern the evolution of alignment $\theta_t$ in ill-conditioned landscapes?} \\
\textbf{Problem 2.} \emph{Why do Dominant Projected SGD (DSGD) fail to reduce the loss while Bulk Projected SGD (BSGD) succeed under high alignment ?}
\end{quote}
We provide complete analytical answers to both questions in the quadratic case under the high-dimensional regime where \(d \to \infty\). Our main results are boxed below:
\\
\medskip
\noindent\fbox{%
\begin{minipage}{\dimexpr\linewidth-2\fboxsep-2\fboxrule\relax}%
\textbf{Answer to Problem 1: Step-size condition of alignment dynamics.}  
We identify an \emph{adaptive critical step size} $\eta_t^*(\bm{x}_t)$ (Eq.~\eqref{eq:def-etatstar-body}) that separates two distinct regimes:
\begin{itemize}
\item \textbf{Low-alignment regime} ($\theta_t \leq g_{\mathrm{gap}}$, where $g_{\mathrm{gap}}$ is a threshold smaller than 1): \\alignment is \emph{dependent} on the step size $\eta_t$.  
    If $\eta_t < \eta_t^*$, then $\mathbb{E}[\theta_{t+1} \mid \bm{x}_t] < \theta_t$;  
    if $\eta_t > \eta_t^*$, then $\mathbb{E}[\theta_{t+1} \mid \bm{x}_t] > \theta_t$.
\item \textbf{High-alignment regime} ($\theta_t \geq \theta_t^*$, where $\theta_t^*$ is another threshold smaller than 1): \\
alignment is \emph{self-correcting}.  
    $\mathbb{E}[\theta_{t+1} \mid \bm{x}_t] < \theta_t$ for \emph{any} $\eta_t > 0$.
\end{itemize}
Here $\mathbb{E}[\cdot|\bm{x}_t]$ means taking the expectation with respect to the gradient noise at $t$ step. These regimes are separated by a stable interval $(g_{\mathrm{gap}}, \theta_t^*)$, into which $\theta_t$ tends to oscillate. {Furthermore, $g_{\mathrm{gap}} \to 1$,$\theta_t^* \to 1$ as $\lambda_k / \lambda_{k+1} \to \infty$ }. This is an informal summary of Theorems~\ref{thm:step-dec-below}, \ref{thm:step-inc-above}, \ref{thm:step-dec-large}, and \ref{thm:regime-separation}.
\end{minipage}%
}
\medskip
\noindent\fbox{%
\begin{minipage}{\dimexpr\linewidth-2\fboxsep-2\fboxrule\relax}%
\textbf{Answer to Problem 2: Stability disparity between subspaces.}  
For projected updates, expected loss decreases iff $\eta_t < \eta^{\mathrm{loss}}_{\mathcal{S}}(\bm{x}_t)$, where for $\mathcal{S} \in \{\mathcal{D}, \mathcal{B}\}$,
\[
\eta^{\mathrm{loss}}_{\mathcal{S}}(\bm{x}_t) = \frac{2 \sum_{i \in \mathcal{S}} \lambda_i^2 c_{i,t}^2}{\sum_{i \in \mathcal{S}} \lambda_i^4 c_{i,t}^2 + \sum_{i \in \mathcal{S}} \lambda_i \kappa_i^2},
\]
with $c_{i,t} = \langle \bm{x}_t, \bm{u}_i \rangle$ and $\kappa_i^2 = \bm{u}_i^\top \boldsymbol{\Sigma} \bm{u}_i$. There exists a critical alignment $\theta^{\mathrm{crit}}_t \in (0,1)$ such that
\[
\theta_t < \theta^{\mathrm{crit}}_t \quad \Longleftrightarrow \quad \eta^{\mathrm{loss}}_{\mathcal{D}} < \eta^{\mathrm{loss}}_{\mathcal{B}}.
\]
As $\lambda_k / \lambda_{k+1} \to \infty$, $\theta^{\mathrm{crit}}_t \to 1$, so for nearly all $\theta_t < 1$, there exists an unstabele regime $(\eta^{\mathrm{loss}}_{\mathcal{D}}, \eta^{\mathrm{loss}}_{\mathcal{B}})$ for the dominant update. When the step size is inside $(\eta^{\mathrm{loss}}_{\mathcal{D}}, \eta^{\mathrm{loss}}_{\mathcal{B}})$, DSGD increases loss while BSGD decreases it, resolving the paradox of \citet{song2024does}. This is an informal summary of Theorems~\ref{thm:proj-loss-corr}, \ref{thm:loss-crossover-rigorous}, and \ref{thm:loss-crossover-asymptotic}.
\end{minipage}%
}
\begin{figure}[H]
    \centering
    \includegraphics[width=0.7\linewidth]{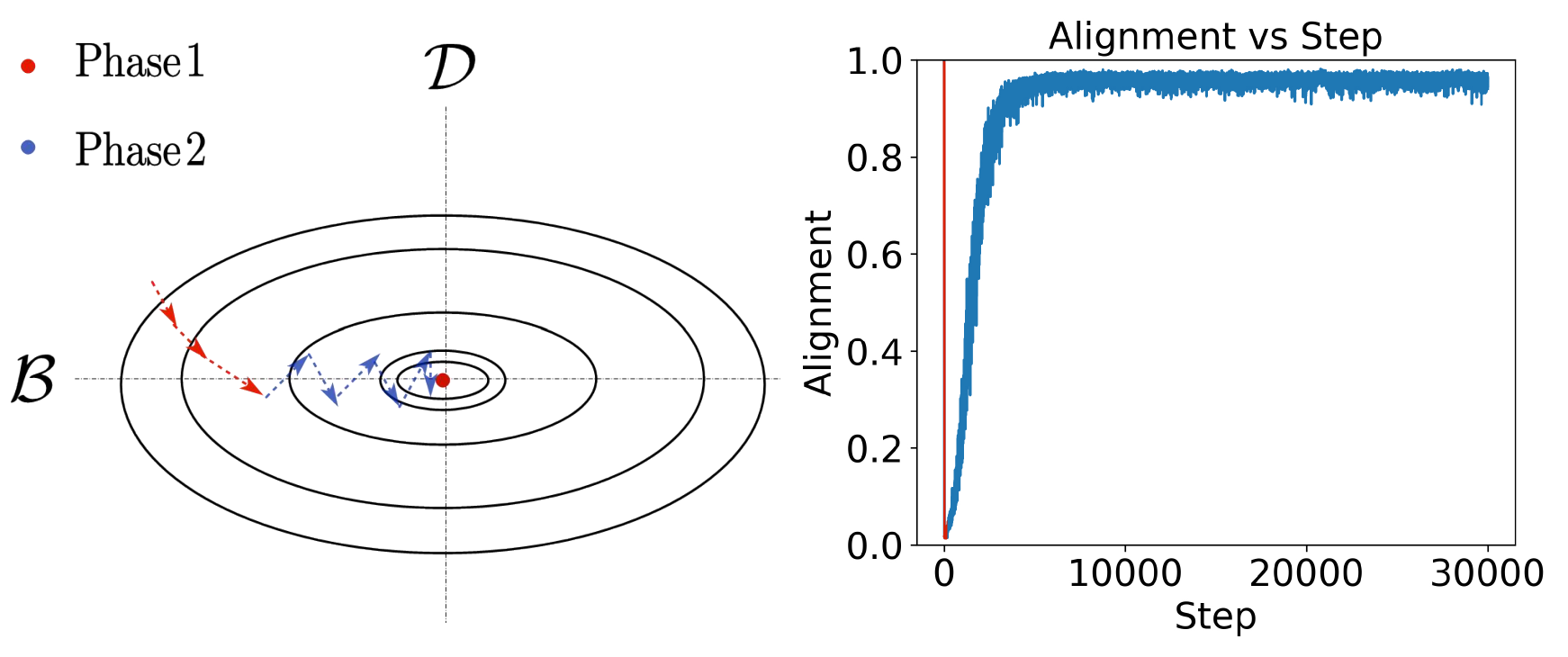}
    \caption{Two-phase SGD Dynamic: $\mathcal{D},\mathcal{B}$ represent dominant and bulk direction respectively. {This is a simulation experiment with constant step size for SGD with spectrum gap $\frac{\lambda_k}{\lambda_{k+1}}=100$, the details can be referred to Section \ref{NES}.} \addressedyaoqing{Remember to replace this figure with a non-conceptual one.}}
    \label{twophase}
\end{figure}
Building on this theoretical foundation, we further analyze the long-term alignment dynamics under \emph{constant step size} SGD (CSGD). In Section~\ref{csgdsection}, we characterize a distinct \emph{two-phase behavior} of $\theta_t$ as show in Figure \ref{twophase}: an initial transient phase where alignment monotonically decreases (Theorem~\ref{thm:monotone-decrease-expected-alignment}), followed by a late-time equilibrium where $\theta_t$ converges to a noise- and spectrum-dependent limit $\theta_\infty$ (Theorem~\ref{thm:late-theta}). This provides a complete picture of how ill-conditioning, noise geometry, and step size jointly shape the alignment trajectory of SGD.

\section{Related Work}
\paragraph{Ill-conditioned Loss Landscape}
{The performance of deep neural networks is strongly affected by the geometric properties of their loss landscapes, as illustrated in the visualization in \citet{li2018visualizing}. Initial spectral analysis by \citet{sagun2016eigenvalues, sagun2017empirical} identified a bulk-dominant structure in the Hessian, characterized by a vast majority of near-zero eigenvalues alongside a small number of dominant outlier eigenvalues. Subsequent research has further characterized the properties of this structure across various tasks and architectures \citep{ghorbani2019investigation, papyan2019measurements, papyan2020traces}. Specifically, \citet{ghorbani2019investigation} observed that gradients predominantly align with the few outlier eigenvectors, while \citet{wu2020dissecting} demonstrated that the Hessian remains consistently low-rank, with its significant curvature confined to a tiny fraction of the parameter space. These empirical observations are supported by theoretical proofs: \citet{liao2021hessian} explained how outlier eigenvalues and bulk eigenvalues emerge from the coupling between the model and the data structure, while \citet{singh2021analytic} derived exact analytic formulas for the low-rank nature of Hessian maps. More recently, \citet{hodgkinson2025models} proposed that the Hessian eigenvalues of optimal models approximately follow a Heavy-Tailed Marchenko-Pastur (HTMP) distribution. A key characteristic of this distribution is its heavy-tailed nature, which is more likely to produce outlier eigenvalues than light-tail distribution. Furthermore, certain dynamic properties of this structure can exert unexpected impacts on optimization. In PINNs, \citet{rathore2024challenges} showed that the presence of multiple conflicting loss terms leads to a highly unstable outlier structure where outlier eigenvalues frequently shift and overlap during training, triggering severe instabilities. In addition to these primary studies, other works have explored broader connections between this structure and various optimization phenomena \citep{zhang2019gradient, yang2021taxonomizing, garipov2018loss, pan2021eigencurve}.}
\paragraph{Gradient Alignment Theory}
In the context of such ill-conditioned loss landscapes, \citet{gur2018gradient} found that SGD consistently aligns with the large eigenvalue directions in the later stages. Since the large eigenvalue space is a very low-dimensional subspace, they posited that SGD operates within a \emph{tiny subspace}. However, subsequent observations by \citet{song2024does} overturned this conclusion. By projecting SGD updates onto the corresponding subspaces during the middle stage of training, they discovered that updates projected onto the large eigenvalue space do not reduce the loss. Instead, updates projected onto the small eigenvalue space lead to further loss reduction. Recent theoretical work on loss landscapes has also pertained to the alignment phenomenon of gradient under ill-conditioned structures. 
For example, \citet{li2022understanding} proved in their Theorem 3.1 that, for edge of stability (EOS) phenomenon~\citep{cohen2021gradient} and for quadratic programming problems, if gradient descent uses a normalized step-size schedule, the gradient $\mA \bm{x}$ aligns with the largest eigenvector of the Hessian as time approaches infinity. Later, in Theorem 4.4, they established gradient alignment results for normalized gradient flow on manifolds that satisfy certain assumptions. In the context of studying LLM step size schedules, \citet{wen2024understanding} provided in Lemma A.13 that, for a specific loss landscape structure, after a certain time $ t $, the ratio of the projections of the gradient flow onto the large and small eigenspaces (which is the alignment) has an asymptotic upper bound with $t$. {But, we note that the SGD results in \citet{wen2024understanding} concerns track the trajectory near the river only, and it does not have a result on trajectory alignment. In conclusion, these theoretical alignment results are based on gradient descent rather than SGD.}\yaoqing{The river-valley loss landscape paper also studies SGD, right?}\shenyang{They don't have alignment result on SGD, SGD is an approximation to GD flow in distance case.}\yaoqing{In that case, you can perhaps say something like "We note that the results in XXX concerns XXX only, and it does not have a result on trajectory alignment."} Moreover, the works above are not specifically designed to study the loss landscape or the impact of step-size schedules on the alignment phenomenon. Furthermore, we cannot directly use their lemma to explain the observations reported by \citet{song2024does}.

\section{Preliminary}
\subsection{Common Notation}
Before moving on to the main content, we will first introduce some common notation. We denote \(\mathbb{R}^d\) as the \(d\)-dimensional Euclidean space. Bold lowercase letters (such as \(\bm{u}\), \(\bm{v}\)) denote column vectors; Bold uppercase letters (such as \(\mM\)) denote matrices; lowercase letters (such as \(a\), \(b\)) denote scalars. Sets are denoted by calligraphic letters (such as \(\mathcal{D}\), \(\mathcal{B}\)). The norm \(\|\cdot\|_p\) represents the \(p\)-norm for vectors or the \(p\)-operator norm for matrices; \(\|\cdot\|_F\) represents the Frobenius norm for matrices. $\text{Tr} (\mM) $ denotes the trace of matrix $\mM$.

\subsection{Setup and Assumptions}
We study a quadratic optimization problem under SGD with arbitrary noise of finite second moment, focusing on how the Hessian eigenspectrum and the noise covariance influence gradient alignment with the dominant eigenspace of the Hessian.

\paragraph{Quadratic Loss and Spectral Decomposition}
We consider the quadratic loss function
\[
L(\bm{x}) = \frac{1}{2} \bm{x}^\top \mA \bm{x}, \qquad \nabla L(\bm{x}) = \mA \bm{x},
\]
where $\mA \in \mathbb{R}^{d \times d}$ is symmetric positive definite. The matrix $\mA$ has the spectral decomposition
\[
\mA = \mU \mLambda \mU^\top, \quad \mLambda = \mathrm{diag}(\lambda_1, \dots, \lambda_d), \quad \lambda_1 \geq \cdots\geq \lambda_k  >\lambda_{k+1} \geq \cdots\geq \lambda_d > 0,
\] 
with orthonormal eigenvectors $\mU = [\bm{u}_1, \dots, \bm{u}_d]$. We express the state $\bm{x}_t$ in the eigenbasis as $\bm{x}_t = \sum_{i=1}^d c_{i,t} \bm{u}_i$, where $c_{i,t} := \langle \bm{x}_t, \bm{u}_i \rangle$.

\paragraph{Dominant and Bulk Spectral}
Recall that we partition the spectrum into a dominant block and a bulk block:
\[
\mathcal{D} = \{1, \dots, k(d)\}, \qquad \mathcal{B} = \{k(d)+1, \dots, d\}.
\]
Here, we assume $k(d)$ is a function of $d$. For simplicity, in the following, we denote $k(d)$ by $k$, where $k \in \{1, \dots, d-1\}$.
We define projection matrices for the dominant and bulk subspaces:
\[
\bm{P}^{\mathcal{D}} := \sum_{i \in \mathcal{D}} \bm{u}_i \bm{u}_i^\top, \quad \bm{P}^{\mathcal{B}} := \sum_{i \in \mathcal{B}} \bm{u}_i \bm{u}_i^\top, 
\]
and the squared alignment function \(\theta(\bm{x}_t)\) of the gradient \(\mA \bm{x}_t\) with the dominant subspace at time \(t\) define as
\[
\theta(\bm{x}_t) := 
\begin{cases}
\dfrac{\|\bm{P}^{\mathcal{D}} \nabla L(\bm{x}_t)\|_2^2}{\|\nabla L(\bm{x}_t)\|_2^2} = \dfrac{\|\bm{P}^{\mathcal{D}}\mA\bm{x}_t\|_2^2}{\|\mA\bm{x}_t\|_2^2} = \dfrac{\sum_{i \in \mathcal{D}} \lambda_i^2 c_{i,t}^2}{\sum_{i=1}^d \lambda_i^2 c_{i,t}^2} \in [0,1] & \text{if } \bm{x}_t \neq \bm{0}, \\
0 & \text{if } \bm{x}_t = \bm{0},
\end{cases}
\]
which captures the proportion of the gradient's norm in the top-$k$ dominant eigenspace (with the convention \(\theta(\bm{0}) = 0\) ensuring well-definedness in stochastic analysis). We will use $\theta_t$ for short in the later section. To analyze the alignment, we need to introduce the following notation for quantities defined in certain subspaces:
\begin{align}
\notag &\psi_{\mathcal D} := \sum_{i \in \mathcal D} \lambda_i^2, \quad \psi_{\mathcal B} := \sum_{i \in \mathcal B} \lambda_i^2,  \\
s_t^{\mathcal{D}} := \sum_{i \in \mathcal{D}} \lambda_i^2 c_{i,t}^2, \quad &s_t^{\mathcal{B}} := \sum_{i \in \mathcal{B}} \lambda_i^2 c_{i,t}^2, \quad s_t := s_t^{\mathcal{D}} + s_t^{\mathcal{B}} = \sum_{i=1}^d \lambda_i^2 c_{i,t}^2. \label{eq:definitions}
\end{align}
We assume the following eigenvalue gaps to ensure a clear separation between the bulk and dominant eigenvalues. The gap is defined at the boundary between the $k$-th and $(k+1)$-th eigenvalues.
\[
\mathrm{gap}_1 := \lambda_k - \lambda_{k+1} > 0, \quad \mathrm{gap}_2 := \lambda_k^2 - \lambda_{k+1}^2 > 0.
\]
\paragraph{SGD Update with Noise}
The SGD update with step size $\eta_t > 0$ is given by
\[
\bm{x}_{t+1} = \bm{x}_t - \eta_t (\mA \bm{x}_t + \vxi_t), \quad \mathbb{E}[\vxi_t]=  \bm{0}, \text{Cov}[\vxi_t]=\mSigma, \quad \mSigma = \mSigma^\top \succeq 0,
\]
where $\vxi_t\sim\mathcal{N}(\bm{0},\mSigma)$ is a Gaussian random vector. Define $\mC := \mU^\top \mSigma \mU$ as the noise covariance in the $\mA$-basis, with per-eigendirection variances $\kappa_i^2 := (\mC)_{ii} = \bm{u}_i^\top \mSigma \bm{u}_i$. We introduce the following quantities to represent block-wise noise energy:
\[
\, \quad e_{\mathcal{D}} := \sum_{i \in \mathcal{D}} \lambda_i^2 \kappa_i^2, \quad e_{\mathcal{B}} := \sum_{i \in \mathcal{B}} \lambda_i^2 \kappa_i^2.\
\]
The noise covariance eigenvalues $s_{\min} := \lambda_{\min}(\mSigma)$ and $s_{\max} := \lambda_{\max}(\mSigma)$ yield bounds:
\[
\,s_{\min} \psi_{\mathcal{D}} \leq e_{\mathcal{D}} \leq s_{\max} \psi_{\mathcal{D}}, \quad s_{\min} \psi_{\mathcal{B}} \leq e_{\mathcal{B}} \leq s_{\max} \psi_{\mathcal{B}}.\,
\]
\paragraph{Why this Model?}
This quadratic framework, by separately modeling the components of $\bm{x}_t$ in high-curvature (dominant) and low-curvature (bulk) directions as well as the corresponding noise($\vxi_t$) components, facilitates obtaining more fine-grained analytical conclusions and understanding the impact of alignment on optimization.
It provides a tractable model for studying how step sizes and noise geometry drive gradient alignment in SGD.
\begin{assumption}[Asymptotic Spectral Assumptions]\label{asp:standing}
When we consider the high-dimensional regime, where both $d$ and $k(d) \to \infty$, we assume the following conditions, summarized in Table~\ref{tab:assumptions}.
\begin{table}[h]
\centering
\caption{Asymptotic Spectral Assumptions}
\label{tab:assumptions}
\begin{tabular}{>{\raggedright\arraybackslash}p{4cm} >{\raggedright\arraybackslash}p{8cm}}
\toprule
\textbf{Assumption} & \textbf{Description} \\
\midrule
Trajectory boundedness & The state remains bounded: $\sup_t \limsup_{d \to \infty} \frac{1}{d} \sum_{i=1}^d c_{i,t}^2 < \infty$. \\
Block proportion & The subspace dimension ratio is a fixed constant: $\rho :=\frac{k}{d-k} \in (0, \infty)$. Which implies $k=\frac{\rho}{1+\rho} d$. \\
Block spectral moments & For $p \in \{2, 3, 4, 6, 8\}$, the block-wise spectral moments converge: $\frac{1}{k} \sum_{i \in \mathcal{D}} \lambda_i^p \to \lambda_{\mathcal{D},p} \in (0, \infty)$, $\frac{1}{d-k} \sum_{i \in \mathcal{B}} \lambda_i^p \to \lambda_{\mathcal{B},p} \in [0, \infty)$. \\
Noise spectral bounds & The noise covariance has a bounded trace: $Tr(\mSigma) \in (0, +\infty)$ \\
\bottomrule
\end{tabular}
\end{table}
\end{assumption}

\section{Step Size Condition Theory}\label{sec:step_size_condition}
For the SGD update $\bm{x}_{t+1}=\bm{x}_t-\eta_t(\mA \bm{x}_t+\vxi_t)$ with $\vxi_t\sim\mathcal{N}(\bm{0},\mSigma)$,
we define the adaptive critical step size at time $t$ as
\begin{equation}\label{eq:def-etatstar-body}
\eta_t^*(\bm{x}_t)
:=\frac{2\Big(\big(\sum_{i\in\mathcal B}\lambda_i^2 c_{i,t}^2\big)\big(\sum_{j\in\mathcal D}\lambda_j^3 c_{j,t}^2\big)
-\big(\sum_{i\in\mathcal D}\lambda_i^2 c_{i,t}^2\big)\big(\sum_{j\in\mathcal B}\lambda_j^3 c_{j,t}^2\big)\Big)}{
\big(\sum_{i\in\mathcal B}\lambda_i^2 c_{i,t}^2\big)\Big(\sum_{j\in\mathcal D}\lambda_j^4 c_{j,t}^2+e_{\mathcal D}\Big)
-\big(\sum_{i\in\mathcal D}\lambda_i^2 c_{i,t}^2\big)\Big(\sum_{j\in\mathcal B}\lambda_j^4 c_{j,t}^2+e_{\mathcal B}\Big)}.
\end{equation}
Recall $e_{\mathcal D}=\sum_{i\in\mathcal D}\lambda_i^2\kappa_i^2$, $e_{\mathcal B}=\sum_{i\in\mathcal B}\lambda_i^2\kappa_i^2$ with $\kappa_i^2=\bm u_i^\top\mSigma\bm u_i$.

\begin{theorem}[\emph{\textbf{Decrease Condition}}]\label{thm:step-dec-below}
Under Assumption \ref{asp:standing}, if $0<\eta_t<\eta_t^*(\bm{x}_t)$, then
\[
\lim_{d\to\infty}\,\mathbb E\big[\theta_{t+1}\,\big|\,\bm{x}_t\big]\ <\ \theta_t.
\]

\end{theorem}
The detailed proof of Theorem \ref{thm:step-dec-below} is in Appendix~\ref{proof-step-size-theory}.

\noindent\textbf{Interpretation}
When the step size is smaller than the adaptive threshold \eqref{eq:def-etatstar-body}, the expected alignment function decreases in one step.
This is the ``small-step'' phase. The adaptive threshold depends on $(\mA,\mSigma)$ and the current state $\bm{x}_t$ through the function (\ref{eq:def-etatstar-body}). 
\begin{theorem}[\emph{\textbf{Increase Condition}}]\label{thm:step-inc-above}
Under Assumption \ref{asp:standing}, let
\[
\,g_{\mathrm{gap}}:=\frac{1}{1+\frac{s_{\max}}{s_{\min}}\cdot \frac{1}{\rho}\left(\frac{\lambda_{k+1}}{\lambda_k}\right)^2}\in(0,1).\,
\]
For $t$ and $\bm{x}_t$. If $\theta_t\le g_{\mathrm{gap}}$ and $\eta_t>\eta_t^*(\bm{x}_t)$, then
\[
\lim_{d\to\infty}\,\mathbb E\big[\theta_{t+1}\,\big|\,\bm{x}_t\big]\ >\ \theta_t.
\]

\end{theorem}
 The detailed proof of Theorem \ref{thm:step-inc-above} is in Appendix \ref{proof-step-size-theory}.\\
\noindent\textbf{Interpretation}
If the current alignment is in the ``small regime'' (meaning less than $g_{\mathrm{gap}}$), and the current step size $\eta_t$ exceeds the adaptive threshold, the dominant alignment increases. Notice that $g_{\mathrm{gap}}$ is an upper bound which does not depend on $\bm{x}_t$. As the dominant-bulk gap grows ($\lambda_k/\lambda_{k+1}\to\infty$), $g_{\mathrm{gap}}\to 1$, so the ``increase region'' $[0,g_{\mathrm{gap}})$ expands to nearly all $\theta_t\in[0,1)$. 
\begin{theorem}[\emph{\textbf{Large Alignment Regime Condition}}]\label{thm:step-dec-large}
Under Assumption \ref{asp:standing}, there exists a critical alignment threshold $\theta^*_t \in (0,1)$ such that if the current alignment $\theta_t \ge \theta^*_t$, then for any positive step size $\eta_t > 0$, the expected alignment decreases:
\[
\lim_{d\to\infty}\,\mathbb E\big[\theta_{t+1}\,\big|\,\bm{x}_t\big]\ <\ \theta_t.
\]
The threshold is given by $\theta^*_t := r_{0,t}/(1+r_{0,t})$, where $r_{0,t}$ is the positive root of the quadratic equation:
\[
a_t^{\mathrm{aux}} r^2+(a_t^{\mathrm{aux}}-m_t^{\mathrm{aux}}-h_t^{\mathrm{aux}})r-h_t^{\mathrm{aux}} = 0.
\]
Recall that $\psi_{\mathcal S} := \sum_{i \in \mathcal S} \lambda_i^2$ ($\mathcal{S} \in \{\mathcal{D}, \mathcal{B}\}$) and $s_t:=\sum_{i=1}^d \lambda_i^2 c_{i,t}^2$, the coefficients in the quadratic equation are defined as:
\begin{align*}
a_t^{\mathrm{aux}} &:=s_{\min}\psi_{\mathcal B}, \\
h_t^{\mathrm{aux}} &:=s_{\max}\psi_{\mathcal D}, \\
m_t^{\mathrm{aux}} &:=s_t(\lambda_1^2-\lambda_d^2).
\end{align*}
Furthermore, as $\psi_{\mathcal D}/\psi_{\mathcal B}\to\infty$, we have $\theta^*_t\to 1$. 
\end{theorem}
 The detailed proof of Theorem \ref{thm:step-dec-large} is in Appendix \ref{proof-step-size-theory}.
 {\begin{theorem}[\emph{\textbf{Asymptotic rate of $\theta_t^*$}}]\label{thm:theta-star-rate}
Under Assumption \ref{asp:standing}, let $m:=\lambda_k/\lambda_{k+1}>1$. There exist constants $\alpha,\beta>0$ such that
\[
1 - \frac{\beta}{m^2}
\;\le\;
\theta_t^*
\;\le\;
1 - \frac{\alpha}{m^2}.
\]
In particular,
\[
\theta_t^* = 1 - \Theta\!\left(\frac{1}{m^2}\right).
\]
\textbf{Interpretation} Theorem \ref{thm:step-dec-large} and \ref{thm:theta-star-rate}  shows that a very large alignment value can be \emph{self-correct}: regardless of the step size, the expected alignment decreases for $\theta_t$ beyond a computable threshold $\theta^*_t$. As the dominant-bulk gap grows ($\lambda_k/\lambda_{k+1}\to\infty$), $\theta^*_t\to 1$, so the ``decrease'' region $(\theta^*_t,1)$ shrinks toward the maximum alignment value of 1.
\end{theorem}}
\begin{theorem}[\emph{\textbf{Separation of Alignment Regimes}}]\label{thm:regime-separation}
Under Assumption \ref{asp:standing}, for any nontrivial problem with a non-zero spectral gap $(\mathrm{gap}_2 > 0)$ and bounded noise $(s_{\max} < \infty)$, the low-alignment threshold $g_{\mathrm{gap}}$ is strictly less than the high-alignment threshold $\theta^*_t$:
\[
g_{\mathrm{gap}} < \theta^*_t.
\]
\end{theorem}
 The detailed proof of Theorem \ref{thm:regime-separation} is in Appendix \ref{proof-step-size-theory}.
\paragraph{Summary}
The preceding theorems delineate two primary ``alignment regimes'' based on the current alignment value $\theta_t$. These regimes, which govern the core behavior of the SGD dynamics, are presented below.

\medskip

\noindent\fbox{%
\begin{minipage}{\dimexpr\linewidth-2\fboxsep-2\fboxrule\relax}
\centering
\begin{tabular}{c@{\qquad\qquad}c}
\fbox{$\theta_t \in (0, g_{\mathrm{gap}}]$} & \fbox{$\theta_t \in [\theta^*_t, 1)$} \\
& \\
$\underbrace{\text{\parbox{5.5cm}{\centering Alignment is dependent on $\eta_t$:\\ $\mathbb{E}[\theta_{t+1} \mid \bm{x}_t] > \theta_t$ if $\eta_t > \eta_t^*$, \\ $\mathbb{E}[\theta_{t+1} \mid \bm{x}_t] < \theta_t$ if $\eta_t < \eta_t^*$
}}}_{\text{Low-Alignment Regime}}$
&
$\underbrace{\text{\parbox{5.5cm}{\centering Alignment is self-correcting: \\ $\mathbb{E}[\theta_{t+1} \mid \bm{x}_t] < \theta_t$ \\ for any $\eta_t > 0$
}}}_{\text{High-Alignment Regime}}$
\end{tabular}
\vspace{0.2cm}
\end{minipage}%
}
\medskip
These two primary regimes are separated by an intermediate interval, $(g_{\mathrm{gap}}, \theta^*_t)$, which can be viewed as a \emph{stable regime}. For a sufficiently large step size (i.e., $\eta_t > \eta_t^*$), the dynamics create an oscillating behavior: if alignment drops below $g_{\mathrm{gap}}$, it is pushed back up, and if it exceeds $\theta^*_t$, it is forced back down. Consequently, the alignment $\theta_t$ will tend to  \emph{oscillate within this stable interval} $(g_{\mathrm{gap}}, \theta^*_t)$.
Crucially, this stable interval is dependent on the spectral structure of the Hessian. As the gap between the dominant and bulk eigenvalues grows (i.e., as $\lambda_k/\lambda_{k+1} \to \infty$), both boundaries of this interval converge towards 1. This means the interval itself asymptotically approaches the single point $\{1\}$, driving the alignment to stay close to the maximum value. To provide an intuitive understanding of the connection between $\eta_t^*$ and the current state of $\bm{x}_t$, we present the following theorem:

\begin{theorem}[\emph{\textbf{State- and gap-aware lower bounds on $\eta_t^*$ (with $\|\bm{x}_t\|_2$)}}]\label{thm:bounds-xnorm}
Under Assumption \ref{asp:standing}, for any $\bm{x}_t$, we have 
\[
\eta_t^*\ \ge\ \frac{2\,\mathrm{gap}_1}{(\lambda_1^2-\lambda_d^2)\ +\ \dfrac{s_{\max}\psi_{\mathcal D}}{\lambda_d^2\,\|\bm{x}_t\|^2_2\,\theta_t}}.
\]

\end{theorem}
The detailed proof of Theorem \ref{thm:bounds-xnorm} is in Appendix \ref{proof-step-size-theory}.\\
\textbf{Remark}
The lower bounds \emph{increase} with the parameter norm $\|\bm{x}_t\|$ and with alignment $\theta_t$ (or decrease with $1-\theta_t$ by symmetry), and with the first-order gap $\mathrm{gap}_1$. Or, more essentially, based on the expression for $\theta_t$, we can conclude that $\theta_t > \frac{\lambda_{k}^2\|\bm{P}^{\mathcal{D}}\bm{x}_t\|_2^2}{\lambda_1^2\|\bm{x}_t\|_2^2}$. Finally, we obtain: 
$$\eta_t^*\ \ge\ \frac{2\,\mathrm{gap}_1}{(\lambda_1^2-\lambda_d^2)\ +\ \dfrac{\lambda_1^2\,s_{\max}\psi_{\mathcal D}}{\lambda_k^2\lambda_d^2\,\|\bm{P}^{\mathcal{D}}\bm{x}_t\|_2^2}}.$$  This means that when the dominant space has a sufficiently large projection norm (for example, at initialization, a large random initialization can satisfy the condition with high probability), $\eta_t^*$ 's lower bound $ \approx \frac{2 \mathrm{gap}_1}{\lambda_1^2 - \lambda_d^2}$. If the eigenvalues in different subspaces can be approximated (i.e., $\forall i \in \mathcal{D}, \lambda_1 \approx \lambda_i$, and $\forall i \in \mathcal{B}, \lambda_d \approx \lambda_i$), we further obtain $\eta_t^*$'s lower bound $ \approx \frac{2}{\lambda_1 + \lambda_d}$.


{\begin{theorem}[\emph{\textbf{State- and gap-aware upper bounds on $\eta^*$}}]\label{thm:state-gap-aware}
Under Assumption \ref{asp:standing}, providing that the alignment satisfies $\theta_t \ge \frac{e_{\mathcal{B}}}{e_{\mathcal{B}} + e_{\mathcal{D}}}$, we have:
\[
\eta_t^* \ \le \ \frac{(\lambda_1 - \lambda_d)}{\lambda_k \lambda_1 - \lambda_{k+1} \lambda_d}.
\]
\end{theorem}

\begin{corollary}[\emph{\textbf{The comparison between $\eta_t^\ast$ and $\frac{2}{\lambda_1}$}}]\label{cor:conservative-limit2}
Suppose the conditions of Theorem \ref{thm:state-gap-aware} hold.  We have the following relation between $\eta_t^*$ and the convergence step size for Quadratic Programming:
\[
\eta_t^* \ \le \ \frac{2}{\lambda_1}.
\]
\end{corollary}
\textbf{Remark} Here, Theorem \ref{thm:state-gap-aware} offers the intuition that when the alignment function $\theta_t$ exceeds a critical threshold $\frac{e_{\mathcal{B}}}{e_{\mathcal{B}} + e_{\mathcal{D}}}$ (Note that this threshold depends solely on the noise covariance matrix), we can obtain Corollary \ref{cor:conservative-limit2} that $\eta^*$ is smaller than the critical step size required for gradient descent convergence in quadratic programming($\frac{2}{\lambda_1}$).
}

\section{Behavioral Analysis of Projected SGD under Alignment Conditions}
\label{secBA}

Recalling the finding by \citet{gur2018gradient} that late-stage gradients are primarily confined to the dominant subspace (high alignment, $\theta_t \to 1$), the subsequent work of \citet{song2024does} revealed a seemingly paradoxical phenomenon. They demonstrated that in this high-alignment regime, updates projected onto this supposedly critical dominant direction do not necessarily decrease the loss; instead, it is often the projection onto the orthogonal bulk subspace that guarantees a loss reduction. We resolve this apparent contradiction by developing a precise theory for the step-size conditions under which each projected update is guaranteed to decrease the expected loss.

\paragraph{Projected SGD}
To rigorously analyze this phenomenon, we define two idealized algorithms where the stochastic gradient update is explicitly projected onto the dominant and bulk subspaces, respectively. The update rules for the Projected SGD algorithms are given as follows: 

\begin{align}
&\textit{Dominant Projected SGD (DSGD):} && \bm{x}_{t+1} = \bm{x}_t - \eta_t \, \mP^{\mathcal{D}}(\mA \bm{x}_t + \boldsymbol{\vxi}_t). \label{eq:dominant_sgd} \\
&\textit{Bulk Projected SGD (BSGD):} && \bm{x}_{t+1} = \bm{x}_t - \eta_t \, \mP^{\mathcal{B}}(\mA \bm{x}_t + \boldsymbol{\vxi}_t). \label{eq:bulk_sgd}
\end{align}
To analyze their decay condition, we first define the necessary state- and noise-dependent quantities for a subspace $\mathcal{S} \in \{\mathcal{D}, \mathcal{B}\}$:

\begin{equation}
\begin{aligned}
\tau_t^{\mathcal S} &:= \sum_{i\in\mathcal S}\lambda_i^3 c_{i,t}^2, \qquad 
n^{\mathrm{loss}}_{\mathcal S} &:= \sum_{i\in\mathcal S}\lambda_i \kappa_i^2, \qquad \mu_t^{\mathcal S}(p, q) := \frac{\sum_{i\in\mathcal S}\lambda_i^p c_{i,t}^2}{\sum_{i\in\mathcal S}\lambda_i^q c_{i,t}^2}.
\end{aligned}
\label{eq:combined-formulas}
\end{equation}

\begin{theorem}[ \emph{\textbf{Step Size Condition for Projected Updates}}]\label{thm:proj-loss-corr}
For a given state $\bm{x}_t$, the one-step expected loss decreases if and only if the step size $\eta_t$ satisfies the following conditions for the dominant and bulk projected updates, respectively:
\\
\noindent%
\begin{minipage}{0.97\linewidth}
\begin{align*}
\text{For } \mathcal{D}: \quad \mathbb E[L(\bm{x}_{t+1})-L(\bm{x}_t)\mid \bm{x}_t]<0 \quad &\Longleftrightarrow \quad 0<\eta_t< \eta^{\mathrm{loss}}_{\mathcal D}(\bm{x}_t):=\frac{2\,s_t^{\mathcal D}}{\tau_t^{\mathcal D}+n^{\mathrm{loss}}_{\mathcal D}} \\[2.5ex]
\text{For } \mathcal{B}: \quad \mathbb E[L(\bm{x}_{t+1})-L(\bm{x}_t)\mid \bm{x}_t]<0 \quad &\Longleftrightarrow \quad 0<\eta_t< \eta^{\mathrm{loss}}_{\mathcal B}(\bm{x}_t):=\frac{2\,s_t^{\mathcal B}}{\tau_t^{\mathcal B}+n^{\mathrm{loss}}_{\mathcal B}}
\end{align*}
\end{minipage}

\end{theorem}
The detailed proof of Theorem \ref{thm:proj-loss-corr} is in Appendix \ref{secpfpst}. The relationship between these two thresholds is not fixed; it is dynamically determined by the alignment of the current state. The following theorems characterize this crossover phenomenon. 

\begin{theorem}[\emph{\textbf{Condition Differences on Different Alignment Regime}}]\label{thm:loss-crossover-rigorous}
Under Assumption \ref{asp:standing}, the relative ordering of the two loss thresholds $\eta^{\mathrm{loss}}_{\mathcal D}(\bm{x}_t) $ and $ \eta^{\mathrm{loss}}_{\mathcal B}(\bm{x}_t)$ is determined by the alignment $\theta_t$ of a given $\bm{x}_t$ relative to the unique critical threshold, $\theta^{\text{crit}}_t \in (0,1)$. This threshold is the root of the quadratic equation $h(\theta) = \alpha_t \theta^2 + \beta_t \theta + \gamma_t = 0$, with coefficients defined by:

\begin{align}
\alpha_t &=  s_t \left( \mu_t^{\mathcal D}(3,2) - \mu_t^{\mathcal B}(3,2) \right) \\
\beta_t &= -\alpha_t + n^{\mathrm{loss}}_{\mathcal B} + n^{\mathrm{loss}}_{\mathcal D}  \\
\gamma_t &= -n^{\mathrm{loss}}_{\mathcal D}
\end{align}
Where $n^{\mathrm{loss}}_{\mathcal S} ,\mu_t^{\mathcal S}(p, q)$ is defined in formula (\ref{eq:combined-formulas}). The ordering of the thresholds in the two resulting regimes is as follows:
\[
\begin{aligned}
\theta_t < \theta^{\text{crit}}_t \quad & \Longleftrightarrow \quad \eta^{\mathrm{loss}}_{\mathcal D}(\bm{x}_t) < \eta^{\mathrm{loss}}_{\mathcal B}(\bm{x}_t), \\
\theta_t > \theta^{\text{crit}}_t \quad & \Longleftrightarrow \quad \eta^{\mathrm{loss}}_{\mathcal D}(\bm{x}_t) > \eta^{\mathrm{loss}}_{\mathcal B}(\bm{x}_t).
\end{aligned}
\]
\end{theorem}
The detailed proof of Theorem \ref{thm:loss-crossover-rigorous} is in Appendix \ref{secpfpst}. Furthermore, as stated in the following theorem, this critical alignment threshold exhibits a predictable behavior in ill-conditioned settings, which converges to a state of full alignment 1.

\begin{theorem}[\emph{\textbf{Asymptotic Limit of the Alignment Threshold}}]\label{thm:loss-crossover-asymptotic}
Under Assumption \ref{asp:standing}, let the spectral gap be denoted by $m := \lambda_k/\lambda_{k+1}$. In the limit as the gap grows infinitely large, the critical threshold in Theorem \ref{thm:loss-crossover-rigorous} converges to 1:
\[
\lim_{m \to \infty} \theta^{\eta}_{\mathrm{crit}}(\bm{x}_t) = 1.
\]
\end{theorem}

\paragraph{Summary}
The central finding in this section is the existence of a critical alignment threshold, $\theta^{\text{crit}}_t$, that dictates the relative stability between  DSGD and BSGD. This threshold partitions the alignment space into two regimes with inverted stability orderings, as summarized below:
\\
\noindent\fbox{%
\begin{minipage}{0.97\linewidth}
\centering
\begin{tabular}{c@{\qquad\qquad}c}
\boxed{\theta_t < \theta^{\text{crit}}_t} & \boxed{\theta_t > \theta^{\text{crit}}_t} \\[1.5ex]
$\underbrace{\eta^{\mathrm{loss}}_{\mathcal D}(\bm{x}_t) < \eta^{\mathrm{loss}}_{\mathcal B}(\bm{x}_t)}_{\text{Low-Alignment Regime}}$
&
$\underbrace{\eta^{\mathrm{loss}}_{\mathcal D}(\bm{x}_t) > \eta^{\mathrm{loss}}_{\mathcal B}(\bm{x}_t)}_{\text{High-Alignment Regime}}$
\end{tabular}
\vspace{0.2cm}
\end{minipage}
}
\\
Our asymptotic analysis shows that as the spectral gap grows ($m := \lambda_k/\lambda_{k+1} \to \infty$), the critical threshold converges to one: $\lim_{m\to\infty} \theta^{\text{crit}}_t = 1$. This has a direct consequence on the operational regimes: the high-alignment regime $(\theta^{\text{crit}}_t, 1)$ asymptotically vanishes to \{1\}. Conversely, the low-alignment regime $[0, \theta^{\text{crit}}_t)$ expands to occupy nearly the entire alignment space $[0,1)$. The dynamics within this now-predominant low-alignment regime are therefore of primary importance. Here, the stability thresholds are ordered such that $\eta^{\mathrm{loss}}_{\mathcal D} < \eta^{\mathrm{loss}}_{\mathcal B}$, which leads to the following different descent behaviors:
\[
\forall\, \eta_t \in (\eta^{\mathrm{loss}}_{\mathcal D}, \eta^{\mathrm{loss}}_{\mathcal B}),
\quad
\begin{cases}
\text{For } \mathcal{D}: \quad \mathbb E[L(\bm{x}_{t+1})-L(\bm{x}_t)\mid \bm{x}_t] > 0 & \text{(Loss Increases)} \\
\text{For } \mathcal{B}: \quad \mathbb E[L(\bm{x}_{t+1})-L(\bm{x}_t)\mid \bm{x}_t] < 0 & \text{(Loss Decreases)}
\end{cases}
\]
This provides a theoretical explanation for the empirical findings of \citet{song2024does}, which observed that updates along dominant directions can increase while those in the bulk remain decreases, although the gradient is aligned with dominant space. However, it should be noted that the two alignment regimes here are not directly related to those in the previous Section~\ref{sec:step_size_condition}, although they exhibit the same behavior in the asymptotic case.

{\begin{theorem}[\emph{\textbf{Rate of the critical alignment threshold}}]
\label{thm:theta-crit-rate}
Under Assumption~\ref{asp:standing}, let $m := \frac{\lambda_k}{\lambda_{k+1}} > 1.$ Then the critical alignment threshold $\theta_t^{\mathrm{crit}}\in(0,1)$ satisfies
\[
\frac{n_{\mathcal{B}}^{\mathrm{loss}}}
{s_t\,\lambda_{k+1}(m-1)
+n_{\mathcal{B}}^{\mathrm{loss}}
+n_{\mathcal{D}}^{\mathrm{loss}}}
\;\le\;
1-\theta_t^{\mathrm{crit}}
\;\le\;
\frac{n_{\mathcal{B}}^{\mathrm{loss}}}
{s_t\,\lambda_{k+1}(m-1)}.
\]
If $\lambda_{k+1}=\Theta(1)$, consequently,
\[
\lambda_{k}=\Theta(m),  1-\theta_t^{\mathrm{crit}}
\in
\Theta\!\left(\frac{1}{s_t\,(m-1)}\right).
\]
\end{theorem}
\textbf{Remark} Note that the asymptotic order here is simultaneously related to both $s_t$ and $m$.  $s_t$ is positively correlated with the loss function $L(\bm{x}_t)$. Although we discussed in Theorem \ref{thm:theta-star-rate} that $\theta_t^\ast$ is asymptotically 1 with order $\Theta\!\left(\frac{1}{m^2}\right)$. Notice that $\theta^\ast_t>\theta_t^{crit}$ with large enough $m$, which implies that the alignment stable regime may drop in Bulk projection unstable regime. This seems to contradict the empirical phenomena, but actually it doesn't. We recall the phenomenon observed by \citet{song2024does}: when the loss has not yet converged (i.e., $L(\bm{x}_t)$ has not reached a very small regime), even the alignment is already very high at this point, if we start performing subspace projection on the updates, we will observe the suspicious alignment phenomenon.
The conditions here, in addition to the inherently ill-conditioned structure of the loss Hessian itself, also require that the loss $L(\bm{x}_t)$ is still in the pre-convergence stage.

}

\section{Constant Step Size SGD: Two-Phase Dynamics of $\theta_t$}
\label{csgdsection}
Building upon the step size condition theory from the preceding sections, we now investigate the long-term dynamics under the \emph{constant step size} SGD, which we will refer to as CSGD in the later content. This specialization allows us to characterize a distinct two-phase learning dynamic: an initial, transient phase driven by the initial state $\bm{x}_0$, and a late-time, steady-state equilibrium driven by noise. In this section, we make the following assumptions about the step size and initialization for CSGD.
Recall that $c_{i,t} := \langle \bm{x}_t, \bm{u}_i \rangle$ is the projection of the state $\bm{x}_t$ onto the $i$-th eigenvector of matrix $\mA$. To give a further analysis, we define several key quantities:
\[
\beta_i := \frac{\eta\,\kappa_i^2}{2\lambda_i-\eta\lambda_i^2} > 0, \varrho_{\mathcal{D}} := \sum_{i \in \mathcal{D}} (c_{i,0}^2 - \beta_i).
\]
\[
\delta := \frac{s_{\max}\psi_{\mathcal D}\lambda_1^2}{\lambda_d^2 \lambda_k^2 \left(\frac{2(\lambda_k-\lambda_{k+1})}{\eta} - (\lambda_1^2 - \lambda_d^2)\right)}.
\]

\begin{assumption}
\label{assum-cssa}
Our analysis is based on the following assumptions:
\begin{table}[H]
\centering
\caption{Assumptions for CSGD}
\label{tab:const_eta_assumptions}
\begin{tabular}{ll}
\hline
\textbf{Assumption} & \textbf{Description} \\ \hline \\
Constant Step Size & 
$\eta_t = \eta, \quad 0 < \eta < \min\left\{ \frac{2}{\lambda_1}, \frac{2(\lambda_k - \lambda_{k+1})}{\lambda_1^2 - \lambda_d^2} \right\}$ \\[4ex]
Initialization for $\bm{x}_0$& 
$\begin{array}{l}
\forall i \in \mathcal{D} ,\quad c_{i,0}^2 > \beta_i, \\
\varrho_{\mathcal{D}} > \delta - \sum_{i \in \mathcal{D}} \beta_i
\end{array}$ \\ \\ \hline
\end{tabular}
\end{table}
\end{assumption}

\begin{theorem}\label{thm:monotone-decrease-expected-alignment}
(Initial Decrease Phase) Under Assumption \ref{asp:standing} and Assumption \ref{assum-cssa}, let the  $t^*$ be defined as
\[
t^* := \left\lfloor \frac{\log\left( \dfrac{\varrho_{\mathcal{D}}}{\delta - \sum_{i \in \mathcal{D}} \beta_i} \right)}{-2\log(1-\eta\lambda_1)} \right\rfloor.
\]
Where $\left\lfloor \cdot \right\rfloor$ is the floor function. Then for all time steps $t \in \{0, 1, \dots, t^*-1\}$, the expected alignment is strictly decreasing:
\[
\lim_{d\to\infty}\,\mathbb{E}[\theta_{t+1}] < \lim_{d\to\infty}\mathbb{E}[\theta_t].
\]
\end{theorem}
The proof of Theorem \ref{thm:monotone-decrease-expected-alignment} is in Appendix \ref{sectionPCSD}.\\
\textbf{Interpretation}  Theorem \ref{thm:monotone-decrease-expected-alignment} establishes the existence of a predictable initial phase in SGD dynamics, guaranteeing that for a sufficiently large initialization, the expected alignment $\mathbb{E}[\theta_t]$ will monotonically decrease for a calculable duration of $t^*$ steps. The formula for $t^*$ explicitly links this phase duration to the initial state: a positive-length phase is guaranteed when the initial $\varrho_{\mathcal{D}}$, exceeds a threshold. Furthermore, the theorem shows that the length of this phase, $t^*$, grows logarithmically with this initial $\varrho_{\mathcal{D}}$.

\begin{theorem}\label{thm:late-theta}(Late Phase)
Under Assumption~\ref{asp:standing} and Assumption~\ref{assum-cssa}, the late-time asymptotic expected alignment is given by
\[
\theta_\infty := \lim_{t \to \infty} \lim_{d \to \infty} \mathbb{E}_t[\theta_t]
= \frac{\lim_{d \to \infty} \sum_{i \in \mathcal{D}} \lambda_i^2 \beta_i}{\lim_{d \to \infty} \sum_{i=1}^d \lambda_i^2 \beta_i},
\]
where $\beta_i = \dfrac{\eta \kappa_i^2}{2\lambda_i - \eta \lambda_i^2} > 0$. Equivalently,
\[
\theta_\infty = \frac{\lim_{d \to \infty} \sum_{i \in \mathcal{D}} \dfrac{\eta \lambda_i^2 \kappa_i^2}{2\lambda_i - \eta \lambda_i^2}}{\lim_{d \to \infty} \sum_{i=1}^d \dfrac{\eta \lambda_i^2 \kappa_i^2}{2\lambda_i - \eta \lambda_i^2}}.
\]
\end{theorem}
The proof of Theorem \ref{thm:late-theta} is in Appendix \ref{sectionPCSD}.\\
\textbf{Interpretation} After the transient decay, the alignment settles to a stable level $\theta_\infty$ determined solely by the Hessian spectrum $\{\lambda_i\}$, the noise covariance $\{\kappa_i^2\}$, and the step size $\eta$. In the case of isotropic noise ($\mSigma = \sigma^2 \bm{I}$, so $\kappa_i^2 = \sigma^2$ for all $i$) and small step size $\eta \ll 1$, we have the approximation
\[
\theta_\infty \approx \frac{\lim_{d \to \infty} \sum_{i \in \mathcal{D}} \lambda_i}{\lim_{d \to \infty} \sum_{i=1}^d \lambda_i}.
\]
Furthermore, if the dominant--bulk eigenvalue gap grows such that $\lambda_k / \lambda_{k+1} \to \infty$ while the bulk spectrum remains bounded away from zero, then $\theta_\infty \to 1$. This provides a theoretical explanation for the long-run alignment of SGD with the dominant eigenspace.
\section{Numerical Experiment}
\label{NES}
\paragraph{Numerical Simulation with Different Spectrum Gap}

We performed several numerical simulations of constant step size SGD under varying spectral gaps. In these experiments, we fixed the dimension $d=500$, the dominant space dimension $k=50$, the step size $\eta=0.003$, and the total number of steps $T=30,000$. For each spectral gap 
$$m = \lambda_k / \lambda_{k+1} \in \{5, 10, 20, 50, 100, 200, 300, 400, 500\},$$
We randomly initialized positive definite symmetric matrices $A$ with the corresponding spectral gap and conducted comparative experiments. {For brevity,} We selected four experiments to present in the main text, with simulations for $m=5, 20, 50, 200$. We also plotted the loss decline curves to ensure that the algorithm converges to the convergence stage, with specific results shown in Figure \ref{simfig}. The complete experimental results {and setup
}are available at
Appendix \ref{appendix_simulation}. It can be observed that all experiments exhibit a two-phase phenomenon, where the alignment function initially decreases in the short term and then increases. Furthermore, as $m$ increases, the stable region of the alignment function gradually approaches 1. 

\begin{figure}[h]
    \centering
    \subfigure[m=5]{\includegraphics[width=0.45\textwidth]{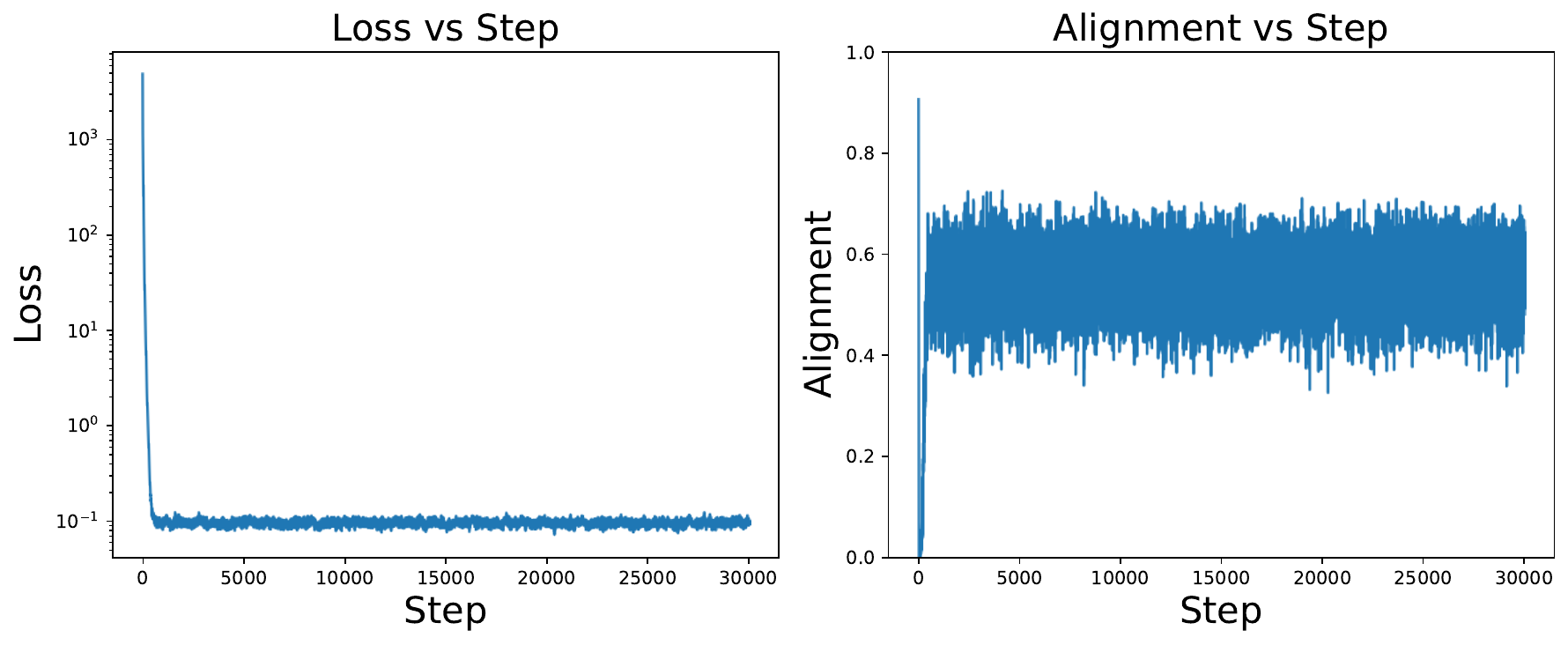}}%
    \subfigure[m=20]{\includegraphics[width=0.45\textwidth]{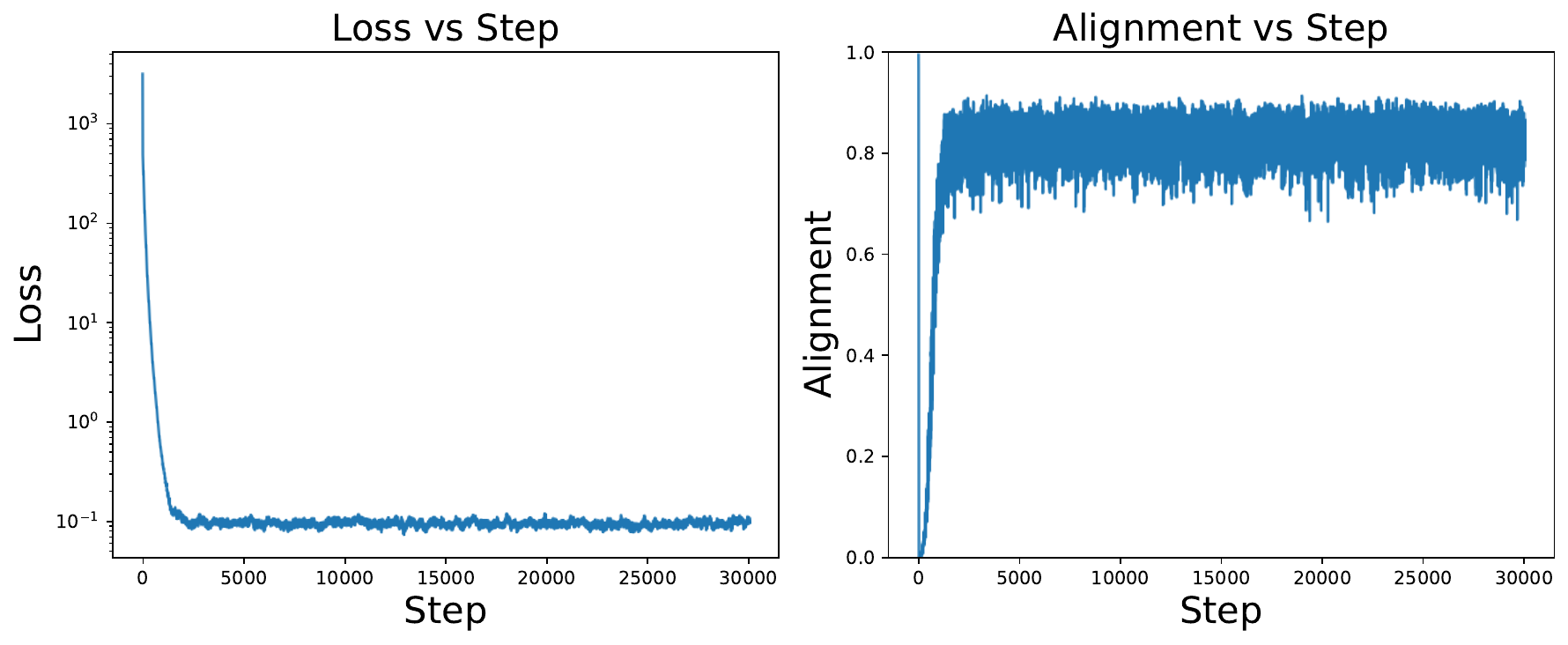}}%
     \\
    \subfigure[m=50]{\includegraphics[width=0.45\textwidth]{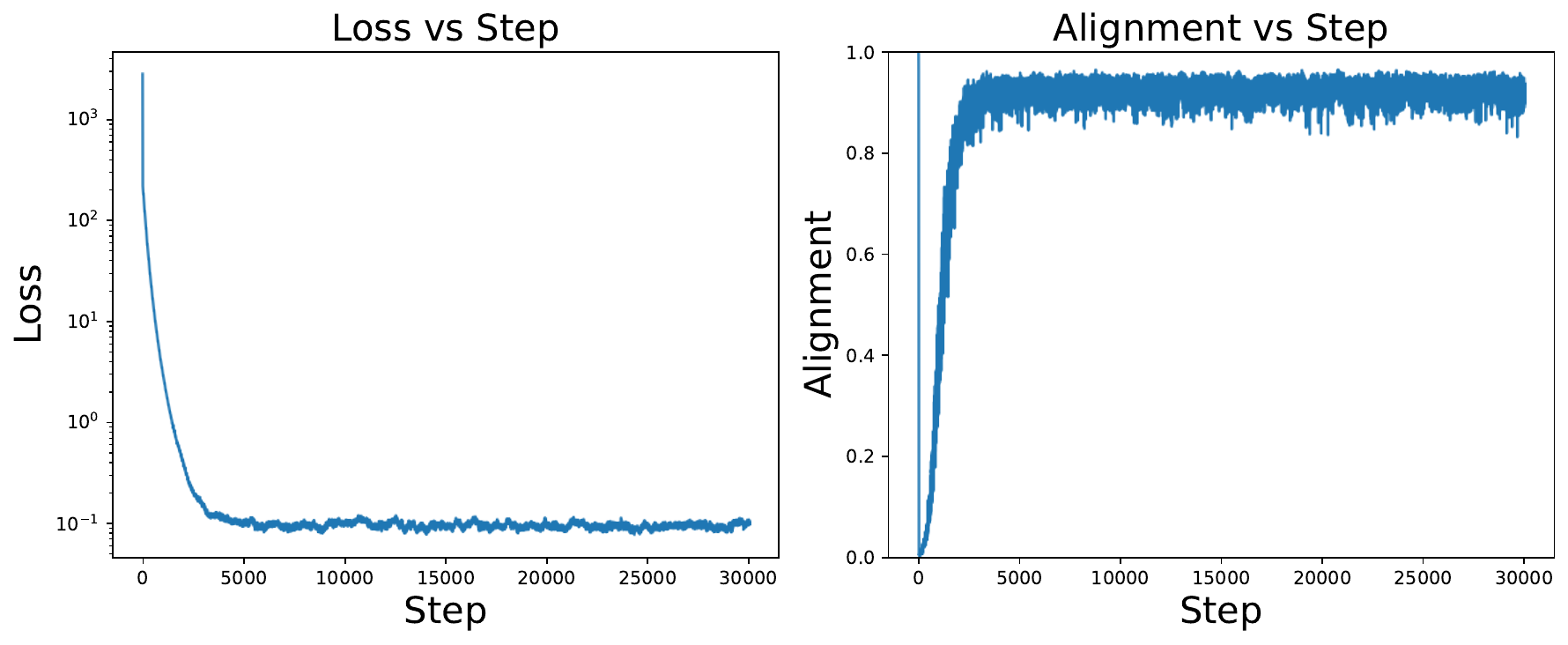}}%
    \subfigure[m=200]
{\includegraphics[width=0.45\textwidth]{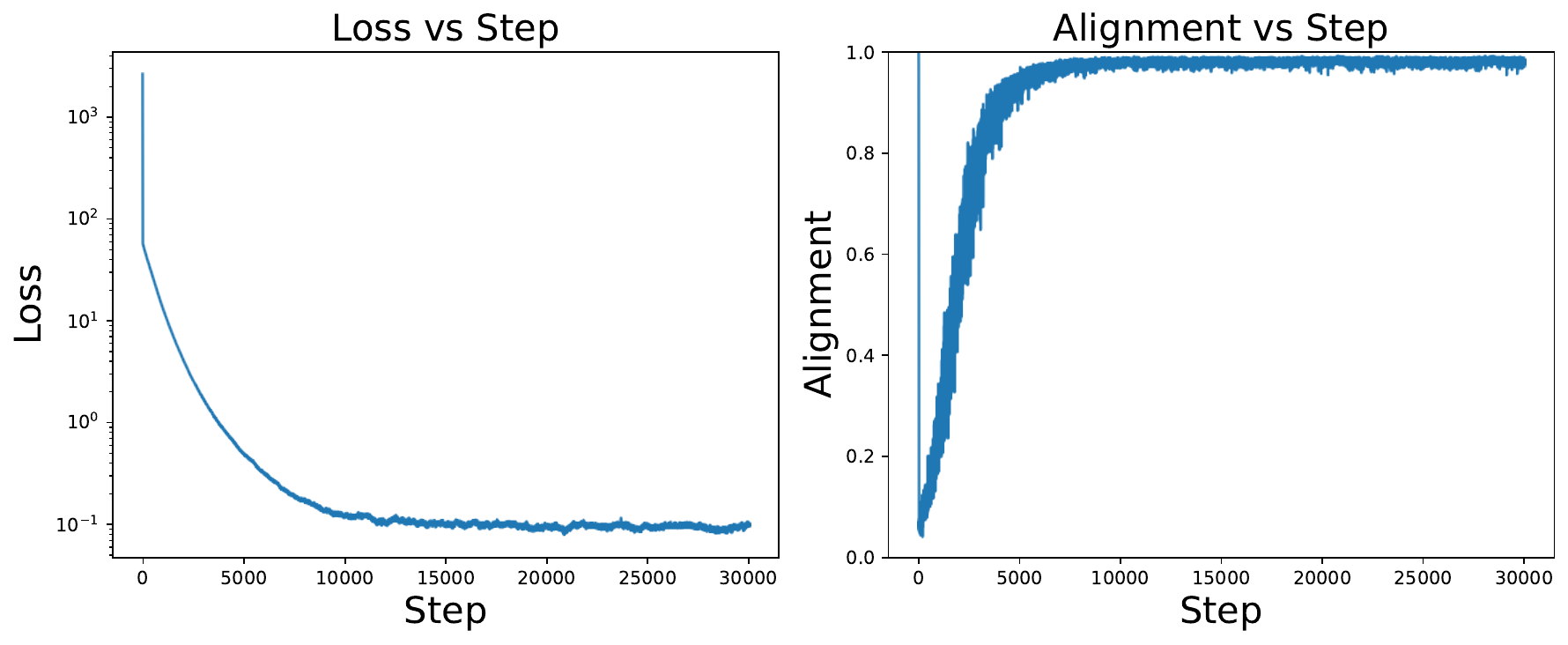}}%
    \caption{Numerical simulation experiments with different spectral gaps (\(m=\lambda_k/\lambda_{k+1}\))}\label{simfig}
\end{figure}

\addressedtianyu{I think for the following section We need to clarify which case this $\theta_{\infty}$ corresponds to— the constant lr setting, the decaying lr setting, or both (I think it is const lr? since you refer to theorem 16). Otherwise, it looks like the experiments above and below partially overlap.
 }
{\paragraph{The Simulation for $\mathbb{E}[\theta_{\infty}]$ and Spectrum Gap}
Following the setup from the previous experiment (detailed in Appendix \ref{appendix_simulation}, where a constant step size is employed for SGD), we performed a simulation study to investigate the relationship between the limiting stable value $\theta_\infty$ and the spectral gap $m = \lambda_k / \lambda_{k+1}$ in the expectation context. Here, $\theta_T$ denotes the value at time $T$, which marks the beginning of the second phase where the system enters stability, and $T_{\text{end}}$ denotes the end time of the experiment. We use the statistic:
\[
\mathbb{E}[\text{Alignment}] = \frac{1}{T_{\text{end}} - T} \sum_{t=T}^{T_{\text{end}}} \theta_t
\]
to estimate $\mathbb{E}[\theta_\infty]$. We plotted the trend of $\mathbb{E}[\text{Alignment}]$ and recorded the variance of this statistic as functions of $m$, the latter reflecting the trend in the volatility of $\theta_t$ in the second phase, as shown in Figure~\ref{AE2}. We also performed tests with different random seeds, the detail results can be found in Appendix \ref{appendix_simulation}. It can be observed that as $m$ gradually increases, $\mathbb{E}[\text{Alignment}]$ gradually approaches 1 with fewer fluctuations, which is consistent with the result of Theorem \ref{thm:late-theta}.
\begin{figure}
    \centering
    \includegraphics[width=0.8\linewidth]{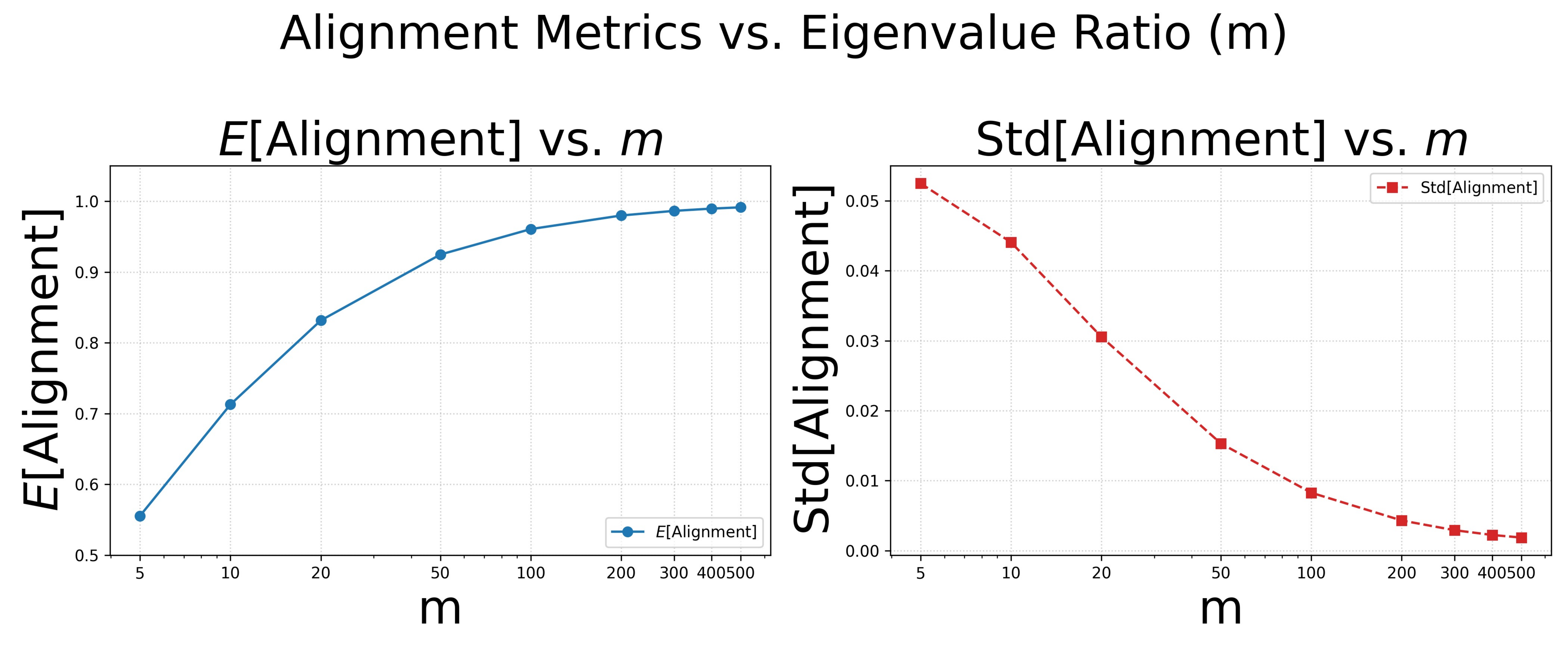}
    \caption{$\mathbb{E}[\text{Alignment}]$ and $\text{Std}[\text{Alignment}]$ vs $m = \lambda_k / \lambda_{k+1}$}
    \label{AE2}
\end{figure}
Besides the primary numerical simulation results discussed above, the Appendix\ref{appendix_simulation} contains additional simulation analyses, including results on the decay rate in the first phase and a discussion of the order in $m$ at which $\theta_\infty$ asymptotically approaches 1. The code for our simulation experiments is available at \href{https://github.com/xuan-lgbq/Suspicious-Alignment-of-SGD.git}{link}. }

\section{Conclusion}
Combining Sections~\ref{sec:step_size_condition}, \ref{secBA}, and \ref{csgdsection}, this work points out that for ill-conditioned loss structures, when optimized using SGD, the gradient dynamics exhibit two distinct alignment trends depending on the choice of step size. In the early stages of optimization, following sufficiently large initialization, if the step size $\eta_t$ is relatively large compared to the norm of the optimization variable $\bm{x}_t$, the gradient gradually aligns with the bulk space. However, as the optimization progresses and the norm of the optimization variable $\bm{x}_t$ decreases, entering a small-scale regime, if $\eta_t$ does not keep pace with the decay of the norm of $\bm{x}_t$, the gradient gradually aligns with the dominant space (Theorems~\ref{thm:step-dec-below}, \ref{thm:step-inc-above}, and \ref{thm:step-dec-large}). However, this alignment does not imply that the dominant direction of the gradient is beneficial to loss reduction. On the contrary, the direction that truly facilitates loss reduction remains the bulk space (Theorems~\ref{thm:proj-loss-corr}, \ref{thm:loss-crossover-rigorous}, and \ref{thm:loss-crossover-asymptotic}). Using the ``river-valley'' structure analogy \citep{wen2024understanding}, each step of SGD’s update has a critical step size $\eta_t^*$. When $\bm{x}_t$ is sufficiently far from the ``river'' or the update step size is smaller than $\eta_t^*$, $\bm{x}_t$ moves closer to the ``river.'' When $\bm{x}_t$ is sufficiently close to the ``river'' and the update step size is larger than $\eta_t^*$, $\bm{x}_t$ moves toward the ``valley.'' This critical step size $\eta_t^*$ gradually decreases and converges to approximately lower than $\approx \dfrac{1}{\lambda_1 + \lambda_d}$ (Theorem~\ref{thm:regime-separation}). This suggests that to track the ``river'' during the optimization process, the step size should be set to less than $\approx \dfrac{1}{\lambda_1 + \lambda_d}$. However, to ensure sufficiently fast movement in the direction of the ``river,'' the step size needs to be sufficiently large. This indicates that to balance tracking the ``river'' and maintaining fast optimization progress, the step size should be at most $\approx \dfrac{1}{\lambda_1 + \lambda_d}$. This proves that for SGD, tracking the ``river'' while preserving optimization efficiency is constrained by the ill-conditioned structure. This provides an explanation for SGD to maintain optimization effectiveness while ensuring efficiency. To achieve better optimization efficiency in such cases, further improvements to the descent direction are necessary, such as those found in many existing preconditioning methods, including methods that approximate the inverse Hessian or assign step sizes to individual parameters \citep{kingma2014adam,gupta2018shampoo,yao2021adahessianadaptivesecondorder,song2025through}.  The analysis of the relationship between step size selection and alignment for these preconditioning methods can be left as a direction for future research.

\acks{We thank our colleagues and funding agencies. This work is supported by DOE under Award Number DE-SC0025584, Dartmouth College, and Lambda AI.}

\bibliography{references}

\appendix
\section{Proof of Theorems}

\subsection{Preliminaries, Notation, and Assumptions (for reference)}
\paragraph{Spectral notation and projections.}
We recall $\mA=\mU\mLambda \mU^\top$ with $\mU=[\bm u_1,\dots,\bm u_d]$ orthonormal and $\Lambda=\mathrm{diag}(\lambda_1,\dots,\lambda_d)$ strictly decreasing:
\[
\lambda_1 \geq \cdots\geq \lambda_k  >\lambda_{k+1} \geq \cdots\geq \lambda_d > 0.
\]
Any $\bm{x}_t\in\mathbb \mathbb{R}^d$ is expanded in the eigenbasis as
\[
\bm{x}_t=\sum_{i=1}^d c_{i,t}\,\bm u_i,\qquad c_{i,t}:=\langle \bm{x}_t,\bm u_i\rangle\in\mathbb R.
\]
For a dominant–bulk partition $\mathcal D=\{1,\dots,k\}$, $\mathcal B=\{k+1,\dots,d\}$, define
\[
\mP^{\mathcal D}:=\sum_{i\in\mathcal D}\bm u_i\bm u_i^\top,\qquad
\mP^{\mathcal B}:=\sum_{i\in\mathcal B}\bm u_i\bm u_i^\top,
\]
and, for $p\in\mathbb N$,
\[
\mP^{\mathcal S}_{\lambda^p}:=\sum_{i\in\mathcal S}\lambda_i^p\,\bm u_i\bm u_i^\top\qquad (\mathcal S\in\{\mathcal D,\mathcal B\}).
\]

\paragraph{Alignment and subspace notion.}
The squared alignment function \(\theta(\bm{x}_t)\) of the gradient \(\mA \bm{x}_t\) with the dominant subspace at time \(t\) is defined as
\[
\theta(\bm{x}_t) := 
\begin{cases}
\dfrac{\|\bm{P}^{\mathcal{D}} \nabla L(\bm{x}_t)\|_2^2}{\|\nabla L(\bm{x}_t)\|_2^2} = \dfrac{\|\bm{P}^{\mathcal{D}}\mA\bm{x}_t\|_2^2}{\|\mA\bm{x}_t\|_2^2} = \dfrac{\sum_{i \in \mathcal{D}} \lambda_i^2 c_{i,t}^2}{\sum_{i=1}^d \lambda_i^2 c_{i,t}^2} \in [0,1] & \text{if } \bm{x}_t \neq \bm{0}, \\
0 & \text{if } \bm{x}_t = \bm{0},
\end{cases}
\]
The convention \(\theta(\bm{0}) = 0\) ensuring well-definedness in stochastic analysis).We will use $\theta_t$ for short in the later section. Define the subspace quantities as follow:
\begin{align}
\notag &\psi_{\mathcal D} := \sum_{i \in \mathcal D} \lambda_i^2, \quad \psi_{\mathcal B} := \sum_{i \in \mathcal B} \lambda_i^2,  \\
s_t^{\mathcal{D}} := \sum_{i \in \mathcal{D}} \lambda_i^2 c_{i,t}^2, \quad &s_t^{\mathcal{B}} := \sum_{i \in \mathcal{B}} \lambda_i^2 c_{i,t}^2, \quad s_t := s_t^{\mathcal{D}} + s_t^{\mathcal{B}} = \sum_{i=1}^d \lambda_i^2 c_{i,t}^2. 
\end{align}
and the spectral gaps is defined as:
\[
\mathrm{gap}_1:=\lambda_k-\lambda_{k+1}>0,\qquad
\mathrm{gap}_2:=\lambda_k^2-\lambda_{k+1}^2=(\lambda_k-\lambda_{k+1})(\lambda_k+\lambda_{k+1})>0.
\]

\paragraph{Noise in the eigenbasis.}
The SGD update is
\[
\bm{x}_{t+1}=\bm{x}_t-\eta_t(\mA \bm{x}_t+\boldsymbol{\xi}_t),\qquad \boldsymbol{\xi}_t\sim\mathcal N(0,\mSigma),\quad \mSigma=\mSigma^\top\succeq 0,\quad \{\boldsymbol{\xi}_t\}\ \text{i.i.d.}
\]
Let $\mC:=\mU^\top \mSigma \mU$ and define per-direction noise variances
\[
\kappa_i^2:=(C)_{ii}=\bm u_i^\top\mSigma \bm u_i\ge 0,\qquad
s_{\min}:=\lambda_{\min}(\mSigma),\ \ s_{\max}:=\lambda_{\max}(\mSigma).
\]
Set the block-wise noise energy be:
\[
e_{\mathcal D}:=\sum_{i\in\mathcal D}\lambda_i^2 \kappa_i^2,
\qquad
e_{\mathcal B}:=\sum_{i\in\mathcal B}\lambda_i^2 \kappa_i^2,
\]
and recall $\psi_{\mathcal S} := \sum_{i \in \mathcal S} \lambda_i^2$ for $\mathcal S \in \{\mathcal D, \mathcal B\}$. Then we have: $$s_{\min}\psi_{\mathcal D}\le e_{\mathcal D}\le s_{\max}\psi_{\mathcal D},s_{\min}\psi_{\mathcal B}\le e_{\mathcal B}\le s_{\max}\psi_{\mathcal B}.$$

And our assumptions are:
\begin{assumption-restated}\ref{asp:standing} (\textbf{Asymptotic Spectral Assumptions})
When we consider the high-dimensional regime, where both $d$ and $k(d) \to \infty$, we assume the following conditions, summarized in Table~\ref{tab:assumptions}:
\begin{table}[H]
\centering
\caption{Asymptotic Spectral Assumptions}

\begin{tabular}{>{\raggedright\arraybackslash}p{4cm} >{\raggedright\arraybackslash}p{8cm}}
\toprule
\textbf{Assumption} & \textbf{Description} \\
\midrule
Trajectory boundedness & The state remains bounded: $\sup_t \limsup_{d \to \infty} \frac{1}{d} \sum_{i=1}^d c_{i,t}^2 < \infty$. \\
Block proportion & The subspace dimension ratio is a fixed constant: $\rho :=\frac{k}{d-k} \in (0, \infty)$. Which implies $k=\frac{\rho}{1+\rho} d$. \\
Block spectral moments & For $p \in \{2, 3, 4, 6, 8\}$, the block-wise spectral moments converge: $\frac{1}{k} \sum_{i \in \mathcal{D}} \lambda_i^p \to \lambda_{\mathcal{D},p} \in (0, \infty)$, $\frac{1}{d-k} \sum_{i \in \mathcal{B}} \lambda_i^p \to \lambda_{\mathcal{B},p} \in [0, \infty)$. \\
Noise spectral bounds & The noise covariance has a bounded trace: $Tr(\mSigma) \in (0, +\infty)$  \\
\bottomrule
\end{tabular}
\end{table}
\end{assumption-restated}
\subsection{Next step alignment function}

\paragraph{Exact one-step alignment transform.}
The SGD update projected onto the eigenbasis gives the evolution of the coefficients $c_{i,t}$:
\[
c_{i,t+1} = \langle \bm{x}_{t+1}, \bm{u}_i \rangle = \langle \bm{x}_t - \eta_t(\mA\bm{x}_t + \boldsymbol{\xi}_t), \bm{u}_i \rangle = c_{i,t} - \eta_t(\lambda_i c_{i,t} + \zeta_{i,t}) = (1-\eta_t\lambda_i)c_{i,t} - \eta_t\zeta_{i,t},
\]
where $\zeta_{i,t} := \langle \boldsymbol{\xi}_t, \bm{u}_i \rangle$.
 At time $t+1$,$s_{t+1}^{\mathcal{S}} = \sum_{i \in \mathcal{S}} \lambda_i^2 c_{i,t+1}^2$. Thus,
\begin{equation}\label{eq:theta-next}
\theta_{t+1}
=\frac{\sum_{i\in\mathcal D}\lambda_i^2\big((1-\eta_t\lambda_i)c_{i,t} - \eta_t\zeta_{i,t}\big)^2}{\sum_{i=1}^d\lambda_i^2\big((1-\eta_t\lambda_i)c_{i,t} - \eta_t\zeta_{i,t}\big)^2}
=\frac{1}{1+\frac{s_{t+1}^{\mathcal B}}{s_{t+1}^{\mathcal D}}},
\end{equation}
where
\begin{equation}\label{eq:apbp-def}
s_{t+1}^{\mathcal S}:=\sum_{i\in\mathcal S}\lambda_i^2\big((1-\eta_t\lambda_i)c_{i,t} - \eta_t\zeta_{i,t}\big)^2 \qquad (\mathcal{S} \in \{\mathcal{D}, \mathcal{B}\}).
\end{equation}
\textbf{Remark.} When we fix a $\bm{x}_t$, we can reasonably assume that $s_{t+1} \ne0,s_{t+1}^{\mathcal{D}}\ne 0$, since the case that $\{\xi|s_{t+1}(\xi)=0\},\{\xi|s_{t+1}^{\mathcal{D}}(\xi)=0\}$ are zero measure set for $\xi$. \\
Notice that,when we fix $\bm{x}_t$ (such as we consider the conditional expectation $\mathbb{E}[\cdot|\bm{x}_t]$ later). $\theta_{t+1},s_{t+1}=\sum_{i=1}^d\lambda_i^2\big((1-\eta_t\lambda_i)c_{i,t} - \eta_t\zeta_{i,t}\big)^2$ is a function of $\xi_t$. Let $ \bm{x}_t^{\prime}$ be a vector that:
$$\forall i \in \mathcal{D},<\bm{x}_t^{\prime},\bm{u}_i>= \dfrac{(1-\eta_t\lambda_i)c_{i,t} }{\eta_t}$$
If $\xi_t=\bm{x}_t^{\prime}$, $s^{D}_{t+1}=0$. Since $\mathcal{V}=\{\xi_t| \forall i \in \mathcal{D},<\xi_t,\bm{u}_i>= \dfrac{(1-\eta_t\lambda_i)c_{i,t} }{\eta_t}\}$ is a $d-k$ dim hyperplane, and since $\xi$ is a $d$ dim gaussian. By \textit{Sard} Theorem, $\mathcal{V}$ is a zero measure set for $\xi$. Since $\theta_{t+1} $ is a bounded function, then $\theta_{t+1} $ is $\xi-\textit{integrable}$ ($\mathbb{E}_{\xi}[\theta_{t+1}]<\infty$). Then: 
$$
\mathbb{E}_{\xi}[\theta_{t+1}]=\int_{\mathbb{R}^d}\theta_{t+1} d \xi= \int_{\mathbb{R}^d \setminus \mathcal{V}}\theta_{t+1} d \xi+\int_{\mathcal{V}}\theta_{t+1} d\xi=\int_{\mathbb{R}^d \setminus \mathcal{V}}\theta_{t+1} d \xi.
$$ Therefore, we can assume $s_{t+1}^\mathcal{D} \ne 0$,and it does not affect the expectation on $\theta_{t+1}$. The case of $\{\xi|s_{t+1}^{\mathcal{D}}(\xi)=0\}$ is analogous.
\paragraph{Comparison functional and its meaning.}
To more conveniently obtain our conclusions in the asymptotic setting, we define the comparison functional:
\begin{equation}\label{eq:f-compare}
f_t(\eta_t):=s_t^{\mathcal B}\,s_{t+1}^{\mathcal D}-s_{t+1}^{\mathcal B}\,s_t^{\mathcal D},
\end{equation}
so that
\[
\theta_{t+1}>\theta_t\ \Longleftrightarrow\ f_t(\eta_t)>0,\qquad
\theta_{t+1}<\theta_t\ \Longleftrightarrow\ f_t(\eta_t)<0.
\]
The sign of $f_t(\eta_t)$ will be controlled via its conditional expectation given $\bm{x}_t$.

\subsection{Probabilistic Lemmas}

\begin{lemma}[Variance of Gaussian Forms]\label{lem:gauss-var}
Let $\bm{z}\sim\mathcal N(\bm{0},\mC)$ with $\mC=\mC^\top\succeq 0$ and $\|\mC\|_2<\infty$.
For any deterministic vector $\bm{g}\in\mathbb \mathbb{R}^d$ and any deterministic diagonal matrix $\mD=\mathrm{diag}(d_1,\dots,d_d)$, the following standard results hold:
\[
\mathrm{Var}(\bm{g}^\top \bm{z})=\bm{g}^\top \mC \bm{g}\le \|\mC\|_2\,\|\bm{g}\|_2^2,
\]
\[
\mathrm{Var}(\bm{z}^\top \mD \bm{z})=2\,\mathrm{tr}((\mC \mD)^2)\le 2\|\mC\|_2^2\,\mathrm{tr}(\mD^2).
\]
\end{lemma}

\begin{lemma}[A Weak Law of Large Numbers for Block Averages]\label{lem:wlln-blocks}
Let $\{\mathcal{S}_d\}_{d\in\mathbb{N}}$ be a sequence of index sets such that $|\mathcal{S}_d|\to\infty$ as $d\to\infty$. For each $d$, let $\{y_i\}_{i\in\mathcal S_d}$ be a collection of scalar random variables satisfying $\mathrm{Var}(\sum_{i\in\mathcal S_d} y_i)=O(|\mathcal S_d|)$. Then
\[
\frac{1}{|\mathcal S_d|}\sum_{i\in\mathcal S_d} y_i - \mathbb{E}\left[\frac{1}{|\mathcal S_d|}\sum_{i\in\mathcal S_d} y_i\right] \xrightarrow[d\to\infty]{p} 0.
\]
Furthermore, if $\lim_{d\to\infty} \frac{1}{|\mathcal S_d|}\sum_{i\in\mathcal S_d}\mathbb E[y_i] = \ell$ for some constant $\ell$, then $\frac{1}{|\mathcal S_d|}\sum_{i\in\mathcal S_d} y_i \xrightarrow[d\to\infty]{p} \ell$.
\end{lemma}
\begin{proof}
Let $s_d = \sum_{i \in \mathcal{S}_d} y_i$ and $\bar{s}_d = \frac{s_d}{|\mathcal{S}_d|}$. We want to show that for any $\epsilon > 0$, $\lim_{d\to\infty} P(|\bar{s}_d - \mathbb{E}[\bar{s}_d]| \ge \epsilon) = 0$.
By Chebyshev's inequality,
\[
P(|\bar{s}_d - \mathbb{E}[\bar{s}_d]| \ge \epsilon) \le \frac{\mathrm{Var}(\bar{s}_d)}{\epsilon^2}.
\]
We analyze the variance term:
\[
\mathrm{Var}(\bar{s}_d) = \mathrm{Var}\left(\frac{1}{|\mathcal{S}_d|}\sum_{i \in \mathcal{S}_d} y_i\right) = \frac{1}{|\mathcal{S}_d|^2} \mathrm{Var}\left(\sum_{i \in \mathcal{S}_d} y_i\right).
\]
By assumption, there exists a constant $k_0 < \infty$ such that $\mathrm{Var}(\sum_{i \in \mathcal{S}_d} y_i) \le k_0|\mathcal{S}_d|$. Substituting this into the inequality gives
\[
P(|\bar{s}_d - \mathbb{E}[\bar{s}_d]| \ge \epsilon) \le \frac{k_0|\mathcal{S}_d|}{|\mathcal{S}_d|^2 \epsilon^2} = \frac{k_0}{|\mathcal{S}_d|\epsilon^2}.
\]
As $d\to\infty$, we have $|\mathcal{S}_d|\to\infty$, so $\frac{k_0}{|\mathcal{S}_d|\epsilon^2} \to 0$. This proves convergence in probability. The second part of the lemma follows directly.
\end{proof}

\begin{lemma}[Continuous Mapping Theorem with Dominated Convergence]\label{lem:cmt-dc}
Let $\{y_n\}_{n\in\mathbb{N}}$ be a sequence of scalar random variables such that $y_n\xrightarrow[n\to\infty]{p}y_0$, where $y_0$ is a constant. Let $f$ be a function that is continuous at $y_0$ and bounded. Then $\lim_{n\to\infty}\mathbb E[f(y_n)]= f(y_0)$.
\end{lemma}

\paragraph{Remarks.} The variance formulas in Lemma~\ref{lem:gauss-var} are standard results for Gaussian distributions. The weak law in Lemma~\ref{lem:wlln-blocks} is proven directly from Chebyshev's inequality. Lemma~\ref{lem:cmt-dc} is a standard result combining the Continuous Mapping Theorem with the Bounded Convergence Theorem.

\subsection{Asymptotic Property}
In this section, we need to introduce some important asymptotic properties of variables under our asymptotic region settings. In this section, all probability measures \( p \) we consider are conditional measures based on \( \bm{x}_t \), defined as \( p(\cdot | \bm{x}_t) \).
\begin{lemma}[Asymptotic ratio of expectations]\label{lem:ratio-asymp-app}
Under Assumption~\ref{asp:standing}, the normalized block sums converge in probability to their conditional expectations:
\[
\frac{1}{k}\,s_{t+1}^{\mathcal D}\xrightarrow[d \to \infty]{p}\ell'_{\mathcal D}(\bm{x}_t,\mSigma),\qquad
\frac{1}{d-k}\,s_{t+1}^{\mathcal B}\xrightarrow[d \to \infty]{p}\ell'_{\mathcal B}(\bm{x}_t,\mSigma),
\]
where $\ell'_{\mathcal S}(\bm{x}_t,\mSigma) = \lim_{d \to \infty} \frac{1}{|\mathcal{S}|}\mathbb{E}[s_{t+1}^{\mathcal S} \mid \bm{x}_t]$ are finite constants (or functions depend on $\bm{x}_t,\mSigma$). Consequently, the limit of the expected alignment is the ratio of the limits of expected energies:
\begin{equation}\label{eq:ratio-asymp}
\lim_{d\to\infty}\,\mathbb E\big[\theta_{t+1}\,\big|\,\bm{x}_t\big]
=\lim_{d\to\infty}\,\mathbb E\Big[\frac{1}{1+ \frac{1}{\rho}\rho\frac{s_{t+1}^{\mathcal B}}{s_{t+1}^{\mathcal D}}}\Big]
=\frac{1}{1+\frac{1}{\rho}\frac{\lim_{d \to \infty} \mathbb E[\frac{1}{d-k}s_{t+1}^{\mathcal B}\mid \bm{x}_t]}{\lim_{d \to \infty} \mathbb E[\frac{1}{k}s_{t+1}^{\mathcal D}\mid \bm{x}_t]}}
=\frac{1}{1+\frac{\ell'_{\mathcal B}(\bm{x}_t,\mSigma)}{\rho\,\ell'_{\mathcal D}(\bm{x}_t,\mSigma)}}
.
\end{equation}
\end{lemma}
\begin{proof}
We analyze the structure of the scalar sum $s_{t+1}^{\mathcal S}$. By expanding its definition, we can decompose it into terms that are constant, linear, and quadratic with respect to the noise vector $\bm{z}_t$ (whose components are $\zeta_{i,t}$).
\begin{align*}
s_{t+1}^{\mathcal S}
&= \chi_{\mathcal{S},t} + \bm{g}_{\mathcal{S},t}^\top \bm{z}_t + \bm{z}_t^\top D_{\mathcal{S},t} \bm{z}_t,
\end{align*}
where the scalar $\chi_{\mathcal{S},t}$ (constant given $\bm{x}_t$), vector $\bm{g}_{\mathcal{S},t}$, and diagonal matrix $D_{\mathcal{S},t}$ are defined as:
\begin{align*}
\chi_{\mathcal{S},t} &:= \sum_{i\in\mathcal S}\lambda_i^2(1-\eta_t\lambda_i)^2 c_{i,t}^2, \\
(\bm{g}_{\mathcal{S},t})_i &:= \begin{cases} -2\eta_t \lambda_i^2(1-\eta_t\lambda_i)c_{i,t} & \text{if } i \in \mathcal{S} \\ 0 & \text{if } i \notin \mathcal{S} \end{cases}, \\
(D_{\mathcal{S},t})_{ii} &:= \begin{cases} \eta_t^2 \lambda_i^2 & \text{if } i \in \mathcal{S} \\ 0 & \text{if } i \notin \mathcal{S} \end{cases}.
\end{align*}
The variance of $s_{t+1}^{\mathcal S}$ is the variance of the sum of the linear and quadratic forms. By Lemma~\ref{lem:gauss-var} and Assumption~\ref{asp:standing}, $\mathrm{Var}(s_{t+1}^{\mathcal S}) = O(|\mathcal S|)$, which satisfies the conditions for Lemma~\ref{lem:wlln-blocks}. This gives the convergence in probability.
Let the scalar random variable $y_{k,d}$ be the ratio
\[
y_{k,d}:=\frac{s_{t+1}^{\mathcal B}}{s_{t+1}^{\mathcal D}}
=\frac{\frac{1}{d-k}s_{t+1}^{\mathcal B}}{\frac{k}{d-k}\cdot \frac{1}{k}s_{t+1}^{\mathcal D}}
\xrightarrow[k,d \to \infty]{p}\ \frac{\ell'_{\mathcal B}(\bm{x}_t,\mSigma)}{\rho\,\ell'_{\mathcal D}(\bm{x}_t,\mSigma)}=:y_0\in(0,\infty).
\]
With $f(y)=(1+y)^{-1}$ being bounded and continuous, Lemma~\ref{lem:cmt-dc} yields
\[
\lim_{d\to\infty}\mathbb E[\theta_{t+1} \mid \bm{x}_t]
=\lim_{d\to\infty}\mathbb E\Big[\frac{1}{1+y_{k,d}}\Big]
=\frac{1}{1+y_0}=\frac{1}{1+\frac{\ell'_{\mathcal B}(\bm{x}_t,\mSigma)}{\rho\,\ell'_{\mathcal D}(\bm{x}_t,\mSigma)}}.
\]
The last identity in \eqref{eq:ratio-asymp} follows from the definition of $\ell'_{\mathcal S}(\bm{x}_t,\Sigma)$.
\end{proof}

Lemma~\ref{lem:ratio-asymp-app} justifies that, asymptotically, the expectation can be interchanged with the ratio defining $\theta_{t+1}$. 
Which implies if we want to have a decrease for $\theta_t $ asymptotically, we have:
\begin{equation}
 \begin{aligned}
&\lim_{d\to\infty}\,\mathbb E\big[\theta_{t+1}\,\big|\,\bm{x}_t\big]
=\frac{1}{1+\frac{1}{\rho}\frac{\lim_{d \to \infty} \mathbb E[\frac{1}{d-k}s_{t+1}^{\mathcal B}\mid \bm{x}_t]}{\lim_{d \to \infty} \mathbb E[\frac{1}{k}s_{t+1}^{\mathcal D}\mid \bm{x}_t]}}&<\lim_{d\to\infty}\,\mathbb \theta_{t}=\frac{1}{1-\frac{1}{\rho}\frac{\lim_{d \to \infty}\frac{1}{d-k}s_t^\mathcal{B}}{\lim_{d \to \infty}\frac{1}{k}s_t^\mathcal{D}}}\\
\implies &\lim_{d \to \infty} E\big[s_t^{\mathcal B}\,s_{t+1}^{\mathcal D}\,\big|\,\bm{x}_t\big]-\lim_{d \to \infty} E\big[s_{t+1}^{\mathcal B}\,s_t^{\mathcal D}\,\big|\,\bm{x}_t\big]<0
\\
\implies &\lim_{d \to \infty} E\big[f_t(\eta_t):=s_t^{\mathcal B}\,s_{t+1}^{\mathcal D}-s_{t+1}^{\mathcal B}\,s_t^{\mathcal D}\,\big|\,\bm{x}_t\big]<0
\end{aligned}   
\end{equation}

Therefore, to decide the sign of the drift of $\theta_t$ in expectation, it suffices to analyze the sign of the comparison functional’s expectation $\mathbb E[f_t(\eta_t)\mid \bm{x}_t]$ (a quadratic in $\eta_t$). This motivates computing elementwise expectations next.

\subsection{Comparison Lemma}
\begin{lemma}[Elementwise expectation of $\bm{x}_{t+1}$]\label{lem:blk-exp-t-app}
For each $i\in\{1,\dots,d\}$ and for a given $\bm{x}_t$,
\[
\mathbb E\big[\lambda_i^2 c_{i,t+1}^2 \mid \bm{x}_t\big]=(1-\eta_t\lambda_i)^2 \lambda_i^2 c_{i,t}^2+\eta_t^2\lambda_i^2\,\kappa_i^2.
\]
Consequently,
\begin{align}
\mathbb E\big[s_{t+1}^{\mathcal D}\mid \bm{x}_t\big]&=s_t^{\mathcal D}-2\eta_t\sum_{i\in\mathcal D}\lambda_i^3 c_{i,t}^2+\eta_t^2\Big(\sum_{i\in\mathcal D}\lambda_i^4 c_{i,t}^2+e_{\mathcal D}\Big),\label{eq:Eat}\\
\mathbb E\big[s_{t+1}^{\mathcal B}\mid \bm{x}_t\big]&=s_t^{\mathcal B}-2\eta_t\sum_{i\in\mathcal B}\lambda_i^3 c_{i,t}^2+\eta_t^2\Big(\sum_{i\in\mathcal B}\lambda_i^4 c_{i,t}^2+e_{\mathcal B}\Big).\label{eq:Ebt}
\end{align}
\end{lemma}
\begin{proof}
We expand the square for $c_{i,t+1}$ and take the expectation conditional on $\bm{x}_t$:
\[
\mathbb E\big[ \lambda_i^2 c_{i,t+1}^2 \mid \bm{x}_t \big] = \mathbb E\big[ \lambda_i^2 \left( (1-\eta_t\lambda_i)c_{i,t} - \eta_t\zeta_{i,t} \right)^2 \mid \bm{x}_t \big].
\]
Expanding the squared term gives
\[
\mathbb E\big[ \lambda_i^2 \left( (1-\eta_t\lambda_i)^2 c_{i,t}^2 - 2\eta_t(1-\eta_t\lambda_i)c_{i,t}\zeta_{i,t} + \eta_t^2\zeta_{i,t}^2 \right) \mid \bm{x}_t\big].
\]
Using $\mathbb E[\zeta_{i,t}\mid \bm{x}_t]=0$ and $\mathbb E[\zeta_{i,t}^2\mid \bm{x}_t]=\kappa_i^2$, the cross-term vanishes. We are left with
\begin{align}
\mathbb E\big[ \lambda_i^2 c_{i,t+1}^2 \mid \bm{x}_t \big] &= \lambda_i^2 (1-\eta_t\lambda_i)^2 c_{i,t}^2 + \eta_t^2 \lambda_i^2 \kappa_i^2 \\
&= (1-2\eta_t\lambda_i+\eta_t^2\lambda_i^2)\lambda_i^2 c_{i,t}^2 + \eta_t^2 \lambda_i^2 \kappa_i^2.
\end{align}
Summing over $i \in \mathcal{D}$ (or $\mathcal{B}$) and regrouping terms yields
\[
\mathbb E\big[s_{t+1}^{\mathcal S}\mid \bm{x}_t\big]=\sum_{i\in\mathcal S} \lambda_i^2 c_{i,t}^2 - 2\eta_t \sum_{i\in\mathcal S} \lambda_i^3 c_{i,t}^2 + \eta_t^2 \left( \sum_{i\in\mathcal S} \lambda_i^4 c_{i,t}^2 + \sum_{i\in\mathcal S} \lambda_i^2 \kappa_i^2 \right).
\]
This simplifies to the expressions in \eqref{eq:Eat} and \eqref{eq:Ebt} using the definitions of $s_t^{\mathcal S}$ and $e_{\mathcal S}$.
\end{proof}

\begin{lemma}[Quadratic form for step size condition]\label{lem:quad-t-app}
With $f_t(\eta_t)$ as in \eqref{eq:f-compare}, for a given $\bm{x}_t$, its conditional expectation is a quadratic function of $\eta_t$ with no constant term:
\begin{equation}\label{eq:quad-form}
\mathbb E[f_t(\eta_t)\mid \bm{x}_t]=p_t\,\eta_t^2+q_t\,\eta_t,
\end{equation}
where the scalars $p_t$ and $q_t$ are defined as
\begin{align}
q_t&=2\Big(s_t^{\mathcal D}\sum_{i\in\mathcal B}\lambda_i^3 c_{i,t}^2-s_t^{\mathcal B}\sum_{i\in\mathcal D}\lambda_i^3 c_{i,t}^2\Big)\le 2s_t^{\mathcal B}s_t^{\mathcal D}(\lambda_{k+1}-\lambda_k)<0,\label{eq:q-sign}\\
p_t&=s_t^{\mathcal B}\Big(\sum_{j\in\mathcal D}\lambda_j^4 c_{j,t}^2+e_{\mathcal D}\Big)\ -\ s_t^{\mathcal D}\Big(\sum_{i\in\mathcal B}\lambda_i^4 c_{i,t}^2+e_{\mathcal B}\Big).\label{eq:p-def}
\end{align}
If $p_t\neq 0$, we define the adaptive critical step size as the non-zero root $\eta_t^*:=-q_t/p_t$.
\end{lemma}
\begin{proof}
We explicitly compute the conditional expectation of $f_t(\eta_t)$ by substituting the results from Lemma~\ref{lem:blk-exp-t-app} into its definition and collecting terms by powers of $\eta_t$.
\begin{align*}
\mathbb E[f_t(\eta_t)\mid \bm{x}_t] &= s_t^{\mathcal B}\mathbb E[s_{t+1}^{\mathcal D}\mid \bm{x}_t] - s_t^{\mathcal D}\mathbb E[s_{t+1}^{\mathcal B}\mid \bm{x}_t] \\
&= s_t^{\mathcal B}\left(s_t^{\mathcal D}-2\eta_t\sum_{i\in\mathcal D}\lambda_i^3 c_{i,t}^2+\eta_t^2\Big(\sum_{j\in\mathcal D}\lambda_j^4 c_{j,t}^2+e_{\mathcal D}\Big)\right) \\
&\quad - s_t^{\mathcal D}\left(s_t^{\mathcal B}-2\eta_t\sum_{i\in\mathcal B}\lambda_i^3 c_{i,t}^2+\eta_t^2\Big(\sum_{i\in\mathcal B}\lambda_i^4 c_{i,t}^2+e_{\mathcal B}\Big)\right) \\
&= (s_t^{\mathcal B}s_t^{\mathcal D} - s_t^{\mathcal D}s_t^{\mathcal B}) \\
&\quad + \eta_t \left( -2s_t^{\mathcal B}\sum_{i\in\mathcal D}\lambda_i^3 c_{i,t}^2 + 2s_t^{\mathcal D}\sum_{i\in\mathcal B}\lambda_i^3 c_{i,t}^2 \right) \\
&\quad + \eta_t^2 \left( s_t^{\mathcal B}\Big(\sum_{j\in\mathcal D}\lambda_j^4 c_{j,t}^2+e_{\mathcal D}\Big) - s_t^{\mathcal D}\Big(\sum_{i\in\mathcal B}\lambda_i^4 c_{i,t}^2+e_{\mathcal B}\Big) \right) \\
&= \eta_t \underbrace{2\Big(s_t^{\mathcal D}\sum_{i\in\mathcal B}\lambda_i^3 c_{i,t}^2-s_t^{\mathcal B}\sum_{i\in\mathcal D}\lambda_i^3 c_{i,t}^2\Big)}_{q_t} \\
&\quad + \eta_t^2 \underbrace{ \left( s_t^{\mathcal B}\Big(\sum_{j\in\mathcal D}\lambda_j^4 c_{j,t}^2+e_{\mathcal D}\Big)\ -\ s_t^{\mathcal D}\Big(\sum_{i\in\mathcal B}\lambda_i^4 c_{i,t}^2+e_{\mathcal B}\Big) \right) }_{p_t}.
\end{align*}
This confirms the quadratic form $\mathbb E[f_t(\eta_t)\mid \bm{x}_t] = p_t\eta_t^2+q_t\eta_t$.
The bound on $q_t$ follows from $\sum_{i \in\mathcal B}\lambda_i^3 c_{i,t}^2\le \lambda_{k+1}s_t^{\mathcal B}$ and $\sum_{j \in\mathcal D}\lambda_j^3 c_{j,t}^2\ge \lambda_k s_t^{\mathcal D}$, which makes $q_t$ strictly negative since $\lambda_k > \lambda_{k+1}$.
\end{proof}

\paragraph{Remark.} The expected one-step alignment change functional $\mathbb E[f_t(\eta_t)\mid \bm{x}_t]$ is a quadratic function of the step size $\eta_t$ that always passes through the origin $(\eta_t=0, \mathbb E[f_t(\eta_t)\mid \bm{x}_t]=0)$. The behavior for any positive step size $\eta_t > 0$ is entirely determined by the sign of the quadratic coefficient $p_t$, as illustrated in Figure~\ref{fig:eta_behavior}.

\begin{figure}[h!]
\centering
\begin{minipage}{0.45\textwidth}
    \centering
    \includegraphics[width=\linewidth]{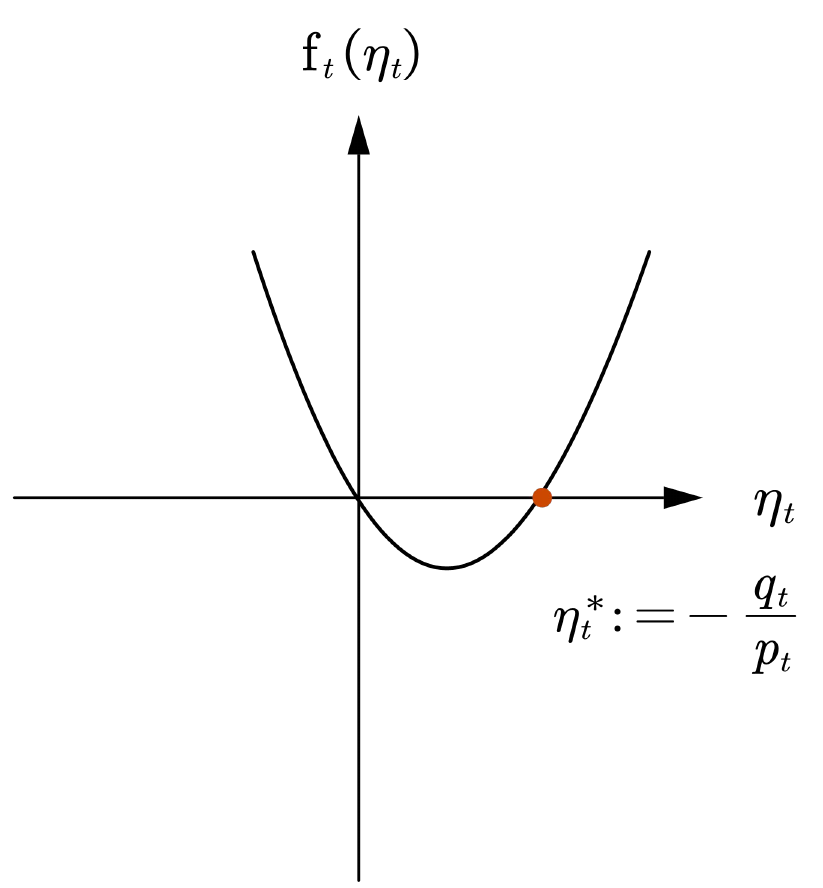}
    \centerline{(a)}
\end{minipage}
\hfill
\begin{minipage}{0.45\textwidth}
    \centering
    \includegraphics[width=\linewidth]{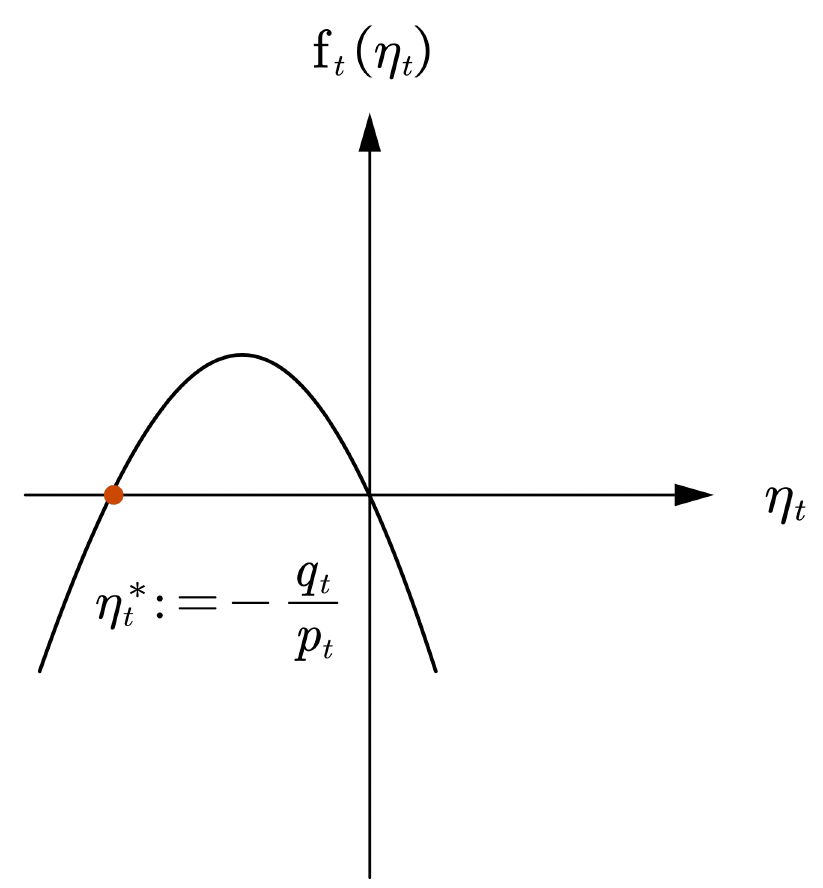}
    \centerline{(b)}
\end{minipage}
\caption{The sign of $\mathbb E[f_t(\eta_t)\mid \bm{x}_t]$ as a function of $\eta_t$. (a) shows the behavior when $p_t > 0$, where the parabola opens upwards. (b) shows the behavior when $p_t < 0$, where the parabola opens downwards.}
\label{fig:eta_behavior}
\end{figure}

When $p_t > 0$ (Figure~\ref{fig:eta_behavior}a), the parabola opens upwards. Since $q_t < 0$, there exists a positive critical step size $\eta_t^* = -q_t/p_t > 0$. For a small step size $0 < \eta_t < \eta_t^*$, we have $\mathbb E[f_t(\eta_t)] < 0$, implying the expected alignment $\mathbb E[\theta_{t+1}]$ decreases. For a large step size $\eta_t > \eta_t^*$, we have $\mathbb E[f_t(\eta_t)] > 0$, implying the expected alignment increases.

When $p_t < 0$ (Figure~\ref{fig:eta_behavior}b), the parabola opens downwards. Since both coefficients $p_t$ and $q_t$ are negative, the functional $\mathbb E[f_t(\eta_t)]$ is negative for all positive step sizes $\eta_t > 0$. This implies that, in this regime, any choice of step size will lead to a decrease in the expected alignment.

\subsection{Proofs of Step Size Condition Theory}
\label{proof-step-size-theory}
\begin{theorem-restated}\ref{thm:step-dec-below}(\textbf{\textit{Decreases Condition}})
Under Assumption \ref{asp:standing}, if $0<\eta_t<\eta_t^*(\bm{x}_t)$, then
\[
\lim_{d\to\infty}\,\mathbb E\big[\theta_{t+1}\,\big|\,\bm{x}_t\big]\ <\ \theta_t.
\]
\end{theorem-restated}

\begin{proof}
By Lemma~\ref{lem:quad-t-app}, $\mathbb E[f_t(\eta_t)\mid \bm{x}_t]=p_t\eta_t^2+q_t\eta_t$ with $q_t<0$. If $p_t > 0$, then for $0<\eta_t<\eta_t^*=-q_t/p_t$, we have $\mathbb E[f_t(\eta_t)\mid \bm{x}_t]<0$. Lemma~\ref{lem:ratio-asymp-app} transfers the sign to $\lim_{d\to\infty}\mathbb E[\theta_{t+1}\mid \bm{x}_t]<\theta_t$.
\end{proof}

\begin{theorem-restated}\ref{thm:step-inc-above}(\textbf{\textit{Increase Condition}})
Under Assumption \ref{asp:standing}, let
\[
\,g_{\mathrm{gap}}:=\frac{1}{1+\frac{s_{\max}}{s_{\min}}\cdot \frac{1}{\rho}\left(\frac{\lambda_{k+1}}{\lambda_k}\right)^2}\in(0,1).\,
\]
For $t$ and $\bm{x}_t$. If $\theta_t\le g_{\mathrm{gap}}$ and $\eta_t>\eta_t^*(\bm{x}_t)$, then
\[
\lim_{d\to\infty}\,\mathbb E\big[\theta_{t+1}\,\big|\,\bm{x}_t\big]\ >\ \theta_t.
\]
\end{theorem-restated}
\begin{proof}
We will show $p_t>0$ when $\theta_t\le g_{\mathrm{gap}}$, then apply Lemma \ref{lem:ratio-asymp-app} and ~\ref{lem:quad-t-app}. \\
First, recall the expression for $p_t$ (formula \ref{eq:p-def}). Let's consider the term which do not contain $ e_{\mathcal S}$ in $p_t$:
\begin{align}
s_t^{\mathcal B}\sum_{j\in\mathcal D}\lambda_j^4 c_{j,t}^2-s_t^{\mathcal D}\sum_{i\in\mathcal B}\lambda_i^4 c_{i,t}^2
&=\sum_{i\in\mathcal B}\sum_{j\in\mathcal D}\lambda_i^2\lambda_j^2 c_{i,t}^2 c_{j,t}^2(\lambda_j^2-\lambda_i^2)\nonumber\\
&\ge s_t^{\mathcal B}s_t^{\mathcal D}(\lambda_k^2-\lambda_{k+1}^2)>0.\label{eq:ptc-lb-app}
\end{align}
Next, using $e_{\mathcal D}\ge s_{\min}\psi_{\mathcal D}$ and $e_{\mathcal B}\le s_{\max}\psi_{\mathcal B}$,
\begin{equation}\label{eq:ptn-lb-app}
s_t^{\mathcal B} e_{\mathcal D}-s_t^{\mathcal D} e_{\mathcal B}\ \ge\ s_t\Big((1-\theta_t) s_{\min}\psi_{\mathcal D}-\theta_t s_{\max}\psi_{\mathcal B}\Big). 
\end{equation}
If $\theta_t\le g_{\mathrm{noise}}:=\big(1+\frac{s_{\max}\psi_{\mathcal B}}{s_{\min}\psi_{\mathcal D}}\big)^{-1}$, then the bracket is non-negative and so the term which do  contain $ e_{\mathcal S}$ is $\ge 0$. Using $\psi_{\mathcal D}\ge k\lambda_k^2$ and $\psi_{\mathcal B}\le (d-k)\lambda_{k+1}^2$,we can have $ g_{\mathrm{gap}} \le g_{\mathrm{noise}}$. Hence $p_t>0$ whenever $\theta_t\le g_{\mathrm{gap}}$. Therefore, if additionally $\eta_t>\eta_t^*$, then $\mathbb E[f_t(\eta_t)\mid \bm{x}_t]>0$ and Lemma~\ref{lem:ratio-asymp-app} yields
\[
\lim_{d\to\infty}\mathbb E[\theta_{t+1}\mid \bm{x}_t]>\theta_t.
\]
As $\lambda_k/\lambda_{k+1}\to\infty$, $(\lambda_{k+1}/\lambda_k)^2\to 0$ implies $g_{\mathrm{gap}}\to 1$.
\end{proof}

\begin{theorem-restated}\ref{thm:step-dec-large}(\textbf{\textit{Large Alignment Regime Condition}})
Under Assumption \ref{asp:standing}, there exists a critical alignment threshold $\theta^*_t \in (0,1)$ such that if the current alignment $\theta_t \ge \theta^*_t$, then for any positive step size $\eta_t > 0$, the expected alignment decreases:
\[
\lim_{d\to\infty}\,\mathbb E\big[\theta_{t+1}\,\big|\,\bm{x}_t\big]\ <\ \theta_t.
\]
The threshold is given by $\theta^*_t := r_{0,t}/(1+r_{0,t})$, where $r_{0,t}$ is the positive root of the quadratic equation:
\[
a_t^{\mathrm{aux}} r^2+(a_t^{\mathrm{aux}}-m_t^{\mathrm{aux}}-h_t^{\mathrm{aux}})r-h_t^{\mathrm{aux}} = 0.
\]
The coefficients are defined using the second-order spectral masses $\psi_{\mathcal S} := \sum_{i \in \mathcal S} \lambda_i^2$ and the total energy $s_t:=\sum_{i=1}^d \lambda_i^2 c_{i,t}^2$ as:
\begin{align*}
a_t^{\mathrm{aux}} &:=s_{\min}\psi_{\mathcal B}, \\
h_t^{\mathrm{aux}} &:=s_{\max}\psi_{\mathcal D}, \\
m_t^{\mathrm{aux}} &:=s_t(\lambda_1^2-\lambda_d^2).
\end{align*}
As $\psi_{\mathcal D}/\psi_{\mathcal B}\to\infty$, we have $\theta^*_t\to 1$.
\end{theorem-restated}

\begin{proof}
We derive a sufficient condition for $p_t\le 0$. Using
\[
\sum_{j\in\mathcal D}\lambda_j^4 c_{j,t}^2\le \lambda_1^2 s_t^{\mathcal D},\quad
\sum_{i\in\mathcal B}\lambda_i^4 c_{i,t}^2\ge \lambda_d^2 s_t^{\mathcal B},\quad
e_{\mathcal D}\le s_{\max}\psi_{\mathcal D},\quad e_{\mathcal B}\ge s_{\min}\psi_{\mathcal B},
\]
we get
\begin{align}
\label{eqptupbound}
p_t
&\le s_t^{\mathcal B}(\lambda_1^2 s_t^{\mathcal D}+s_{\max}\psi_{\mathcal D})-s_t^{\mathcal D}(\lambda_d^2 s_t^{\mathcal B}+s_{\min}\psi_{\mathcal B})\nonumber\\
&=s_t^2\theta_t(1-\theta_t)(\lambda_1^2-\lambda_d^2)+s_t\big((1-\theta_t)s_{\max}\psi_{\mathcal D}-\theta_t s_{\min}\psi_{\mathcal B}\big).
\end{align}
Let $r_t=\theta_t/(1-\theta_t)$ and define $a_t^{\mathrm{aux}}:=s_{\min}\psi_{\mathcal B}$, $h_t^{\mathrm{aux}}:=s_{\max}\psi_{\mathcal D}$, $m_t^{\mathrm{aux}}:=s_t(\lambda_1^2-\lambda_d^2)$. Then $p_t\le 0$ is implied by
$r_t m_t^{\mathrm{aux}} + (1+r_t)h_t^{\mathrm{aux}} - r_t(1+r_t)a_t^{\mathrm{aux}}\le 0$, which is equivalent to
\[
a_t^{\mathrm{aux}} r_t^2+(a_t^{\mathrm{aux}}-m_t^{\mathrm{aux}}-h_t^{\mathrm{aux}})r_t-h_t^{\mathrm{aux}}\ \ge\ 0.
\]
Hence, if $\theta_t$ is large enough so that the quadratic inequality holds, $p_t\le 0$ and, for any $\eta_t>0$, combine with Lemma~\ref{lem:ratio-asymp-app}:
\[
\lim_{d\to \infty}\mathbb E[f_t(\eta_t)\mid \bm{x}_t]\le 0\ \ \Rightarrow\ \ \lim_{d\to\infty}\mathbb E[\theta_{t+1}\mid \bm{x}_t]<\theta_t,
\]

\end{proof}

\begin{theorem-restated}\ref{thm:regime-separation} (\emph{\textbf{Separation of Alignment Regimes}})
Under Assumption \ref{asp:standing}, for any nontrivial problem with a non-zero spectral gap $(\mathrm{gap}_2 > 0)$ and bounded noise $(s_{\max} < \infty)$, the low-alignment threshold $g_{\mathrm{gap}}$ is strictly less than the high-alignment threshold $\theta^*_t$:
\[
g_{\mathrm{gap}} < \theta^*_t.
\]
\end{theorem-restated}

\begin{proof}
According to the proof of Theorem \ref{thm:step-inc-above}:
$$\theta_t<g_{gap}\implies p_t>0,$$
if $g_{gap}>\theta_t^*, \exists \theta_t \in (\theta^*_t,g_{gap})$, by the proof of Theorem \ref{thm:step-dec-large}. we have:
$$\theta_t>\theta_t^* \implies p_t<0$$
which contradicts the previous conclusion. Therefore $g_{gap}<\theta_t^*$
\end{proof}

{\begin{theorem-restated}\ref{thm:theta-star-rate}(\textbf{\textit{Asymptotic rate of $\theta_t^*$}})
Under Assumption \ref{asp:standing}, with $m := \lambda_k/\lambda_{k+1}>1$, there exist constants $\alpha, \beta>0$ such that
\[
1 - \frac{\beta}{m^2}
\;\le\;
\theta_t^*
\;\le\;
1 - \frac{\alpha}{m^2},
\]
and hence
\[
\theta_t^* = 1 - \Theta\!\left(\frac{1}{m^2}\right).
\]
\end{theorem-restated}

\begin{proof}
By Theorem~\ref{thm:step-dec-large} we have
\[
\theta_t^* \;=\; \frac{r_{0,t}}{\,1+r_{0,t}\,},
\]
where $r_{0,t}>0$ is the positive root of the quadratic
\[
a_t^{\mathrm{aux}}\,r^2 + \bigl(a_t^{\mathrm{aux}} - m_t^{\mathrm{aux}} - h_t^{\mathrm{aux}}\bigr)\,r - h_t^{\mathrm{aux}} = 0.
\]
Here
\[
a_t^{\mathrm{aux}} = s_{\min}\sum_{i\in\mathcal{B}}\lambda_i^2,\quad
h_t^{\mathrm{aux}} = s_{\max}\sum_{i\in\mathcal{D}}\lambda_i^2,\quad
m_t^{\mathrm{aux}} = s_t\bigl(\lambda_1^2 - \lambda_d^2\bigr),
\]
are all nonnegative scalars. Using the quadratic formula, the positive root can be written as
\[
r_{0,t}
=\frac{
m_t^{\mathrm{aux}} + h_t^{\mathrm{aux}} - a_t^{\mathrm{aux}}
+
\sqrt{
\bigl(m_t^{\mathrm{aux}}+h_t^{\mathrm{aux}}-a_t^{\mathrm{aux}}\bigr)^2
+ 4\,a_t^{\mathrm{aux}}\,h_t^{\mathrm{aux}}
}
}{2\,a_t^{\mathrm{aux}}}.
\]
Under Assumption~\ref{asp:standing} the eigenvalues satisfy
\[
\lambda_{k+1}^2 \le \lambda_i^2 \le \lambda_k^2\quad\text{for }i\in\mathcal{D},
\]
and
\[
\lambda_{k+1}^2 \le \lambda_i^2\quad\text{for }i\in\mathcal{B},
\]
so that, with the shorthand
\[
\ell_{\mathcal{B}} := |\mathcal{B}|\cdot\lambda_{k+1}^2,\qquad
u_{\mathcal{B}} := |\mathcal{B}|\cdot m^2\lambda_{k+1}^2,
\]
we have $\ell_{\mathcal{B}}\le\sum_{i\in\mathcal{B}}\lambda_i^2\le u_{\mathcal{B}}$ and
$\rho\,\ell_{\mathcal{B}}\le\sum_{i\in\mathcal{D}}\lambda_i^2\le\rho\,u_{\mathcal{B}}$, where $\rho = k/(d-k)$. 
Using the noise spectral bounds $s_{\min}\le\lambda_{\min}(\Sigma)\le\lambda_{\max}(\Sigma)\le s_{\max}$, it follows that
\[
s_{\min}\,\ell_{\mathcal{B}}\;\le\;a_t^{\mathrm{aux}}\;\le\;s_{\min}\,u_{\mathcal{B}},\qquad
s_{\max}\,\rho\,\ell_{\mathcal{B}}\;\le\;h_t^{\mathrm{aux}}\;\le\;s_{\max}\,\rho\,u_{\mathcal{B}}.
\]
Define the fixed scalar $\beta := \rho\,(s_{\max}/s_{\min})>0$. Using $m_t^{\mathrm{aux}}\ge 0$, we obtain
\[
r_{0,t}\;\ge\;
\frac{s_{\max}\,\rho\,\ell_{\mathcal{B}}-s_{\min}\,u_{\mathcal{B}}}{s_{\min}\,u_{\mathcal{B}}}
= \frac{\beta}{m^2}-1,
\]
and hence
\[
\theta_t^* = \frac{r_{0,t}}{1+r_{0,t}}\;\ge\;
\frac{\beta/m^2 - 1}{1 + (\beta/m^2 - 1)}
=1-\frac{1}{\beta/m^2}.
\]
For the upper bound, since $a_t^{\mathrm{aux}}\ge s_{\min}\ell_{\mathcal{B}}$, $h_t^{\mathrm{aux}}\le s_{\max}\rho\,u_{\mathcal{B}}$, and $m_t^{\mathrm{aux}}\le s_t(\lambda_1^2-\lambda_d^2)$, there exists a fixed scalar $\alpha>0$ such that
\[
r_{0,t}\;\le\;\alpha\,m^2.
\]
Then
\[
\theta_t^* = \frac{r_{0,t}}{1+r_{0,t}}\;\le\;
\frac{\alpha\,m^2}{1+\alpha\,m^2}
= 1 - \frac{1}{1+\alpha\,m^2},
\]
which gives
\[
1-\frac{\beta}{m^2}\;\le\; \theta_t^*\;\le\;
1-\frac{\alpha}{m^2}.
\]
The \(\Theta(1/m^2)\) form follows from these two-sided inequalities.
\end{proof}

}

\begin{theorem-restated}\ref{thm:bounds-xnorm}(\textbf{\textit{State- and gap-aware bounds on $\eta_t^*$ (with $\|\bm{x}_t\|_2$)}})
Under Assumption \ref{asp:standing}, for any $x_t$, we have 
\[
\eta_t^*\ \ge\ \frac{2\,\mathrm{gap}_1}{(\lambda_1^2-\lambda_d^2)\ +\ \dfrac{s_{\max}\psi_{\mathcal D}}{\lambda_d^2\,\|\bm{x}_t\|^2_2\,\theta_t}},
\]

\end{theorem-restated}

\begin{proof}
The bounds for $\eta_t^*=-q_t/p_t$ are derived by finding lower and upper bounds for $-q_t$ and $p_t$. \\
By formula \eqref{eq:q-sign}, we have:

$$-q_t\ge 2s_t^2\theta_t(1-\theta_t)\mathrm{gap}_1$$. 

And by formula \eqref{eqptupbound}, we have:
$$p_t\le s_t^2\theta_t(1-\theta_t)(\lambda_1^2-\lambda_d^2)+s_t(1-\theta_t)s_{\max}\psi_{\mathcal D}$$.
Therefore
\[
\eta_t^*=\frac{-q_t}{p_t}\ \ge\ \frac{2\,\mathrm{gap}_1}{(\lambda_1^2-\lambda_d^2)+\dfrac{s_{\max}\psi_{\mathcal D}}{s_t\,\theta_t}}.
\]
Using $s_t\ge \lambda_d^2\|\bm{x}_t\|_2^2$ yields the stated lower bounds.\\
\end{proof}

{\begin{theorem-restated}\ref{thm:state-gap-aware}(\textbf{\textit{State- and gap-aware upper bounds on $\eta^*$}})
Under Assumption \ref{asp:standing}, providing that the alignment satisfies $\theta_t \ge \frac{e_{\mathcal{B}}}{e_{\mathcal{B}} + e_{\mathcal{D}}}$, where $e_{\mathcal{B}}, e_{\mathcal{D}}$ represent the noise magnitudes in the principal and tail subspaces respectively, we have:
\[
\eta_t^* \ \le \ \frac{2(\lambda_1 - \lambda_d)}{\lambda_k \lambda_1 - \lambda_{k+1} \lambda_d}.
\]
\end{theorem-restated}
\begin{proof}
Recall that the optimal step size is given by the ratio $\eta_t^* = -q_t / p_t$. We first derive an upper bound for $-q_t$. By definition, and noting that $q_t$ contains no noise terms under the orthogonality assumption, we have:
\begin{equation}
\begin{aligned}
    -q_t \ &= \ s_t^{\mathcal{B}}\sum_{i\in\mathcal{D}}\lambda_i^3 c_{i,t}^2 - s_t^{\mathcal{D}}\sum_{i\in\mathcal{B}}\lambda_i^3 c_{i,t}^2 \\
    &\le \ s_t^{\mathcal{B}} \lambda_1 \left(\sum_{i\in\mathcal{D}}\lambda_i^2 c_{i,t}^2\right) - s_t^{\mathcal{D}} \lambda_d \left(\sum_{i\in\mathcal{B}}\lambda_i^2 c_{i,t}^2\right) \\
    &= \ \lambda_1 s_t^{\mathcal{B}} s_t^{\mathcal{D}} - \lambda_d s_t^{\mathcal{D}} s_t^{\mathcal{B}} \\
    &= \ (\lambda_1 - \lambda_d) s_t^{\mathcal{D}} s_t^{\mathcal{B}}.
\end{aligned}
\end{equation}

Next, we derive a lower bound for $p_t$. Expanding the expression for $p_t$ yields:
\begin{equation}
\begin{aligned}
    p_t \ &= \ s_t^{\mathcal{B}}\Big(\sum_{j\in\mathcal{D}}\lambda_j^4 c_{j,t}^2+e_{\mathcal{D}}\Big) - s_t^{\mathcal{D}}\Big(\sum_{i\in\mathcal{B}}\lambda_i^4 c_{i,t}^2+e_{\mathcal{B}}\Big) \\
    &= \ \left( s_t^{\mathcal{B}}\sum_{j\in\mathcal{D}}\lambda_j^4 c_{j,t}^2 - s_t^{\mathcal{D}}\sum_{i\in\mathcal{B}}\lambda_i^4 c_{i,t}^2 \right) + (s_t^{\mathcal{B}}e_{\mathcal{D}} - s_t^{\mathcal{D}}e_{\mathcal{B}}).
\end{aligned}
\end{equation}
Applying the spectral bounds $\lambda_j \ge \lambda_k$ for $j \in \mathcal{D}$ and $\lambda_i \le \lambda_{k+1}$ for $i \in \mathcal{B}$ to the signal components:
\begin{equation}
    p_t \ \ge \ (\lambda_k \lambda_1 - \lambda_{k+1} \lambda_d) s_t^{\mathcal{D}} s_t^{\mathcal{B}} + (s_t^{\mathcal{B}}e_{\mathcal{D}} - s_t^{\mathcal{D}}e_{\mathcal{B}}).
\end{equation}
The given condition $\theta_t \ge \frac{e_{\mathcal{B}}}{e_{\mathcal{B}} + e_{\mathcal{D}}}$ implies:
\[
\frac{s_t^{\mathcal{D}}}{s_t^{\mathcal{D}} + s_t^{\mathcal{B}}} \ge \frac{e_{\mathcal{B}}}{e_{\mathcal{B}} + e_{\mathcal{D}}} \iff s_t^{\mathcal{D}} e_{\mathcal{D}} + s_t^{\mathcal{D}} e_{\mathcal{B}} \ge s_t^{\mathcal{D}} e_{\mathcal{B}} + s_t^{\mathcal{B}} e_{\mathcal{B}} \iff s_t^{\mathcal{B}} e_{\mathcal{D}} - s_t^{\mathcal{D}} e_{\mathcal{B}} \ge 0.
\]
Since the noise residual term is non-negative, we may lower bound $p_t$ by omitting it:
\begin{equation}
    p_t \ \ge \ (\lambda_k \lambda_1 - \lambda_{k+1} \lambda_d) s_t^{\mathcal{D}} s_t^{\mathcal{B}}.
\end{equation}
Substituting the bounds for $-q_t$ and $p_t$ back into the expression for $\eta_t^*$:
\begin{equation}
    \eta_t^* \ \le \ \frac{(\lambda_1 - \lambda_d) s_t^{\mathcal{D}} s_t^{\mathcal{B}}}{(\lambda_k \lambda_1 - \lambda_{k+1} \lambda_d) s_t^{\mathcal{D}} s_t^{\mathcal{B}}} \ = \ \frac{\lambda_1 - \lambda_d}{\lambda_k \lambda_1 - \lambda_{k+1} \lambda_d}.
\end{equation}
\end{proof}

\begin{corollary-restated}\ref{cor:conservative-limit}(\emph{\textbf{The comparison between $\eta_t^\ast$ and $\frac{2}{\lambda_1}$}})\label{cor:conservative-limit}
Suppose the conditions of Theorem \ref{thm:state-gap-aware} hold.  We have the following relation between $\eta_t^*$ and the convergence step size for Quadratic Programming:
\[
\eta_t^* \ \le \ \frac{2}{\lambda_1}.
\]
\end{corollary-restated}

\begin{proof}
Using the tight bound derived in the proof of Theorem \ref{thm:state-gap-aware}:
\[
\eta_t^* \le \frac{(\lambda_1 - \lambda_d)}{\lambda_k \lambda_1 - \lambda_{k+1} \lambda_d}.
\]
We first verify $\eta_t^* \le \frac{1}{\lambda_1}$, it suffices to show:
\[
\lambda_1 (\lambda_1 - \lambda_d) \le \lambda_k \lambda_1 - \lambda_{k+1} \lambda_d.
\]
Noting that $\lambda_k \le \lambda_1$, the inequality holds if $\lambda_1^2 - \lambda_1 \lambda_d \le \lambda_1^2 - \lambda_{k+1} \lambda_d$, which simplifies to $\lambda_{k+1} \le \lambda_1$. This is naturally satisfied by the definition of principal and tail eigenvalues. Therefore:
\[
\eta_t^* \le \frac{1}{\lambda_1}<\frac{2}{\lambda_1}
\]
\end{proof}
}

\subsection{Proofs for Projected SGD Theorem }
\label{secpfpst}
This section provides the detailed proofs for the theorems and lemma presented in the main text. For clarity, we first recall the central definitions for a given subspace $\mathcal{S} \in \{\mathcal{D}, \mathcal{B}\}$:
\begin{align*}
s_t^{\mathcal S} &:=\sum_{i\in\mathcal S}\lambda_i^2 c_{i,t}^2, & \tau_t^{\mathcal S} &:=\sum_{i\in\mathcal S}\lambda_i^3 c_{i,t}^2, & u_t^{\mathcal S} &:=\sum_{i\in\mathcal S}\lambda_i^4 c_{i,t}^2, \\
n^{\mathrm{loss}}_{\mathcal S} &:=\sum_{i\in\mathcal S}\lambda_i \kappa_i^2, & \mu_t^{\mathcal S}(p, q) &:= \frac{\sum_{i\in\mathcal S}\lambda_i^p c_{i,t}^2}{\sum_{i\in\mathcal S}\lambda_i^q c_{i,t}^2}, & \theta_t &:= \frac{s_t^{\mathcal D}}{s_t^{\mathcal D}+s_t^{\mathcal B}}.
\end{align*}

We begin with the proof of the first theorem, which establishes the one-step expected loss condition.

\begin{theorem-restated}\ref{thm:proj-loss-corr}(\textbf{\textit{Condition Differences on Different Alignment Regime}})
For a given state $\bm{x}_t$, the one-step expected loss decreases if and only if the step size $\eta_t$ satisfies the following conditions for the dominant and bulk projected updates, respectively:
\\
\noindent\fbox{%
\begin{minipage}{0.97\linewidth}
\begin{align*}
\text{For } \mathcal{D}: \quad \mathbb E[L(\bm{x}_{t+1})-L(\bm{x}_t)\mid \bm{x}_t]<0 \quad &\Longleftrightarrow \quad 0<\eta_t< \eta^{\mathrm{loss}}_{\mathcal D}(\bm{x}_t):=\frac{2\,s_t^{\mathcal D}}{\tau_t^{\mathcal D}+n^{\mathrm{loss}}_{\mathcal D}} \\[2.5ex]
\text{For } \mathcal{B}: \quad \mathbb E[L(\bm{x}_{t+1})-L(\bm{x}_t)\mid \bm{x}_t]<0 \quad &\Longleftrightarrow \quad 0<\eta_t< \eta^{\mathrm{loss}}_{\mathcal B}(\bm{x}_t):=\frac{2\,s_t^{\mathcal B}}{\tau_t^{\mathcal B}+n^{\mathrm{loss}}_{\mathcal B}}
\end{align*}
\end{minipage}
}
\end{theorem-restated}

\begin{proof}
Let the projected stochastic gradient be $\bm{g}_t^{\mathcal S} = \mP^{\mathcal S}(\mA\bm{x}_t + \boldsymbol{\vxi}_t)$. The update rule is $\bm{x}_{t+1} = \bm{x}_t - \eta_t \bm{g}_t^{\mathcal S}$. The one-step change in the loss function $L(\bm{x}) = \frac{1}{2}\bm{x}^\top A \bm{x}$ is:
\begin{align*}
\Delta L &= L(\bm{x}_{t+1}) - L(\bm{x}_t) = \frac{1}{2}(\bm{x}_t - \eta_t \bm{g}_t^{\mathcal S})^\top \mA (\bm{x}_t - \eta_t \bm{g}_t^{\mathcal S}) - \frac{1}{2}\bm{x}_t^\top \mA \bm{x}_t \\
&= -\eta_t (\bm{g}_t^{\mathcal S})^\top \mA \bm{x}_t + \frac{1}{2}\eta_t^2 (\bm{g}_t^{\mathcal S})^\top \mA \bm{g}_t^{\mathcal S}.
\end{align*}
We now take the expectation with respect to the noise $\boldsymbol{\vxi}_t$, conditioned on $\bm{x}_t$. The linear term becomes $\mathbb{E}[(\bm{g}_t^{\mathcal S})^\top \mA \bm{x}_t] = (\mP^{\mathcal S}\mA\bm{x}_t)^\top A \bm{x}_t = \bm{x}_t^\top \mA \mP^{\mathcal S} \mA \bm{x}_t = s_t^{\mathcal S}$. \addressedtianyu{check this.}
The quadratic term becomes $\mathbb{E}[(\bm{g}_t^{\mathcal S})^\top \mA \bm{g}_t^{\mathcal S}] = (\mP^{\mathcal S}\mA\bm{x}_t)^\top \mA (\mP^{\mathcal S}\mA\bm{x}_t) + \mathbb{E}[(\mP^{\mathcal S}\boldsymbol{\vxi}_t)^\top \mA (\mP^{\mathcal S}\boldsymbol{\vxi}_t)] = \tau_t^{\mathcal S} + n^{\mathrm{loss}}_{\mathcal S}$.
Combining these, the expected change in loss is $\mathbb{E}[\Delta L \mid \bm{x}_t] = -\eta_t s_t^{\mathcal S} + \frac{1}{2}\eta_t^2 (\tau_t^{\mathcal S} + n^{\mathrm{loss}}_{\mathcal S})$. Setting this to be less than zero holds for $0 < \eta_t < \frac{2s_t^{\mathcal S}}{\tau_t^{\mathcal S} + n^{\mathrm{loss}}_{\mathcal S}}$.
\end{proof}

The proof of Theorem \ref{thm:loss-crossover-rigorous} relies on the fundamental properties of its governing quadratic function, which we first prove in the following lemma.

\paragraph{Additional Notion}
To give a further analysis of the projected SGD, we need to introduce the following quantities $\alpha(\bm{x}_t),\beta(\bm{x}_t),\gamma(\bm{x}_t)$ which depend on $\bm{x}_t$. (We will use $\alpha_t,\beta_t,\gamma_t$ later for short.):

\begin{align}
\alpha_t &=  s_t \left( \mu_t^{\mathcal D}(3,2) - \mu_t^{\mathcal B}(3,2) \right) \\
\beta_t &= -\alpha_t + n^{\mathrm{loss}}_{\mathcal B} + n^{\mathrm{loss}}_{\mathcal D}  \\
\gamma_t &= -n^{\mathrm{loss}}_{\mathcal D}
\end{align}

\begin{lemma}
\label{lem:crossover-quadratic-properties}
Let $h(\theta) := \alpha_t \theta^2 + \beta_t \theta + \gamma_t$ be the crossover quadratic, with coefficients as defined above. The following properties hold:
\begin{enumerate}
    \item $\forall\, \bm{x}_t,\ \alpha_t > 0$ \quad (positive definite in $\bm{x}_t$),
    \item $\forall\, \bm{x}_t,\ \gamma_t < 0$ \quad (negative definite in $\bm{x}_t$).
\end{enumerate}
Additionally, $\forall\, \bm{x}_t$, $h(0) = \gamma_t < 0$ and $h(1) = \alpha_t + \beta_t + \gamma_t > 0$.
\end{lemma}

\begin{proof}
First, we prove that the leading coefficient $\alpha_t$ is positive definite. By definition,
\[
\alpha_t = s_t \left( \mu_t^{\mathcal D}(3,2) - \mu_t^{\mathcal B}(3,2) \right) .
\]
For any $\bm{x}_t$, $s_t > 0$, so the sign of $\alpha_t$ is determined by the term in the parenthesis. This term is positive if and only if the following inequality holds:
\[
\mu_t^{\mathcal D}(3,2) >\mu_t^{\mathcal B}(3,2).
\]
For the dominant subspace, $\forall i\in\mathcal{D}, \lambda_i \ge \lambda_k$, which provides a lower bound:
\[
\mu_t^{\mathcal D}(3,2) = \frac{\sum_{i\in\mathcal{D}}\lambda_i \cdot (\lambda_i^2 c_{i,t}^2)}{\sum_{i\in\mathcal{D}}\lambda_i^2 c_{i,t}^2} \ge \frac{\sum_{i\in\mathcal{D}}\lambda_k \cdot (\lambda_i^2 c_{i,t}^2)}{\sum_{i\in\mathcal{D}}\lambda_i^2 c_{i,t}^2} = \lambda_k.
\]
For the bulk subspace, $\forall j\in\mathcal{B}, \lambda_j \le \lambda_{k+1}$, which provides an upper bound:
\[
\mu_t^{\mathcal B}(3,2) = \frac{\sum_{j\in\mathcal{B}}\lambda_j \cdot (\lambda_j^3 c_{j,t}^2)}{\sum_{j\in\mathcal{B}}\lambda_j^2 c_{j,t}^2} \le \frac{\sum_{j\in\mathcal{B}}\lambda_{k+1} \cdot (\lambda_j^2 c_{j,t}^2)}{\sum_{j\in\mathcal{B}}\lambda_j^2 c_{j,t}^2} = \lambda_{k+1}.
\]
The spectral gap condition $\lambda_k > \lambda_{k+1}$ combines these bounds into the strict inequality chain:
\[
\mu_t^{\mathcal D}(3,2) \ge \lambda_k > \lambda_{k+1} \ge \mu_t^{\mathcal B}(3,2).
\]
This proves $\mu_t^{\mathcal D}(3,2) > \mu_t^{\mathcal B}(3,2)$, which implies that $\alpha_t > 0$.

Second, we prove that the constant coefficient $\gamma_t$ is negative definite. By definition,
\[
\gamma_t = -n^{\mathrm{loss}}_{\mathcal D}.
\]
For any $\bm{x}_t$ and noise $\vxi$, $n^{\mathrm{loss}}_{\mathcal D} = \sum_{i\in\mathcal{D}} \lambda_i \kappa_i^2 > 0$. Since $\gamma_t$ is the negative product of two positive definite quantities, it follows that $\gamma_t < 0$.

Finally, we verify the function's values at the boundaries. At $\theta=0$, the value is the constant term, so $h(0) = \gamma_t$. As just shown, this is negative definite. At $\theta=1$, the value is the sum of the coefficients. Substituting their definitions and simplifying yields:
\begin{align*}
h(1) &= \alpha_t + \beta_t + \gamma_t \\
&= \alpha_t + \left( -\alpha_t + n^{\mathrm{loss}}_{\mathcal B} + n^{\mathrm{loss}}_{\mathcal D}\right) + \left( -n^{\mathrm{loss}}_{\mathcal D}\right) \\
&= n^{\mathrm{loss}}_{\mathcal B}.
\end{align*}
Since $n^{\mathrm{loss}}_{\mathcal B} > 0$ , we have $h(1) > 0$. This completes the proof of all claims.
\end{proof}

\begin{lemma}
\label{lem:crossover-quadratic-critical-root}
Under the conditions of Lemma~\ref{lem:crossover-quadratic-properties}, there exists a unique critical function $\theta^{\text{crit}}(\bm{x}_t): \mathbb{R}^d \to (0,1)$. For brevity, we will use $\theta^{\text{crit}}_t$ for $\theta^{\text{crit}}(\bm{x}_t)$ in the following. For  $h(\theta) = \alpha_t \theta^2 + \beta_t \theta + \gamma_t$, the following hold:
\begin{enumerate}
    \item $\forall\, \bm{x}_t,\ \forall\, \theta \in \bigl(0, \theta^{\text{crit}}_t\bigr),\ h(\theta) < 0$,
    \item $\forall\, \bm{x}_t,\ \forall\, \theta \in \bigl(\theta^{\text{crit}}_t, 1\bigr),\ h(\theta) > 0$.
\end{enumerate}
Moreover, $\theta^{\text{crit}}_t$ is the unique root of $h(\theta) = 0$ in $(0,1)$, given explicitly by:
\[
\theta^{\text{crit}}_t = \frac{ -\beta_t + \sqrt{\beta_t^2 - 4 \alpha_t \gamma_t} }{2 \alpha_t}.
\]
\end{lemma}

\begin{proof}
The proof relies on the properties of the quadratic function $h(\theta) = \alpha_t \theta^2 + \beta_t \theta + \gamma_t$ established in Lemma~\ref{lem:crossover-quadratic-properties} for any given $\bm{x}_t$.

First, we analyze the roots of the equation $h(\theta)=0$. The discriminant is $\Delta_t = \beta_t^2 - 4\alpha_t\gamma_t$. From Lemma~\ref{lem:crossover-quadratic-properties}, we have $\alpha_t > 0$ and $\gamma_t < 0$. This implies that the term $-4\alpha_t\gamma_t$ is strictly positive. Therefore,
\[
\Delta_t = \beta_t^2 - 4\alpha_t\gamma_t > 0.
\]
Since $\Delta_t > 0$, the equation $h(\theta)=0$ has two distinct real roots.

Next, we analyze the location of these roots using Vieta's formulas. The product of the roots is given by:
\[
\theta_1 \theta_2 = \frac{\gamma_t}{\alpha_t}.
\]
Since $\gamma_t < 0$ and $\alpha_t > 0$, their ratio is strictly negative, $\frac{\gamma_t}{\alpha_t} < 0$. This proves that one root is positive and the other is negative.

Let the unique positive root be denoted by $\theta^{\text{crit}}_t$. We now prove that this root must lie in the interval $(0,1)$. Lemma~\ref{lem:crossover-quadratic-properties} states that $h(0) < 0$ and $h(1) > 0$. Since $h(\theta)$ is a continuous function, the Intermediate Value Theorem guarantees that there must be at least one root in $(0,1)$. As we have already established that there is only one positive root, this positive root must be the one that lies in $(0,1)$. Thus, there exists a unique root $\theta^{\text{crit}}_t$ in the interval $(0,1)$.

Now, we derive the explicit formula for this root. The two roots of the quadratic equation are given by $\frac{-\beta_t \pm \sqrt{\Delta_t}}{2\alpha_t}$. We must identify which sign corresponds to the positive root. We analyze the magnitude of the square root term:
\[
\Delta_t = \beta_t^2 - 4\alpha_t\gamma_t > \beta_t^2 \implies \sqrt{\Delta_t} > \sqrt{\beta_t^2} = |\beta_t|.
\]
Consider the root with the minus sign: $-\beta_t - \sqrt{\Delta_t}$. Since $\sqrt{\Delta_t} > |\beta_t| \ge -\beta_t$, this term is always negative. Thus, $\frac{-\beta_t - \sqrt{\Delta_t}}{2\alpha_t}$ is the negative root.
Consider the root with the plus sign: $-\beta_t + \sqrt{\Delta_t}$. Since $\sqrt{\Delta_t} > |\beta_t| \ge \beta_t$, this term is always positive. Thus, $\frac{-\beta_t + \sqrt{\Delta_t}}{2\alpha_t}$ is the unique positive root. We conclude:
\[
\theta^{\text{crit}}_t = \frac{ -\beta_t + \sqrt{\beta_t^2 - 4 \alpha_t \gamma_t} }{2 \alpha_t}.
\]
Finally, we prove the sign of $h(\theta)$ on either side of $\theta^{\text{crit}}_t$. From Lemma~\ref{lem:crossover-quadratic-properties}, we know that $h(\theta)$ is an upward-opening parabola ($\alpha_t>0$). A continuous, upward-opening parabola with a single root in an interval must be negative before that root and positive after it within that interval.
\begin{itemize}
    \item For any $\theta \in (0, \theta^{\text{crit}}_t)$, since $h(0)<0$ and the only root in this range is at the endpoint, $h(\theta)$ must remain negative.
    \item For any $\theta \in (\theta^{\text{crit}}_t, 1)$, since $h(1)>0$ and the only root in this range is at the startpoint, $h(\theta)$ must remain positive.
\end{itemize}
This completes the proof of all claims.
\end{proof}

\begin{theorem-restated}[\ref{thm:loss-crossover-rigorous}](\textbf{\textit{Condition Differences on Different Alignment Regime}})
Under Assumption \ref{asp:standing}, the relative ordering of the loss thresholds is determined by the alignment $\theta_t$ of a given $\bm{x}_t$ relative to the unique critical threshold, $\theta^{\text{crit}}_t \in (0,1)$. This threshold is the root of the quadratic equation $h(\theta) = \alpha_t \theta^2 + \beta_t \theta + \gamma_t = 0$, with coefficients defined by:

\begin{align}
\alpha_t &=  s_t \left( \mu_t^{\mathcal D}(3,2) - \mu_t^{\mathcal B}(3,2) \right) \\
\beta_t &= -\alpha_t + n^{\mathrm{loss}}_{\mathcal B} + n^{\mathrm{loss}}_{\mathcal D}  \\
\gamma_t &= -n^{\mathrm{loss}}_{\mathcal D}
\end{align}
The ordering of the thresholds in the two resulting regimes is as follows:
\\
\noindent\fbox{%
\begin{minipage}{0.97\linewidth}
\centering
\begin{tabular}{c@{\qquad\qquad}c}
\boxed{\theta_t < \theta^{\text{crit}}_t} & \boxed{\theta_t > \theta^{\text{crit}}_t} \\[1.5ex]
$\underbrace{\eta^{\mathrm{loss}}_{\mathcal D}(\bm{x}_t) < \eta^{\mathrm{loss}}_{\mathcal B}(\bm{x}_t)}_{\text{Low-Alignment Regime}}$
&
$\underbrace{\eta^{\mathrm{loss}}_{\mathcal D}(\bm{x}_t) > \eta^{\mathrm{loss}}_{\mathcal B}(\bm{x}_t)}_{\text{High-Alignment Regime}}$
\end{tabular}
\vspace{0.2cm}
\end{minipage}
}
\end{theorem-restated}

\begin{proof}
We begin by analyzing the sign of the difference between the two loss thresholds.
\[
\eta^{\mathrm{loss}}_{\mathcal D}(\bm{x}_t) - \eta^{\mathrm{loss}}_{\mathcal B}(\bm{x}_t) = \frac{2s_t^{\mathcal D}}{\tau_t^{\mathcal D}+n^{\mathrm{loss}}_{\mathcal D}} - \frac{2s_t^{\mathcal B}}{u_t^{\mathcal B}+n^{\mathrm{loss}}_{\mathcal B}} = \frac{2\left( s_t^{\mathcal D}(\tau_t^{\mathcal B}+n^{\mathrm{loss}}_{\mathcal B}) - s_t^{\mathcal B}(\tau_t^{\mathcal D}+n^{\mathrm{loss}}_{\mathcal D}) \right)}{(\tau_t^{\mathcal D}+n^{\mathrm{loss}}_{\mathcal D})(\tau_t^{\mathcal B}+n^{\mathrm{loss}}_{\mathcal B})}.
\]
Since the denominator is a product of positive definite terms, the sign of this difference is strictly determined by the sign of the numerator.
\begin{equation}
\mathrm{sign}\left(\eta^{\mathrm{loss}}_{\mathcal{D}}(\bm{x}_t) - \eta^{\mathrm{loss}}_{\mathcal{B}}(\bm{x}_t)\right) = \mathrm{sign}\left(s_t^{\mathcal D}(u_t^{\mathcal{B}} + n^{\mathrm{loss}}_{\mathcal{B}}) - s_t^{\mathcal D}(u_t^{\mathcal{D}} + n^{\mathrm{loss}}_{\mathcal{D}})\right).
\label{eq:Eq-DiffLossSign-1}
\end{equation}

We now perform a detailed expansion of the numerator expression. Notice that $\tau_t^{\mathcal S} = \mu_t^{\mathcal S}(3,2) s_t^{\mathcal S}$ and $\theta_t= \frac{s_t^{\mathcal D}}{s_t}=1-\frac{s_t^{\mathcal B}}{s_t}$:
\begin{align*}
&\quad s_t^{\mathcal D}(\tau_t^{\mathcal B} + n^{\mathrm{loss}}_{\mathcal B}) - s_t^{\mathcal B}(\tau_t^{\mathcal D} + n^{\mathrm{loss}}_{\mathcal D}) \\
&= s_t^{\mathcal D} \left( \mu_t^{\mathcal B}(3,2)s_t^{\mathcal B} + n^{\mathrm{loss}}_{\mathcal B} \right) - s_t^{\mathcal B} \left( \mu_t^{\mathcal D}(3,2)s_t^{\mathcal D} + n^{\mathrm{loss}}_{\mathcal D} \right) \\
&= s_t^{\mathcal D} s_t^{\mathcal B} \left( \mu_t^{\mathcal B}(3,2) - \mu_t^{\mathcal D}(3,2)\right) + s_t^{\mathcal D} n^{\mathrm{loss}}_{\mathcal B}- s_t^{\mathcal B} n^{\mathrm{loss}}_{\mathcal D}\\
&= s_t^2 \theta_t (1-\theta_t) \left( \mu_t^{\mathcal B}(3,2) - \mu_t^{\mathcal D}(3,2)\right) \\
& \qquad + s_t \theta_t n^{\mathrm{loss}}_{\mathcal B}- s_t(1-\theta_t) n^{\mathrm{loss}}_{\mathcal D}.
\end{align*}
For non-zero $\bm{x}_t$, we have $s_t > 0$. We can factor out $s_t$ from the entire expression and collect terms by powers of $\theta_t$:
\begin{align*}
= s_t \cdot \Bigg( &\theta_t^2 \Big[ s_t \left( \mu_t^{\mathcal D}(3,2) - \mu_t^{\mathcal B}(3,2) \right) \Big] \\
+ \quad &\theta_t \Big[ s_t \left( \mu_t^{\mathcal D}(3,2) - \mu_t^{\mathcal B}(3,2) \right)  + n^{\mathrm{loss}}_{\mathcal B} + n^{\mathrm{loss}}_{\mathcal D} \Big] \\
+ \quad &\Big[ -n^{\mathrm{loss}}_{\mathcal D} \Big] \Bigg).
\end{align*}
The expression inside the parentheses is a quadratic function of $\theta_t$. We define this normalized quadratic as our crossover function $h(\theta_t)$:
\[
h(\theta_t) := \alpha_t \theta_t^2 + \beta_t \theta_t + \gamma_t,
\]
Where the coefficients $\alpha_t, \beta_t, \gamma_t$ are defined as:

\begin{align}
\alpha_t &=  s_t \left( \mu_t^{\mathcal D}(3,2) - \mu_t^{\mathcal B}(3,2) \right) \\
\beta_t &= -\alpha_t + n^{\mathrm{loss}}_{\mathcal B} + n^{\mathrm{loss}}_{\mathcal D} \\
\gamma_t &= -n^{\mathrm{loss}}_{\mathcal D}
\end{align}
Which precisely match the coefficients of Lemma~\ref{lem:crossover-quadratic-properties}. We have thus established the central equivalence:
\[
\mathrm{sign}\left(\eta^{\mathrm{loss}}_{\mathcal D}(\bm{x}_t) - \eta^{\mathrm{loss}}_{\mathcal B}(\bm{x}_t)\right) = \mathrm{sign}\left(h(\theta_t(\bm{x}_t))\right).
\]
The remainder of the proof follows directly from Lemma~\ref{lem:crossover-quadratic-properties} and Lemma~\ref{lem:crossover-quadratic-critical-root}. These lemmas establish that for any given $\bm{x}_t$, the function $h(\theta)$ has a unique root $\theta^{\text{crit}}_t \in (0,1)$, is strictly negative for all $\theta \in (0, \theta^{\text{crit}}_t)$, and is strictly positive for all $\theta \in (\theta^{\text{crit}}_t, 1)$. This proves the two regimes stated in the theorem.
\end{proof}

\begin{theorem-restated}\ref{thm:loss-crossover-asymptotic}(\textbf{\textit{Asymptotic Limit of the Alignment Threshold}})
Under Assumption \ref{asp:standing}, let the spectral gap be denoted by $m := \lambda_k/\lambda_{k+1}$. In the limit as the gap grows infinitely large, the critical threshold converges to 1:
\[
\lim_{m \to \infty} \theta^{\eta}_{\mathrm{crit}}(\bm{x}_t) = 1.
\]
\end{theorem-restated}

\begin{proof}
Recall that the critical threshold $\theta^{\text{crit}}_t$ is defined as the alignment where the loss thresholds are equal. Citing the condition from Eq.~\eqref{eq:Eq-DiffLossSign-1}, any state $\bm{x}_t$ at this threshold must satisfy the exact equality:
\[
s_t^{\mathcal D}(\tau_t^{\mathcal B} + n^{\mathrm{loss}}_{\mathcal B}) - s_t^{\mathcal B}(\tau_t^{\mathcal D} + n^{\mathrm{loss}}_{\mathcal D}) = 0.
\]
This equality can be rearranged by separating $\tau_t^{\mathcal S}$ and $n^{\mathrm{loss}}_{\mathcal S}$:
\[
s_t^{\mathcal B}\tau_t^{\mathcal D} - s_t^{\mathcal D}\tau_t^{\mathcal B} = s_t^{\mathcal D}n^{\mathrm{loss}}_{\mathcal B} - s_t^{\mathcal B}n^{\mathrm{loss}}_{\mathcal D}.
\]
Dividing by $s_t^{\mathcal D}s_t^{\mathcal B}$ gives the equivalent :
\[
\mu_t^{\mathcal D}(3,2) - \mu_t^{\mathcal B}(3,2) = \frac{n^{\mathrm{loss}}_{\mathcal B}}{s_t^{\mathcal B}} - \frac{n^{\mathrm{loss}}_{\mathcal D}}{s_t^{\mathcal D}}.
\]
We now analyze this equality in the limit as the spectral gap $m := \lambda_k/\lambda_{k+1} \to \infty$. We can establish a lower bound for the left-hand side (LHS) and an upper bound for the right-hand side (RHS):
\[
\text{LHS} = \mu_t^{\mathcal D}(3,2) - \mu_t^{\mathcal B}(3,2) \ge \lambda_k - \lambda_{k+1},
\]
\[
\text{RHS} = \frac{n^{\mathrm{loss}}_{\mathcal B}}{s_t^{\mathcal B}} - \frac{n^{\mathrm{loss}}_{\mathcal D}}{s_t^{\mathcal D}} < \frac{n^{\mathrm{loss}}_{\mathcal B}}{s_t^{\mathcal B}}.
\]
For the equality to hold, the lower bound of the LHS must be less than the upper bound of the RHS. This provides a necessary condition that any state at the crossover must satisfy:
\[
\lambda_k - \lambda_{k+1} < \frac{n^{\mathrm{loss}}_{\mathcal B}}{s_t^{\mathcal B}}.
\]
This inequality provides a strict upper bound on the third-order bulk energy, $s_t^{\mathcal B}$:
\[
0 \le s_t^{\mathcal B} < \frac{n^{\mathrm{loss}}_{\mathcal B}}{\lambda_k - \lambda_{k+1}}.
\]

\paragraph{Remark}
{\itshape It is crucial to note that the term on the left-hand side of the inequality, $s_t^{\mathcal B}$, is intrinsically dependent on the vector $\bm{x}_t$, since its definition $s_t^{\mathcal B} = \sum_{j\in\mathcal{B}} \lambda_j^3 c_{j,t}^2$ explicitly involves the state's components $c_{j,t}$. Conversely, the upper bound on the right-hand side is independent of $\bm{x}_t$, as both the numerator $n^{\mathrm{loss}}_{\mathcal B} = \sum_{j\in\mathcal{B}} \lambda_j \kappa_j^2$ and the denominator $\lambda_k - \lambda_{k+1}$ are determined solely by the constant Hessian eigenvalues and noise covariance structure. This inequality therefore provides a state-independent upper bound for a state-dependent quantity.}\\

We now show that this upper bound converges to zero as $m \to \infty$, independently of the state components $c_{j,t}$. The numerator is bounded by $n^{\mathrm{loss}}_{\mathcal B} = \sum_{j\in\mathcal{B}} \lambda_j \kappa_j^2 \le \lambda_{k+1} \sum_{j\in\mathcal{B}} \kappa_j^2$. The denominator is $\lambda_{k+1}(m-1)$. Combining these gives:
\[
\frac{n^{\mathrm{loss}}_{\mathcal B}}{\lambda_k - \lambda_{k+1}} \le \frac{\lambda_{k+1} \sum_{j\in\mathcal{B}} \kappa_j^2}{\lambda_{k+1}(m-1)} = \frac{1}{m-1} \sum_{j\in\mathcal{B}} \kappa_j^2\leq \frac{1}{m-1}  Tr(\mSigma).
\]
The term $Tr(\mSigma)$ is a finite constant dependent only on the noise covariance, not on the state $\bm{x}_t$. Thus, the upper bound for $s_t^{\mathcal B}$ converges to zero as $m \to \infty$. By the Squeeze Theorem, we rigorously conclude that $\lim_{m\to\infty} s_t^{\mathcal B} = 0$.

Finally, the limit of the critical alignment $\theta^{\text{crit}}_t = s_t^{\mathcal D} / (s_t^{\mathcal D} + s_t^{\mathcal B})$ can be determined. For a non-trivial state where the total energy does not vanish, we conclude:
\[
\lim_{m\to\infty} \theta^{\text{crit}}_t = \frac{\lim_{m\to\infty} s_t^{\mathcal D}}{\lim_{m\to\infty} s_t^{\mathcal D} + \lim_{m\to\infty} s_t^{\mathcal B}} = \frac{s_t^{\mathcal D}}{s_t^{\mathcal D} + 0} = 1.
\]
\end{proof}

{\begin{theorem-restated}\ref{thm:theta-crit-rate}(\textbf{\textit{Asymptotic rate of $\theta_t^{\mathrm{crit}}$}})
Under Assumption~\ref{asp:standing}, let $m := \frac{\lambda_k}{\lambda_{k+1}} > 1.$ Then the critical alignment threshold $\theta_t^{\mathrm{crit}}\in(0,1)$ satisfies
\[
\frac{n_{\mathcal{B}}^{\mathrm{loss}}}
{s_t\,\lambda_{k+1}(m-1)
+n_{\mathcal{B}}^{\mathrm{loss}}
+n_{\mathcal{D}}^{\mathrm{loss}}}
\;\le\;
1-\theta_t^{\mathrm{crit}}
\;\le\;
\frac{n_{\mathcal{B}}^{\mathrm{loss}}}
{s_t\,\lambda_{k+1}(m-1)}.
\]
If $\lambda_{k+1}=\Theta(1)$, consequently,
\[
\lambda_{k}=\Theta(m),  1-\theta_t^{\mathrm{crit}}
\in
\Theta\!\left(\frac{1}{s_t\,(m-1)}\right).
\]
\end{theorem-restated}

\begin{proof}
By Theorem~\ref{thm:loss-crossover-rigorous}, the critical threshold
$\theta_t^{\mathrm{crit}}\in(0,1)$ is the unique solution in $(0,1)$ of
\[
\alpha_t \theta^2 + \beta_t \theta + \gamma_t = 0,
\]
where
\[
\alpha_t
=
s_t\bigl(
\mu_t^{\mathcal D}(3,2)-\mu_t^{\mathcal B}(3,2)
\bigr),\quad
\beta_t
=
-\alpha_t
+
n_{\mathcal B}^{\mathrm{loss}}
+
n_{\mathcal D}^{\mathrm{loss}},\quad
\gamma_t
=
-\,n_{\mathcal D}^{\mathrm{loss}}.
\]
Introduce the change of variables $\delta := 1-\theta$.
Substituting $\theta = 1-\delta$ into the quadratic equation yields
\[
\alpha_t(1-\delta)^2
+
\bigl(
-\alpha_t
+
n_{\mathcal B}^{\mathrm{loss}}
+
n_{\mathcal D}^{\mathrm{loss}}
\bigr)(1-\delta)
-
n_{\mathcal D}^{\mathrm{loss}}
=0.
\]
Expanding and collecting terms gives
\[
\alpha_t
-
2\alpha_t\delta
+
\alpha_t\delta^2
-
\alpha_t
+
n_{\mathcal B}^{\mathrm{loss}}
+
n_{\mathcal D}^{\mathrm{loss}}
-
\bigl(
-\alpha_t
+
n_{\mathcal B}^{\mathrm{loss}}
+
n_{\mathcal D}^{\mathrm{loss}}
\bigr)\delta
-
n_{\mathcal D}^{\mathrm{loss}}
=0,
\]
which simplifies to
\[
\alpha_t\delta^2
+
\bigl(
n_{\mathcal B}^{\mathrm{loss}}
+
n_{\mathcal D}^{\mathrm{loss}}
\bigr)\delta
-
n_{\mathcal B}^{\mathrm{loss}}
=0.
\]
Rearranging yields the equivalent identity
\[
\delta\bigl(
\alpha_t\delta
+
n_{\mathcal B}^{\mathrm{loss}}
+
n_{\mathcal D}^{\mathrm{loss}}
\bigr)
=
n_{\mathcal B}^{\mathrm{loss}}.
\]
Since $\delta>0$, this identity immediately implies the bounds
\[
\frac{n_{\mathcal B}^{\mathrm{loss}}}
{\alpha_t
+
n_{\mathcal B}^{\mathrm{loss}}
+
n_{\mathcal D}^{\mathrm{loss}}}
\;\le\;
\delta
\;\le\;
\frac{n_{\mathcal B}^{\mathrm{loss}}}{\alpha_t}.
\]
Next, by definition,
\[
\mu_t^{\mathcal D}(3,2)
=
\frac{\sum_{i\in\mathcal D}\lambda_i^3 c_{i,t}^2}
{\sum_{i\in\mathcal D}\lambda_i^2 c_{i,t}^2}
=
\sum_{i\in\mathcal D}
\left(
\frac{\lambda_i^2 c_{i,t}^2}
{\sum_{j\in\mathcal D}\lambda_j^2 c_{j,t}^2}
\right)\lambda_i,
\]
which is a convex combination of $\{\lambda_i:i\in\mathcal D\}$ and therefore
satisfies $\mu_t^{\mathcal D}(3,2)\ge\lambda_k$. Similarly,
$\mu_t^{\mathcal B}(3,2)\le\lambda_{k+1}$.
Consequently,
\[
\alpha_t
=
s_t\bigl(
\mu_t^{\mathcal D}(3,2)-\mu_t^{\mathcal B}(3,2)
\bigr)
\ge
s_t(\lambda_k-\lambda_{k+1})
=
s_t\,\lambda_{k+1}(m-1).
\]
Substituting this bound into the previous inequalities and recalling that
$\delta = 1-\theta_t^{\mathrm{crit}}$ yields
\[
\frac{n_{\mathcal{B}}^{\mathrm{loss}}}
{s_t\,\lambda_{k+1}(m-1)
+n_{\mathcal{B}}^{\mathrm{loss}}
+n_{\mathcal{D}}^{\mathrm{loss}}}
\;\le\;
1-\theta_t^{\mathrm{crit}}
\;\le\;
\frac{n_{\mathcal{B}}^{\mathrm{loss}}}
{s_t\,\lambda_{k+1}(m-1)}.
\]
If $\lambda_{k+1}=\Theta(1)$, then $\lambda_k=\Theta(m)$ and
$\lambda_k-\lambda_{k+1}=\Theta(m-1)$, and the above bounds imply
\[
1-\theta_t^{\mathrm{crit}}
\in
\Theta\!\left(\frac{1}{s_t\,(m-1)}\right),
\]
which completes the proof.
\end{proof}

}

\subsection{Lemma Of CSGD}
Let $\mathbb{E}_t[\cdot \mid \bm{x}_j],\mathrm{Var}_t[\cdot \mid \bm{x}_j
]$ denote expectation,variance over the noise sequence $\ \{\boldsymbol{\xi}_s\}_{s=j+1}^{t-1}$. 
\begin{lemma}[Asymptotic representation of conditional alignment]\label{lem:asymp-rep}
Under Assumption~\ref{asp:standing}, define the normalized block statistics at time $t$ for $\mathcal{S} \in \{\mathcal{D}, \mathcal{B}\}$ as
\[
s_t^{\prime,\mathcal{S}} := \frac{1}{|\mathcal{S}|} \sum_{i \in \mathcal{S}} \lambda_i^2 c_{i,t}^2, \quad
u_t^{\prime,\mathcal{S}} := \frac{1}{|\mathcal{S}|} \sum_{i \in \mathcal{S}} \lambda_i^4 c_{i,t}^2, \quad
w_t^{\prime,\mathcal{S}} := \frac{1}{|\mathcal{S}|} \sum_{i \in \mathcal{S}} \lambda_i^3 c_{i,t}^2,
\]
and let $\bar{e}_{\mathcal{D}} := \frac{1}{k} e_{\mathcal{D}}$, $\bar{e}_{\mathcal{B}} := \frac{1}{d-k} e_{\mathcal{B}}$.
Define the deterministic function $k: (0,\infty)^6 \to [0,1]$ by
\[
k(s_1, s_2, u_1, u_2, w_1, w_2) 
= \frac{1}{1 + \frac{1}{\rho} \cdot 
\frac{s_2 - 2\eta_t w_2 + \eta_t^2 (u_2 + \bar{e}_{\mathcal{B}})}
{s_1 - 2\eta_t w_1 + \eta_t^2 (u_1 + \bar{e}_{\mathcal{D}})}},
\quad \text{where } \rho = \lim_{d\rightarrow\infty}\frac{k}{d - k}.
\]
Then, under the spectral and noise assumptions, we have:
\begin{align}
    &k(s_1, s_2, u_1, u_2, w_1, w_2) \notag\\
    &= \lim_{d \to \infty} \mathbb{E}_{t+1}\left[ \theta_{t+1} \,\middle|\, 
    s_t^{\prime,\mathcal{D}} = s_1,\, s_t^{\prime,\mathcal{B}} = s_2,\, 
    u_t^{\prime,\mathcal{D}} = u_1,\, u_t^{\prime,\mathcal{B}} = u_2,\, 
    w_t^{\prime,\mathcal{D}} = w_1,\, w_t^{\prime,\mathcal{B}} = w_2 \right]
\end{align}
Notice that for any deterministic $\bm{x}$ satisfy :$$s^{\prime,\mathcal{D}} = s_1,\, s^{\prime,\mathcal{B}} = s_2,\, 
    u^{\prime,\mathcal{D}} = u_1,\, u^{\prime,\mathcal{B}} = u_2,\, 
    w^{\prime,\mathcal{D}} = w_1,\, w^{\prime,\mathcal{B}} = w_2$$
We have:
$$
 k(s_1, s_2, u_1, u_2, w_1, w_2)=\lim_{d \to \infty} \mathbb{E}_{t+1}\left[ \theta_{t+1} \,\middle|\, 
    \bm{x}_t=\bm{x}\right]
$$

\end{lemma}


\begin{proof}
Recall that $\theta_{t+1} = \dfrac{s_{t+1}^{\mathcal{D}}}{s_{t+1}^{\mathcal{D}} + s_{t+1}^{\mathcal{B}}}$, where the unnormalized energies are
\[
s_{t+1}^{\mathcal{D}} = \sum_{i \in \mathcal{D}} \lambda_i^2 c_{i,t+1}^2, \quad
s_{t+1}^{\mathcal{B}} = \sum_{i \in \mathcal{B}} \lambda_i^2 c_{i,t+1}^2.
\]
From Lemma~\ref{lem:blk-exp-t-app}, their conditional expectations (with respect to the noise $\boldsymbol{\xi}_t$ at step $t$) are
\begin{align*}
\mathbb{E}_{t+1}[s_{t+1}^{\mathcal{D}} \mid \bm{x}_t] &= s_t^{\mathcal{D}} - 2\eta_t \sum_{i \in \mathcal{D}} \lambda_i^3 c_{i,t}^2 + \eta_t^2 \left( \sum_{i \in \mathcal{D}} \lambda_i^4 c_{i,t}^2 + e_{\mathcal{D}} \right), \\
\mathbb{E}_{t+1}[s_{t+1}^{\mathcal{B}} \mid \bm{x}_t] &= s_t^{\mathcal{B}} - 2\eta_t \sum_{i \in \mathcal{B}} \lambda_i^3 c_{i,t}^2 + \eta_t^2 \left( \sum_{i \in \mathcal{B}} \lambda_i^4 c_{i,t}^2 + e_{\mathcal{B}} \right).
\end{align*}
Dividing by block sizes yields
\begin{align*}
\frac{1}{k} \mathbb{E}_{t+1}[s_{t+1}^{\mathcal{D}} \mid \bm{x}_t] &= s_t^{\prime,\mathcal{D}} - 2\eta_t w_t^{\prime,\mathcal{D}} + \eta_t^2 (u_t^{\prime,\mathcal{D}} + \bar{e}_{\mathcal{D}}), \\
\frac{1}{d-k} \mathbb{E}_{t+1}[s_{t+1}^{\mathcal{B}} \mid \bm{x}_t] &= s_t^{\prime,\mathcal{B}} - 2\eta_t w_t^{\prime,\mathcal{B}} + \eta_t^2 (u_t^{\prime,\mathcal{B}} + \bar{e}_{\mathcal{B}}).
\end{align*}
Define the deterministic scalars
\[
\vartheta_1 := s_t^{\prime,\mathcal{D}} - 2\eta_t w_t^{\prime,\mathcal{D}} + \eta_t^2 (u_t^{\prime,\mathcal{D}} + \bar{e}_{\mathcal{D}}), \quad
\vartheta_2 := s_t^{\prime,\mathcal{B}} - 2\eta_t w_t^{\prime,\mathcal{B}} + \eta_t^2 (u_t^{\prime,\mathcal{B}} + \bar{e}_{\mathcal{B}}).
\]

We now analyze the concentration of the normalized energies. Expanding the SGD update
\[
\bm{x}_{t+1} = \bm{x}_t - \eta_t (\mA \bm{x}_t + \boldsymbol{\xi}_t)
\]
in the eigenbasis $\mU = [\bm{u}_1, \dots, \bm{u}_d]$, we have for each $i$:
\[
c_{i,t+1} = (1 - \eta_t \lambda_i) c_{i,t} - \eta_t \zeta_{i,t},
\]
where $\boldsymbol{\zeta}_t := \mU^\top \boldsymbol{\xi}_t \sim \mathcal{N}(\bm{0}, \mC)$ with $\mC = \mU^\top \mSigma \mU$ and $\kappa_i^2 = (\mC)_{ii}$. Hence,
\[
s_{t+1}^{\mathcal{S}} = \sum_{i \in \mathcal{S}} \lambda_i^2 c_{i,t+1}^2 
= \text{(deterministic)} + \text{(linear in } \boldsymbol{\zeta}_t) + \text{(quadratic in } \boldsymbol{\zeta}_t).
\]

By Lemma~\ref{lem:gauss-var} and Assumption~\ref{asp:standing}, the conditional variance satisfies
\[
\mathrm{Var}_{t+1}\left( \frac{1}{|\mathcal{S}|} s_{t+1}^{\mathcal{S}} \,\middle|\, \bm{x}_t \right) = O\left( \frac{1}{|\mathcal{S}|} \right).
\]
Thus, by Chebyshev's inequality,
\[
\lim_{d \to \infty} \mathbb{P}\left( \left| \frac{1}{|\mathcal{S}|} s_{t+1}^{\mathcal{S}} - \vartheta_{\mathcal{S}} \right| > \epsilon \,\middle|\, \bm{x}_t \right) = 0,
\]
where $\vartheta_{\mathcal{D}} = \vartheta_1$ and $\vartheta_{\mathcal{B}} = \vartheta_2$. Consequently,
\[
\lim_{d \to \infty} \frac{s_{t+1}^{\mathcal{B}}}{s_{t+1}^{\mathcal{D}}} = \frac{1}{\rho} \cdot \frac{\vartheta_2}{\vartheta_1} \quad \text{in probability (conditional on } \bm{x}_t\text{)}.
\]

Since $f(x) = 1/(1 + x)$ is continuous and bounded on $(0, \infty)$, Lemma~\ref{lem:cmt-dc} implies
\[
\lim_{d \to \infty} \mathbb{E}_{t+1}\big[ \theta_{t+1} \mid \bm{x}_t \big] 
= f\left( \frac{1}{\rho} \cdot \frac{\vartheta_2}{\vartheta_1} \right)
= k\big( s_t^{\prime,\mathcal{D}}, s_t^{\prime,\mathcal{B}}, u_t^{\prime,\mathcal{D}}, u_t^{\prime,\mathcal{B}}, w_t^{\prime,\mathcal{D}}, w_t^{\prime,\mathcal{B}} \big),
\]
which completes the proof.
\end{proof}

\begin{lemma}[Concentration of macroscopic statistics at time $t$ under CSGD]\label{lem:concentration-at-t}
Under Assumptions~\ref{asp:standing} and \ref{assum-cssa}, with deterministic initialization $\bm{x}_0$ (so that $c_{i,0}$ are fixed scalars), define for $\mathcal{S} \in \{\mathcal{D}, \mathcal{B}\}$ the normalized statistics at time $t$:
\[
s_t^{\prime,\mathcal{S}} := \frac{1}{|\mathcal{S}|} \sum_{i \in \mathcal{S}} \lambda_i^2 c_{i,t}^2, \quad
u_t^{\prime,\mathcal{S}} := \frac{1}{|\mathcal{S}|} \sum_{i \in \mathcal{S}} \lambda_i^4 c_{i,t}^2, \quad
w_t^{\prime,\mathcal{S}} := \frac{1}{|\mathcal{S}|} \sum_{i \in \mathcal{S}} \lambda_i^3 c_{i,t}^2.
\]
Then, as $d \to \infty$, each statistic converges in probability to its expectation, which is a deterministic function of the initial eigencoordinates $\{c_{i,0}\}$:
\begin{align*}
s_t^{\prime,\mathcal{S}} &\xrightarrow{p} \bar{s}_t^{\mathcal{S}} := \frac{1}{|\mathcal{S}|} \sum_{i \in \mathcal{S}} \lambda_i^2 \left[ (1 - \eta \lambda_i)^{2t} (c_{i,0}^2 - \beta_i) + \beta_i \right], \\
u_t^{\prime,\mathcal{S}} &\xrightarrow{p} \bar{u}_t^{\mathcal{S}} := \frac{1}{|\mathcal{S}|} \sum_{i \in \mathcal{S}} \lambda_i^4 \left[ (1 - \eta \lambda_i)^{2t} (c_{i,0}^2 - \beta_i) + \beta_i \right], \\
w_t^{\prime,\mathcal{S}} &\xrightarrow{p} \bar{w}_t^{\mathcal{S}} := \frac{1}{|\mathcal{S}|} \sum_{i \in \mathcal{S}} \lambda_i^3 \left[ (1 - \eta \lambda_i)^{2t} (c_{i,0}^2 - \beta_i) + \beta_i \right],
\end{align*}
where $\beta_i = \dfrac{\eta \kappa_i^2}{2\lambda_i - \eta \lambda_i^2} > 0$.
\end{lemma}

\begin{proof}
We prove the result for $s_t^{\prime,\mathcal{D}}$; the other five statistics follow identically by replacing the weight $\lambda_i^2$ with $\lambda_i^4$ or $\lambda_i^3$, and/or changing the block to $\mathcal{B}$. Under constant step size SGD, the eigencoordinate evolves as
\[
c_{i,t} = (1 - \eta \lambda_i)^t c_{i,0} - \eta \sum_{s=0}^{t-1} (1 - \eta \lambda_i)^{t-1-s} \zeta_{i,s},
\]
where $\zeta_{i,s} = \bm{u}_i^\top \boldsymbol{\xi}_s \sim \mathcal{N}(0, \kappa_i^2)$ are independent across $s$. Since $c_{i,0}$ is deterministic, $c_{i,t}$ is Gaussian with mean $\mu_{i,t} = (1 - \eta \lambda_i)^t c_{i,0}$ and variance
\[
\sigma_{i,t}^2 = \eta^2 \kappa_i^2 \sum_{s=0}^{t-1} (1 - \eta \lambda_i)^{2(t-1-s)} = \eta^2 \kappa_i^2 \frac{1 - (1 - \eta \lambda_i)^{2t}}{1 - (1 - \eta \lambda_i)^2} = \beta_i \left(1 - (1 - \eta \lambda_i)^{2t} \right),
\]
where we used the definition $\beta_i = \eta \kappa_i^2 / (2\lambda_i - \eta \lambda_i^2)$ and the identity $1 - (1 - \eta \lambda_i)^2 = \eta \lambda_i (2 - \eta \lambda_i)$. For a Gaussian random variable, $\mathbb{E}_t[c_{i,t}^2] = \mu_{i,t}^2 + \sigma_{i,t}^2$, so
\[
\mathbb{E}_t[c_{i,t}^2] = (1 - \eta \lambda_i)^{2t} c_{i,0}^2 + \beta_i \left(1 - (1 - \eta \lambda_i)^{2t} \right) = (1 - \eta \lambda_i)^{2t} (c_{i,0}^2 - \beta_i) + \beta_i.
\]
Thus, the expectation of $s_t^{\prime,\mathcal{D}}$ is
\[
\mathbb{E}_t[s_t^{\prime,\mathcal{D}}] = \frac{1}{k} \sum_{i \in \mathcal{D}} \lambda_i^2 \mathbb{E}_t[c_{i,t}^2] = \bar{s}_t^{\mathcal{D}},
\]
which is deterministic and depends only on $\{c_{i,0}\}$. Since the noise is independent across eigendirections, the random variables $\{c_{i,t}^2\}_{i=1}^d$ are independent. Hence,
\[
\mathrm{Var}_t(s_t^{\prime,\mathcal{D}}) = \frac{1}{k^2} \sum_{i \in \mathcal{D}} \lambda_i^4 \mathrm{Var}_t(c_{i,t}^2).
\]
For a Gaussian $X \sim \mathcal{N}(\mu, \sigma^2)$, $\mathrm{Var}(X^2) = 2\sigma^4 + 4\mu^2 \sigma^2 \leq 2(\mu^2 + \sigma^2)^2 = 2(\mathbb{E}[X^2])^2$. Therefore,
\[
\mathrm{Var}_t(c_{i,t}^2) \leq 2 \left( \mathbb{E}_t[c_{i,t}^2] \right)^2 \leq 2M^2,
\]
for some constant $M < \infty$, because $c_{i,0}^2$ is bounded (by trajectory boundedness in Assumption~\ref{asp:standing}), $\beta_i \leq \frac{\eta s_{\max}}{2\lambda_d - \eta \lambda_d^2} < \infty$ (since $\lambda_d > 0$ and $\eta < 2/\lambda_1 \leq 2/\lambda_d$), and $(1 - \eta \lambda_i)^{2t} \in [0,1]$. Thus,
\[
\mathrm{Var}_t(s_t^{\prime,\mathcal{D}}) \leq \frac{1}{k^2} \sum_{i \in \mathcal{D}} \lambda_i^4 \cdot 2M^2 = \frac{2M^2}{k} \cdot \left( \frac{1}{k} \sum_{i \in \mathcal{D}} \lambda_i^4 \right).
\]
By Assumption~\ref{asp:standing}, $\frac{1}{k} \sum_{i \in \mathcal{D}} \lambda_i^4 \to \lambda_{\mathcal{D},4} < \infty$, and $k \to \infty$ as $d \to \infty$. Hence,
\[
\mathrm{Var}_t(s_t^{\prime,\mathcal{D}}) = O\left( \frac{1}{k} \right) \xrightarrow{d \to \infty} 0.
\]
By Chebyshev's inequality, for any $\epsilon > 0$,
\[
\mathbb{P}\left( \left| s_t^{\prime,\mathcal{D}} - \mathbb{E}_t[s_t^{\prime,\mathcal{D}}] \right| > \epsilon \right) \leq \frac{\mathrm{Var}_t(s_t^{\prime,\mathcal{D}})}{\epsilon^2} \xrightarrow{d \to \infty} 0.
\]
Therefore, $s_t^{\prime,\mathcal{D}} \xrightarrow{p} \bar{s}_t^{\mathcal{D}}$. The same argument applies to $s_t^{\prime,\mathcal{B}}$ (with $|\mathcal{B}| = d - k \to \infty$), and to $u_t^{\prime,\mathcal{S}}, w_t^{\prime,\mathcal{S}}$ by replacing $\lambda_i^2$ with $\lambda_i^4$ or $\lambda_i^3$ (the boundedness of spectral moments for $p=3,4$ is assumed in Assumption~\ref{asp:standing}). This completes the proof.
\end{proof}

\begin{lemma}[Interchange of expectation and $k(\cdot)$ at time $t$]\label{lem:expectation-interchange-at-t}
Let $k: (0,\infty)^6 \to [0,1]$ be the deterministic function defined in Lemma~\ref{lem:asymp-rep}. Under the conditions of Lemma~\ref{lem:concentration-at-t}, we have
\[
\lim_{d \to \infty} \mathbb{E}_t\left[ k\left( s_t^{\prime,\mathcal{D}}, s_t^{\prime,\mathcal{B}}, u_t^{\prime,\mathcal{D}}, u_t^{\prime,\mathcal{B}}, w_t^{\prime,\mathcal{D}}, w_t^{\prime,\mathcal{B}} \right) \right] 
= k\left( \bar{s}_t^{\mathcal{D}}, \bar{s}_t^{\mathcal{B}}, \bar{u}_t^{\mathcal{D}}, \bar{u}_t^{\mathcal{B}}, \bar{w}_t^{\mathcal{D}}, \bar{w}_t^{\mathcal{B}} \right),
\]
where $\mathbb{E}_t[\cdot] = \mathbb{E}[\cdot \mid \bm{x}_0]$ denotes expectation over the noise up to time $t-1$, and the limits $\bar{s}_t^{\mathcal{S}}, \bar{u}_t^{\mathcal{S}}, \bar{w}_t^{\mathcal{S}}$ are as defined in Lemma~\ref{lem:concentration-at-t}.
\end{lemma}

\begin{proof}
Define the random vector
\[
\bm{Z}_d := \left( s_t^{\prime,\mathcal{D}}, s_t^{\prime,\mathcal{B}}, u_t^{\prime,\mathcal{D}}, u_t^{\prime,\mathcal{B}}, w_t^{\prime,\mathcal{D}}, w_t^{\prime,\mathcal{B}} \right).
\]
By Lemma~\ref{lem:concentration-at-t}, $\bm{Z}_d \xrightarrow{p} \bar{\bm{z}}_t := (\bar{s}_t^{\mathcal{D}}, \bar{s}_t^{\mathcal{B}}, \bar{u}_t^{\mathcal{D}}, \bar{u}_t^{\mathcal{B}}, \bar{w}_t^{\mathcal{D}}, \bar{w}_t^{\mathcal{B}})$ as $d \to \infty$. The function $k(\cdot)$ is continuous at $\bar{\bm{z}}_t$ because its explicit form
\[
k(s_1, s_2, u_1, u_2, w_1, w_2) = \frac{1}{1 + \frac{1}{\rho} \cdot \frac{s_2 - 2\eta w_2 + \eta^2 (u_2 + \bar{e}_{\mathcal{B}})}{s_1 - 2\eta w_1 + \eta^2 (u_1 + \bar{e}_{\mathcal{D}})}}
\]
involves only continuous operations, and the denominator is strictly positive under Assumption~\ref{assum-cssa} (since $\eta < 2/\lambda_1$ ensures $(1 - \eta \lambda_i)^2 > 0$ for all $i$, and the noise terms $\bar{e}_{\mathcal{D}}, \bar{e}_{\mathcal{B}} \geq 0$ add positivity). Moreover, $k(\cdot)$ is bounded between 0 and 1 since it represents the asymptotic limit of the alignment $\theta_{t+1} \in [0,1]$. Therefore, by Lemma~\ref{lem:cmt-dc} (the Continuous Mapping Theorem combined with the Bounded Convergence Theorem), the convergence in probability of $\bm{Z}_d$ to $\bar{\bm{z}}_t$ implies that
\[
\lim_{d \to \infty} \mathbb{E}_t\left[ k(\bm{Z}_d) \right] = k(\bar{\bm{z}}_t),
\]
which is precisely the claimed equality.
\end{proof}

\begin{lemma}[Equivalence of conditional and unconditional alignment expectation]\label{lem:cond-uncond-equivalence}
Under Assumptions~\ref{asp:standing} and \ref{assum-cssa}, with deterministic initialization $\bm{x}_0$, define the deterministic limits
\begin{align*}
\bar{s}_t^{\mathcal{D}} &:= \lim_{d \to \infty} \mathbb{E}_t[s_t^{\prime,\mathcal{D}}], &
\bar{s}_t^{\mathcal{B}} &:= \lim_{d \to \infty} \mathbb{E}_t[s_t^{\prime,\mathcal{B}}], \\
\bar{u}_t^{\mathcal{D}} &:= \lim_{d \to \infty} \mathbb{E}_t[u_t^{\prime,\mathcal{D}}], &
\bar{u}_t^{\mathcal{B}} &:= \lim_{d \to \infty} \mathbb{E}_t[u_t^{\prime,\mathcal{B}}], \\
\bar{w}_t^{\mathcal{D}} &:= \lim_{d \to \infty} \mathbb{E}_t[w_t^{\prime,\mathcal{D}}], &
\bar{w}_t^{\mathcal{B}} &:= \lim_{d \to \infty} \mathbb{E}_t[w_t^{\prime,\mathcal{B}}].
\end{align*}
Then the following equality holds:
\begin{align*}
&\lim_{d \to \infty} \mathbb{E}_{t+1}\!\left[ \theta_{t+1} \,\middle|\, 
\begin{array}{l}
s_t^{\prime,\mathcal{D}} = \bar{s}_t^{\mathcal{D}},\, 
s_t^{\prime,\mathcal{B}} = \bar{s}_t^{\mathcal{B}}, \\
u_t^{\prime,\mathcal{D}} = \bar{u}_t^{\mathcal{D}},\, 
u_t^{\prime,\mathcal{B}} = \bar{u}_t^{\mathcal{B}}, \\
w_t^{\prime,\mathcal{D}} = \bar{w}_t^{\mathcal{D}},\, 
w_t^{\prime,\mathcal{B}} = \bar{w}_t^{\mathcal{B}}
\end{array}
\right] \\
&\qquad = \lim_{d \to \infty} \mathbb{E}_{t+1}[\theta_{t+1}].
\end{align*}
\end{lemma}

\begin{proof}
By Lemma~\ref{lem:asymp-rep}, for any fixed values of the six statistics at time $t$, the conditional expectation satisfies
\[
\lim_{d \to \infty} \mathbb{E}_{t+1}[\theta_{t+1} \mid s_t^{\prime,\mathcal{D}}, s_t^{\prime,\mathcal{B}}, u_t^{\prime,\mathcal{D}}, u_t^{\prime,\mathcal{B}}, w_t^{\prime,\mathcal{D}}, w_t^{\prime,\mathcal{B}}] 
= k(s_t^{\prime,\mathcal{D}}, s_t^{\prime,\mathcal{B}}, u_t^{\prime,\mathcal{D}}, u_t^{\prime,\mathcal{B}}, w_t^{\prime,\mathcal{D}}, w_t^{\prime,\mathcal{B}}).
\]
Substituting the deterministic limits $\bar{s}_t^{\mathcal{D}}, \dots, \bar{w}_t^{\mathcal{B}}$ into this identity yields
\[
\lim_{d \to \infty} \mathbb{E}_{t+1}[\theta_{t+1} \mid s_t^{\prime,\mathcal{D}} = \bar{s}_t^{\mathcal{D}}, \dots, w_t^{\prime,\mathcal{B}} = \bar{w}_t^{\mathcal{B}}] 
= k(\bar{s}_t^{\mathcal{D}}, \bar{s}_t^{\mathcal{B}}, \bar{u}_t^{\mathcal{D}}, \bar{u}_t^{\mathcal{B}}, \bar{w}_t^{\mathcal{D}}, \bar{w}_t^{\mathcal{B}}).
\]

On the other hand, by the law of total expectation,
\[
\mathbb{E}_{t+1}[\theta_{t+1}] = \mathbb{E}_t\!\left[ \mathbb{E}_{t+1}[\theta_{t+1} \mid s_t^{\prime,\mathcal{D}}, \dots, w_t^{\prime,\mathcal{B}}] \right].
\]
Taking the limit as $d \to \infty$ on both sides gives
\[
\lim_{d \to \infty} \mathbb{E}_{t+1}[\theta_{t+1}] 
= \lim_{d \to \infty} \mathbb{E}_t\!\left[ \mathbb{E}_{t+1}[\theta_{t+1} \mid s_t^{\prime,\mathcal{D}}, \dots, w_t^{\prime,\mathcal{B}}] \right].
\]
Applying Lemma~\ref{lem:asymp-rep} inside the expectation, the right-hand side becomes
\[
\lim_{d \to \infty} \mathbb{E}_t\!\left[ k(s_t^{\prime,\mathcal{D}}, s_t^{\prime,\mathcal{B}}, u_t^{\prime,\mathcal{D}}, u_t^{\prime,\mathcal{B}}, w_t^{\prime,\mathcal{D}}, w_t^{\prime,\mathcal{B}}) \right].
\]
Finally, by Lemma~\ref{lem:expectation-interchange-at-t}, this limit equals
\[
k(\bar{s}_t^{\mathcal{D}}, \bar{s}_t^{\mathcal{B}}, \bar{u}_t^{\mathcal{D}}, \bar{u}_t^{\mathcal{B}}, \bar{w}_t^{\mathcal{D}}, \bar{w}_t^{\mathcal{B}}).
\]

Thus, both sides of the claimed equality converge to the same deterministic value, and the result follows.
\end{proof}

\subsection{Proofs of theorems in CSGD}
\label{sectionPCSD}
Recall our assumption and notion:
\begin{assumption-restated}\ref{assum-cssa} Our analysis is based on the following assumptions:
\begin{table}[H]
\centering
\caption{Assumptions for CSGD}
\begin{tabular}{ll}
\hline
\textbf{Assumption} & \textbf{Description} \\ \hline \\
Constant Step Size & 
$\eta_t = \eta, \quad 0 < \eta < \min\left\{ \frac{2}{\lambda_1}, \frac{2(\lambda_k - \lambda_{k+1})}{\lambda_1^2 - \lambda_d^2} \right\}$ \\[4ex]
Initialization for $\bm{x}_0$& 
$\begin{array}{l}
\forall i \in \mathcal{D} ,\quad c_{i,0}^2 > \beta_i, \\
\varrho_{\mathcal{D}} > \delta - \sum_{i \in \mathcal{D}} \beta_i
\end{array}$ \\ \\ \hline
\end{tabular}
\end{table}
\end{assumption-restated}
For the analysis, we define several key quantities:
\[
\beta_i := \frac{\eta\,\kappa_i^2}{2\lambda_i-\eta\lambda_i^2} > 0, \varrho_{\mathcal{D}} := \sum_{i \in \mathcal{D}} (c_{i,0}^2 - \beta_i).
\]
\[
\delta := \frac{s_{\max}\psi_{\mathcal D}\lambda_1^2}{\lambda_d^2 \lambda_k^2 \left(\frac{2(\lambda_k-\lambda_{k+1})}{\eta} - (\lambda_1^2 - \lambda_d^2)\right)}.
\]

\begin{theorem-restated}\ref{thm:monotone-decrease-expected-alignment}(\textbf{\textit{Initial Decrease Phase}})
Under Assumption \ref{asp:standing} and Assumption \ref{assum-cssa}, let the  $t^*$ be defined as
\[
t^* := \left\lfloor \frac{\log\left( \dfrac{\varrho_{\mathcal{D}}}{\delta - \sum_{i \in \mathcal{D}} \beta_i} \right)}{-2\log(1-\eta\lambda_1)} \right\rfloor.
\]
Then for all time steps $t \in \{0, 1, \dots, t^*-1\}$, the expected alignment is strictly decreasing:
\[
\lim_{d\to\infty}\,\mathbb{E}[\theta_{t+1}] < \lim_{d\to\infty}\mathbb{E}[\theta_t].
\]

\end{theorem-restated}

\begin{proof}
By Theorem~\ref{thm:step-dec-below}, the desired inequality follows if we can show that
\[
\lim_{d \to \infty} \mathbb{E}_t[\eta_t^* \mid s_t^{\prime,\mathcal{D}} = \bar{s}_t^{\mathcal{D}}, \dots, w_t^{\prime,\mathcal{B}} = \bar{w}_t^{\mathcal{B}}] > \eta,
\]
where $\bar{s}_t^{\mathcal{S}} = \lim_{d \to \infty} \mathbb{E}_t[s_t^{\prime,\mathcal{S}}]$, etc., are the deterministic limits from Lemma~\ref{lem:concentration-at-t}.

To lower-bound this expectation, we use the state-dependent bound from Theorem~\ref{thm:bounds-xnorm}:
\[
\eta_t^* \geq \frac{2\,\mathrm{gap}_1}{(\lambda_1^2 - \lambda_d^2) + \dfrac{s_{\max} \psi_{\mathcal{D}}}{\lambda_d^2 \|\bm{x}_t\|^2 \theta_t}}.
\]
Taking conditional expectation given the six macroscopic statistics at time $t$, we obtain
\[
\mathbb{E}_t[\eta_t^* \mid \bar{\bm{z}}_t] \geq \mathbb{E}_t\left[ \frac{2\,\mathrm{gap}_1}{(\lambda_1^2 - \lambda_d^2) + \dfrac{s_{\max} \psi_{\mathcal{D}}}{\lambda_d^2 \|\bm{x}_t\|^2 \theta_t}} \,\middle|\, \bar{\bm{z}}_t \right].
\]

Now, observe that $\|\bm{x}_t\|^2 \theta_t$ is a continuous function of the six statistics. Specifically,
\[
\|\bm{x}_t\|^2 \theta_t = \left( \sum_{i=1}^d c_{i,t}^2 \right) \cdot \frac{\sum_{j \in \mathcal{D}} \lambda_j^2 c_{j,t}^2}{\sum_{i=1}^d \lambda_i^2 c_{i,t}^2}.
\]
By Lemma~\ref{lem:concentration-at-t}, each block average $\frac{1}{|\mathcal{S}|} \sum_{i \in \mathcal{S}} c_{i,t}^2$, $\frac{1}{|\mathcal{S}|} \sum_{i \in \mathcal{S}} \lambda_i^2 c_{i,t}^2$, etc., concentrates around its expectation as $d \to \infty$. Therefore, $\|\bm{x}_t\|^2 \theta_t$ concentrates around its deterministic limit
\[
\overline{\|\bm{x}_t\|^2 \theta_t} := \lim_{d \to \infty} \mathbb{E}_t[\|\bm{x}_t\|^2 \theta_t].
\]
Since the function $x \mapsto 1 / (a + b/x)$ is continuous and bounded for $x > 0$, Lemma~\ref{lem:cmt-dc} implies
\[
\lim_{d \to \infty} \mathbb{E}_t\left[ \frac{2\,\mathrm{gap}_1}{(\lambda_1^2 - \lambda_d^2) + \dfrac{s_{\max} \psi_{\mathcal{D}}}{\lambda_d^2 \|\bm{x}_t\|^2 \theta_t}} \,\middle|\, \bar{\bm{z}}_t \right]
= \frac{2\,\mathrm{gap}_1}{(\lambda_1^2 - \lambda_d^2) + \dfrac{s_{\max} \psi_{\mathcal{D}}}{\lambda_d^2 \overline{\|\bm{x}_t\|^2 \theta_t}}}.
\]

We now lower-bound $\overline{\|\bm{x}_t\|^2 \theta_t}$. Using $\lambda_i \geq \lambda_k$ for $i \in \mathcal{D}$ and $\lambda_i \leq \lambda_1$ for all $i$, we have
\[
\overline{\|\bm{x}_t\|^2 \theta_t} \geq \frac{\lambda_k^2}{\lambda_1^2} \lim_{d \to \infty} \mathbb{E}_t\left[ \sum_{i \in \mathcal{D}} c_{i,t}^2 \right]
= \frac{\lambda_k^2}{\lambda_1^2} \left( (1 - \eta \lambda_1)^{2t} \varrho_{\mathcal{D}} + \sum_{i \in \mathcal{D}} \beta_i \right).
\]
By the definition of $t^*$, for all $t < t^*$, 
\[
(1 - \eta \lambda_1)^{2t} \varrho_{\mathcal{D}} + \sum_{i \in \mathcal{D}} \beta_i > \delta.
\]
Therefore,
\[
\overline{\|\bm{x}_t\|^2 \theta_t} > \frac{\lambda_k^2}{\lambda_1^2} \delta.
\]
Substituting the definition of $\delta$ yields
\[
\frac{s_{\max} \psi_{\mathcal{D}}}{\lambda_d^2 \overline{\|\bm{x}_t\|^2 \theta_t}} < \frac{s_{\max} \psi_{\mathcal{D}} \lambda_1^2}{\lambda_d^2 \lambda_k^2 \delta}
= \frac{2(\lambda_k - \lambda_{k+1})}{\eta} - (\lambda_1^2 - \lambda_d^2).
\] 
Hence,
\[
\frac{2\,\mathrm{gap}_1}{(\lambda_1^2 - \lambda_d^2) + \dfrac{s_{\max} \psi_{\mathcal{D}}}{\lambda_d^2 \overline{\|\bm{x}_t\|^2 \theta_t}}}
> \frac{2(\lambda_k - \lambda_{k+1})}{(\lambda_1^2 - \lambda_d^2) + \left( \dfrac{2(\lambda_k - \lambda_{k+1})}{\eta} - (\lambda_1^2 - \lambda_d^2) \right)} = \eta.
\]

Combining the above, we conclude that
\[
\lim_{d \to \infty} \mathbb{E}_t[\eta_t^* \mid \bar{\bm{z}}_t] > \eta,
\]
which implies
\[
\lim_{d \to \infty} \mathbb{E}_{t+1}[\theta_{t+1} \mid \bar{\bm{z}}_t] < \lim_{d \to \infty} \theta_t.
\]
Finally, by Lemma~\ref{lem:cond-uncond-equivalence} and the concentration of $\theta_t$, we obtain
\[
\lim_{d \to \infty} \mathbb{E}_{t+1}[\theta_{t+1}] < \lim_{d \to \infty} \mathbb{E}_{t}[\theta_t]
\]
for all $t \in \{0, 1, \dots, t^* - 1\}$, completing the proof.
\end{proof}

\begin{theorem-restated}\ref{thm:late-theta}(\textbf{\textit{Late Phase}})
Under Assumption~\ref{asp:standing} and Assumption~\ref{assum-cssa}, the late-time asymptotic expected alignment is given by
\[
\theta_\infty := \lim_{t \to \infty} \lim_{d \to \infty} \mathbb{E}_t[\theta_t]
= \frac{\lim_{d \to \infty} \sum_{i \in \mathcal{D}} \lambda_i^2 \beta_i}{\lim_{d \to \infty} \sum_{i=1}^d \lambda_i^2 \beta_i},
\]
where $\beta_i = \dfrac{\eta \kappa_i^2}{2\lambda_i - \eta \lambda_i^2} > 0$. Equivalently,
\[
\theta_\infty = \frac{\lim_{d \to \infty} \sum_{i \in \mathcal{D}} \dfrac{\eta \lambda_i^2 \kappa_i^2}{2\lambda_i - \eta \lambda_i^2}}{\lim_{d \to \infty} \sum_{i=1}^d \dfrac{\eta \lambda_i^2 \kappa_i^2}{2\lambda_i - \eta \lambda_i^2}}.
\]

\end{theorem-restated}

\begin{proof}
From the CSGD dynamics, the second moment of each eigencoordinate satisfies
\[
\mathbb{E}_t[c_{i,t}^2] = (1 - \eta \lambda_i)^{2t} (c_{i,0}^2 - \beta_i) + \beta_i,
\quad \text{where } \beta_i = \frac{\eta \kappa_i^2}{2\lambda_i - \eta \lambda_i^2}.
\]
Since $0 < \eta < 2 / \lambda_1$, we have $|1 - \eta \lambda_i| < 1$ for all $i$, so $(1 - \eta \lambda_i)^{2t} \to 0$ as $t \to \infty$. Therefore,
\[
\lim_{t \to \infty} \mathbb{E}_t[c_{i,t}^2] = \beta_i.
\]

The expected unnormalized energies are
\[
\mathbb{E}_t[s_t^{\mathcal{D}}] = \sum_{i \in \mathcal{D}} \lambda_i^2 \mathbb{E}_t[c_{i,t}^2], \quad
\mathbb{E}_t[s_t] = \sum_{i=1}^d \lambda_i^2 \mathbb{E}_t[c_{i,t}^2].
\]
Taking the limit as $t \to \infty$, we obtain
\[
\lim_{t \to \infty} \mathbb{E}_t[s_t^{\mathcal{D}}] = \sum_{i \in \mathcal{D}} \lambda_i^2 \beta_i, \quad
\lim_{t \to \infty} \mathbb{E}_t[s_t] = \sum_{i=1}^d \lambda_i^2 \beta_i.
\]

By Lemma~\ref{lem:concentration-at-t}, for each fixed $t$, the normalized statistics $s_t^{\prime,\mathcal{D}} = s_t^{\mathcal{D}}/k$ and $s_t^{\prime} = s_t/d$ concentrate around their expectations as $d \to \infty$:
\[
\lim_{d \to \infty} \mathbb{E}_t[\theta_t] = \frac{\lim_{d \to \infty} \mathbb{E}_t[s_t^{\mathcal{D}}]}{\lim_{d \to \infty} \mathbb{E}_t[s_t]}.
\]
Moreover, the convergence $\mathbb{E}_t[c_{i,t}^2] \to \beta_i$ is uniform in $i$ under Assumption~\ref{asp:standing} (bounded spectrum and noise). Therefore, we can interchange the limits $t \to \infty$ and $d \to \infty$, yielding
\[
\theta_\infty = \lim_{t \to \infty} \lim_{d \to \infty} \mathbb{E}_t[\theta_t]
= \lim_{t \to \infty} \lim_{d \to \infty} \frac{\mathbb{E}_t[s_t^{\mathcal{D}}]}{\mathbb{E}_t[s_t]}
= \frac{\lim_{d \to \infty} \sum_{i \in \mathcal{D}} \lambda_i^2 \beta_i}{\lim_{d \to \infty} \sum_{i=1}^d \lambda_i^2 \beta_i}.
\]
This completes the proof.
\end{proof}

\section{Numerical Simulation Results}
\label{appendix_simulation}
{
In this subsection, we detail the parameter configurations for our numerical simulations. We consider constant-step-size SGD on a deep linear network with the following global settings:
\begin{enumerate}
    \item  Input dimension: $d=500$, target rank $k=50$.
    \item Learning rate: $\eta=0.003$, total steps $T=30,000$.
    \item  Initialization: The matrices $A$ are randomly initialized as positive definite symmetric matrices $A$ with different spectral gap $m=\lambda_k/\lambda_{k+1}$. We conduct experiments across a range of values: $m \in \{5, 10, 20, 50, 100, 200, 300, 400, 500\}$.
    \item  Random seeds: We choose the various random seeds $\{42, 87, 568, \dots, 4008001\}$.
\end{enumerate}
The code is available via \href{https://github.com/xuan-lgbq/Suspicious-Alignment-of-SGD.git}{link}. 
\subsection{General Trends across Spectral Gaps}
The alignment results for different values of $m$ are presented in Figure \ref{complete_sim_fig}. These results demonstrate how the spectral gap influences the speed and stability of the alignment process. For each experiment group, the left panel plots loss as a function of training steps, and the right panel plots alignment as a function of training steps.
}
\begin{figure}[h!] 
    \centering
    \subfigure[m=5]{\includegraphics[width=0.45\textwidth]{img/quad_sim_m_5.pdf}}
    \subfigure[m=10]{\includegraphics[width=0.45\textwidth]{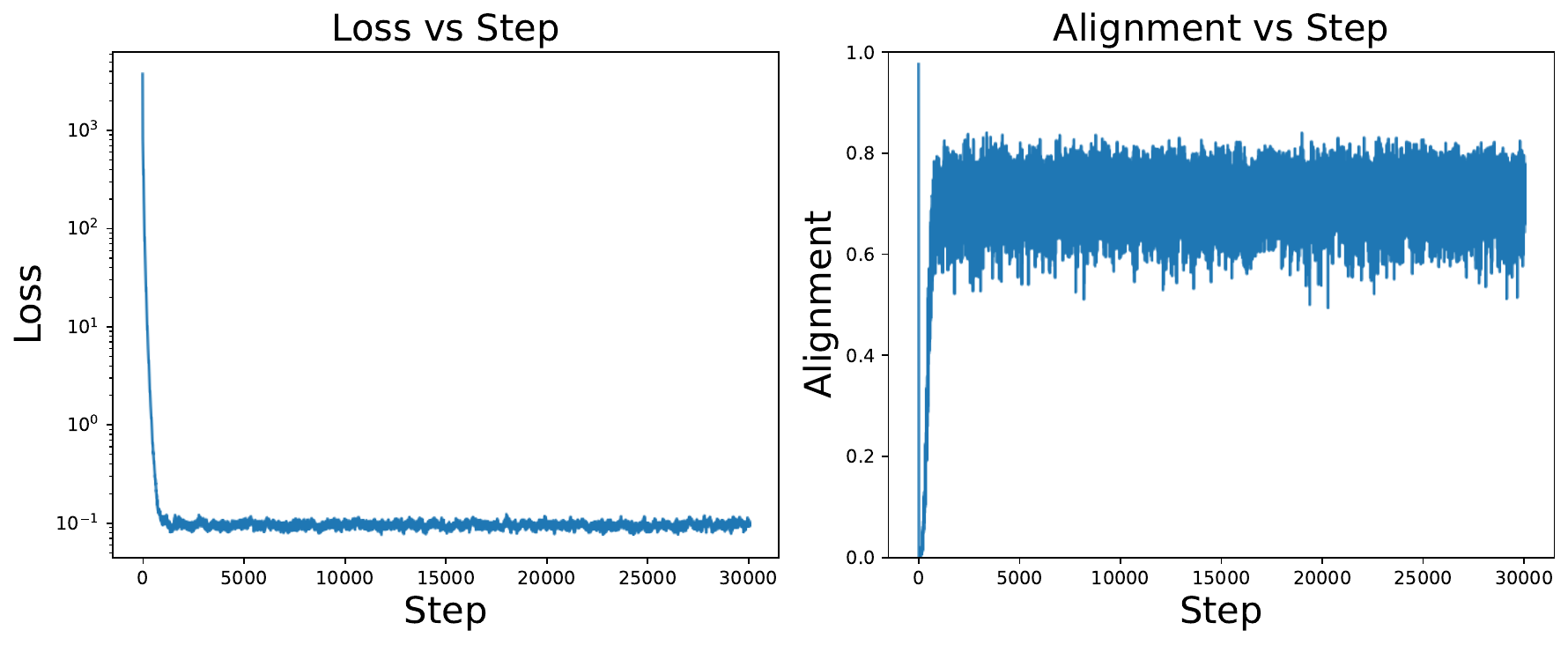}}
     \\
    \subfigure[m=20]{\includegraphics[width=0.45\textwidth]{img/quad_sim_m_20.pdf}}
    \subfigure[m=50]{\includegraphics[width=0.45\textwidth]{img/quad_sim_m_50.pdf}}
    \\
    \subfigure[m=100]{\includegraphics[width=0.45\textwidth]{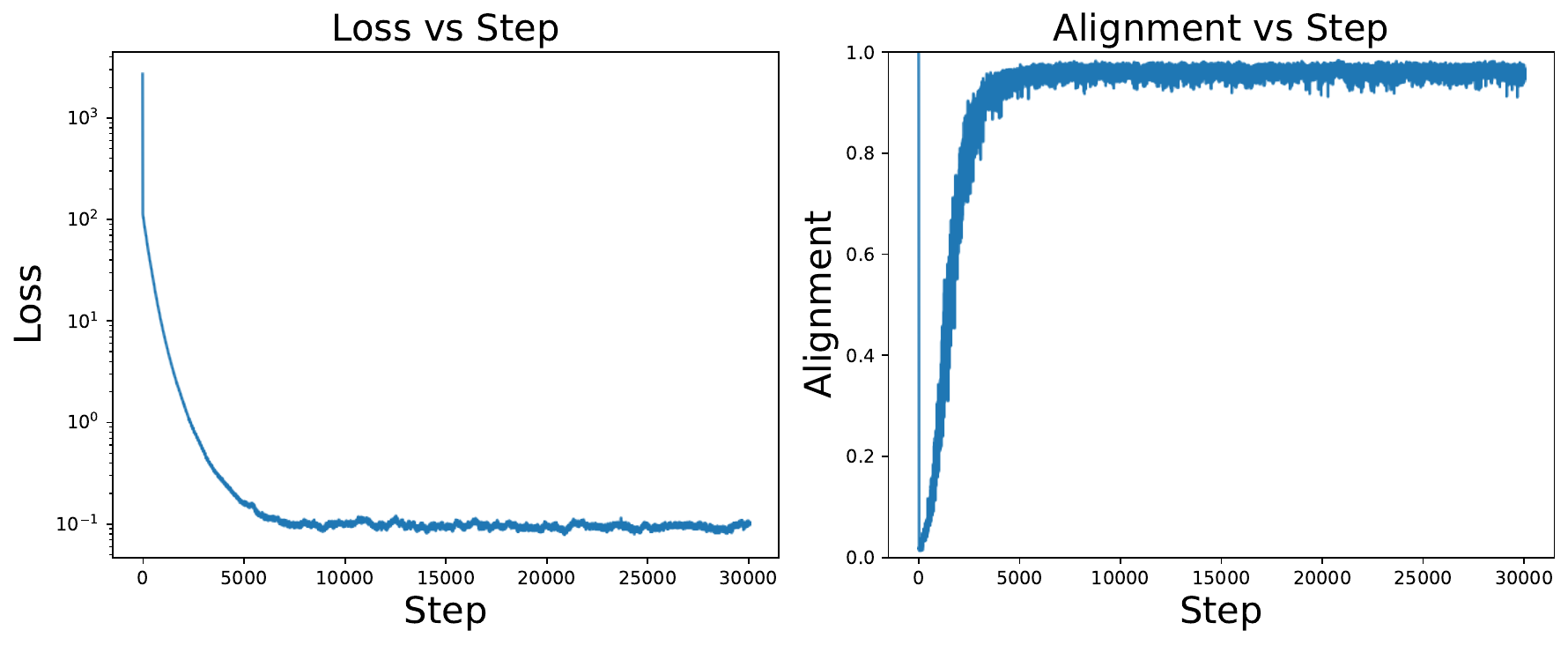}}
    \subfigure[m=200]{\includegraphics[width=0.45\textwidth]{img/quad_sim_m_200.pdf}}
    \\
    \subfigure[m=300]{\includegraphics[width=0.45\textwidth]{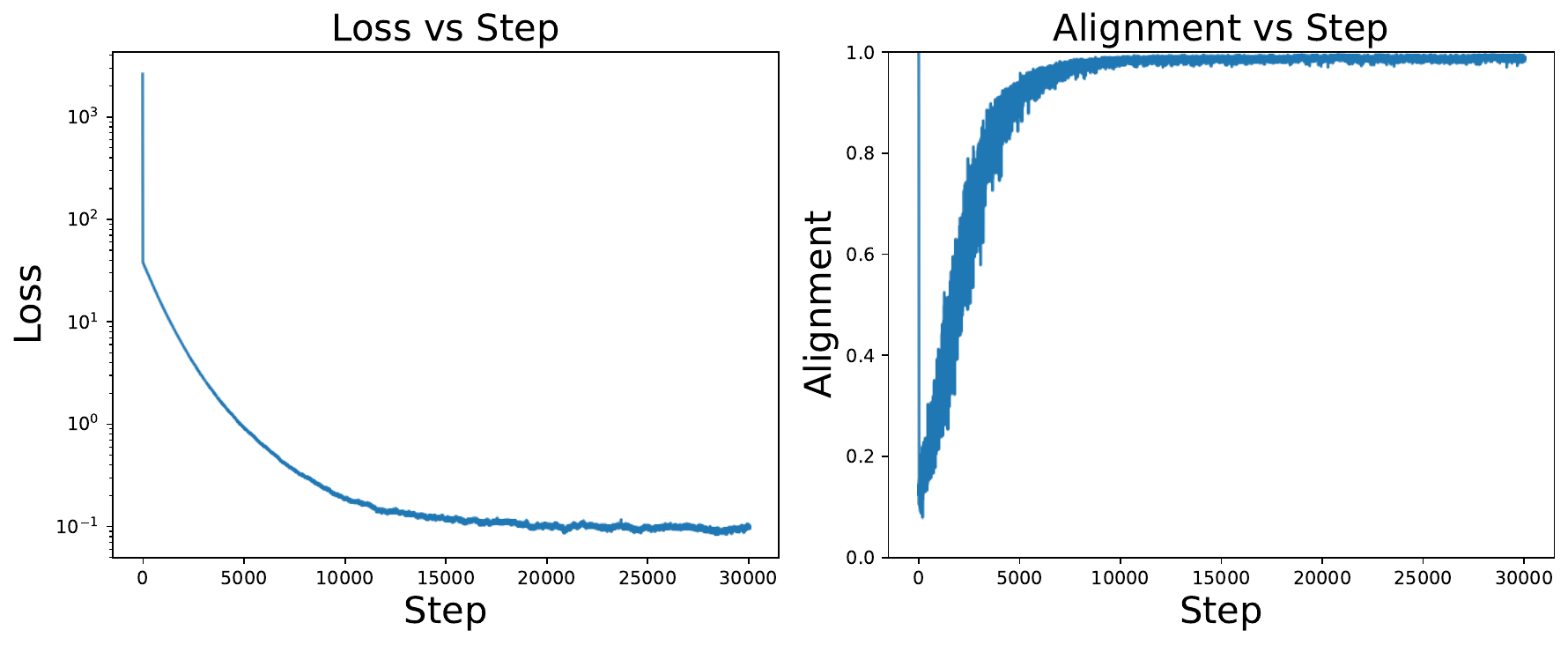}}
    \subfigure[m=400]{\includegraphics[width=0.45\textwidth]{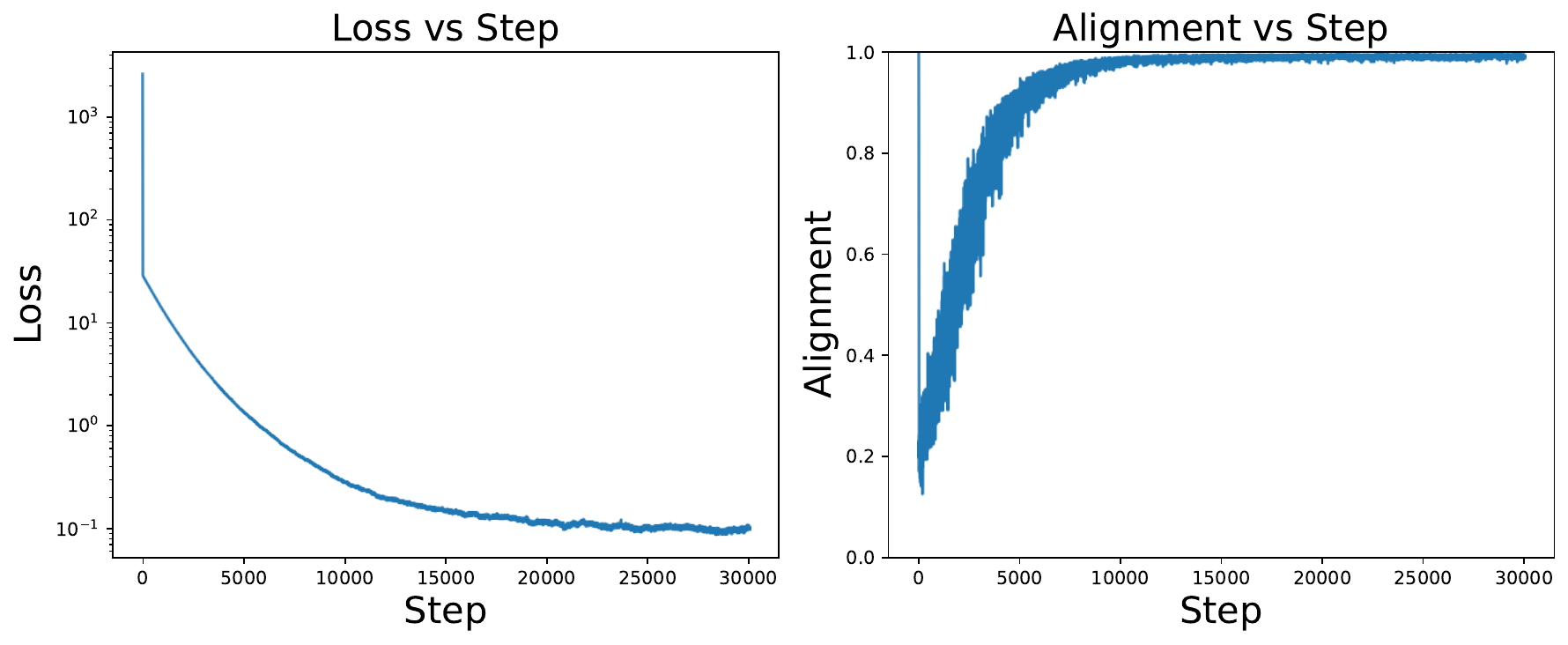}}
\\
  \subfigure[m=500]{\includegraphics[width=0.45\textwidth]{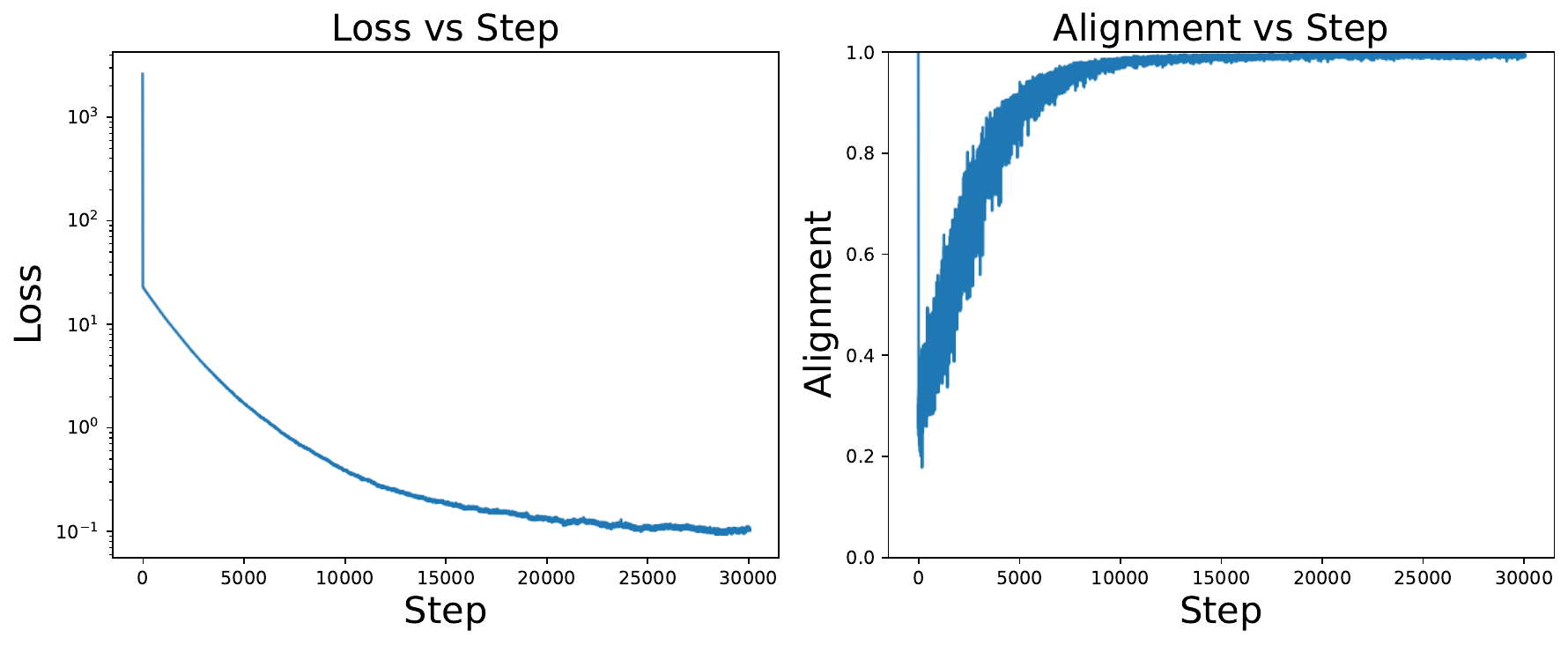}}
\caption{Numerical simulation experiments with different spectral gaps (\(m=\lambda_k/\lambda_{k+1}\))}\label{complete_sim_fig}
\end{figure}

{
\subsection{Expectation and standard deviation of the alignment for different seeds} To account for randomness in initialization and SGD trajectories, we repeated the experiments across various random seeds $\{42, 87, 568, \dots, 4008001\}$. Figure \ref{expectation and std} illustrates the expectation and standard deviation of the alignment, confirming the consistency of our theoretical predictions.}
\begin{figure}[h!] 
    \centering
    \subfigure[seed=42]{
        \includegraphics[width=0.45\textwidth, height=3.65cm, keepaspectratio]{ALT_final/Alignment_expectation_and_std/seed_42.pdf}
    }
    \subfigure[seed=87]{
        \includegraphics[width=0.45\textwidth, height=3.65cm, keepaspectratio]{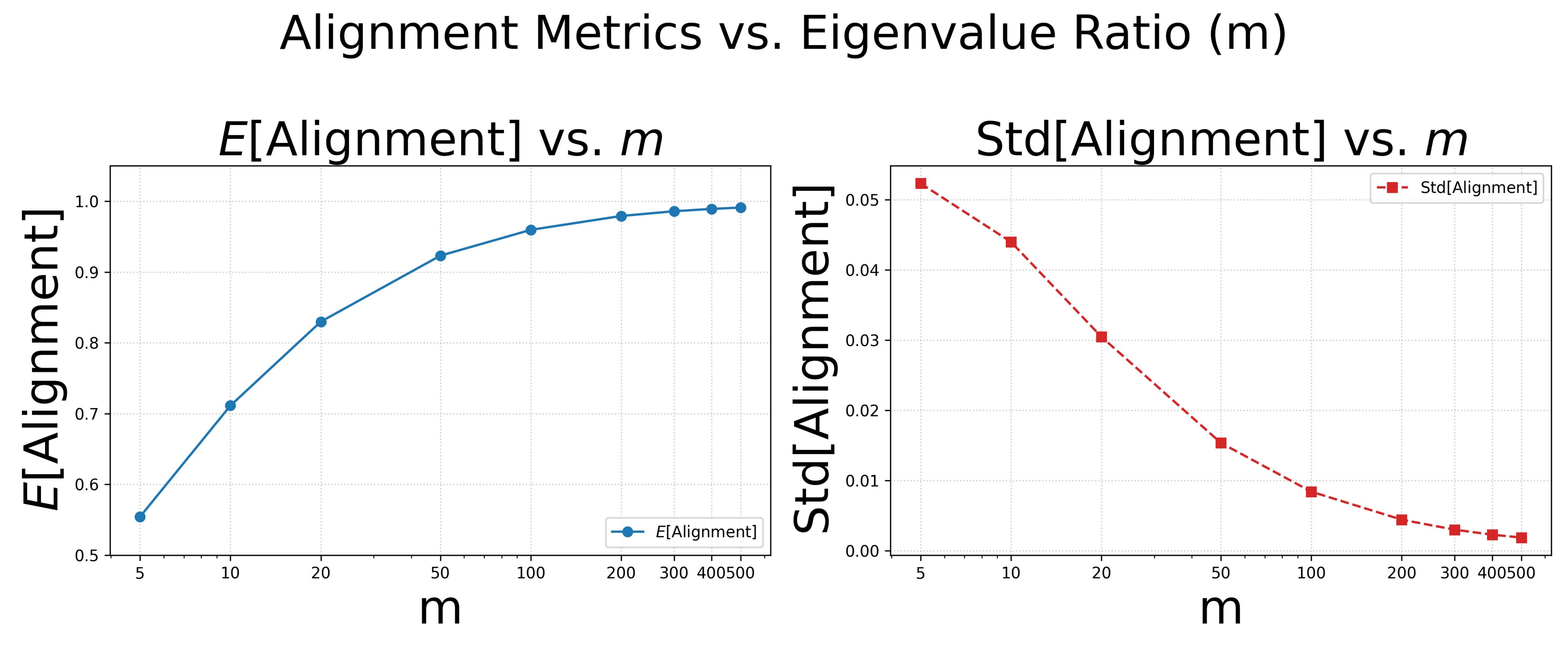}}
    \\ 
    \subfigure[seed=568]{
        \includegraphics[width=0.45\textwidth, height=3.65cm, keepaspectratio]{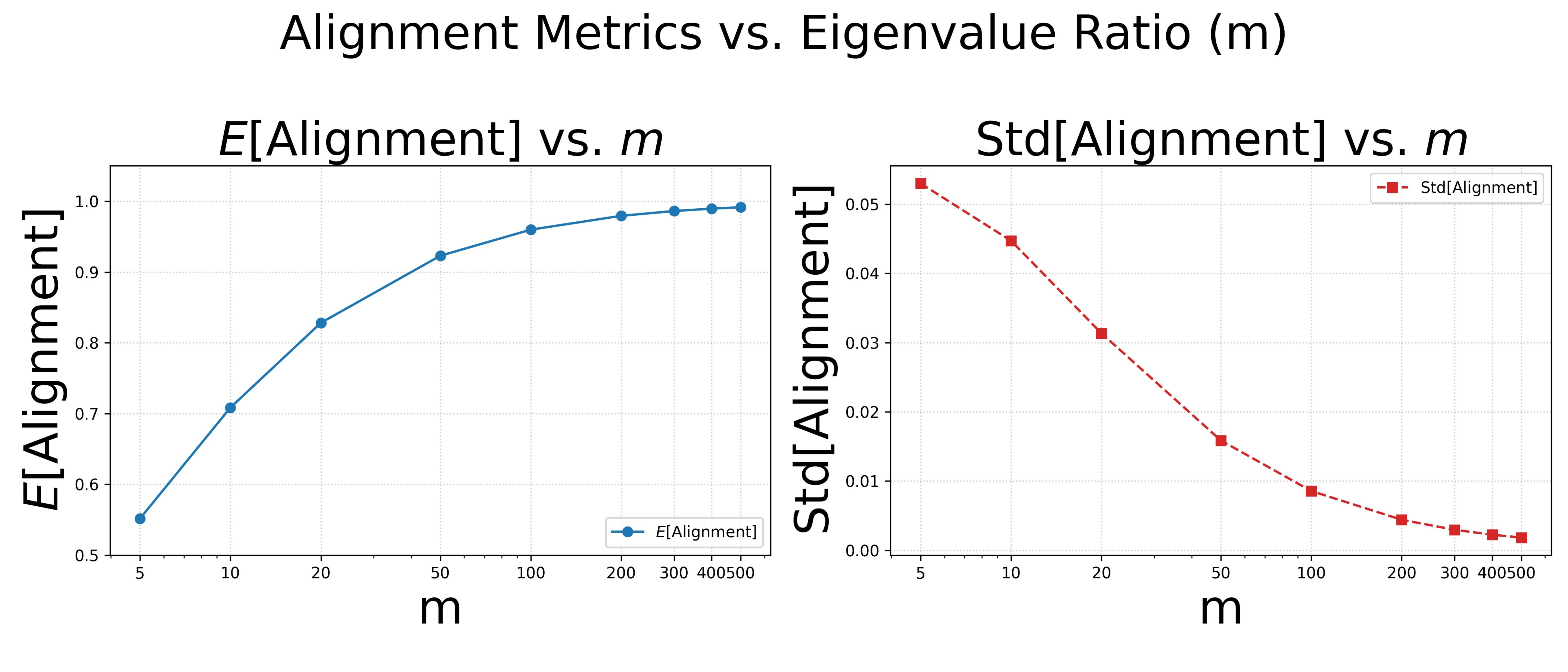}}
    \subfigure[seed=1101]{
        \includegraphics[width=0.45\textwidth, height=3.65cm, keepaspectratio]{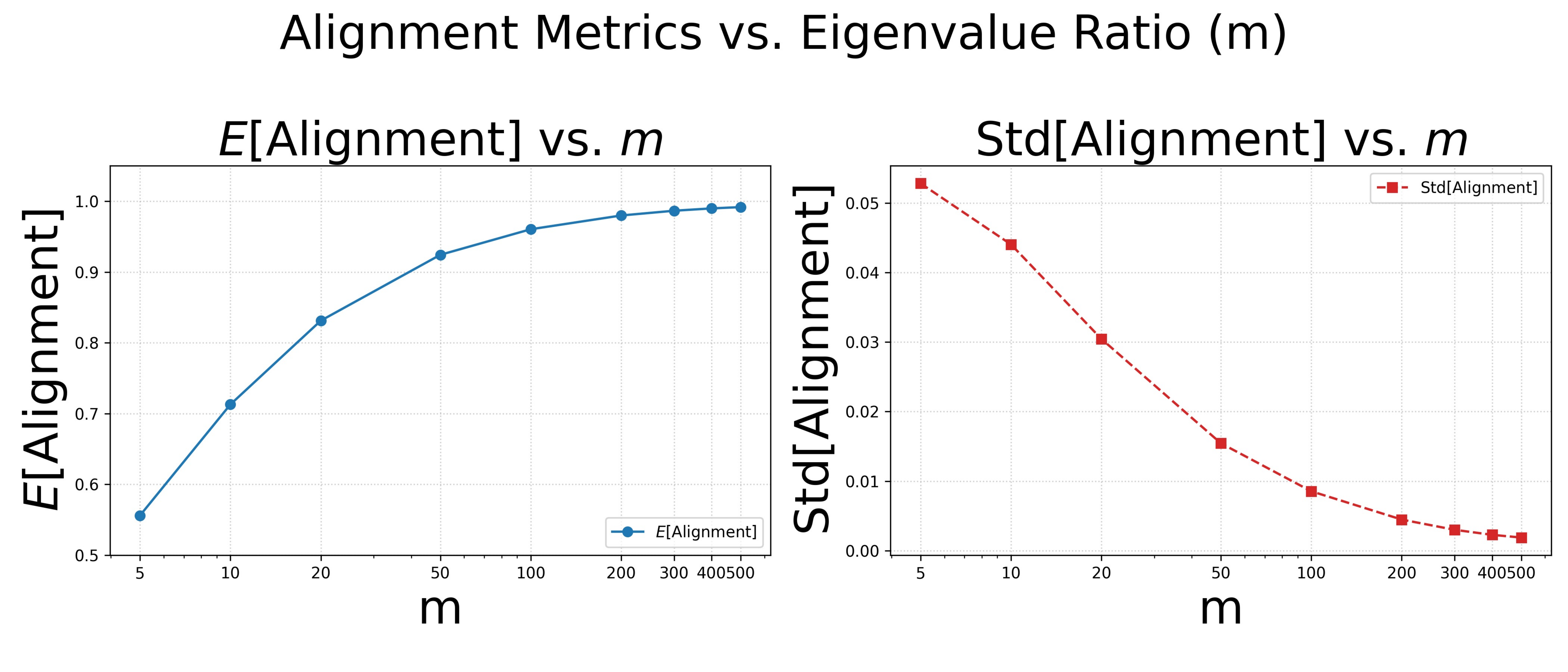}}
    
    \subfigure[seed=12138]{
        \includegraphics[width=0.45\textwidth, height=3.65cm, keepaspectratio]{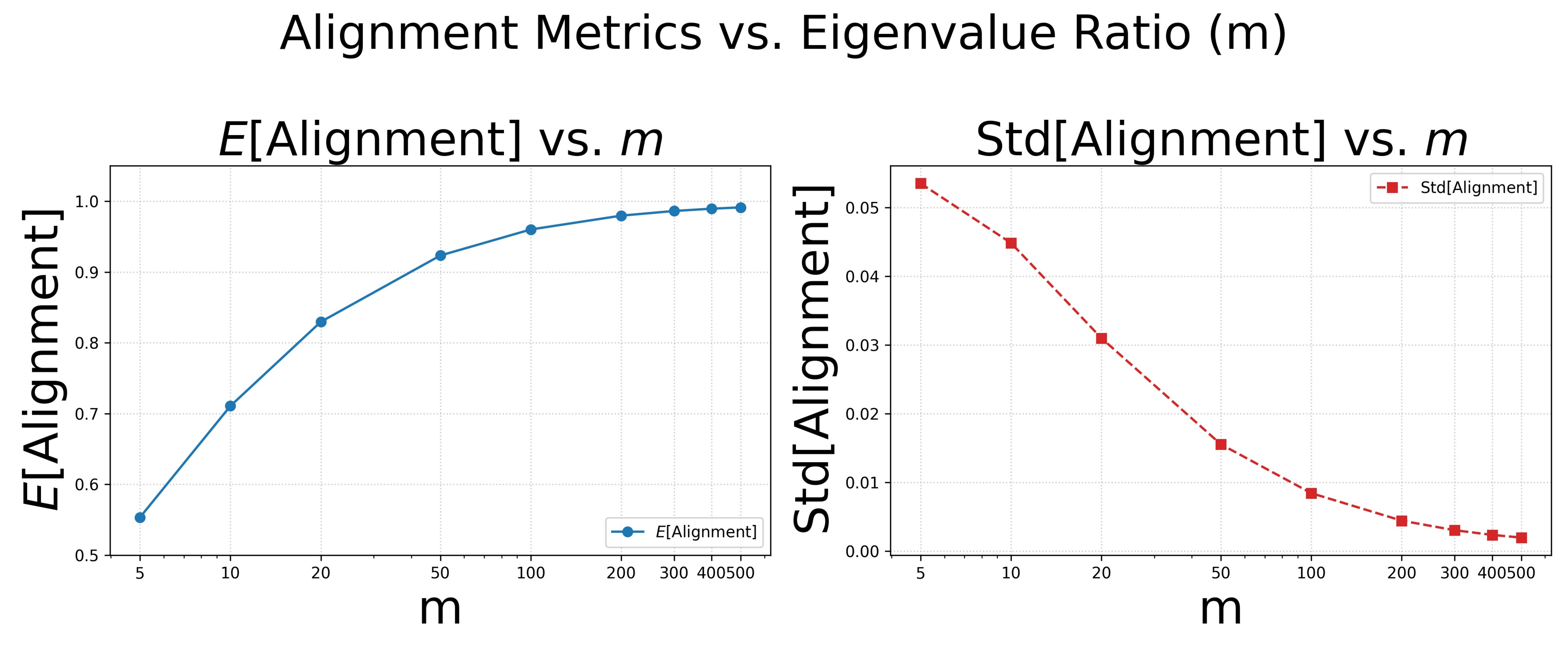}}
    \subfigure[seed=70425]{
        \includegraphics[width=0.45\textwidth, height=3.65cm, keepaspectratio]{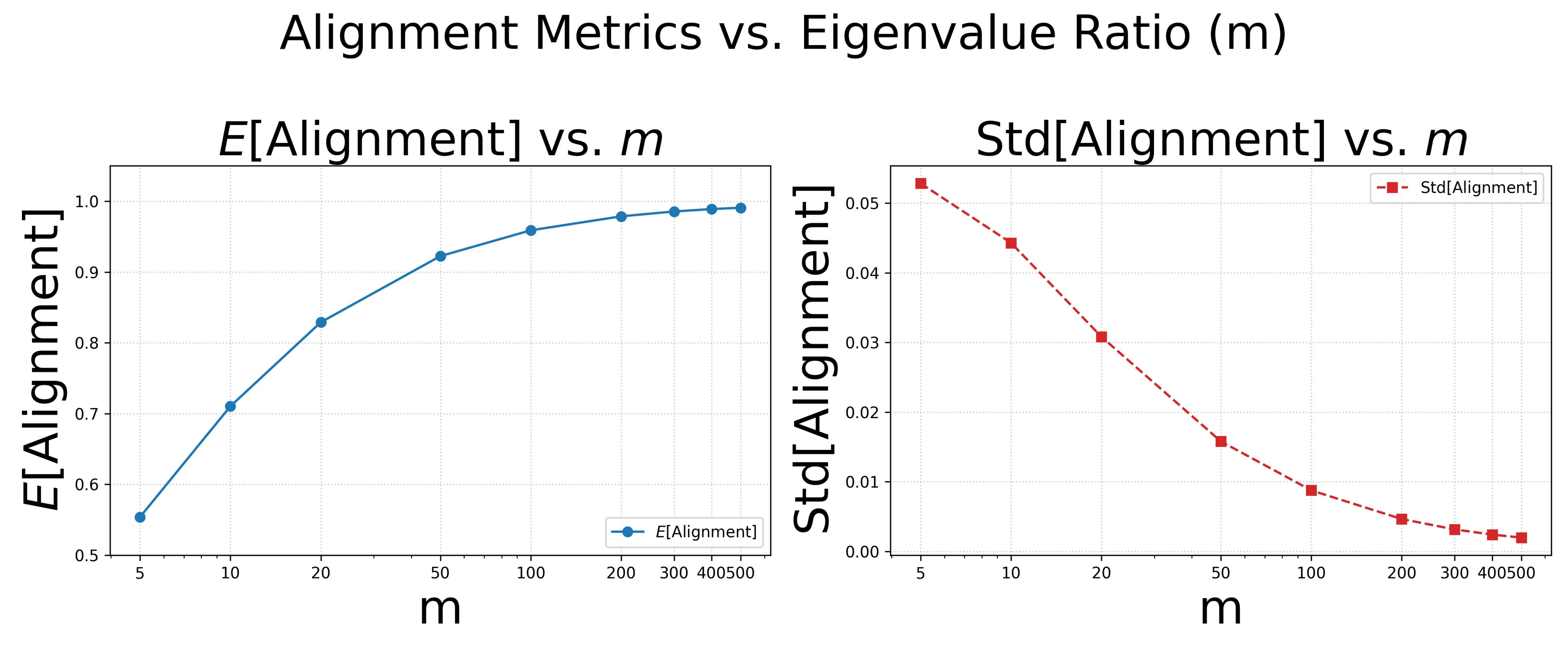}}
    \\ 
    \subfigure[seed=4008001]{
        \includegraphics[width=0.45\textwidth, height=3.65cm, keepaspectratio]{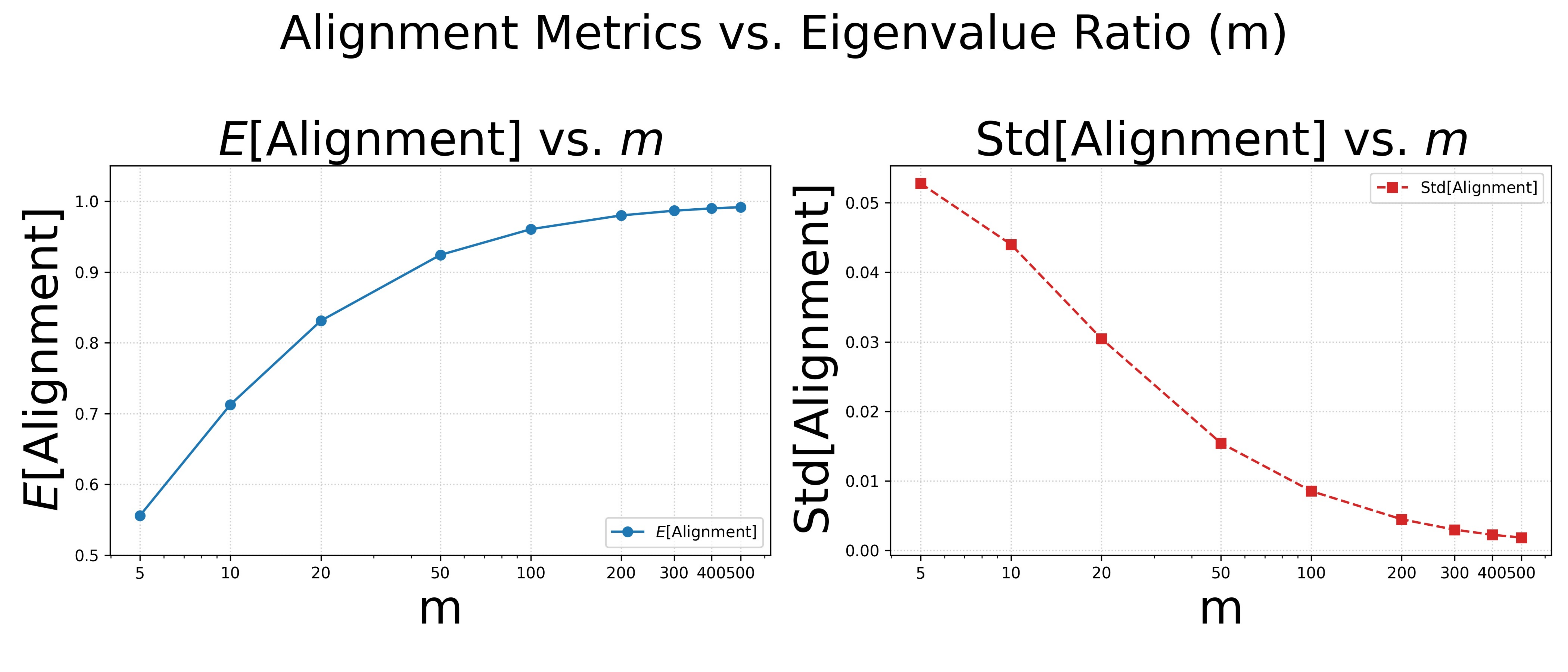}}
    
     \caption{The limitation of the expectation and standard deviation of alignment with respect to $m=\frac{\lambda_k}{\lambda_{k+1}}$. \label{expectation and std}} 
\end{figure}

{
\subsection{Analysis of Phase I Alignment Decay Rate}
As illustrated in Figures \ref{decay, seed=42} through \ref{fig:decay_4008001}, we analyze the decay rate of alignment for various seeds. The empirical results consistently indicate that the alignment exhibits a polynomial decay rate during Phase I.}
{
\subsection{Convergence Rates and Spectral Gap Scaling}
Finally, we investigate how the spectral gap $m$ influences the final alignment state. We plot the relationship between the alignment and the parameter $m$ in Figure \ref{fig:rate_convergence}, which shows that the alignment scales logarithmically with respect to $m$. This logarithmic dependency suggests that while a larger spectral gap facilitates alignment, the marginal benefit diminishes as $m$ increases.
}

\begin{figure}[p]
    \centering
    \subfigure[m=5]{\includegraphics[width=0.9\linewidth, height=3.65cm, keepaspectratio]{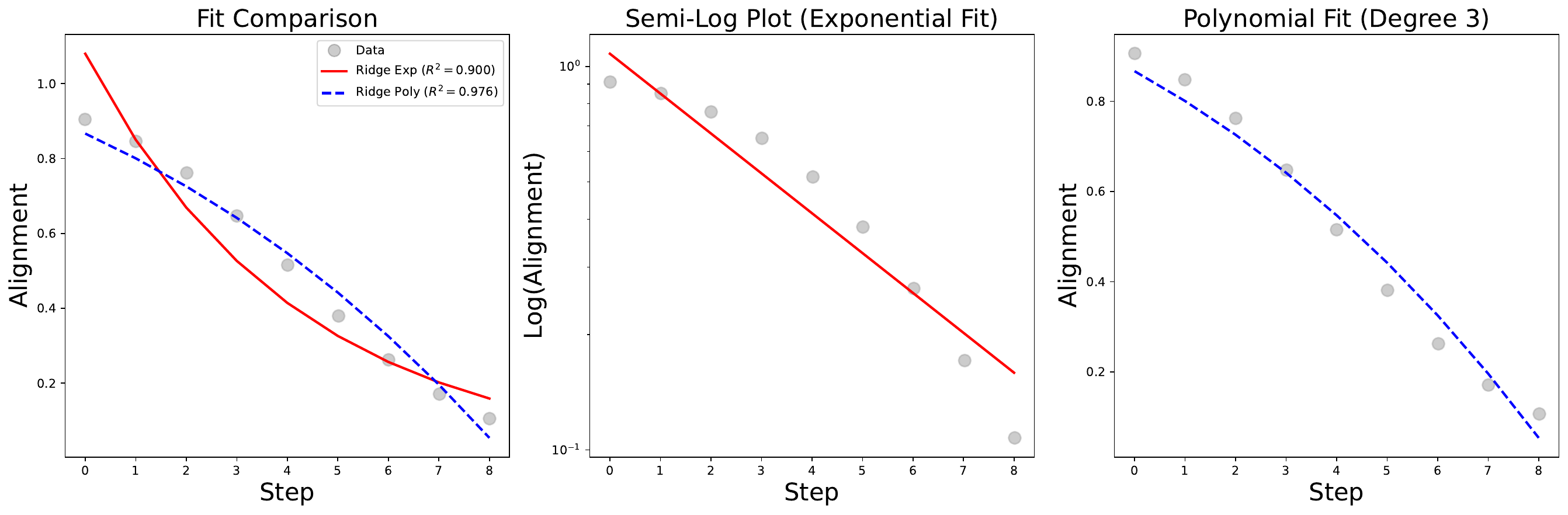}} \\
    \subfigure[m=10]{\includegraphics[width=0.9\linewidth, height=3.65cm, keepaspectratio]{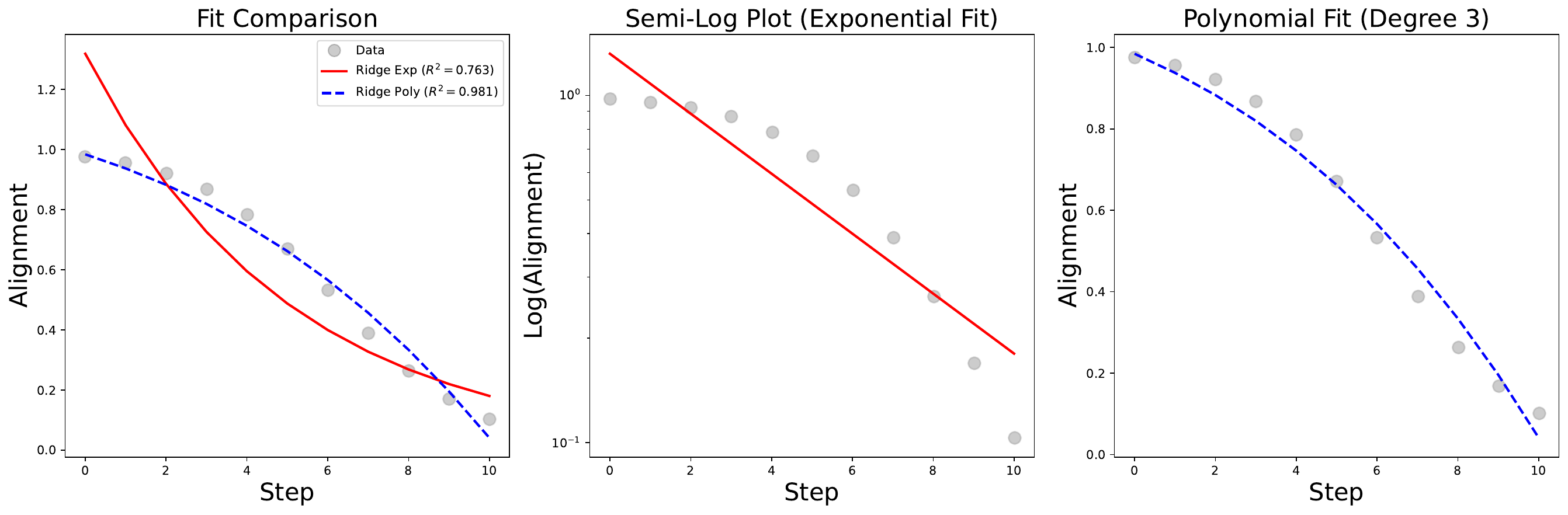}} \\
    \subfigure[m=20]{\includegraphics[width=0.9\linewidth, height=3.65cm, keepaspectratio]{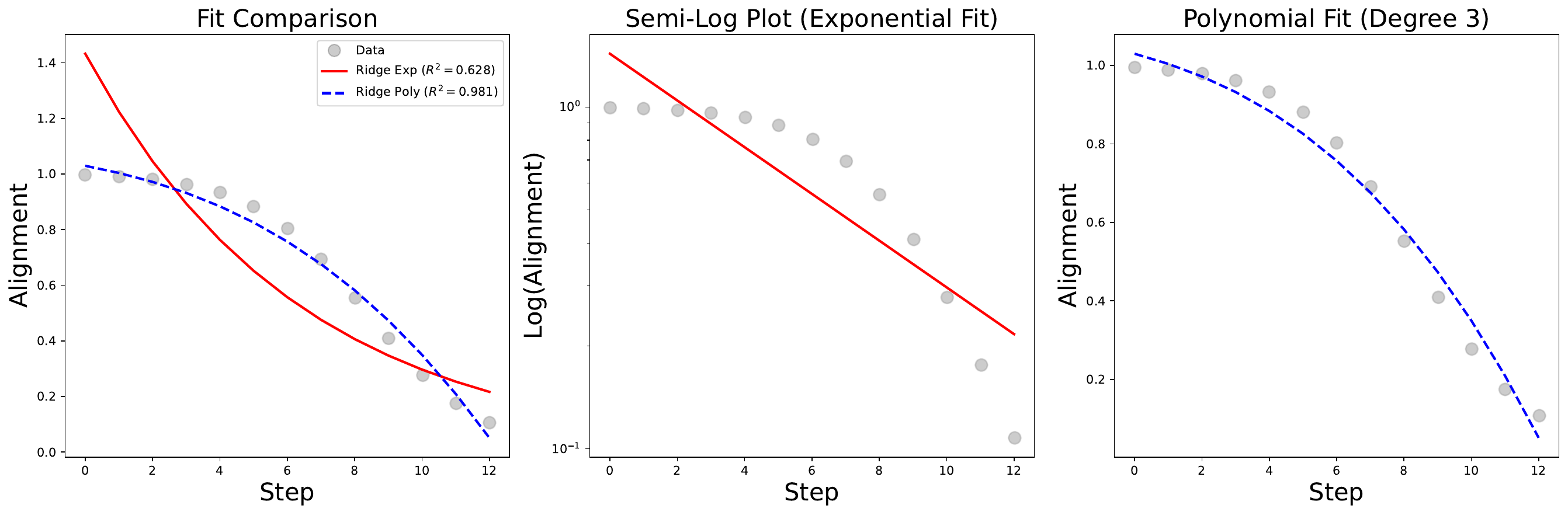}} \\
    \subfigure[m=50]{\includegraphics[width=0.9\linewidth, height=3.65cm, keepaspectratio]{ALT_final/42/phaseI_m_5_seed_42.pdf}}
    
    \caption{The decay rate of phase I when seed=42 for various $m=\frac{\lambda_k}{\lambda_{k+1}}$ (Part 1). \label{decay, seed=42}}
\end{figure}
\clearpage

\begin{figure}[p]
    \ContinuedFloat
    \centering
    \subfigure[m=100]{\includegraphics[width=0.9\linewidth, height=3.65cm, keepaspectratio]{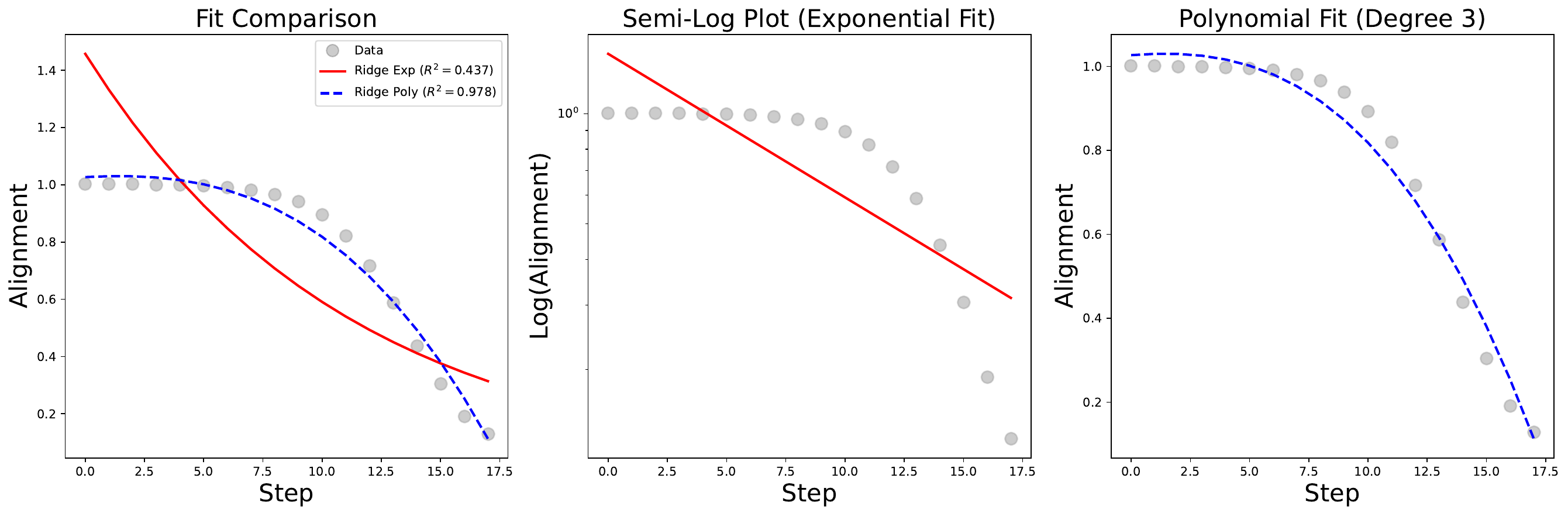}} \\
    \subfigure[m=200]{\includegraphics[width=0.9\linewidth, height=3.65cm, keepaspectratio]{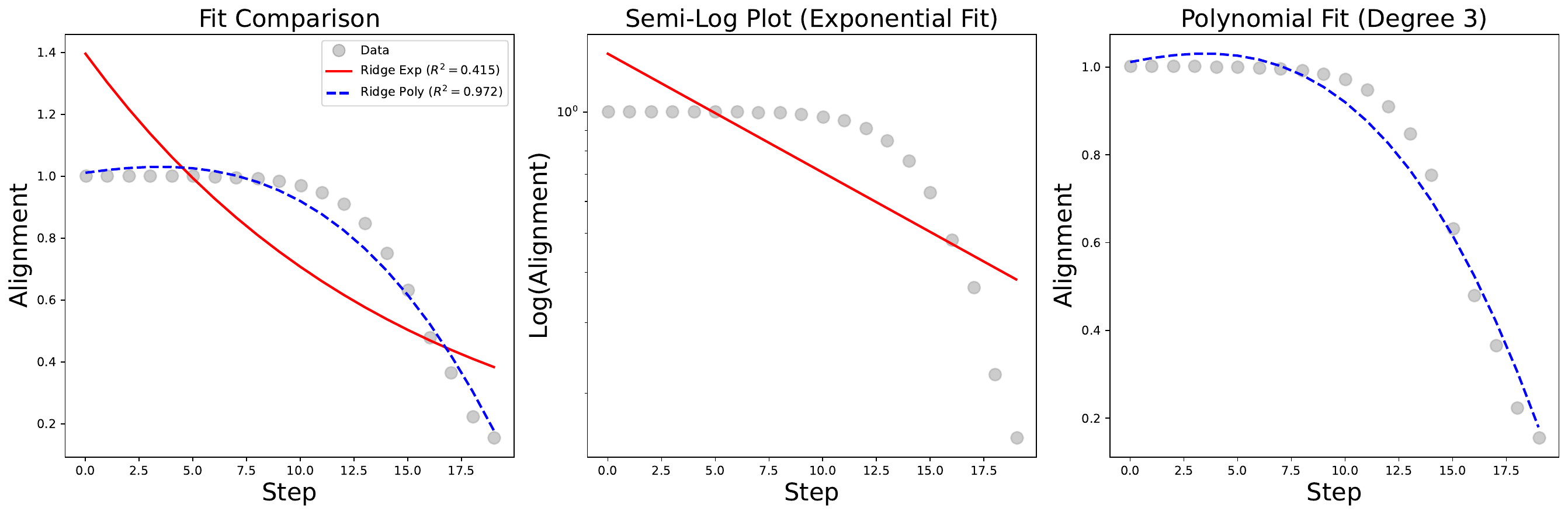}} \\
    \subfigure[m=300]{\includegraphics[width=0.9\linewidth, height=3.65cm, keepaspectratio]{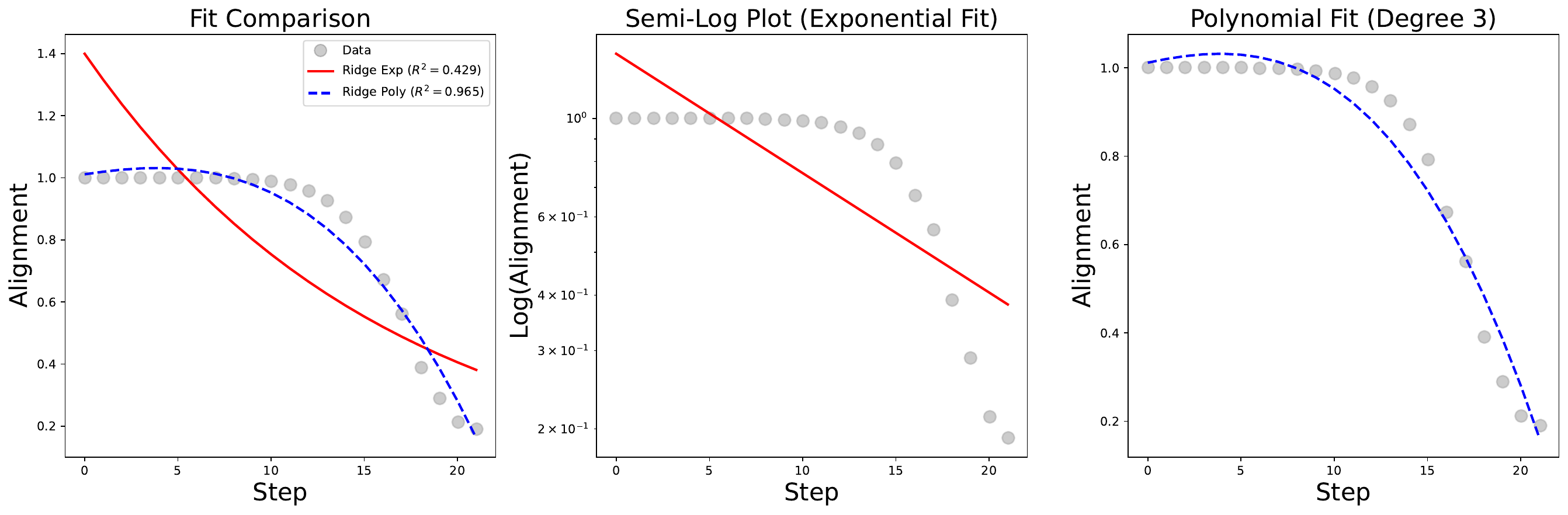}} \\
    \subfigure[m=400]{\includegraphics[width=0.9\linewidth, height=3.65cm, keepaspectratio]{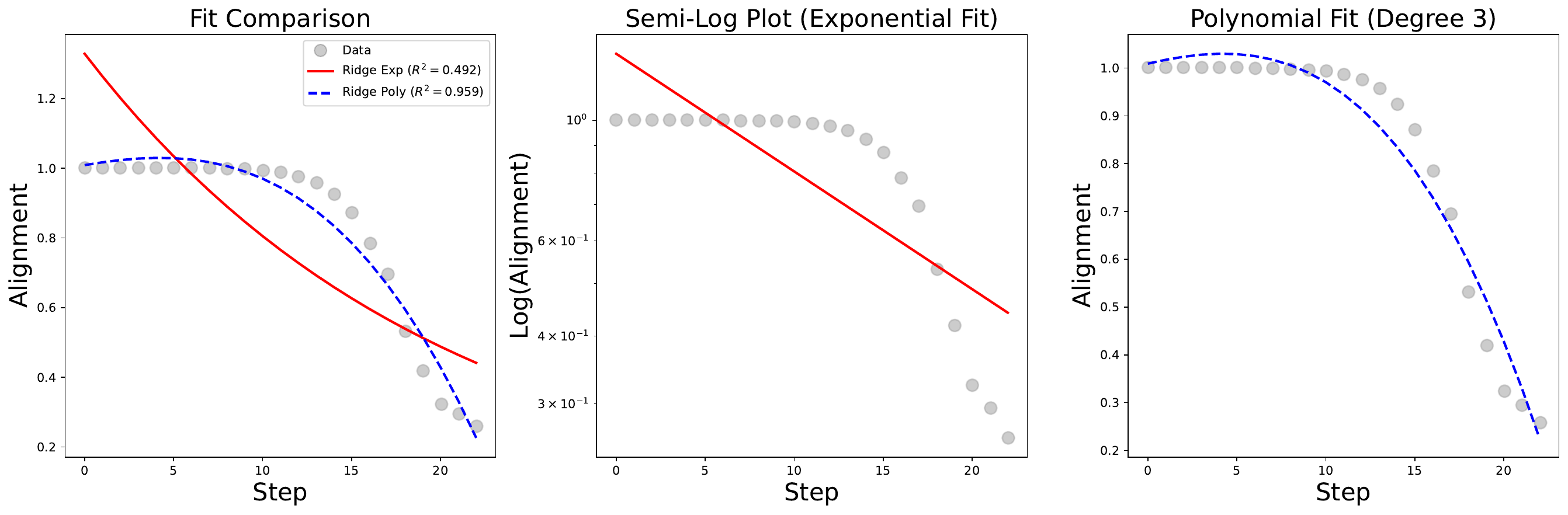}}
    
    \caption{The decay rate of phase I when seed=42 for various $m=\frac{\lambda_k}{\lambda_{k+1}}$ (Part 2). \label{decay, seed=42，2}}
\end{figure}
\clearpage

\begin{figure}[t!]
    \ContinuedFloat
    \centering
    \subfigure[m=500]{\includegraphics[width=0.9\linewidth, height=3.65cm, keepaspectratio]{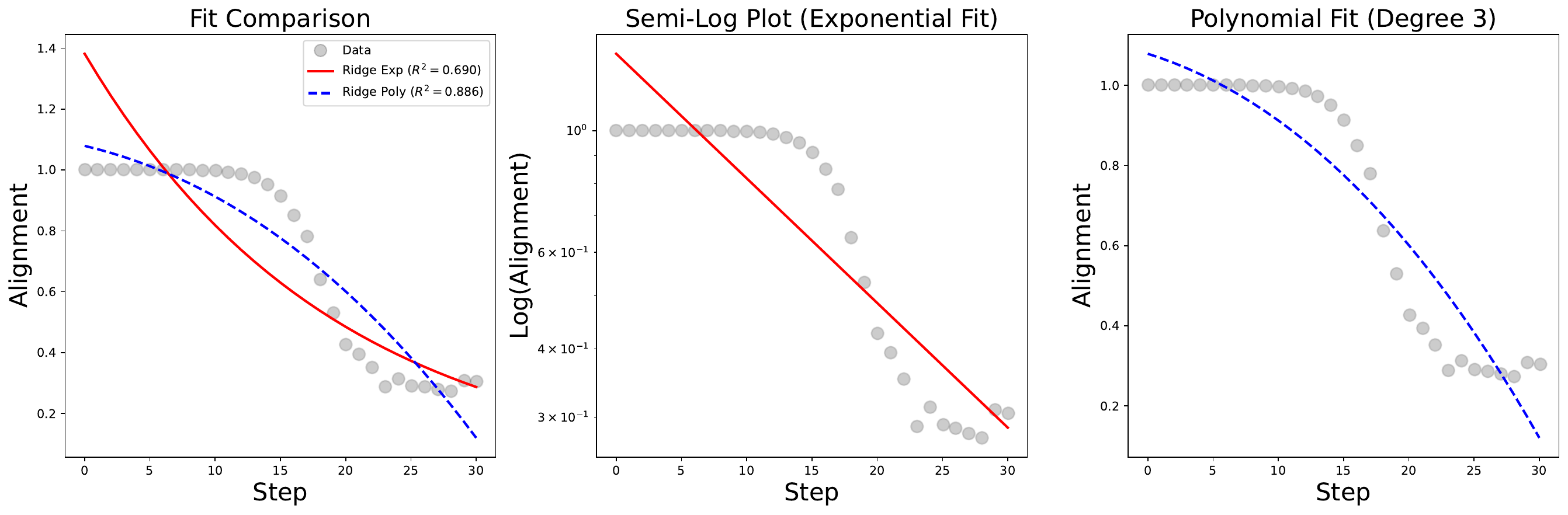}}
    \caption{The decay rate of phase I when seed=42 for various $m=\frac{\lambda_k}{\lambda_{k+1}}$ (Part 3).}
\end{figure}

\begin{figure}[h!]
    \centering
    \subfigure[m=5]{\includegraphics[width=0.9\linewidth, height=3.65cm, keepaspectratio]{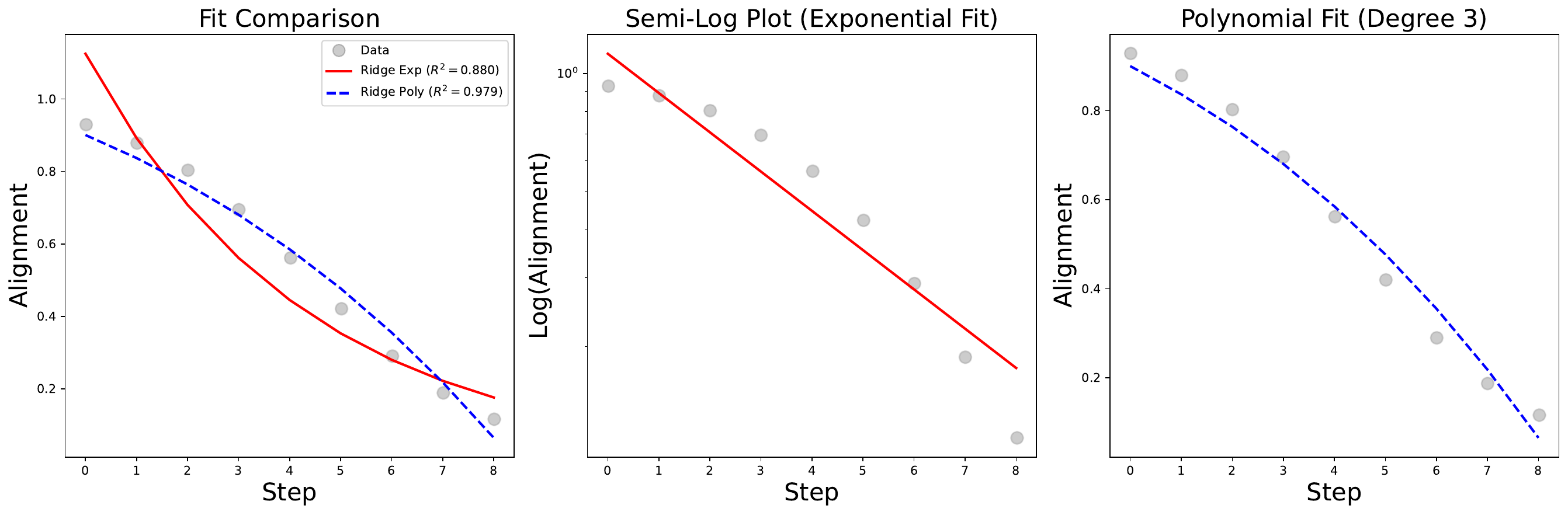}} \\
    \subfigure[m=10]{\includegraphics[width=0.9\linewidth, height=3.65cm, keepaspectratio]{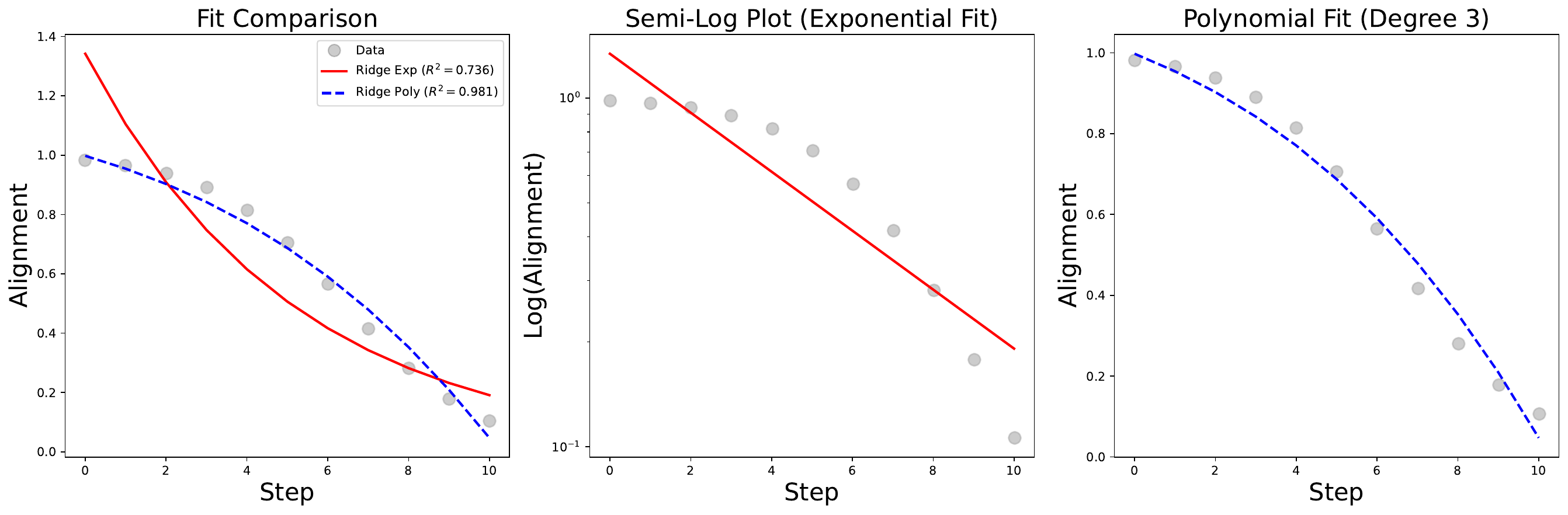}} \\
    \subfigure[m=20]{\includegraphics[width=0.9\linewidth, height=3.65cm, keepaspectratio]{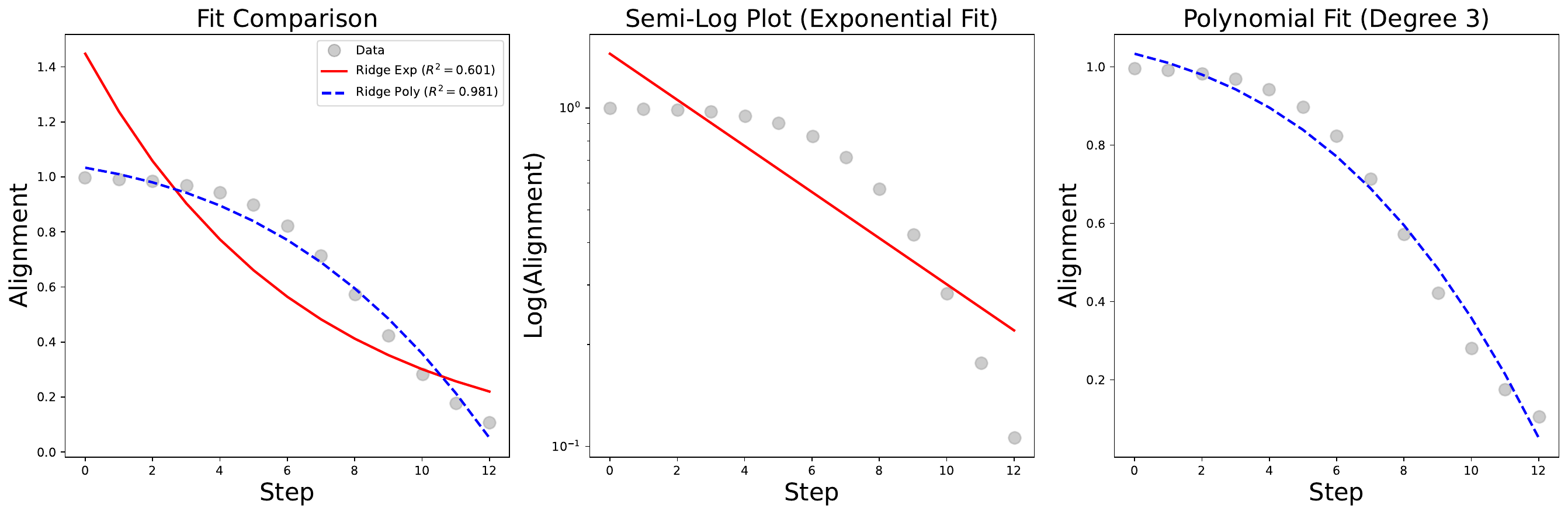}}
    
    \caption{The decay rate of phase I when seed=87 for various $m=\frac{\lambda_k}{\lambda_{k+1}}$ (Part 1)}
    \label{fig:decay_87}
\end{figure}
\clearpage

\begin{figure}[p]
    \ContinuedFloat
    \centering
    \subfigure[m=50]{\includegraphics[width=0.9\linewidth, height=3.65cm, keepaspectratio]{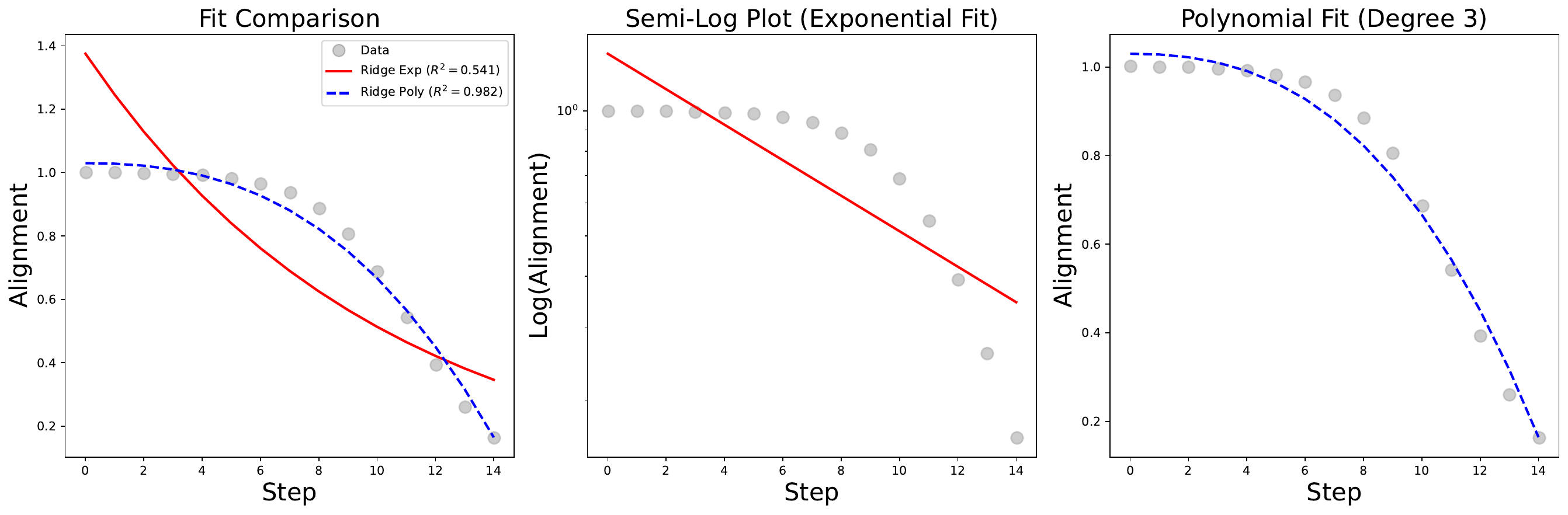}} \\
    \subfigure[m=100]{\includegraphics[width=0.9\linewidth, height=3.65cm, keepaspectratio]{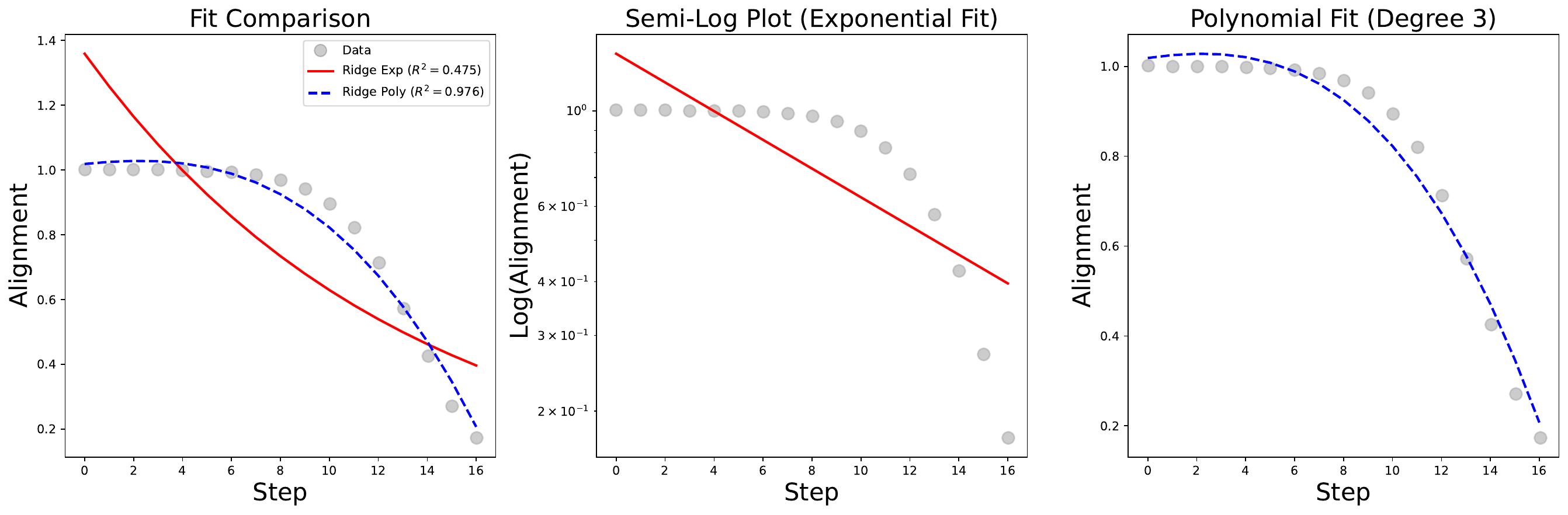}} \\
    \subfigure[m=200]{\includegraphics[width=0.9\linewidth, height=3.65cm, keepaspectratio]{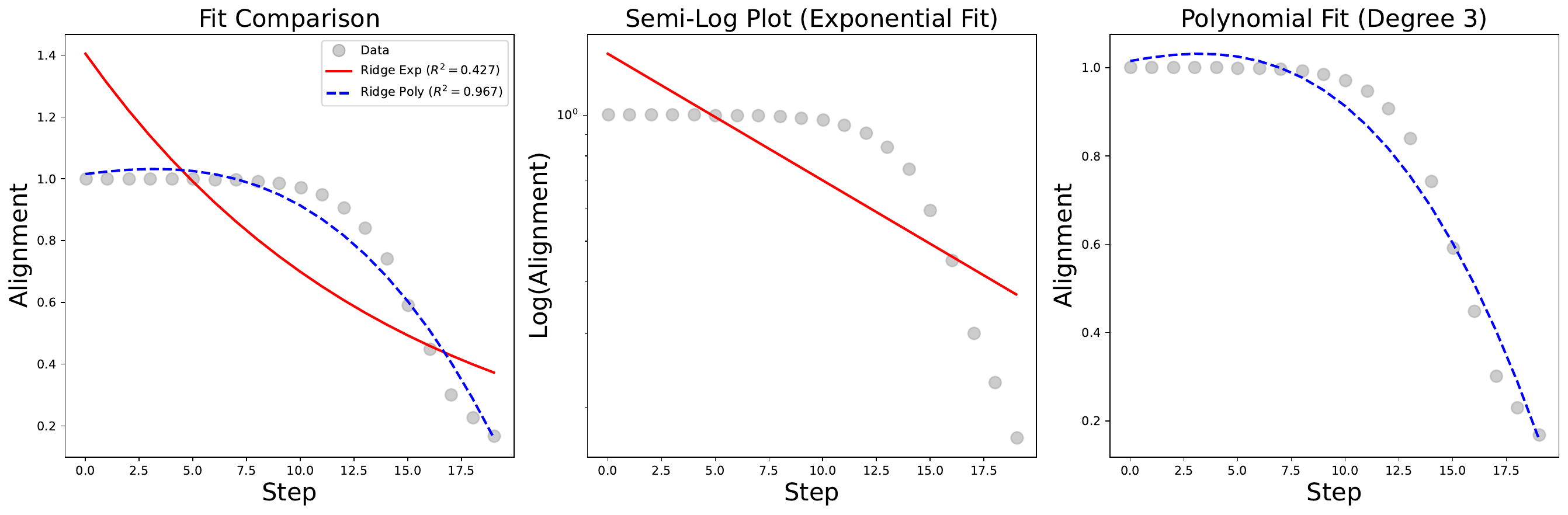}} \\
    \subfigure[m=300]{\includegraphics[width=0.9\linewidth, height=3.65cm, keepaspectratio]{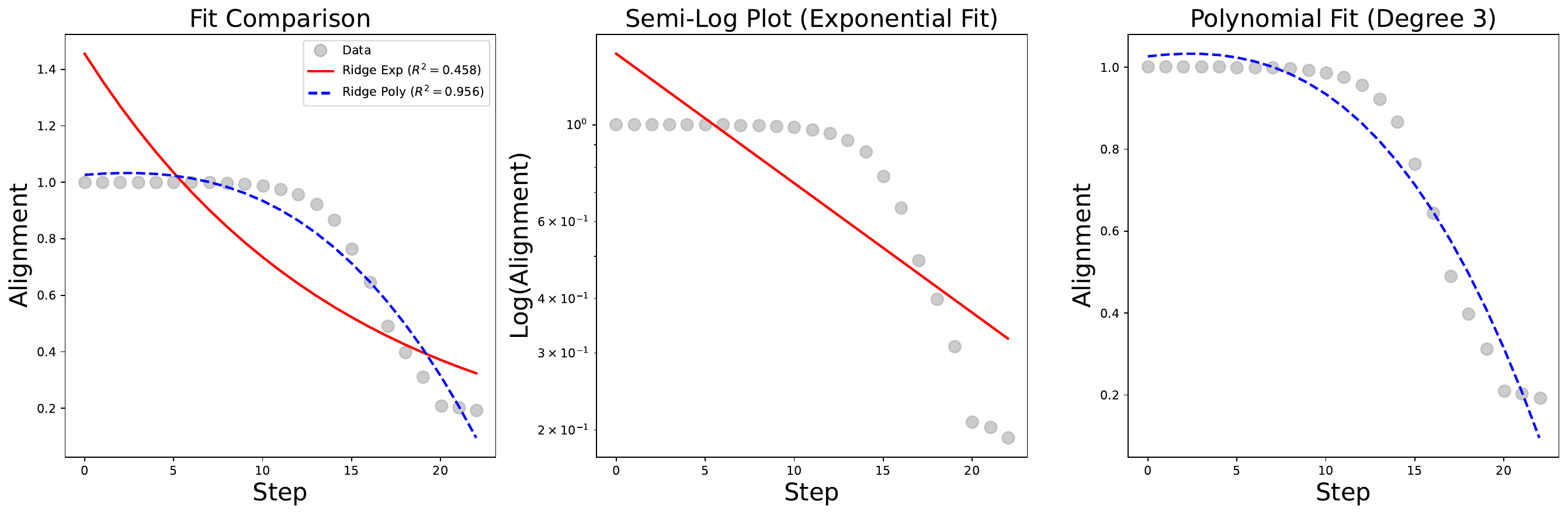}}
    
    \caption{The decay rate of phase I when seed=87 for various $m=\frac{\lambda_k}{\lambda_{k+1}}$ (Part 2)}
\end{figure}
\clearpage

\begin{figure}[t!]
    \ContinuedFloat
    \centering
    \subfigure[m=400]{\includegraphics[width=0.9\linewidth, height=3.65cm, keepaspectratio]{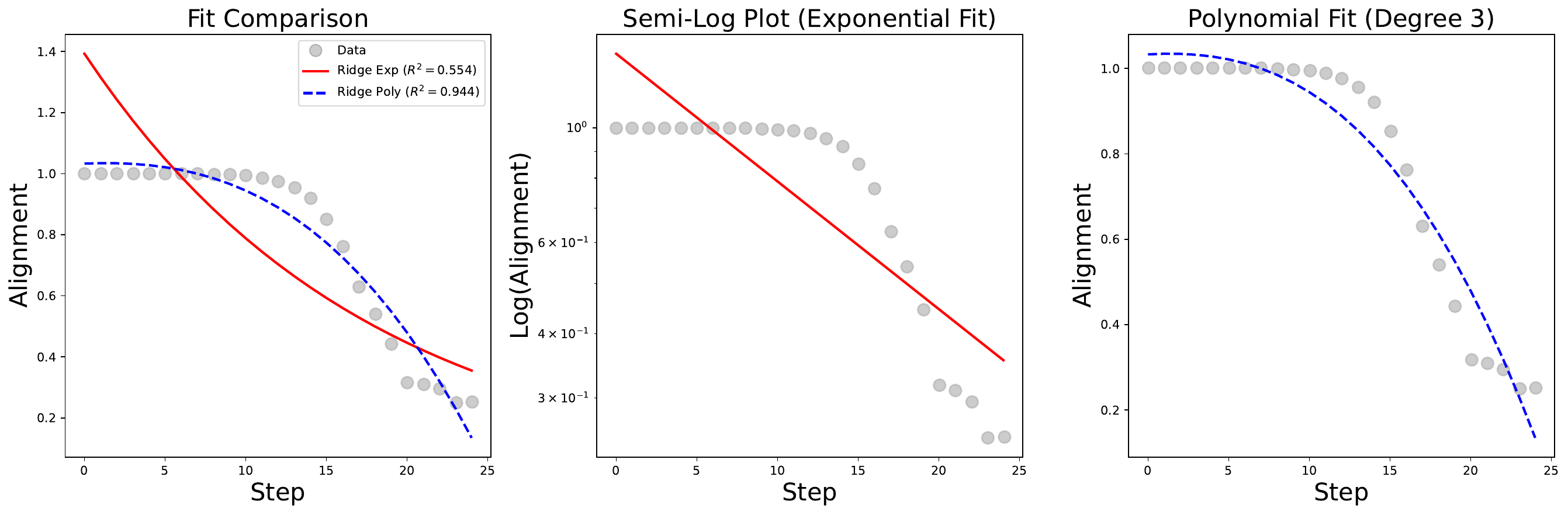}} \\
    \subfigure[m=500]{\includegraphics[width=0.9\linewidth, height=3.65cm, keepaspectratio]{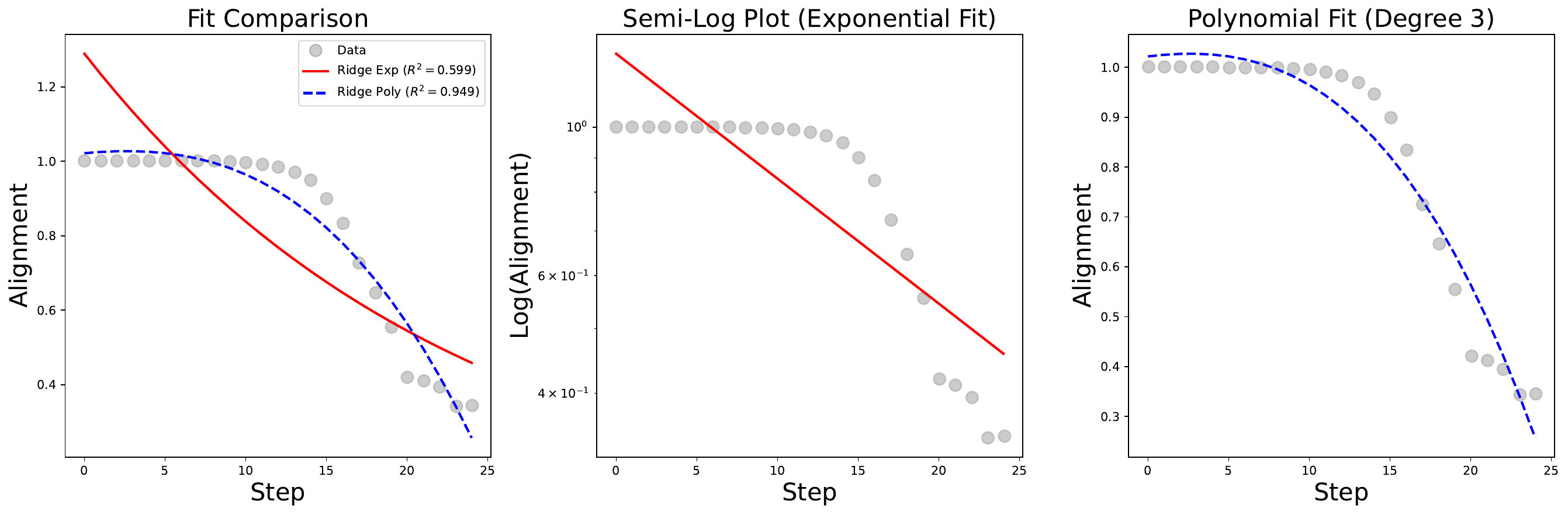}}
    \caption{The decay rate of phase I when seed=87 for various $m=\frac{\lambda_k}{\lambda_{k+1}}$ (Part 3)}
\end{figure}

\begin{figure}[h!]
    \centering
    \subfigure[m=5]{\includegraphics[width=0.9\linewidth, height=3.65cm, keepaspectratio]{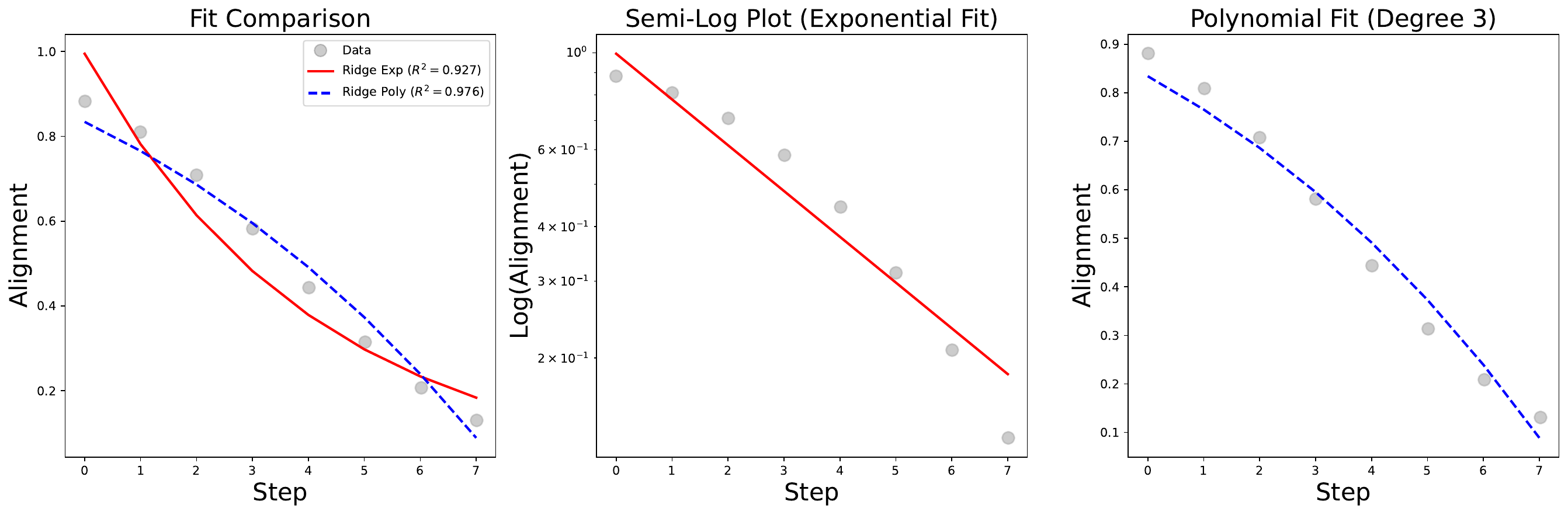}} \\
    \subfigure[m=10]{\includegraphics[width=0.9\linewidth, height=3.65cm, keepaspectratio]{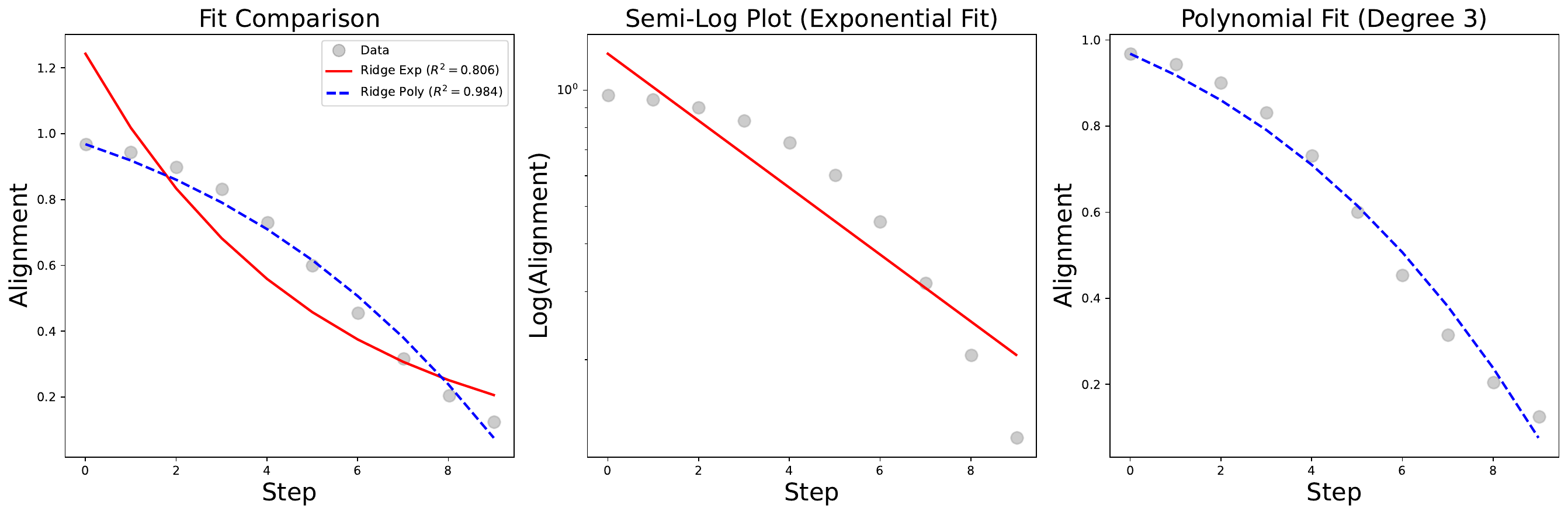}}
    
    \caption{The decay rate of phase I when seed=568 for various $m=\frac{\lambda_k}{\lambda_{k+1}}$ (Part 1)}
    \label{fig:decay_568}
\end{figure}
\clearpage

\begin{figure}[p]
    \ContinuedFloat
    \centering
    \subfigure[m=20]{\includegraphics[width=0.9\linewidth, height=3.65cm, keepaspectratio]{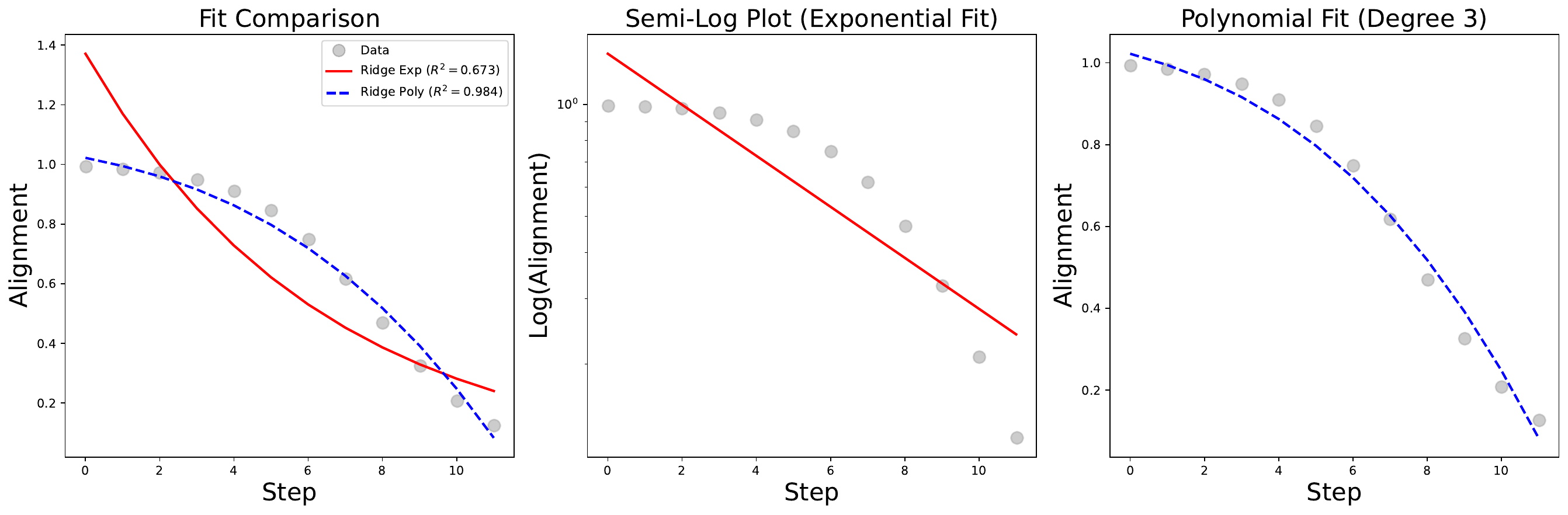}} \\
    \subfigure[m=50]{\includegraphics[width=0.9\linewidth, height=3.65cm, keepaspectratio]{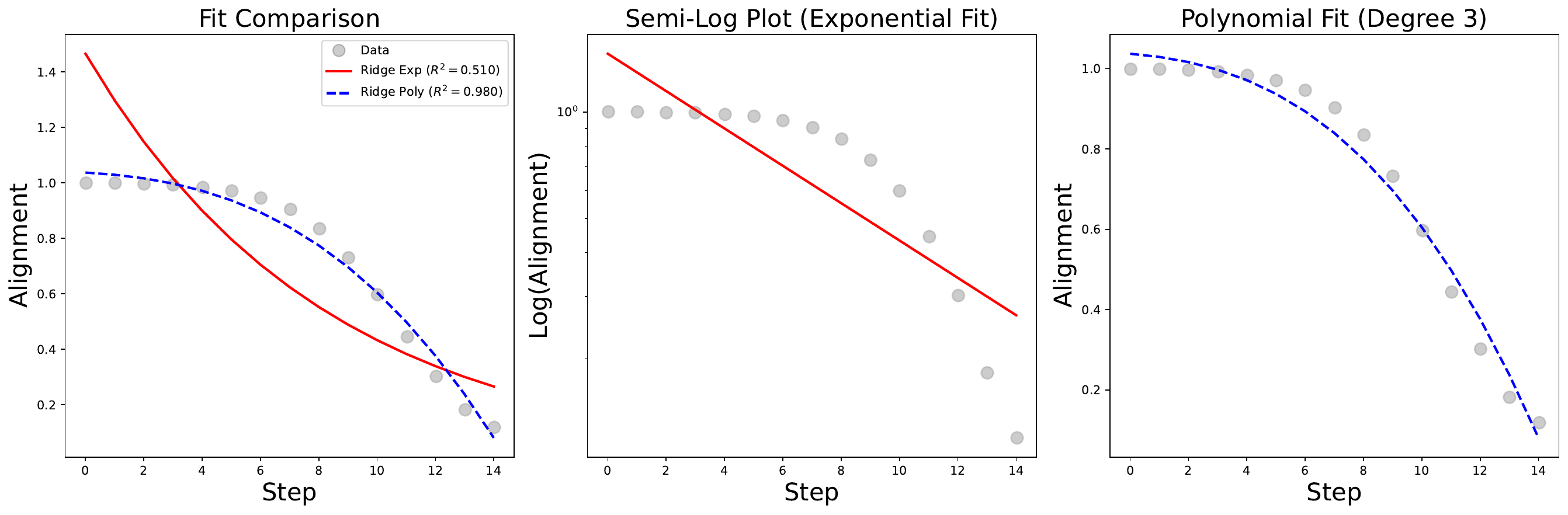}} \\
    \subfigure[m=100]{\includegraphics[width=0.9\linewidth, height=3.65cm, keepaspectratio]{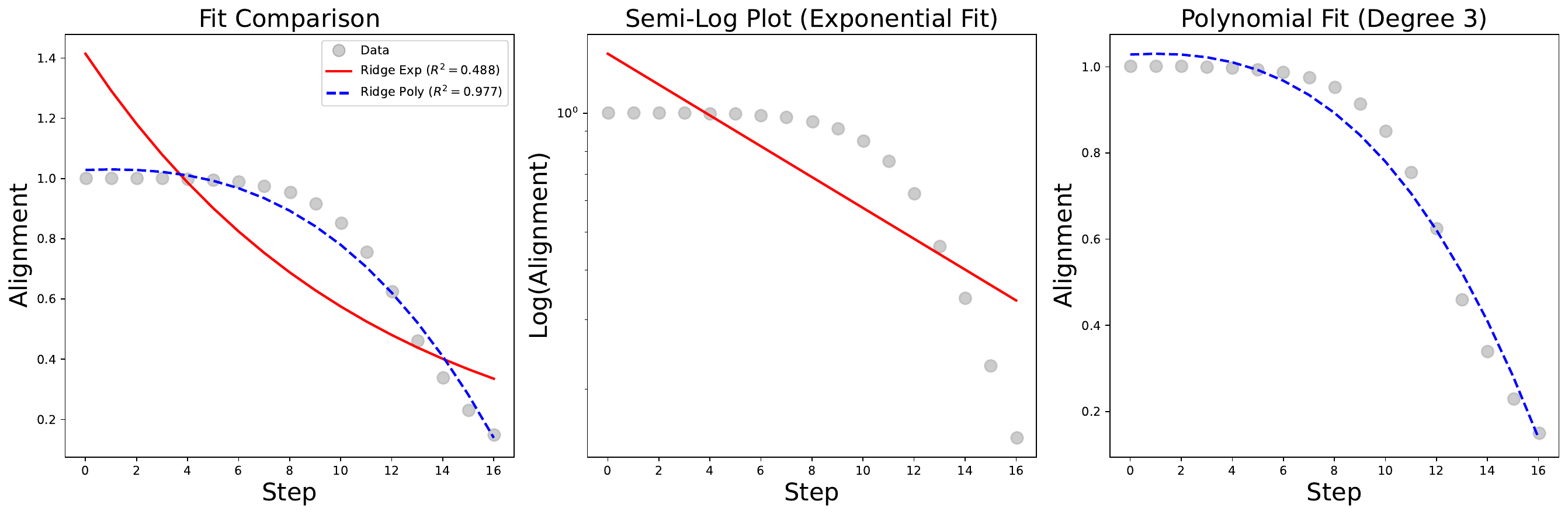}} \\
    \subfigure[m=200]{\includegraphics[width=0.9\linewidth, height=3.65cm, keepaspectratio]{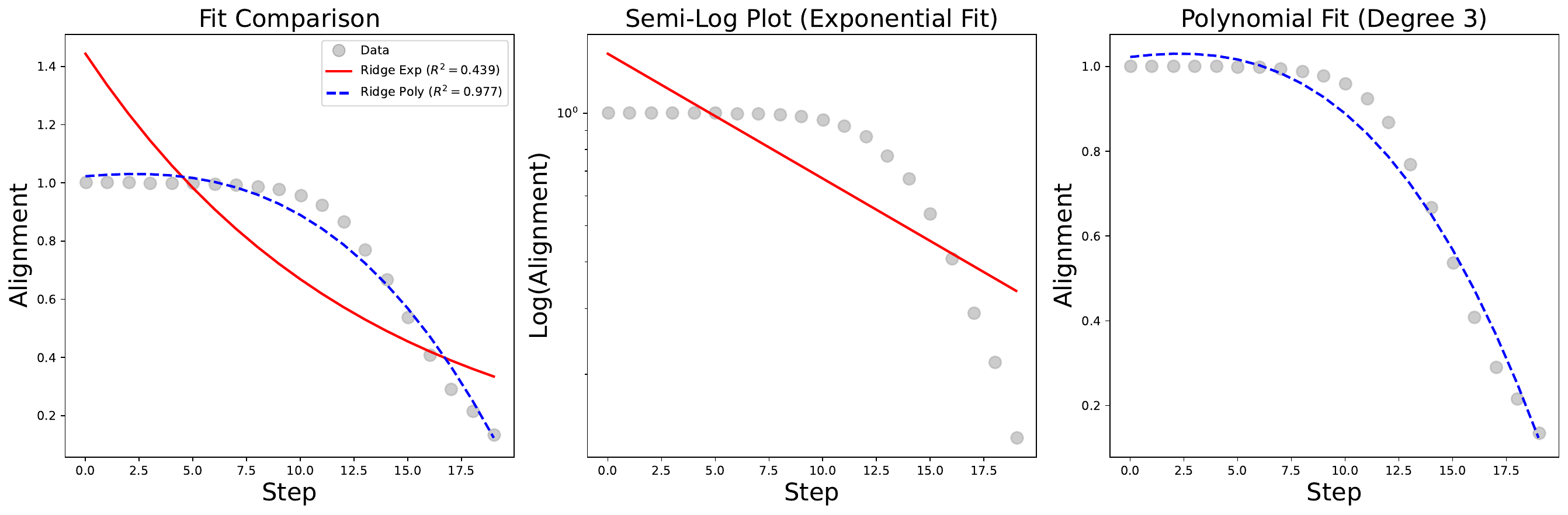}}
    
    \caption{The decay rate of phase I when seed=568 for various $m=\frac{\lambda_k}{\lambda_{k+1}}$ (Part 2)}
\end{figure}
\clearpage

\begin{figure}[t!]
    \ContinuedFloat
    \centering
    \subfigure[m=300]{\includegraphics[width=0.9\linewidth, height=3.65cm, keepaspectratio]{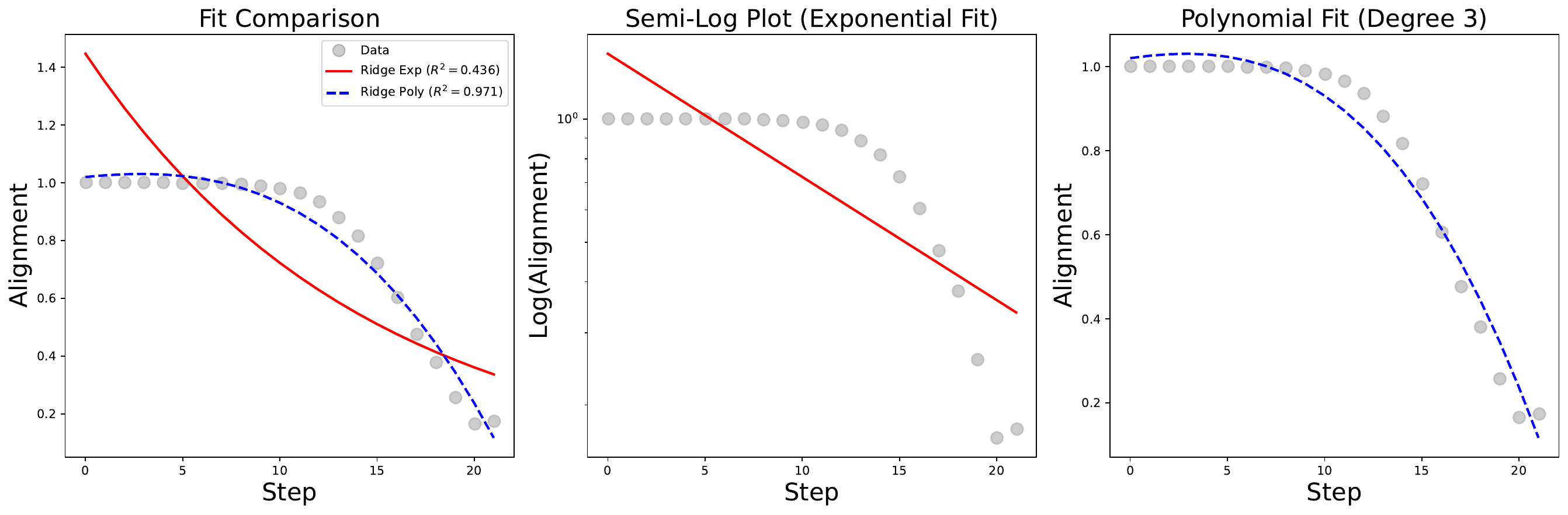}} \\
    \subfigure[m=400]{\includegraphics[width=0.9\linewidth, height=3.65cm, keepaspectratio]{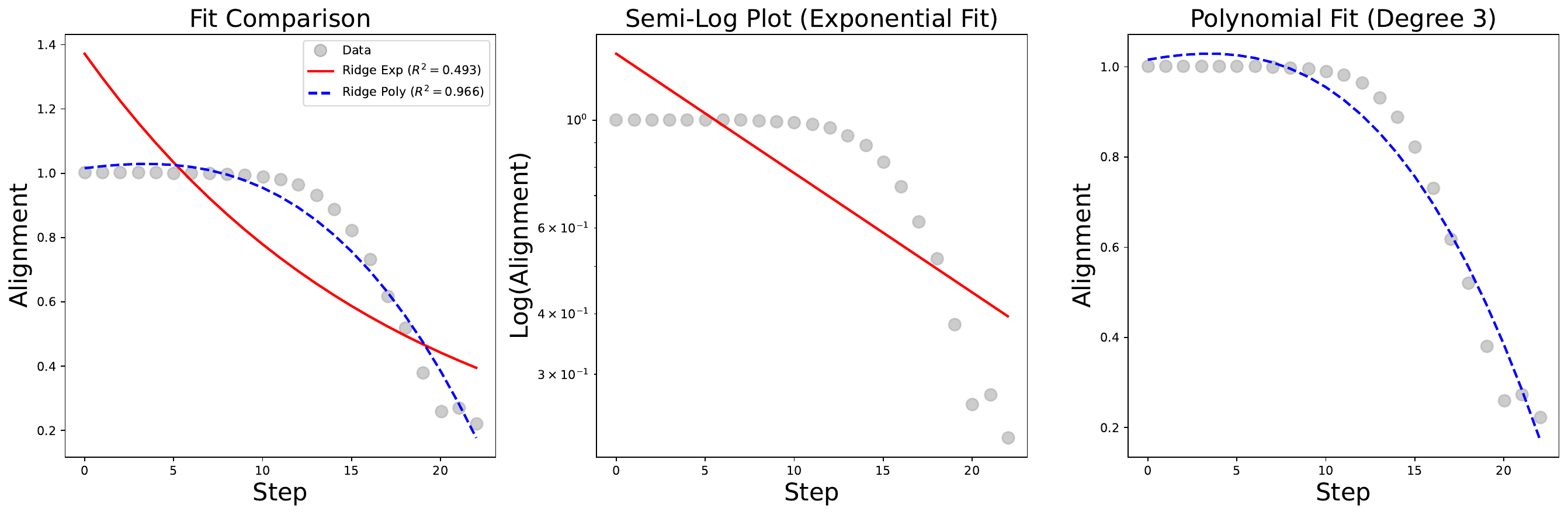}} \\
    \subfigure[m=500]{\includegraphics[width=0.9\linewidth, height=3.65cm, keepaspectratio]{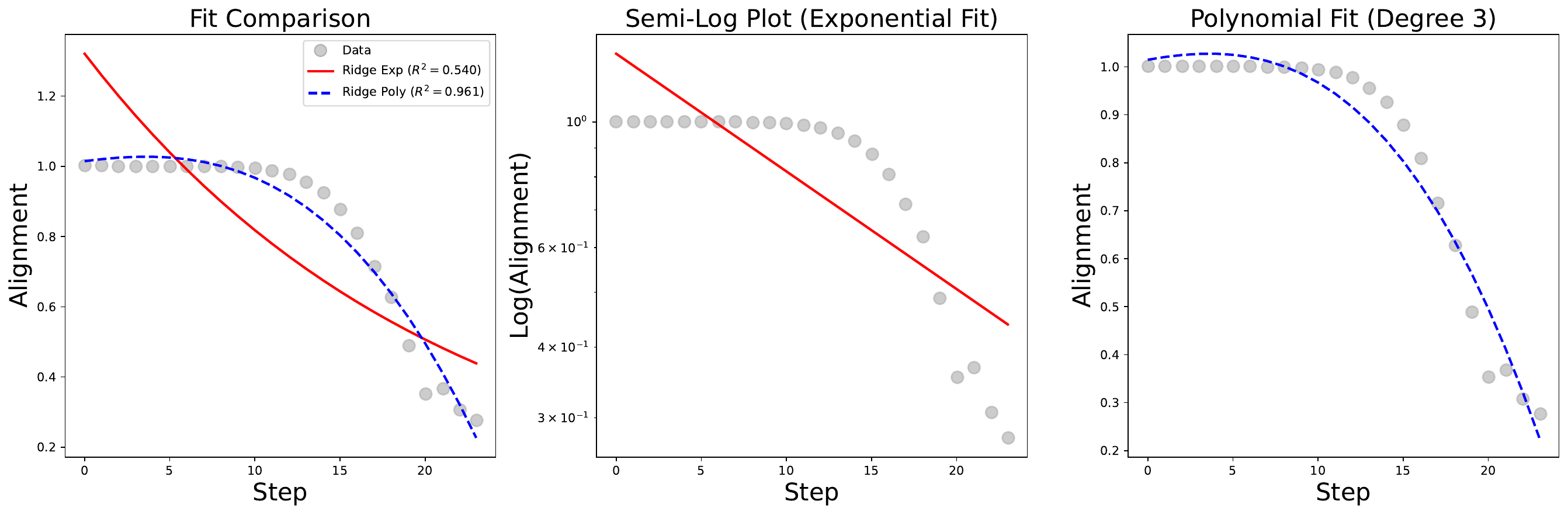}}
    \caption{The decay rate of phase I when seed=568 for various $m=\frac{\lambda_k}{\lambda_{k+1}}$ (Part 3)}
\end{figure}

\begin{figure}[h!]
    \centering
    \subfigure[m=5]{\includegraphics[width=0.9\linewidth, height=3.65cm, keepaspectratio]{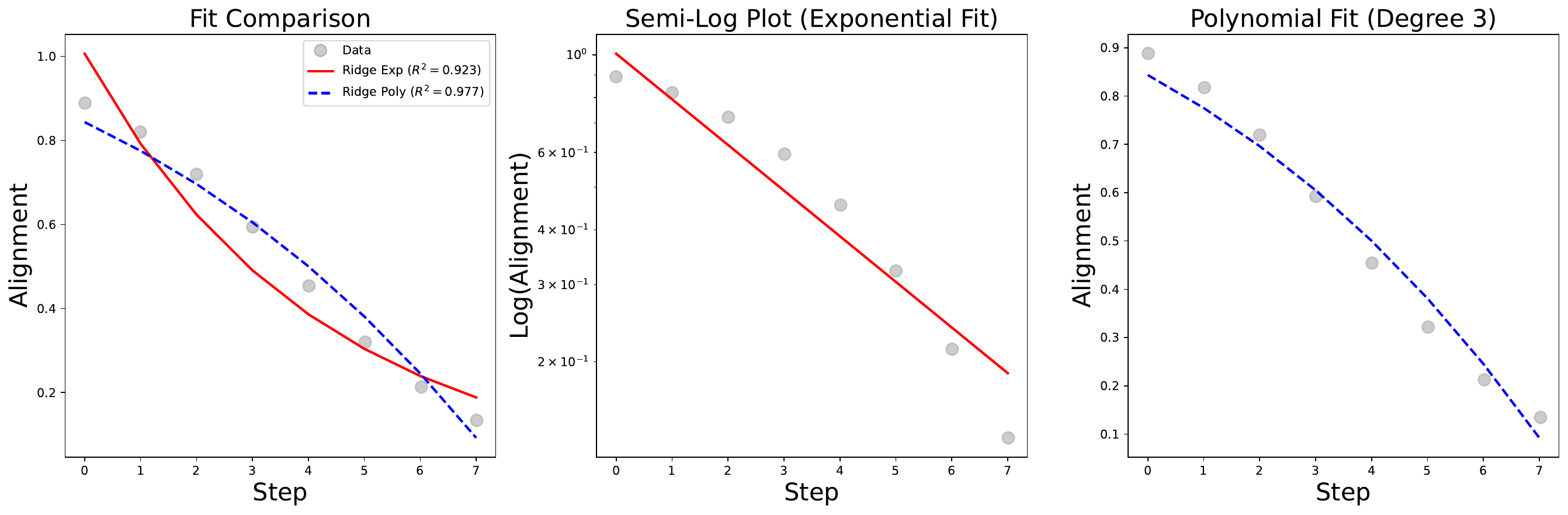}}
    
    \caption{The decay rate of phase I when seed=1101 for various $m=\frac{\lambda_k}{\lambda_{k+1}}$ (Part 1)}
    \label{fig:decay_1101}
\end{figure}
\clearpage

\begin{figure}[p]
    \ContinuedFloat
    \centering
    \subfigure[m=10]{\includegraphics[width=0.9\linewidth, height=3.65cm, keepaspectratio]{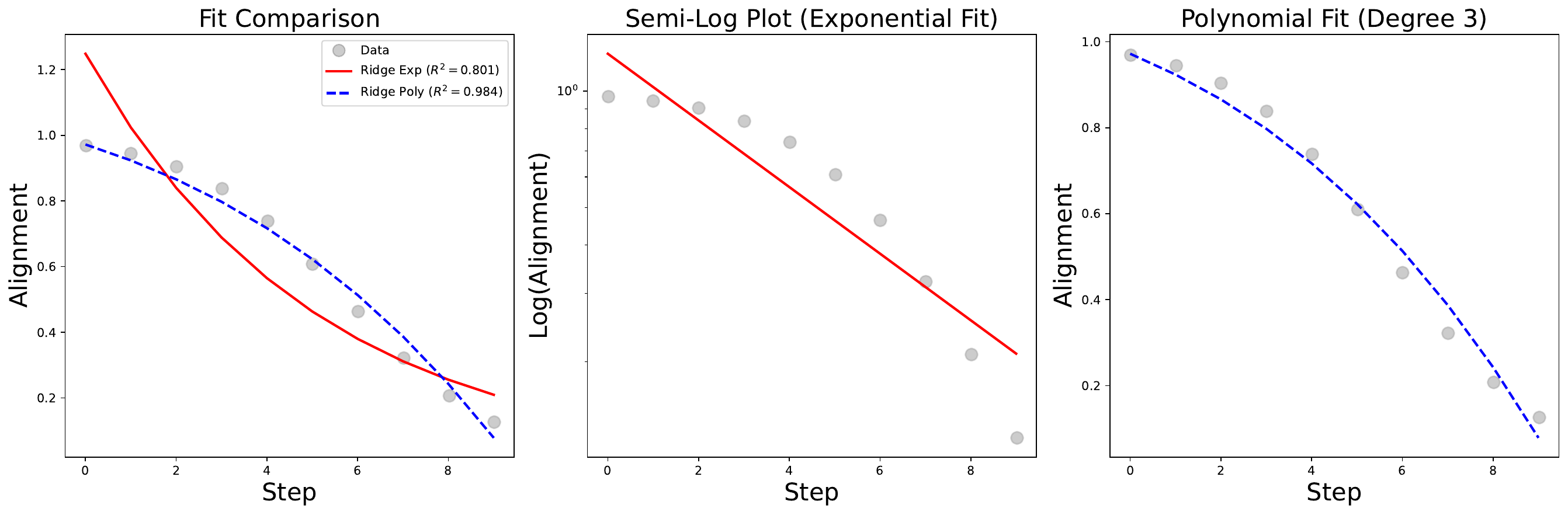}} \\
    \subfigure[m=20]{\includegraphics[width=0.9\linewidth, height=3.65cm, keepaspectratio]{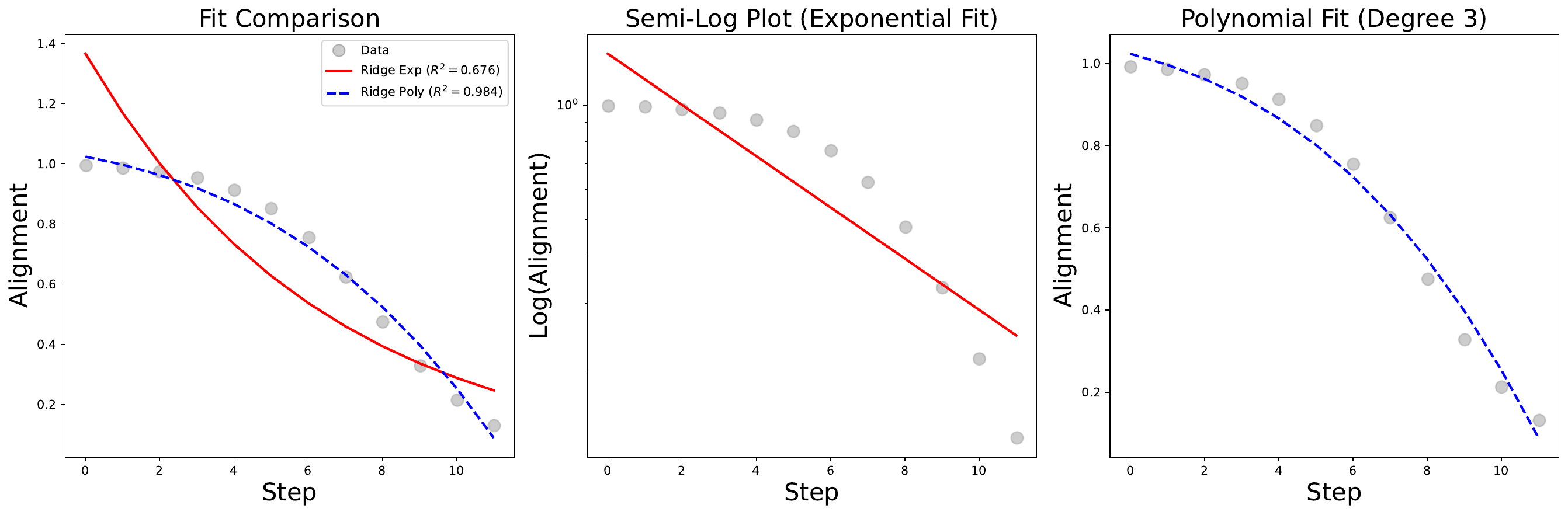}} \\
    \subfigure[m=50]{\includegraphics[width=0.9\linewidth, height=3.65cm, keepaspectratio]{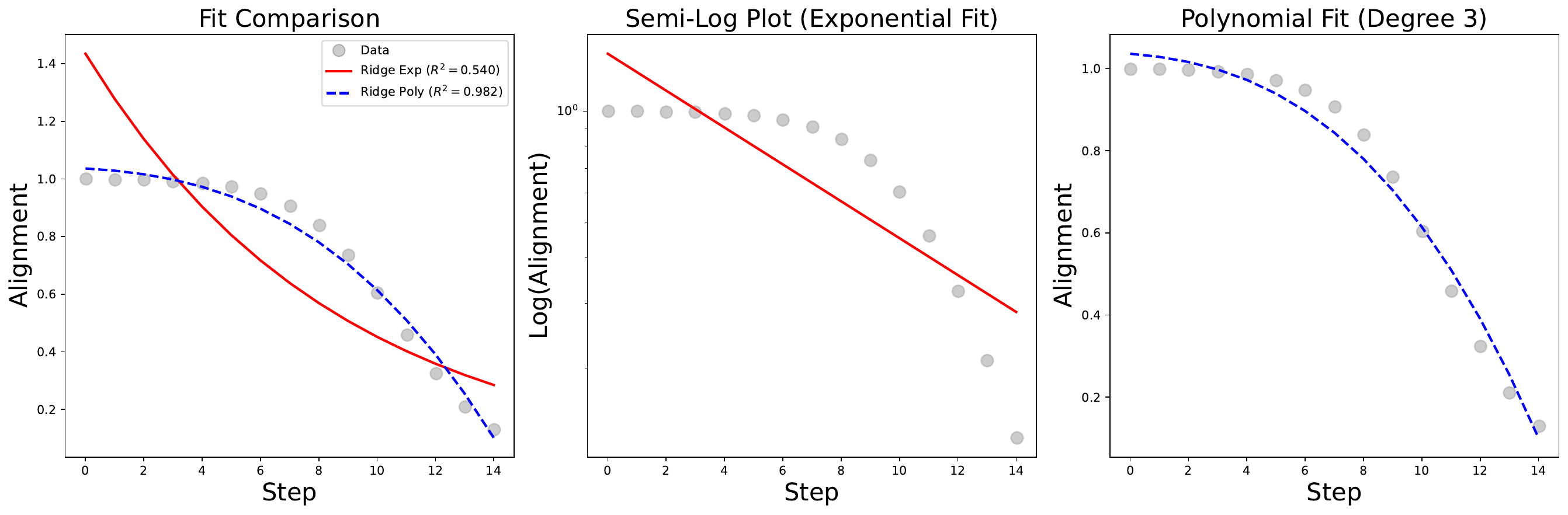}} \\
    \subfigure[m=100]{\includegraphics[width=0.9\linewidth, height=3.65cm, keepaspectratio]{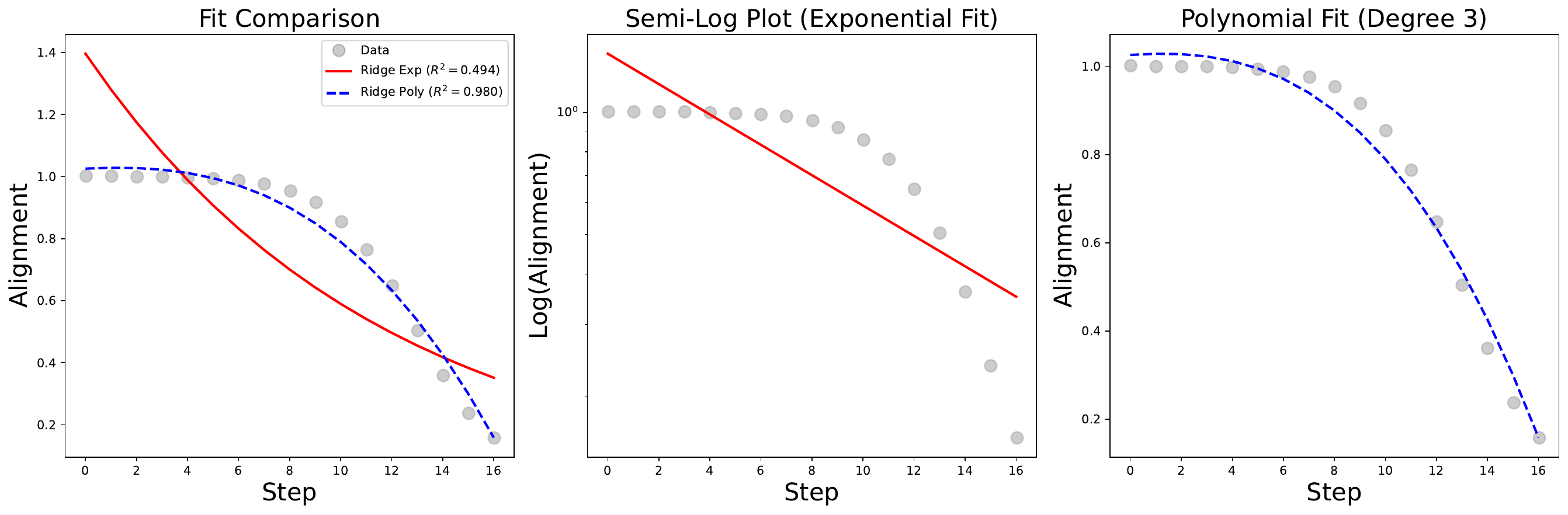}}
    
    \caption{The decay rate of phase I when seed=1101 for various $m=\frac{\lambda_k}{\lambda_{k+1}}$ (Part 2)}
\end{figure}
\clearpage

\begin{figure}[p]
    \ContinuedFloat
    \centering
    \subfigure[m=200]{\includegraphics[width=0.9\linewidth, height=3.65cm, keepaspectratio]{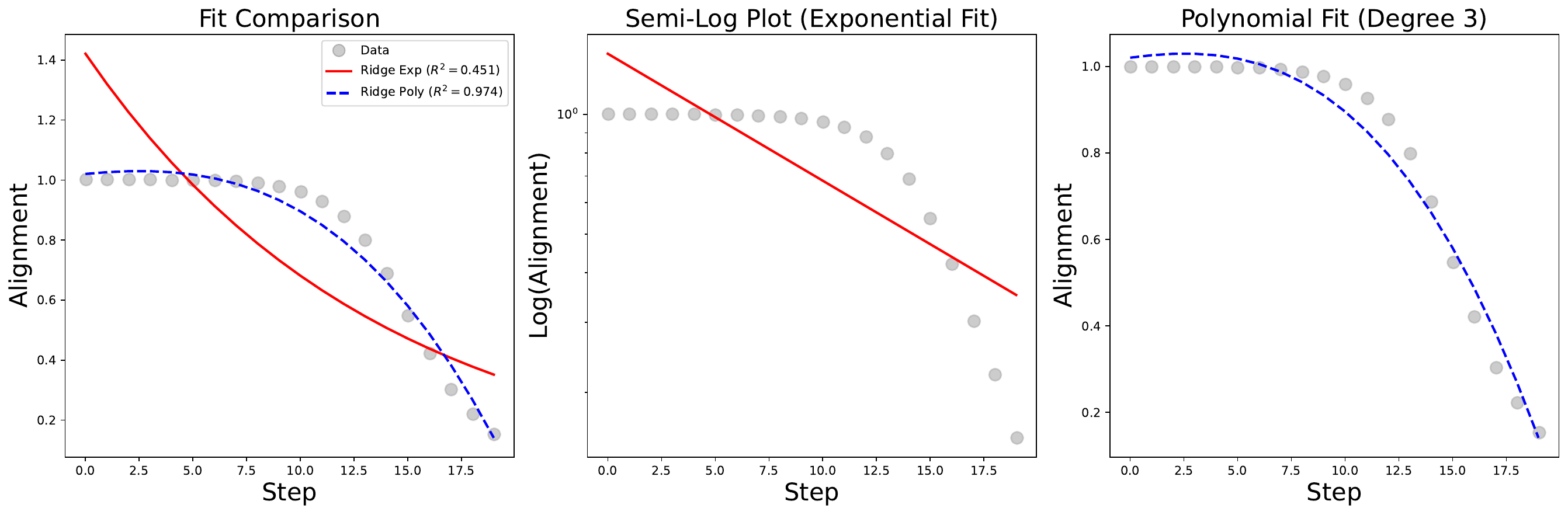}} \\
    \subfigure[m=300]{\includegraphics[width=0.9\linewidth, height=3.65cm, keepaspectratio]{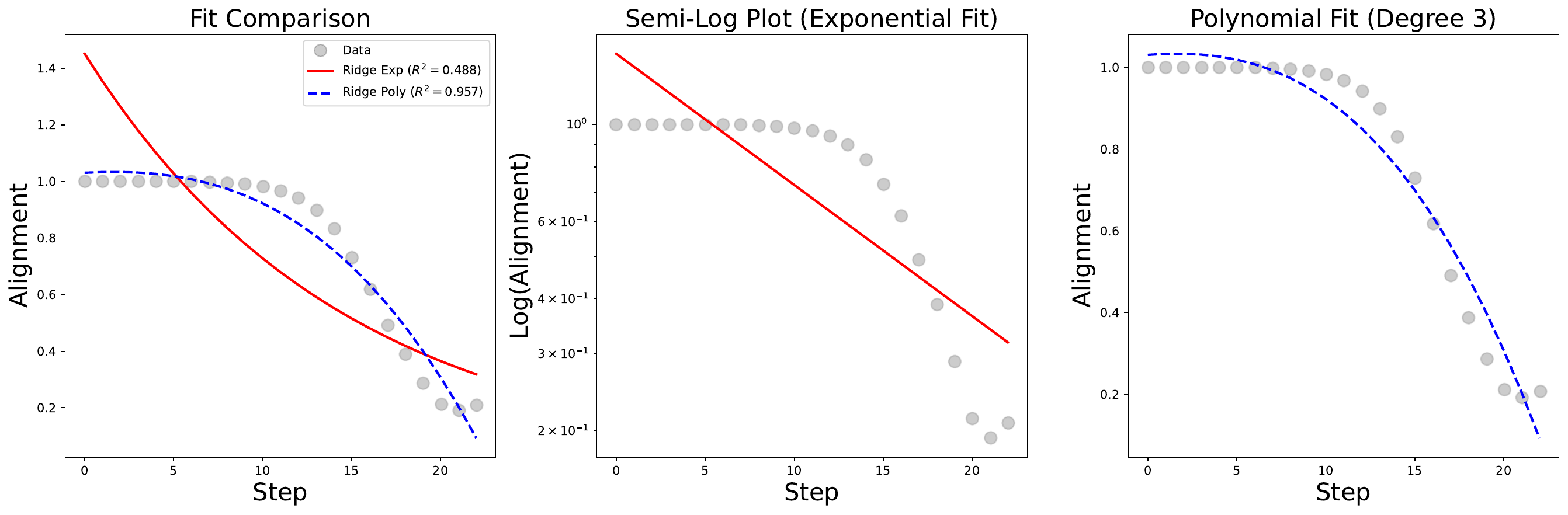}} \\
    \subfigure[m=400]{\includegraphics[width=0.9\linewidth, height=3.65cm, keepaspectratio]{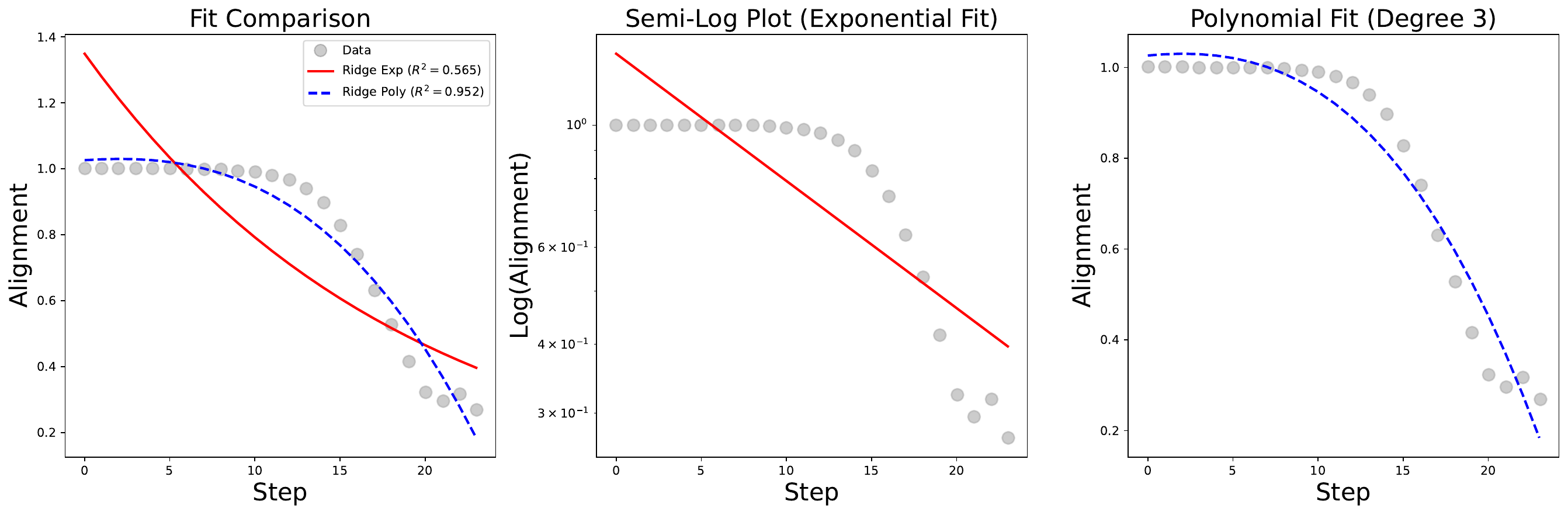}} \\
    \subfigure[m=500]{\includegraphics[width=0.9\linewidth, height=3.65cm, keepaspectratio]{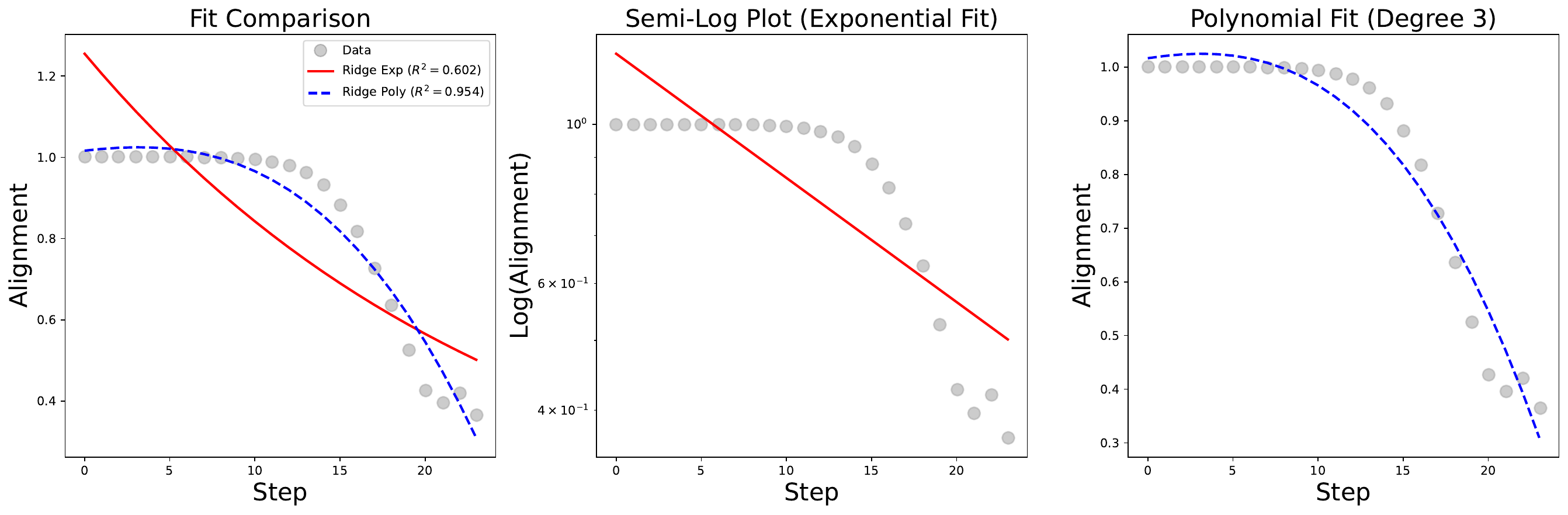}}
    
    \caption{The decay rate of phase I when seed=1101 for various $m=\frac{\lambda_k}{\lambda_{k+1}}$ (Part 3)}
\end{figure}
\clearpage

\begin{figure}[p]
    \centering
    \subfigure[m=5]{\includegraphics[width=0.9\linewidth, height=3.65cm, keepaspectratio]{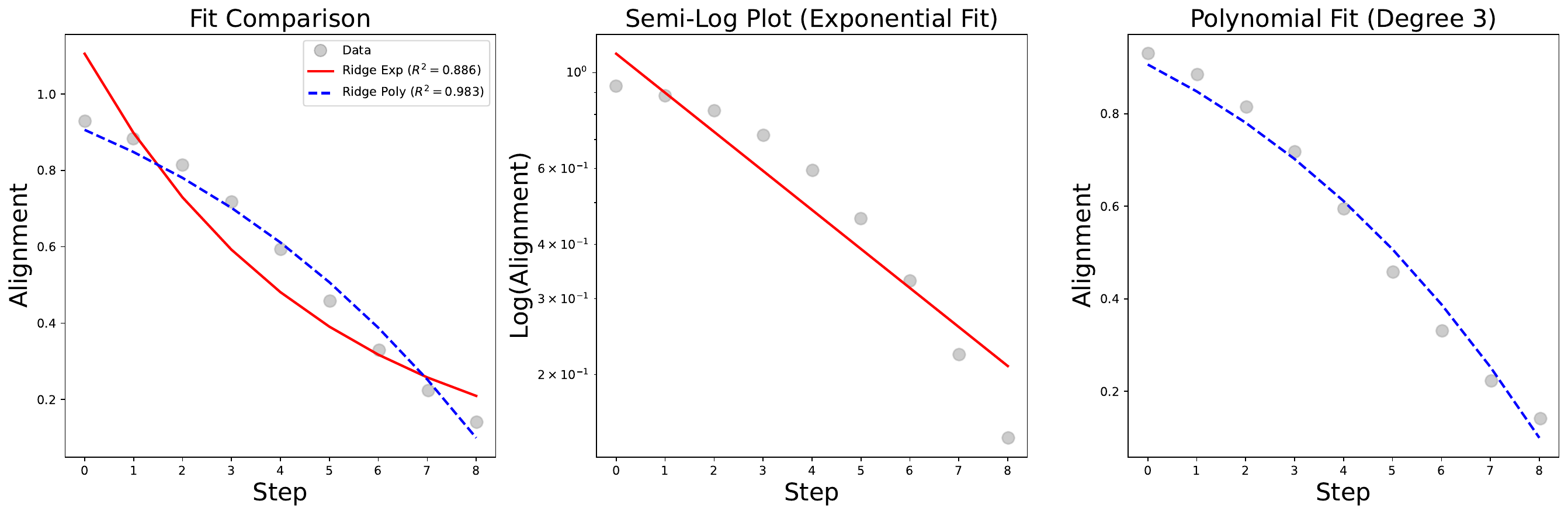}} \\
    \subfigure[m=10]{\includegraphics[width=0.9\linewidth, height=3.65cm, keepaspectratio]{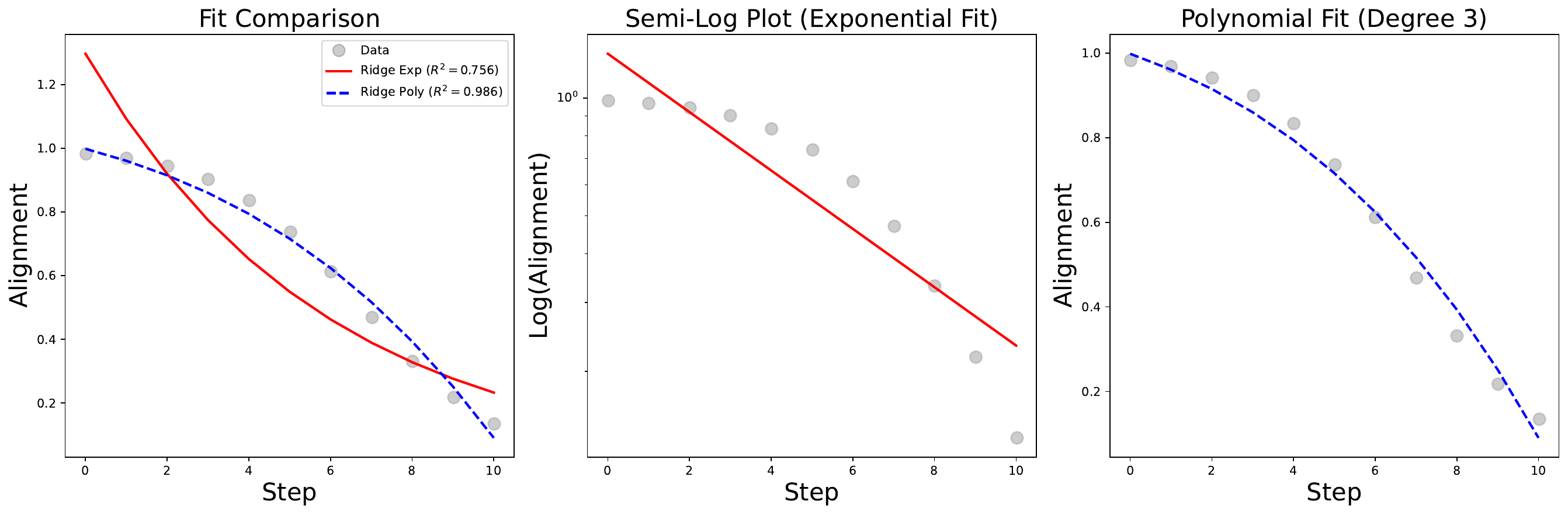}} \\
    \subfigure[m=20]{\includegraphics[width=0.9\linewidth, height=3.65cm, keepaspectratio]{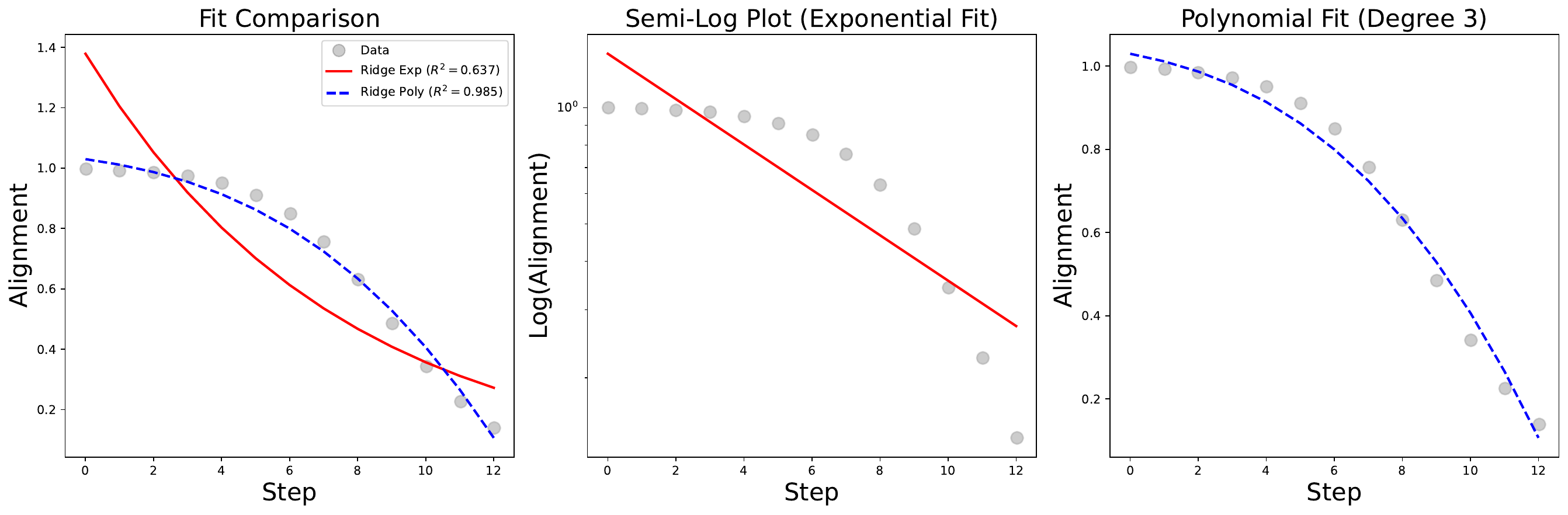}} \\
    \subfigure[m=50]{\includegraphics[width=0.9\linewidth, height=3.65cm, keepaspectratio]{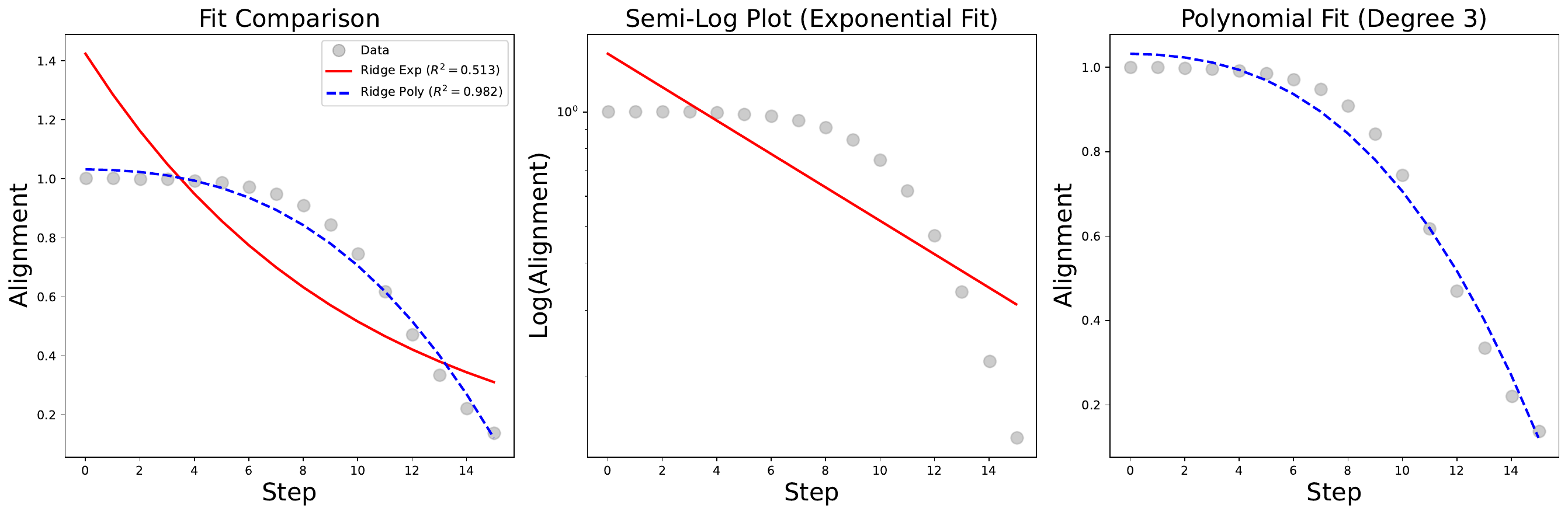}}
    
    \caption{The decay rate of phase I when seed=12138 for various $m=\frac{\lambda_k}{\lambda_{k+1}}$ (Part 1)}
    \label{fig:decay_12138}
\end{figure}
\clearpage

\begin{figure}[p]
    \ContinuedFloat
    \centering
    \subfigure[m=100]{\includegraphics[width=0.9\linewidth, height=3.65cm, keepaspectratio]{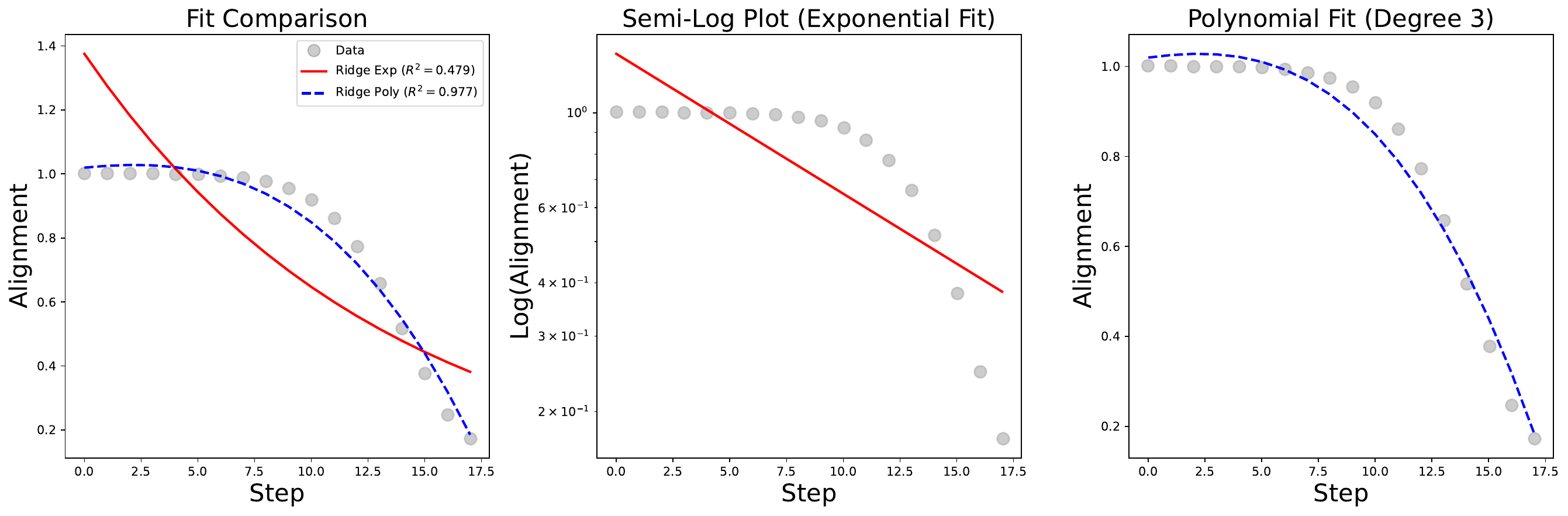}} \\
    \subfigure[m=200]{\includegraphics[width=0.9\linewidth, height=3.65cm, keepaspectratio]{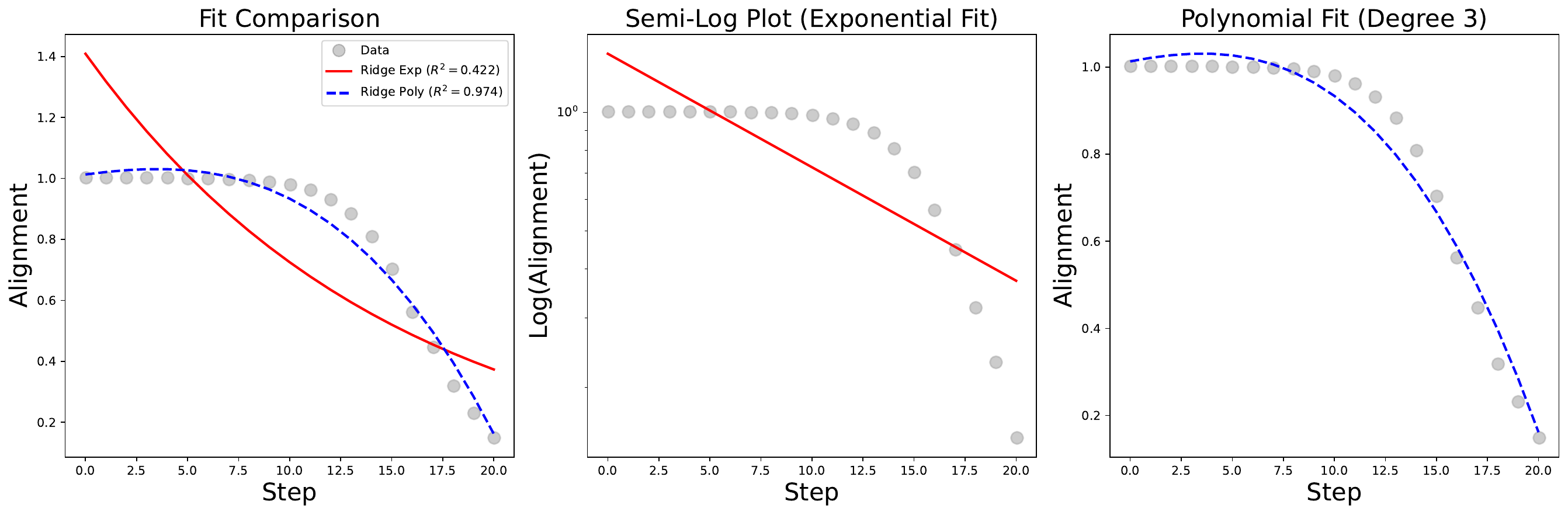}} \\
    \subfigure[m=300]{\includegraphics[width=0.9\linewidth, height=3.65cm, keepaspectratio]{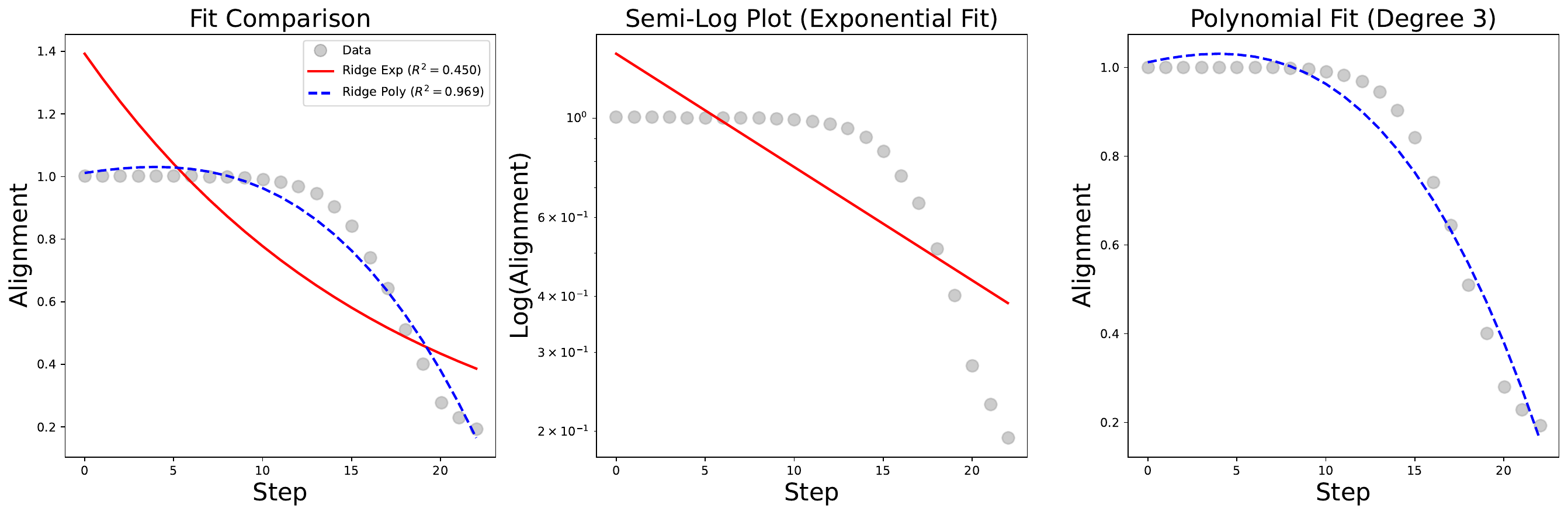}} \\
    \subfigure[m=400]{\includegraphics[width=0.9\linewidth, height=3.65cm, keepaspectratio]{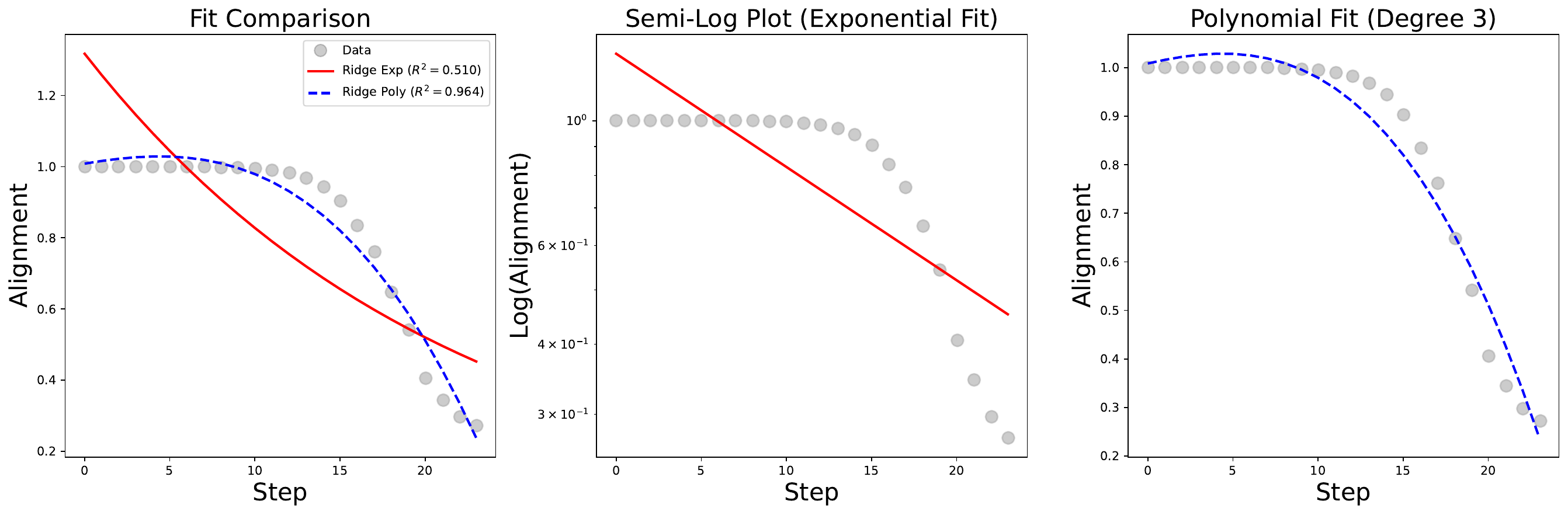}}
    
    \caption{The decay rate of phase I when seed=12138 for various $m=\frac{\lambda_k}{\lambda_{k+1}}$ (Part 2)}
\end{figure}
\clearpage

\begin{figure}[t!]
    \ContinuedFloat
    \centering
    \subfigure[m=500]{\includegraphics[width=0.9\linewidth, height=3.65cm, keepaspectratio]{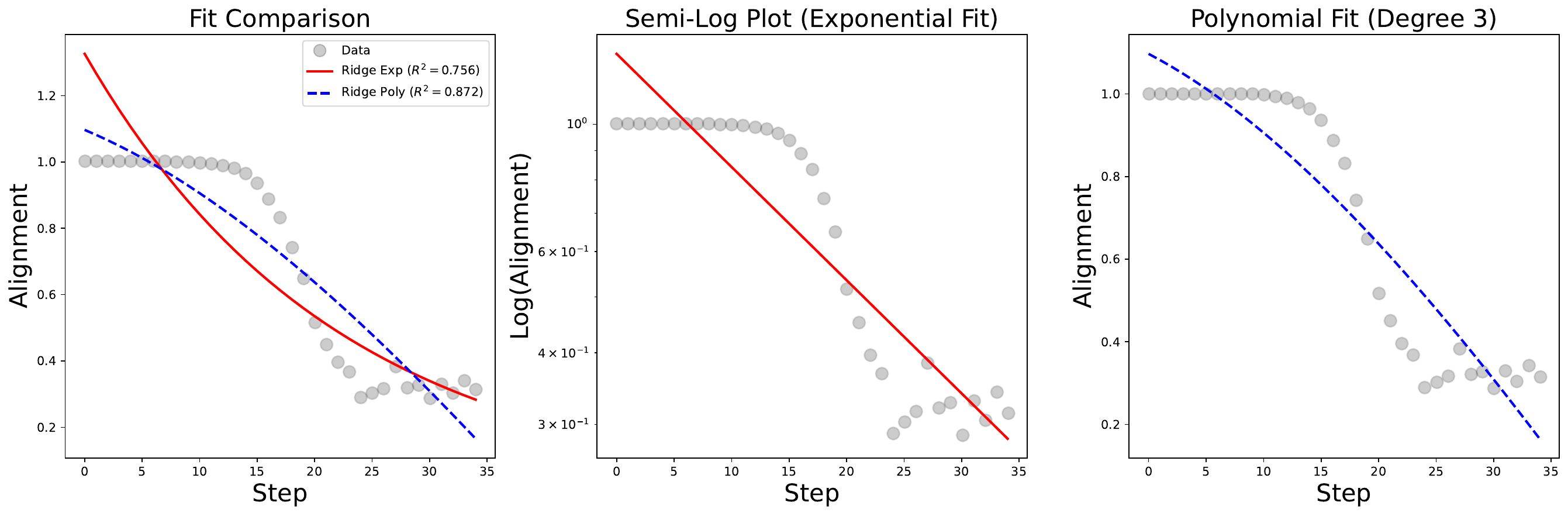}}
    \caption{The decay rate of phase I when seed=12138 for various $m=\frac{\lambda_k}{\lambda_{k+1}}$ (Part 3)}
\end{figure}

\begin{figure}[h!]
    \centering
    \subfigure[m=5]{\includegraphics[width=0.9\linewidth, height=3.65cm, keepaspectratio]{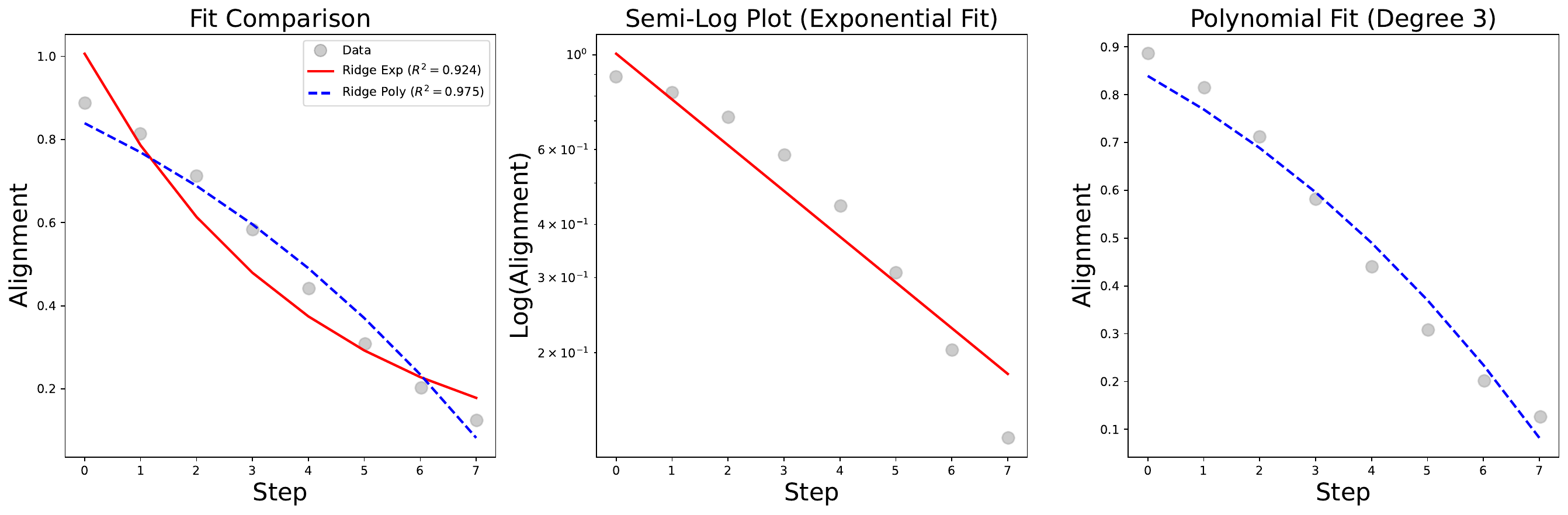}} \\
    \subfigure[m=10]{\includegraphics[width=0.9\linewidth, height=3.65cm, keepaspectratio]{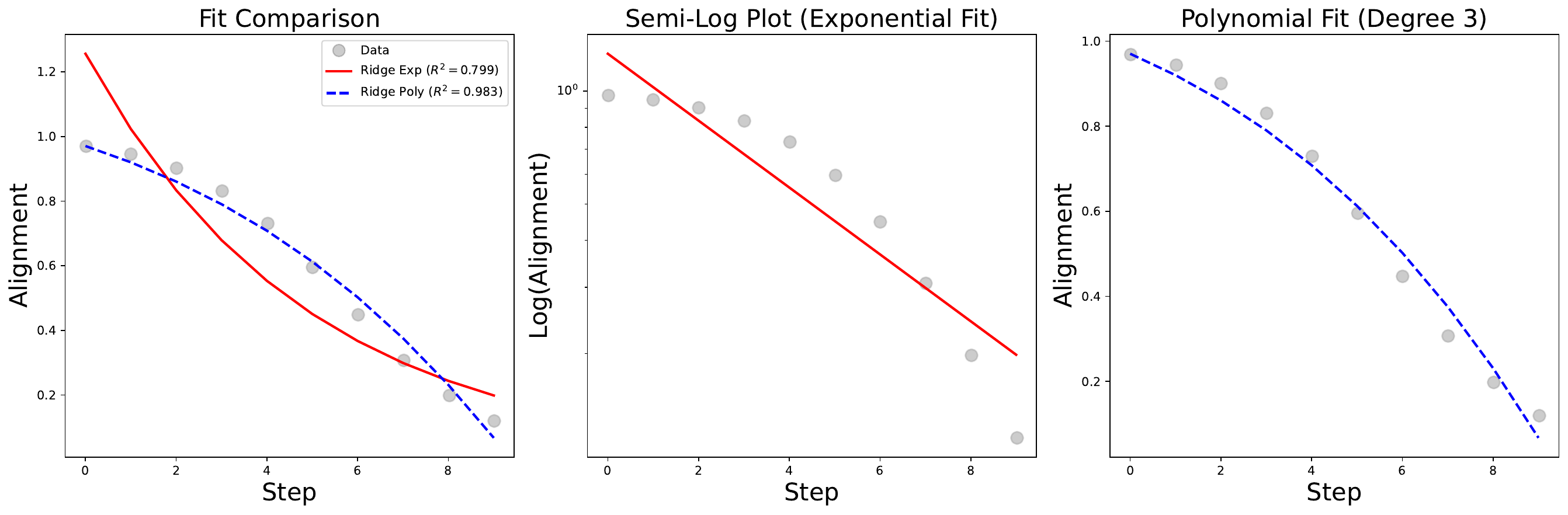}} \\
    \subfigure[m=20]{\includegraphics[width=0.9\linewidth, height=3.65cm, keepaspectratio]{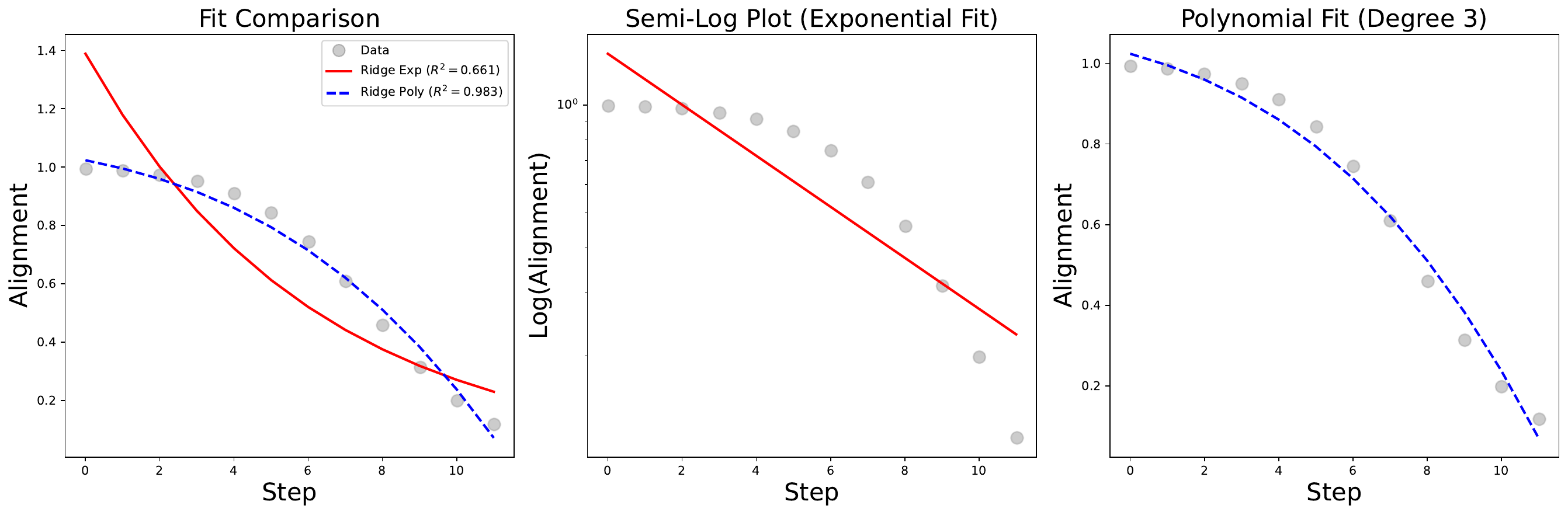}}
    
    \caption{The decay rate of phase I when seed=70425 for various $m=\frac{\lambda_k}{\lambda_{k+1}}$ (Part 1)}
    \label{fig:decay_70425}
\end{figure}
\clearpage

\begin{figure}[p]
    \ContinuedFloat
    \centering
    \subfigure[m=50]{\includegraphics[width=0.9\linewidth, height=3.65cm, keepaspectratio]{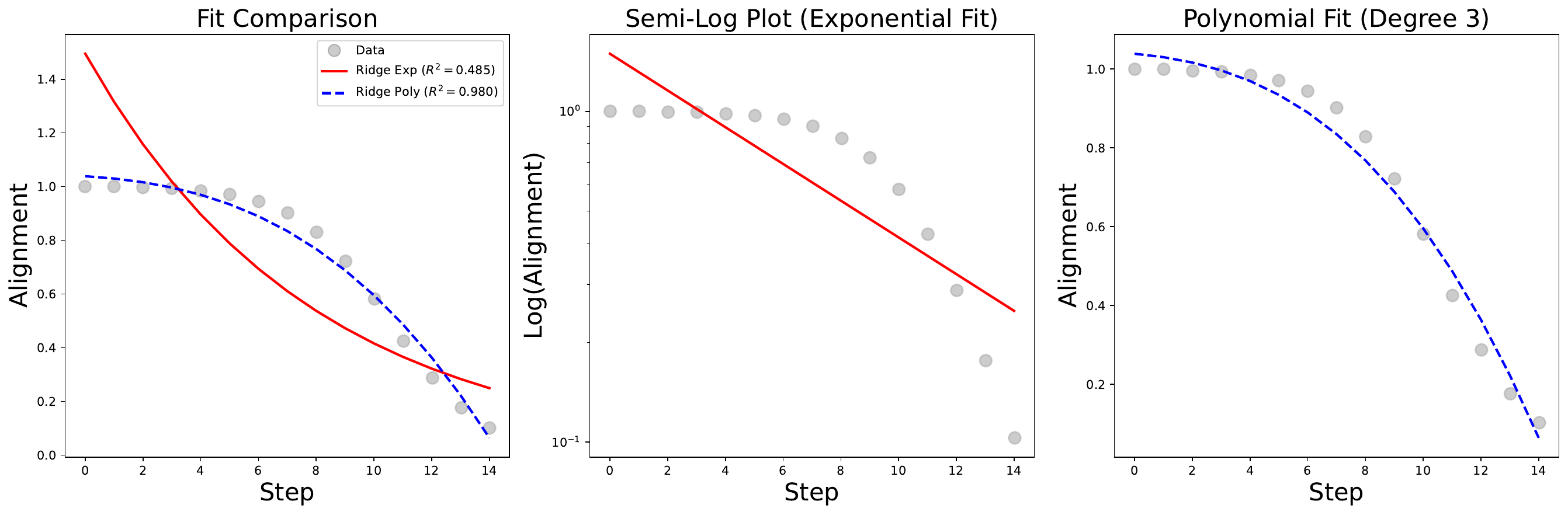}} \\
    \subfigure[m=100]{\includegraphics[width=0.9\linewidth, height=3.65cm, keepaspectratio]{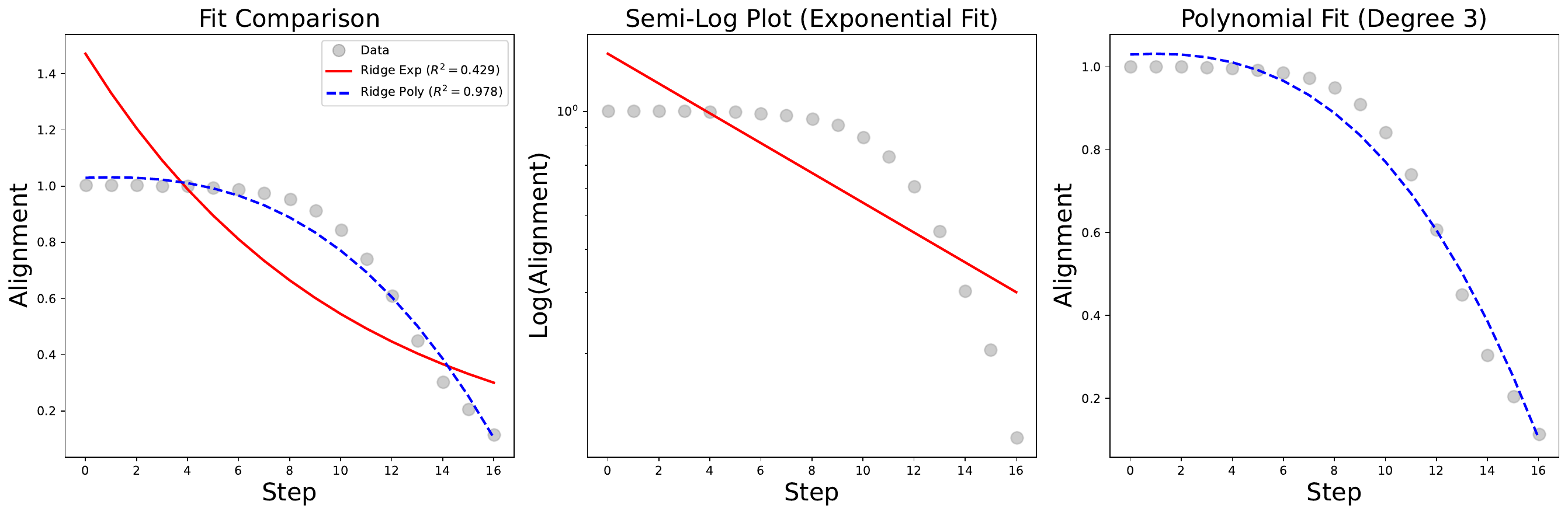}} \\
    \subfigure[m=200]{\includegraphics[width=0.9\linewidth, height=3.65cm, keepaspectratio]{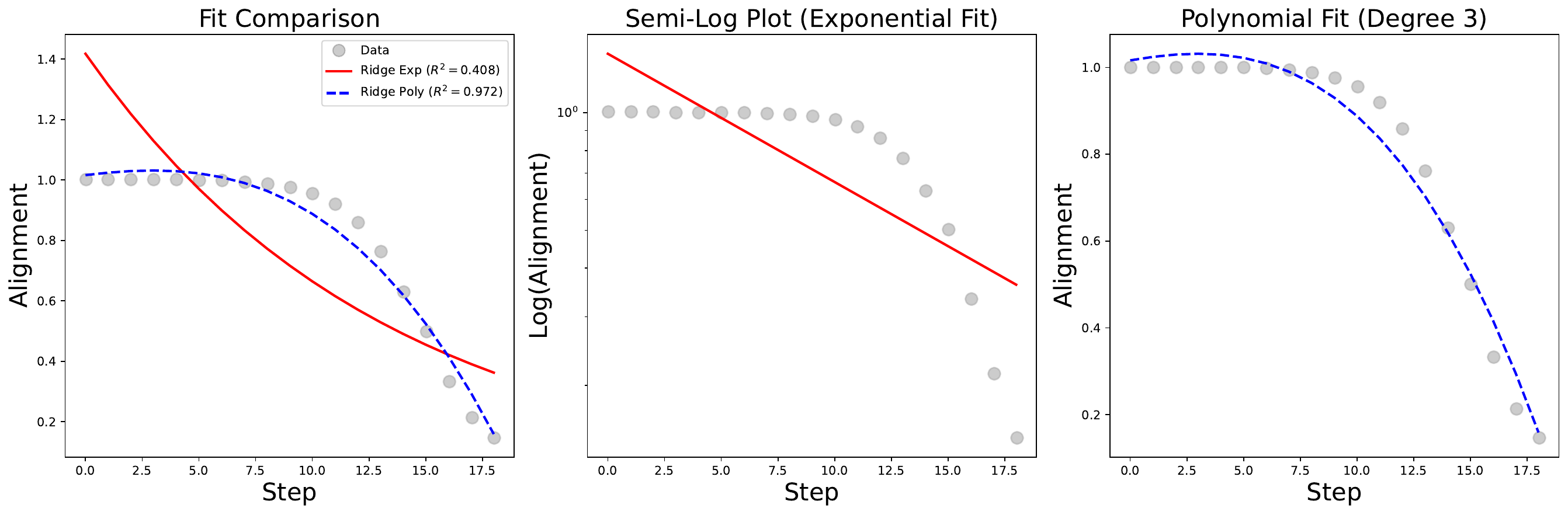}} \\
    \subfigure[m=300]{\includegraphics[width=0.9\linewidth, height=3.65cm, keepaspectratio]{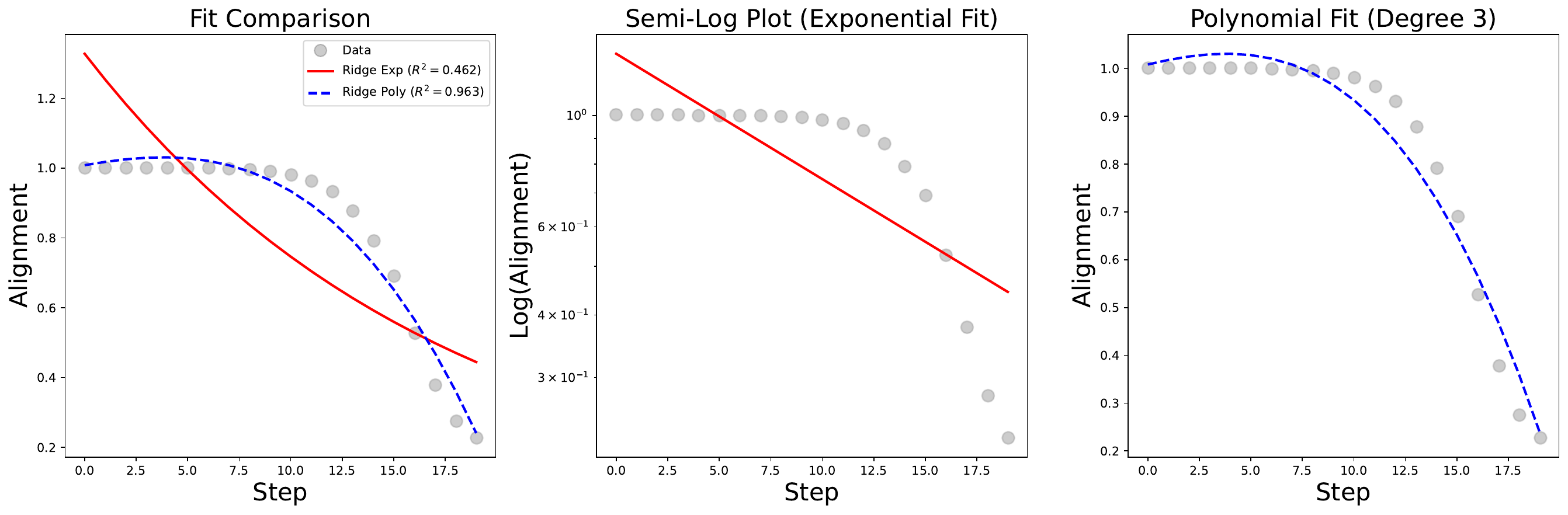}}
    
    \caption{The decay rate of phase I when seed=70425 for various $m=\frac{\lambda_k}{\lambda_{k+1}}$ (Part 2)}
\end{figure}
\clearpage

\begin{figure}[t!]
    \ContinuedFloat
    \centering
    \subfigure[m=400]{\includegraphics[width=0.9\linewidth, height=3.65cm, keepaspectratio]{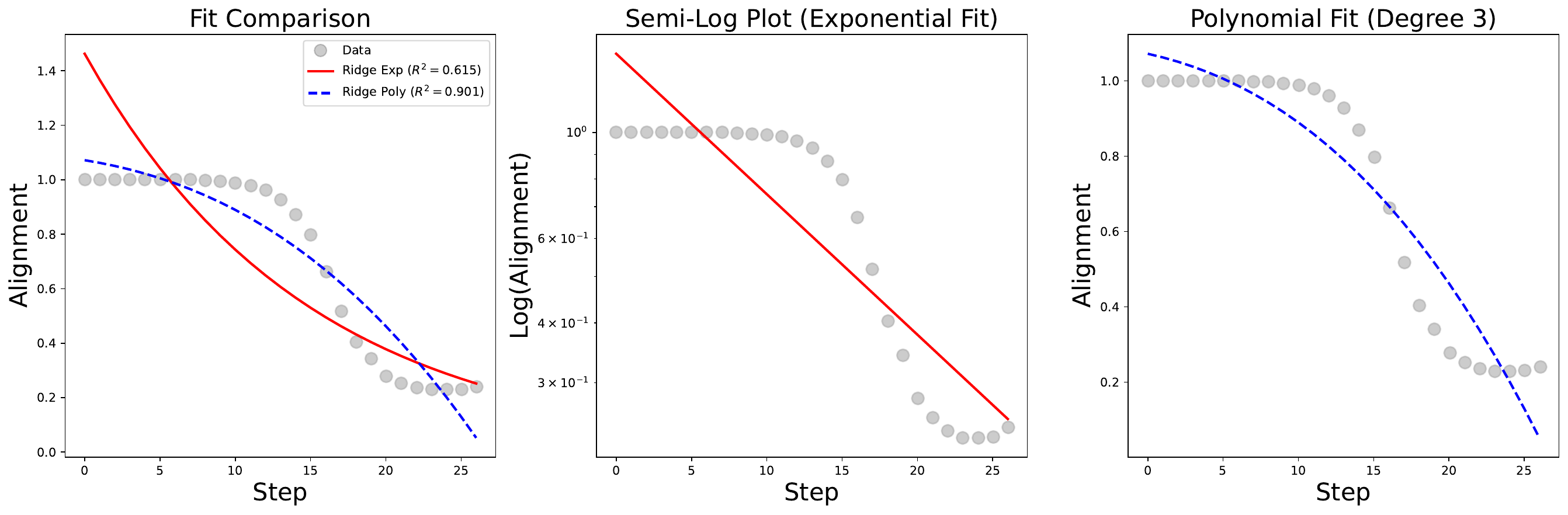}} \\
    \subfigure[m=500]{\includegraphics[width=0.9\linewidth, height=3.65cm, keepaspectratio]{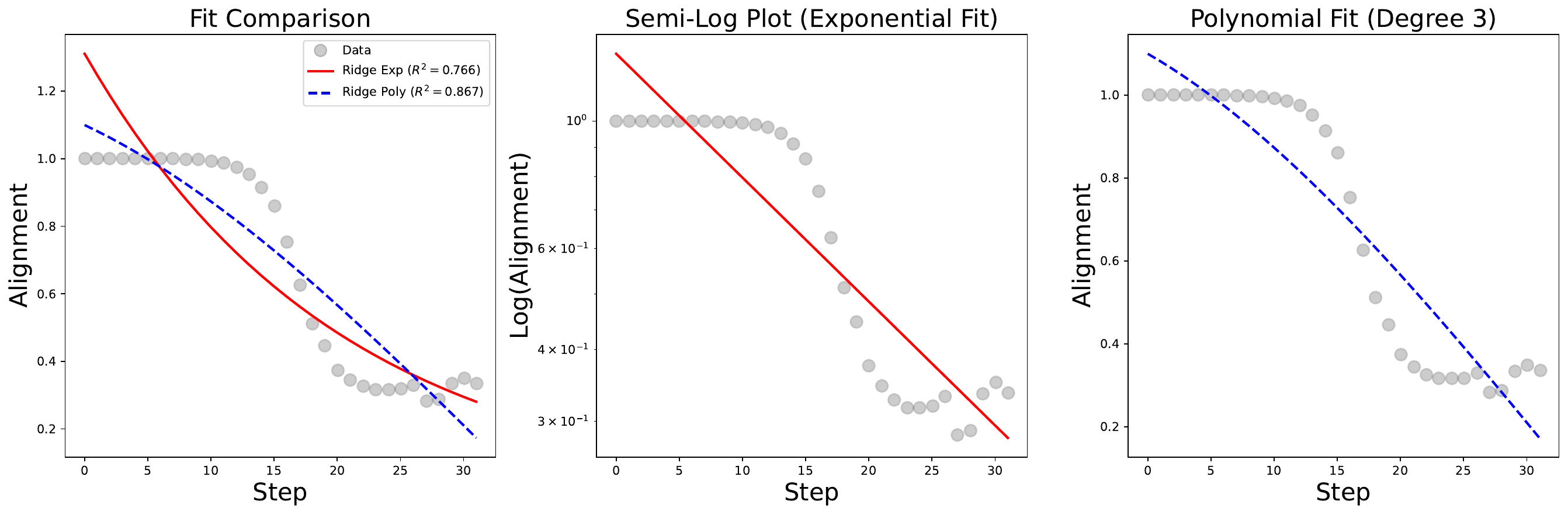}}
    \caption{The decay rate of phase I when seed=70425 for various $m=\frac{\lambda_k}{\lambda_{k+1}}$ (Part 3)}
\end{figure}

\begin{figure}[h!]
    \centering
    \subfigure[m=5]{\includegraphics[width=0.9\linewidth, height=3.65cm, keepaspectratio]{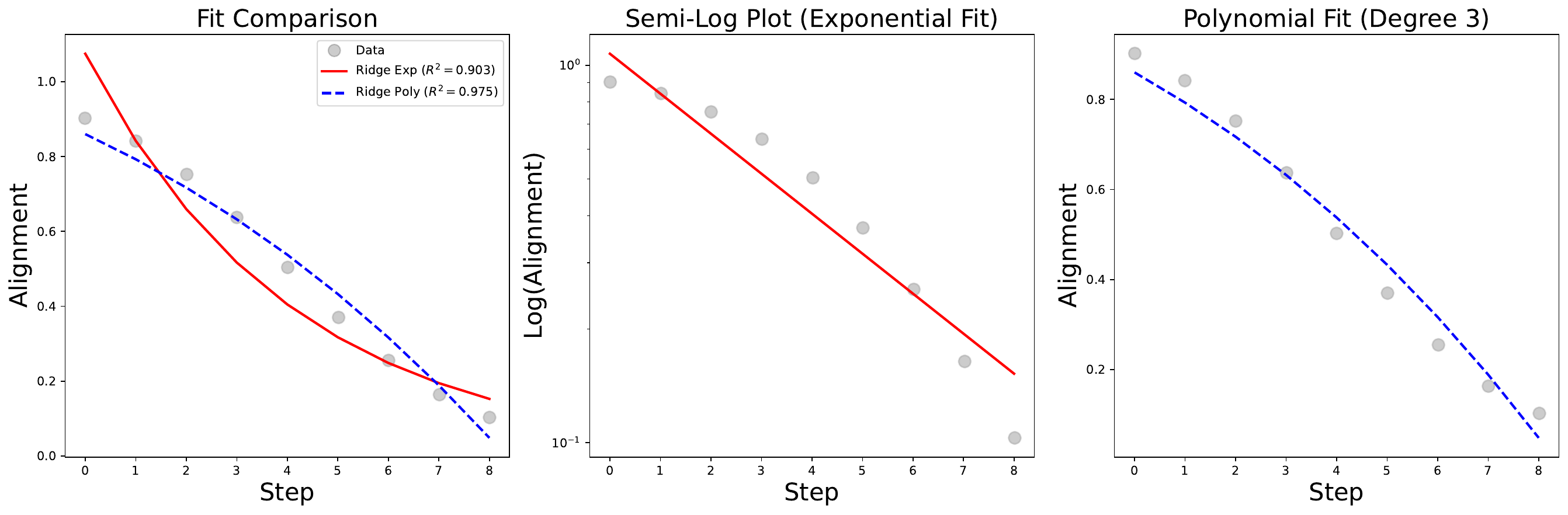}} \\
    \subfigure[m=10]{\includegraphics[width=0.9\linewidth, height=3.65cm, keepaspectratio]{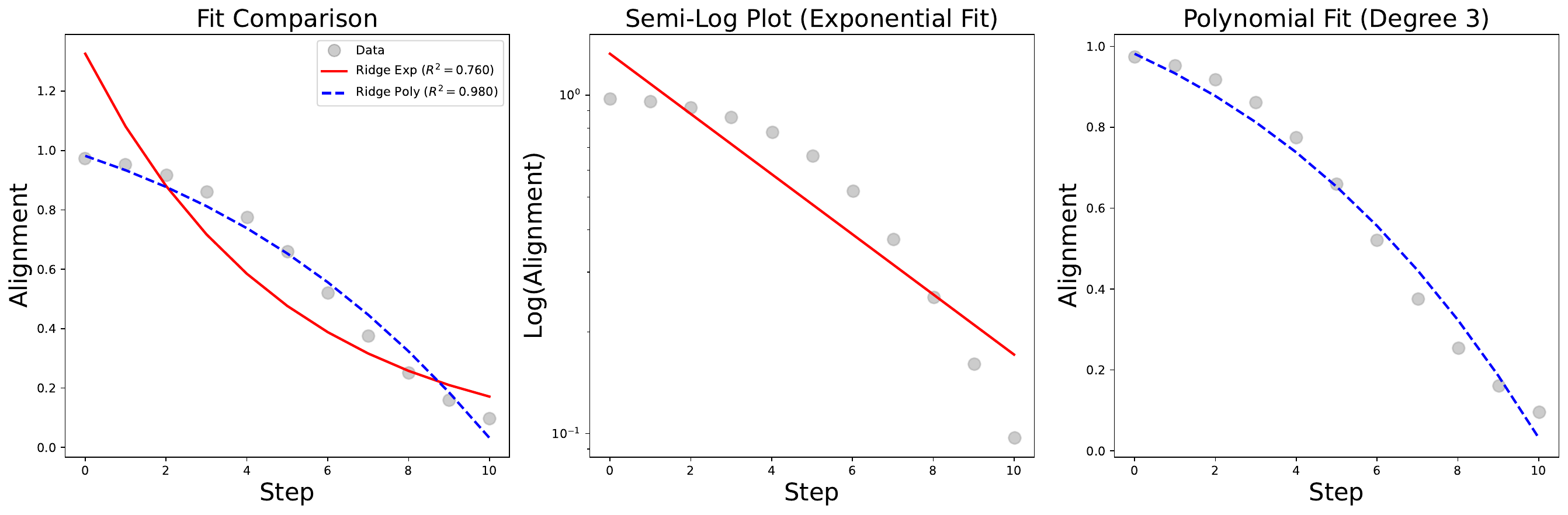}}
    
    \caption{The decay rate of phase I when seed=4008001 for various $m=\frac{\lambda_k}{\lambda_{k+1}}$ (Part 1)}
    \label{fig:decay_4008001}
\end{figure}
\clearpage

\begin{figure}[p]
    \ContinuedFloat
    \centering
    \subfigure[m=20]{\includegraphics[width=0.9\linewidth, height=3.65cm, keepaspectratio]{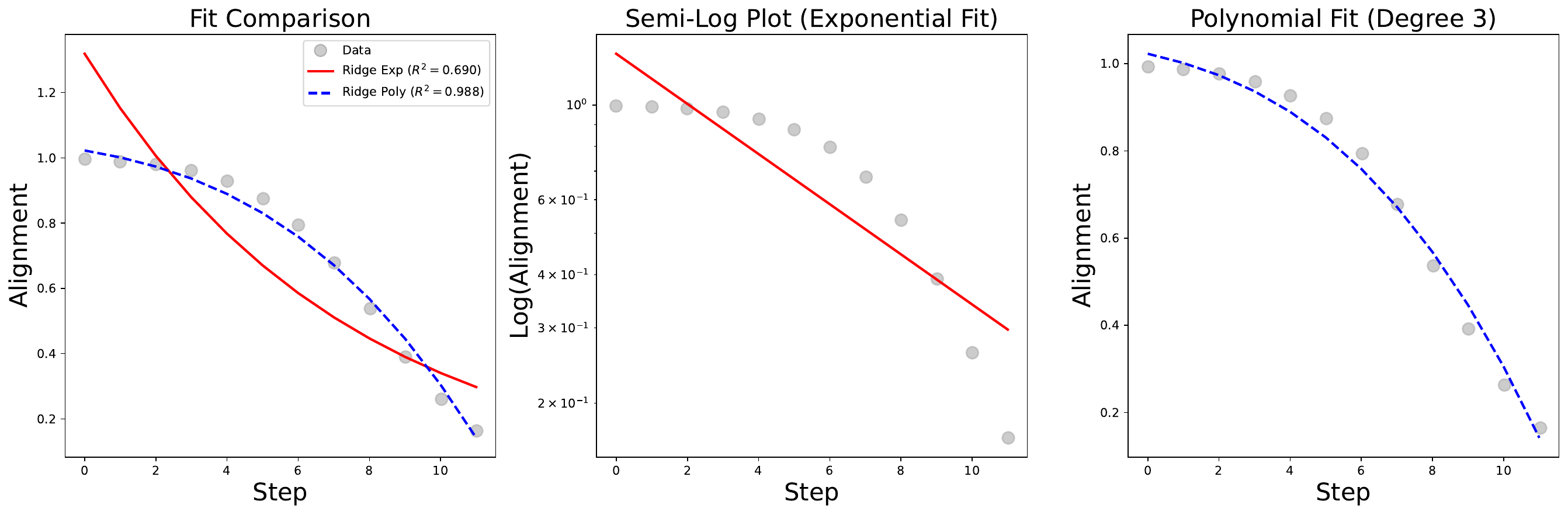}} \\
    \subfigure[m=50]{\includegraphics[width=0.9\linewidth, height=3.65cm, keepaspectratio]{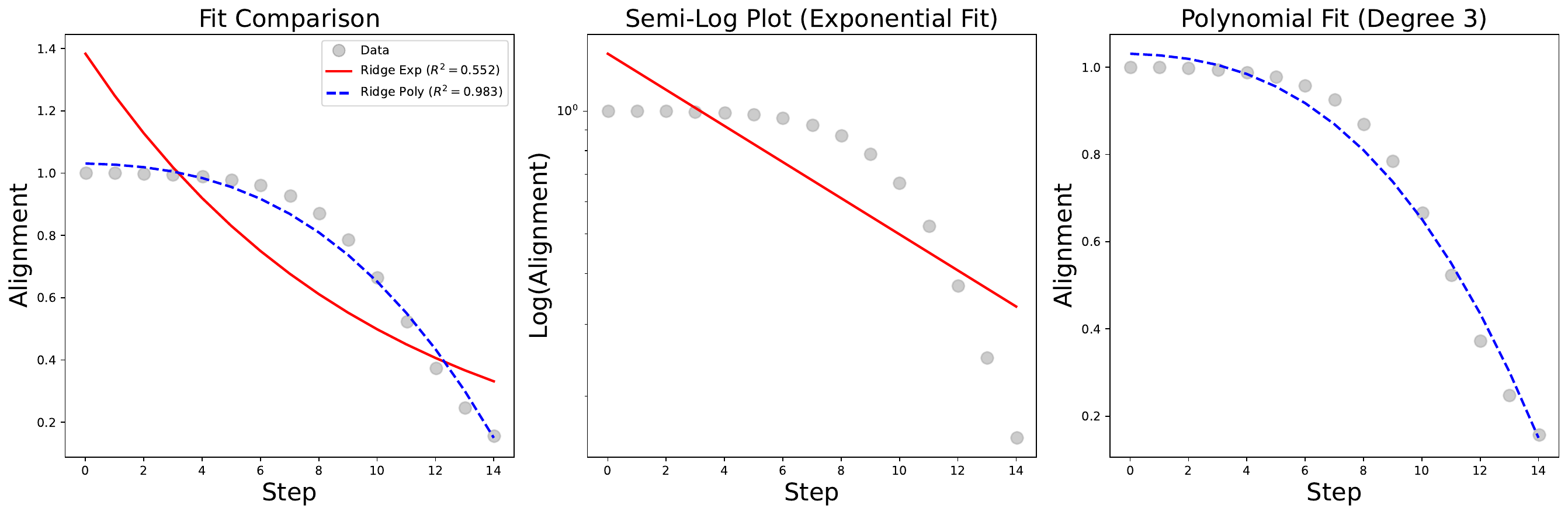}} \\
    \subfigure[m=100]{\includegraphics[width=0.9\linewidth, height=3.65cm, keepaspectratio]{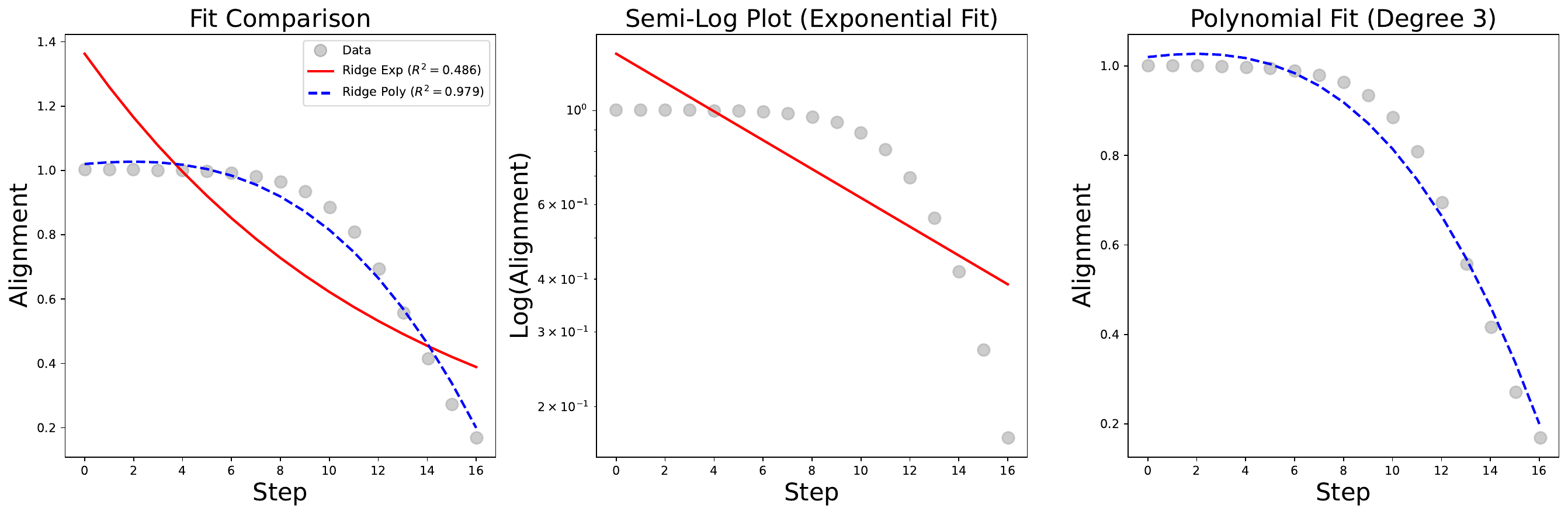}} \\
    \subfigure[m=200]{\includegraphics[width=0.9\linewidth, height=3.65cm, keepaspectratio]{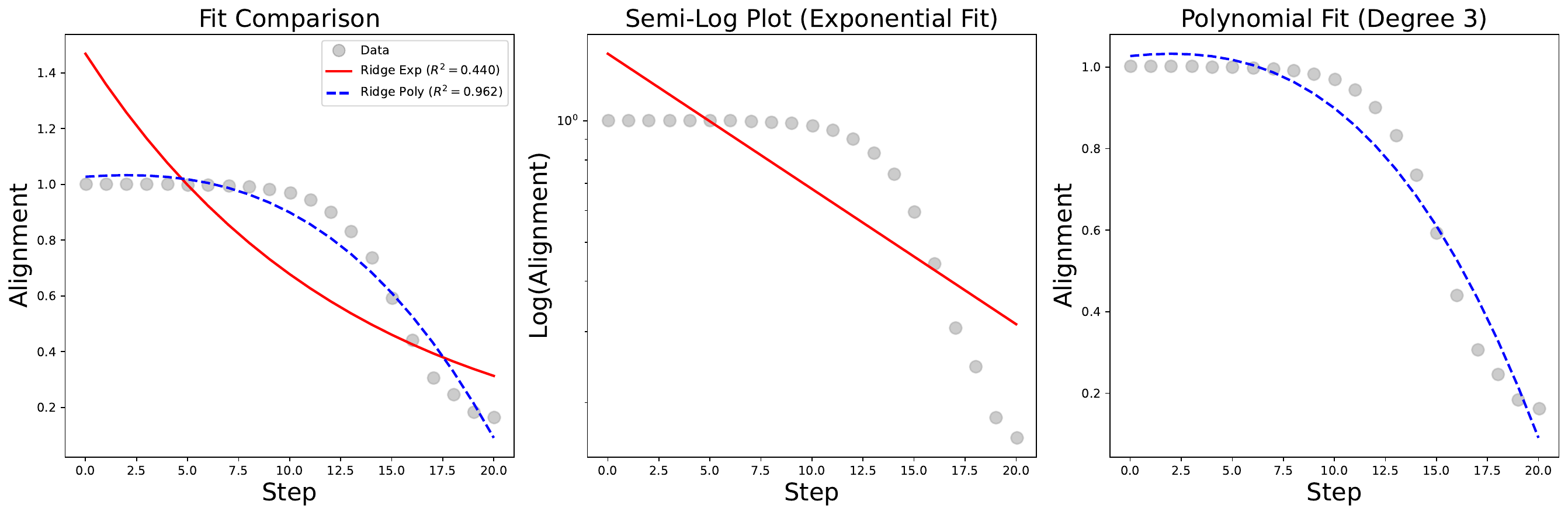}}
    
    \caption{The decay rate of phase I when seed=4008001 for various $m=\frac{\lambda_k}{\lambda_{k+1}}$ (Part 2)}
\end{figure}
\clearpage

\begin{figure}[t!]
    \ContinuedFloat
    \centering
    \subfigure[m=300]{\includegraphics[width=0.9\linewidth, height=3.65cm, keepaspectratio]{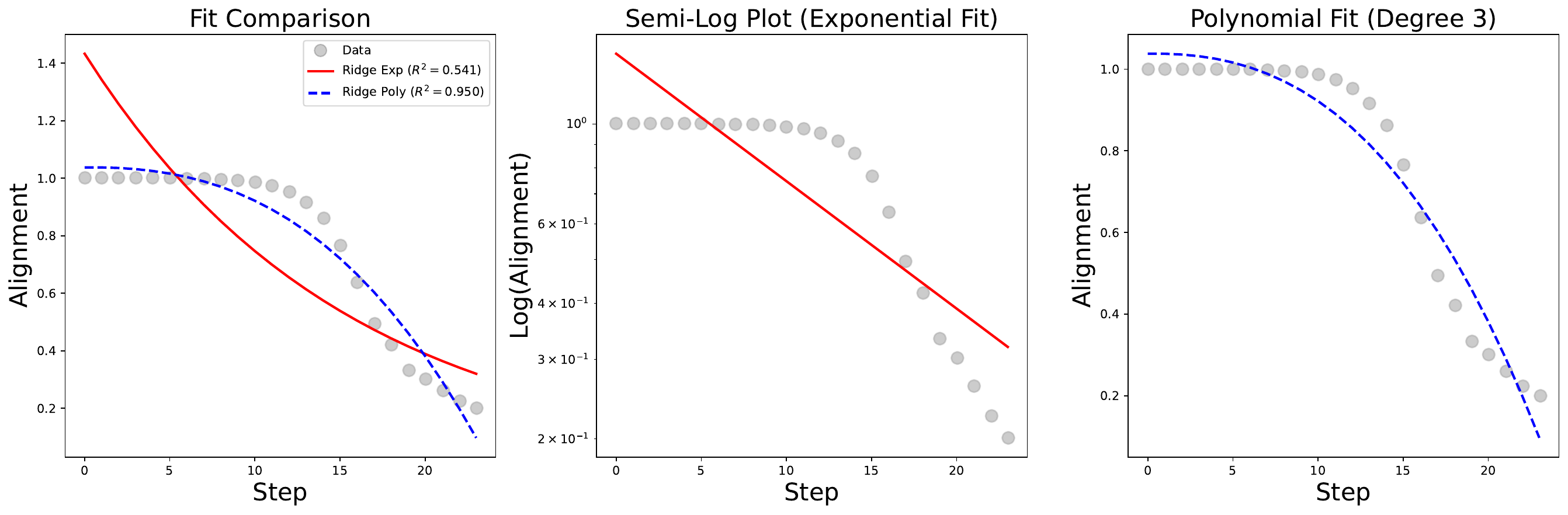}} \\
    \subfigure[m=400]{\includegraphics[width=0.9\linewidth, height=3.65cm, keepaspectratio]{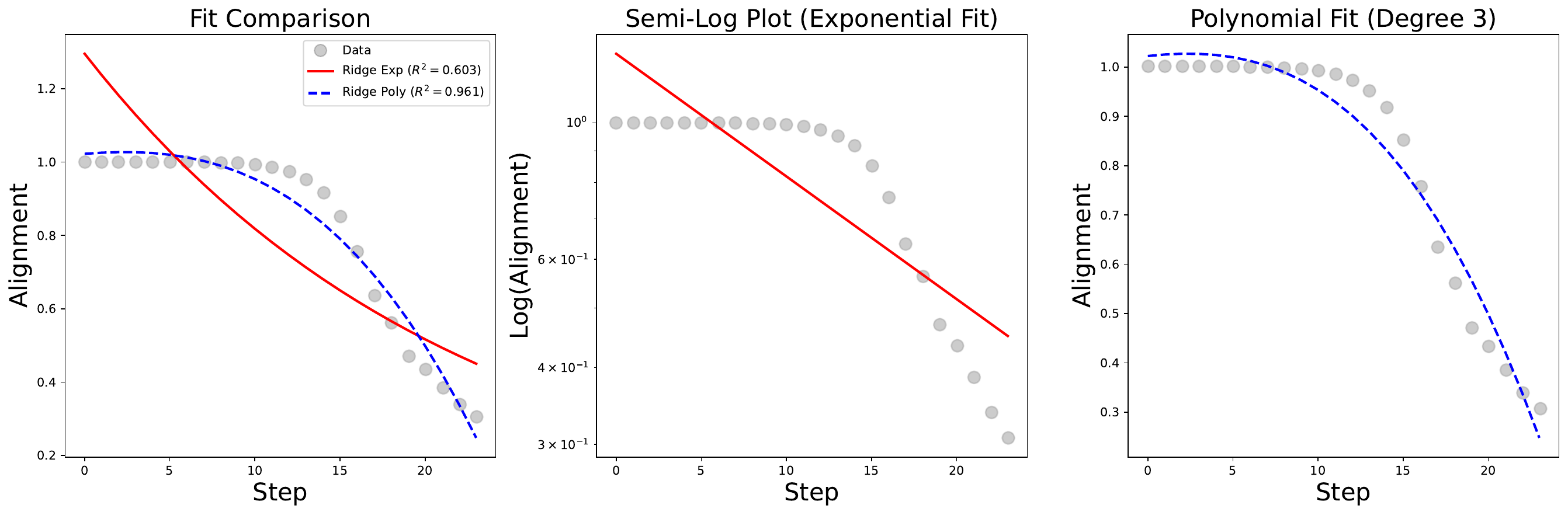}} \\
    \subfigure[m=500]{\includegraphics[width=0.9\linewidth, height=3.65cm, keepaspectratio]{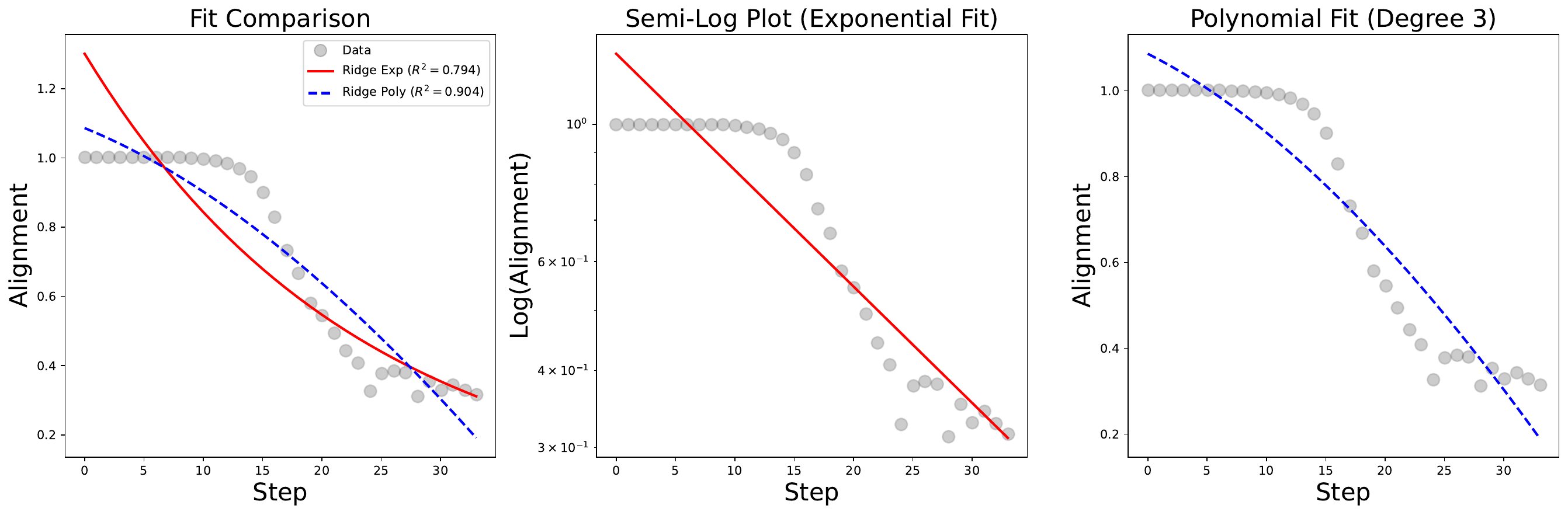}}
    
    \caption{The decay rate of phase I when seed=4008001 for various $m=\frac{\lambda_k}{\lambda_{k+1}}$ (Part 3) \label{fig:decay_4008001，2}}
\end{figure}

\begin{figure}[h!]
    \centering
    \subfigure[seed=42]{\includegraphics[width=0.9\linewidth, height=3.65cm, keepaspectratio]{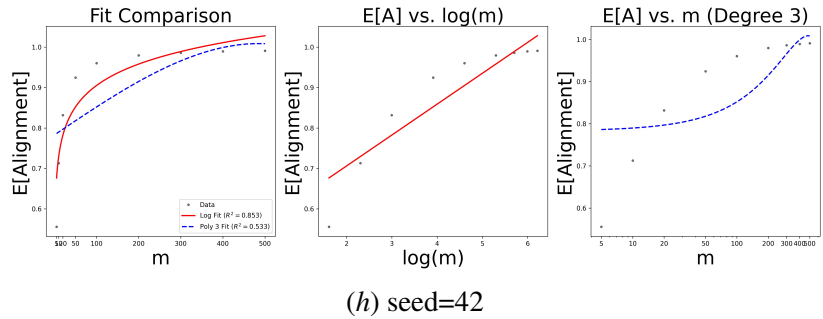}}
    
    \caption{The rate of the expected alignment convergence with respect to $m=\frac{\lambda_k}{\lambda_{k+1}}$ (Part 1)}
\end{figure}
\clearpage

\begin{figure}[p]
    \ContinuedFloat
    \centering
    \subfigure[seed=87]{\includegraphics[width=0.9\linewidth, height=3.65cm, keepaspectratio]{ALT_final/m_rate/seed_87.pic.pdf}} \\
    \subfigure[seed=568]{\includegraphics[width=0.9\linewidth, height=3.65cm, keepaspectratio]{ALT_final/m_rate/seed_568.pic.pdf}} \\
    \subfigure[seed=1101]{\includegraphics[width=0.9\linewidth, height=3.65cm, keepaspectratio]{ALT_final/m_rate/seed_1101.pic.pdf}} \\
    \subfigure[seed=12138]{\includegraphics[width=0.9\linewidth, height=3.65cm, keepaspectratio]{ALT_final/m_rate/seed_12138.pic.pdf}}
    
    \caption{The rate of the expected alignment convergence (Part 2) \label{fig:rate_convergence}}
\end{figure}
\clearpage

\begin{figure}[H]
    \ContinuedFloat
    \centering
    \subfigure[seed=70425]{\includegraphics[width=0.9\linewidth, height=3.65cm, keepaspectratio]{ALT_final/m_rate/seed_70425.pic.pdf}} \\
    \subfigure[seed=4008001]{\includegraphics[width=0.9\linewidth, height=3.65cm, keepaspectratio]{ALT_final/m_rate/seed_4008001.pic.pdf}}
    
    \caption{The rate of the expected alignment convergence (Part 3)}
\end{figure}

\end{document}